%% file: main.tex
\documentclass[twoside,11pt]{article}
\input{preamble}

\title{Online Resource Allocation with an Endogenous Markov State: \\ Fewer LP Solves Earn More}
\author{Zhaohua Chen}
\affil{Taobao \& Tmall Group of Alibaba, Renmin University of China \\ \texttt{chenzhaohua.czh@alibaba-inc.com}}

\date{}

\begin{document}
\maketitle

\begin{abstract}
We study finite-horizon online resource allocation with i.i.d.\ requests and an endogenous Markov state on a finite state space: each action affects the transition of the state that governs future rewards and resource consumption.
In this problem, a transient fluid LP benchmark upper bounds the expected reward of every nonanticipating policy, while a stationary LP supplies randomized state-dependent controls.
We assume that the stationary LP has a unique optimum and identify primal nondegeneracy and irreducibility of the optimal induced kernel as important regularity conditions in this framework. 
With a known request prior, we show that, under nondegeneracy and irreducibility, both frequent and infrequent re-solving attain $\bigO(1)$ regret. 
However, under a degenerate optimum, irreducibility yields the sharp worst-case $\bigTheta(\sqrt{\horizon})$ rate for infrequent re-solving, while frequent re-solving can incur $\bigOmega(\horizon)$ regret. 
Thus, more frequent optimization can perform asymptotically worse.
With an unknown request prior, we develop a three-phase U-shaped infrequent re-solving policy that coordinates learning and inventory correction with $\bigO(\log\log\horizon)$ LP solves. 
When the optimal induced kernel is irreducible and the algorithm is given the optimal target state class and a constant-cost entrance policy, it attains $\bigO(1)$ regret under nondegeneracy and $\bigO(\sqrt{\horizon})$ regret under degeneracy. Without the target-class information, linear minimax regret is unavoidable.
Numerical experiments further illustrate the instability of round-by-round re-solving relative to epoch-wise infrequent re-solving, show that thresholding greatly mitigates its loss, and find that infrequent schemes remain dominant under both known and estimated priors.
\end{abstract}

\newpage

\input{sections/01_introduction}
\input{sections/02_model}
\input{sections/03_regularity_assumptions}
\input{sections/04_resolving_algorithms_and_guarantees}
\input{sections/05_boundary_cases}
\input{sections/06_unknown_request_prior}
\input{sections/07_numerical_experiments}
\input{sections/08_conclusion}

\section*{Acknowledgement}
The analysis of the algorithms and the numerical experiments were developed through extensive interactions with ChatGPT. 
The author has verified the correctness of the proofs and codes and made a thorough re-organization of the results, and the author assumes full responsibility for the contents.

\bibliographystyle{plainnat}
\bibliography{references}

\clearpage
\appendix
\addcontentsline{toc}{part}{Appendix}
\numberwithin{equation}{section}

\input{appendices/01_model_proofs}
\clearpage
\input{appendices/02_regularity_proofs}
\clearpage

\section{Proof of Theorem~\ref{thm:constant-regret}}
\input{appendices/03_algorithm_proofs}
\input{appendices/04_common_ingredients_proofs}
\input{appendices/05_frequent_resolving_proofs}
\input{appendices/06_infrequent_resolving_proofs}

\clearpage

\input{appendices/07_1_boundary_cases_proofs}
\input{appendices/07_2_boundary_cases_proofs}

\clearpage

\input{appendices/08_1_unknown_request_prior_proofs}
\input{appendices/08_2_unknown_request_prior_proofs}

\end{document}

%% file: preamble.tex
\usepackage{authblk}
\usepackage[T1]{fontenc}
\usepackage{lmodern}
\usepackage[margin=1in]{geometry}
\usepackage{setspace}
\usepackage{amsmath,amssymb,amsthm,mathtools}
\usepackage{bm}
\usepackage{array,booktabs,tabularx,multirow}
\usepackage{enumitem}
\usepackage{algorithm}
\usepackage[noend]{algpseudocode}
\usepackage{microtype}
\usepackage[table]{xcolor}
\usepackage{tcolorbox}
\tcbuselibrary{breakable,skins}
\usepackage{graphicx}
\usepackage{subcaption}
\usepackage{xurl}
\usepackage[authoryear,round]{natbib}
\usepackage{hyperref}
\usepackage[nameinlink,capitalise,noabbrev]{cleveref}

\makeatletter
\renewcommand{\theHALG@line}{\thealgorithm.\arabic{ALG@line}}
\makeatother

\graphicspath{{figures/}}

\hypersetup{
  colorlinks=true,
  linkcolor=blue!55!black,
  citecolor=blue!55!black,
  urlcolor=blue!55!black
}

\definecolor{authorcommentframe}{HTML}{C2410C}
\definecolor{authorcommentback}{HTML}{FFF7D6}
\newtcolorbox{authorcomment}[1][Comment]{
  enhanced,
  breakable,
  colback=authorcommentback,
  colframe=authorcommentframe,
  colbacktitle=authorcommentframe,
  coltitle=white,
  fonttitle=\bfseries,
  title={#1},
  boxrule=0.8pt,
  borderline west={3pt}{0pt}{authorcommentframe},
  arc=1.5mm,
  left=6pt,
  right=6pt,
  top=5pt,
  bottom=5pt,
  before skip=8pt,
  after skip=8pt
}

\newtheorem{theorem}{Theorem}[section]
\newtheorem{lemma}[theorem]{Lemma}
\newtheorem{proposition}[theorem]{Proposition}

\newtheorem{fact}[theorem]{Fact}
\newtheorem{assumption}[theorem]{Assumption}
\theoremstyle{definition}

\crefname{theorem}{Theorem}{Theorems}
\Crefname{theorem}{Theorem}{Theorems}
\crefname{lemma}{Lemma}{Lemmas}
\Crefname{lemma}{Lemma}{Lemmas}
\crefname{proposition}{Proposition}{Propositions}
\Crefname{proposition}{Proposition}{Propositions}
\crefname{assumption}{Assumption}{Assumptions}
\Crefname{assumption}{Assumption}{Assumptions}
\crefname{remark}{Remark}{Remarks}
\Crefname{remark}{Remark}{Remarks}
\crefname{fact}{Fact}{Facts}
\Crefname{fact}{Fact}{Facts}
\crefname{algorithm}{Algorithm}{Algorithms}
\Crefname{algorithm}{Algorithm}{Algorithms}

\ExplSyntaxOn
\NewDocumentCommand{\agentvar}{m m O{} O{}}{%
  \NewDocumentCommand{#1}{O{} O{}}{%
    {#2%
      \tl_if_blank:nTF{##1}%
        {\tl_if_blank:nF{#3}{\c_math_subscript_token{#3}}}%
        {\c_math_subscript_token{##1}}%
      \tl_if_blank:nTF{##2}%
        {\tl_if_blank:nF{#4}{\c_math_superscript_token{#4}}}%
        {\c_math_superscript_token{##2}}}}%
}

\ExplSyntaxOff

\agentvar{\reals}{\mathbb R}
\newcommand{\expect}{\mathbb E}
\agentvar{\prob}{\mathbb P}
\agentvar{\indicator}{\mathbf 1}
\agentvar{\onesvector}{\mathbf 1}
\agentvar{\bigO}{O}
\agentvar{\bigOmega}{\Omega}
\agentvar{\bigTheta}{\Theta}
\agentvar{\genericconstant}{C}
\agentvar{\littleo}{o}

\agentvar{\frequentalg}{\mathrm{FR}}
\agentvar{\infrequentalg}{\mathrm{IR}}
\agentvar{\gfrequentalg}{\mathrm{FR}}[][\mathrm{g}]
\agentvar{\ginfrequentalg}{\mathrm{IR}}[][\mathrm{g}]
\agentvar{\unknowninfrequentalg}{\mathrm{UIR\text{-}UP}}

\agentvar{\genericvalue}{x}
\agentvar{\deviationlevel}{x}
\agentvar{\flowmass}{x}
\agentvar{\excursionmass}{x}
\agentvar{\seriesratio}{x}
\agentvar{\numres}{m}
\agentvar{\res}{i}
\agentvar{\state}{s}
\agentvar{\req}{j}
\agentvar{\act}{a}
\agentvar{\period}{t}
\agentvar{\horizon}{T}
\agentvar{\states}{\mathcal S}
\agentvar{\requests}{\mathcal J}
\agentvar{\actions}{\mathcal A}
\agentvar{\prior}{p}
\agentvar{\transition}{P}
\agentvar{\nullstationarylaw}{\mu}
\agentvar{\reward}{r}
\agentvar{\consumption}{c}
\agentvar{\maxreward}{\overline r}
\agentvar{\maxconsumption}{\overline c}
\agentvar{\nominalratio}{\rho}
\agentvar{\stateproc}{S}
\agentvar{\reqproc}{J}
\agentvar{\actproc}{A}
\agentvar{\rewardproc}{R}
\agentvar{\consumptionproc}{C}
\agentvar{\ratioremainder}{\mathcal R}
\agentvar{\inventory}{B}
\agentvar{\stoptime}{\tau}
\agentvar{\policy}{\pi}
\agentvar{\totalreward}{\mathcal R}
\agentvar{\valuefn}{V}
\agentvar{\filtration}{\mathcal F}

\agentvar{\subhorizon}{H}
\agentvar{\nonterminalocc}{X}
\agentvar{\terminalocc}{Y}
\agentvar{\transientstateocc}{Q}
\agentvar{\transientfluid}{\Phi}
\agentvar{\capacityratio}{\beta}
\agentvar{\occupation}{z}
\agentvar{\stateoccupation}{q}
\agentvar{\stationaryfluid}{G}
\agentvar{\inducedkernel}{K}
\agentvar{\resourceprice}{\lambda}
\agentvar{\averagegain}{\gamma}
\agentvar{\statebias}{h}
\agentvar{\statedual}{u}
\agentvar{\boundaryconstant}{\kappa}
\agentvar{\concentrationexponent}{\zeta}

\agentvar{\neighborhoodradius}{r}
\agentvar{\bindingset}{I}[][*]
\agentvar{\localradius}{\delta}
\agentvar{\positivitymargin}{\eta}
\agentvar{\capacityneighborhood}{\mathcal N}
\agentvar{\resourceslack}{\sigma}
\agentvar{\recstates}{\mathcal R}[][*]
\agentvar{\actsupport}{\mathcal E}[][*]
\agentvar{\capacitydistance}{\Delta}
\agentvar{\inventorydeviationnorm}{\mathcal D}
\agentvar{\inventoryimbalance}{x}
\agentvar{\inventorydeviation}{X}
\agentvar{\basicsolution}{\widetilde x}
\agentvar{\flowprojection}{P}[zq][]
\agentvar{\comparatorsolution}{y}
\agentvar{\priortransport}{x}[][0]
\agentvar{\nullsolution}{x}[][\mathrm{null}]
\agentvar{\perturbedcomparator}{\widetilde x}
\agentvar{\basis}{A}[][*]
\agentvar{\basisinverse}{(\basis)}[][-1]
\agentvar{\nonbasiccols}{N}[][*]
\agentvar{\constraintmatrix}{A}
\agentvar{\objectivevector}{d}
\agentvar{\rhsvector}{b}
\agentvar{\reducedcost}{\overline d}
\agentvar{\unitvector}{e}
\agentvar{\perturbation}{\varepsilon}

\agentvar{\entrance}{\mathrm{ent}}
\agentvar{\entranceduration}{N}[][\entrance]
\agentvar{\entrancelast}{\stoptime}[][\entrance]
\agentvar{\entrancerest}{\subhorizon}[][\entrance]
\agentvar{\entrancepolicy}{\policy}[][\entrance]
\agentvar{\entranceinventory}{\inventory}[][\entrance]
\agentvar{\entranceconsumption}{\consumptionproc}[][\entrance]
\agentvar{\entrancereward}{\rewardproc}[][\entrance]
\agentvar{\entrancestate}{\stateproc}[][\entrance]
\agentvar{\entrancecapacityratio}{\capacityratio}[][\entrance]
\agentvar{\entranceinventorydeviation}{\inventorydeviation}[][\entrance]

\agentvar{\firstbadepoch}{N}
\agentvar{\batchsize}{N}
\agentvar{\binomialcount}{N}
\agentvar{\event}{\mathcal E}
\agentvar{\schedulepower}{\alpha}
\agentvar{\epoch}{k}
\agentvar{\epochhorizon}{h}
\agentvar{\numepochs}{K}
\agentvar{\blocklength}{\Lambda}

\agentvar{\cutrest}{\subhorizon}[\perp]

\agentvar{\potential}{F}
\agentvar{\stateconsumption}{g}
\agentvar{\stationaryconsumption}{\overline c}
\agentvar{\stationaryreward}{\overline r}
\agentvar{\trueresourcemean}{\overline c}
\agentvar{\initiallaw}{\mu}
\agentvar{\poissoncorrector}{v}
\agentvar{\vertex}{v}
\agentvar{\poissonbound}{V}
\agentvar{\poissonlip}{C}[\mathrm{Pois}][]
\agentvar{\bindingparameter}{\theta}
\agentvar{\mixtureweight}{\theta}
\agentvar{\bindingbox}{\mathcal B}
\agentvar{\stationaryprojector}{\Pi}
\agentvar{\fundamentalmatrix}{Z}
\agentvar{\maxmartingale}{Z}
\agentvar{\analyticalstop}{\sigma}
\agentvar{\potentialincrement}{\Delta}

\agentvar{\hypotheticalcost}{\widetilde C}
\agentvar{\hypotheticalratio}{\widetilde\beta}
\agentvar{\martdiff}{\xi}
\agentvar{\correcteddeviation}{W}
\agentvar{\martingale}{M}
\agentvar{\quadvariation}{V}
\agentvar{\martindex}{n}
\agentvar{\poststep}{n}
\agentvar{\nullmass}{n}
\agentvar{\prefixidx}{\ell}
\agentvar{\prefixstep}{r}
\agentvar{\variancebudget}{v}
\agentvar{\incrementbound}{b}
\agentvar{\goodevent}{\mathcal G}
\agentvar{\concentrationpower}{\kappa}
\agentvar{\tailpower}{Q}
\agentvar{\epochbuffer}{d}
\agentvar{\capacityshortfall}{d}
\agentvar{\reachedevent}{\mathcal R}
\agentvar{\highprobevent}{\mathcal H}
\agentvar{\highprobparameter}{A}

\agentvar{\optimizerface}{\mathfrak F}[\nominalratio][]
\agentvar{\outsidestates}{\mathcal O}
\agentvar{\selectedsolution}{x}
\agentvar{\basisregion}{\mathcal P}
\agentvar{\bindingpolytope}{\mathcal P}[][+]
\agentvar{\candidatebasis}{B}
\agentvar{\identitymatrix}{I}
\agentvar{\onesideddeviation}{D}
\agentvar{\oneperiodquantity}{Y}
\agentvar{\conditionalmean}{f}
\agentvar{\nominalmean}{\overline f}[][*]
\agentvar{\dualloss}{\ell}

\agentvar{\empiricalprior}{\widehat p}
\agentvar{\requestcount}{N}
\agentvar{\completedrounds}{n}
\agentvar{\coreprojection}{\Pi}[\recstates][]
\agentvar{\jointdistance}{\Delta}
\agentvar{\protectlength}{\stoptime}[][\mathrm{lrn}]
\agentvar{\resolveindex}{\epoch}
\agentvar{\learningresolve}{r}
\agentvar{\earlySchedule}{E}
\agentvar{\midSchedule}{K}
\agentvar{\lateSchedule}{L}
\agentvar{\sqrtepoch}{J}
\agentvar{\scheduleendpoint}{\stoptime}[][\mathrm{hal}]
\agentvar{\schedulecoef}{\omega}
\agentvar{\horizonratio}{u}
\agentvar{\sharpnessconstant}{\chi}
\agentvar{\priortolerance}{a}
\agentvar{\priorerror}{e}
\agentvar{\approxerror}{e}
\agentvar{\excursionact}{e}
\agentvar{\globalremainder}{U}
\agentvar{\requestrowcoefficient}{u}
\agentvar{\stateflowrowcoefficient}{v}
\agentvar{\frozenresourcerate}{\mu}

\agentvar{\empiricalthreshold}{\epsilon}
\agentvar{\regret}{\operatorname{Reg}}
\agentvar{\facecapacityratio}{\bar\beta}
\agentvar{\minimumcapacityratio}{\underline\beta}

\agentvar{\hardstateA}{A}
\agentvar{\hardstateB}{B}
\agentvar{\hardstateC}{C}
\agentvar{\requestbonus}{b}
\agentvar{\requestdiscount}{d}

\agentvar{\highreq}{{\mathrm{hi}}}
\agentvar{\lowreq}{{\mathrm{lo}}}
\agentvar{\classreward}{g}
\agentvar{\classAprob}{x}
\agentvar{\wrongclassprob}{e}
\agentvar{\entranceimbalanceamount}{A}
\agentvar{\thresholdfraction}{c}
\agentvar{\firstnonzeroact}{u}
\agentvar{\secondnonzeroact}{v}

\agentvar{\lowerreward}{\varepsilon}
\agentvar{\recoveryact}{r}
\agentvar{\frequentcomparisonpolicy}{\widehat\pi}[][\mathrm F]

%% file: sections/01_introduction.tex
\section{Introduction}

Consider the following scenario.
A service provider that uses large models repeatedly receives requests from users throughout the day.
Because its computing resources are limited, it has a daily token budget and a daily running-time budget for each cluster.
Upon receiving each request, it decides which model should handle the request and on which cluster the task should run.
Different models require different numbers of tokens.
The reward that the service provider receives depends on the quality of the final output and the running time of the task.

This problem is, in general, an online resource allocation problem.
However, it involves additional factors.
The running time of a task depends strongly on the \emph{state} of the cluster on which it is deployed.
A lightly loaded cluster runs at its normal speed but operates significantly more slowly under a heavy load.
Therefore, the cluster states affect both resource consumption and reward. In turn, the service provider's action and the current cluster states determine the next cluster states, either deterministically or randomly.
This state factor is not captured by the classical model of online resource allocation.

More generally, rewards and resource consumption are often affected by an endogenous environmental ``state'' that evolves according to the chosen action.
Suppose an algorithm engineer maintains a retrieval-augmented generation (RAG) knowledge base subject to a fixed token budget.
For each query, she must decide whether to update the knowledge base using additional tokens before answering or to answer directly.
Without an update, the token cost is lower, but the response quality may be worse when the state indicates that the knowledge base is stale.
As another example, a retailer may be building a long-term relationship with a customer.
Its reward from serving the customer's request depends on how much the customer trusts the retailer.
This trust evolves with the quality of the service that the customer receives.

To capture these real-world scenarios, we augment the online resource allocation model with \emph{an endogenous Markov state}.
To explain our model, we first recall the online resource allocation problem.

In the online resource allocation framework, there are $\horizon$ rounds and $\numres$ resources.
The initial total resource inventory is $\inventory[][1] = \nominalratio \cdot \horizon$, where $\nominalratio \in \reals[++][\numres]$ is a constant vector independent of $\horizon$.
In each round $\period$, the decision maker observes an i.i.d.\ request $\reqproc[][\period]$ and chooses an action $\actproc[][\period]$.
This action produces a reward and consumes resources, as captured by
\[
  \rewardproc[][\period] = \reward(\reqproc[][\period], \actproc[][\period]) \geq 0, \qquad
  \consumptionproc[][\period] = \consumption(\reqproc[][\period], \actproc[][\period]) \geq 0.
\]
The decision maker's goal is to maximize the expected total reward subject to the initial resource constraint, that is,
\[
  \max \expect\left[{\sum_{\period = 1}^{\horizon} \rewardproc[][\period]} \right], \qquad \mathrm{s.t.}\, \sum_{\period = 1}^{\horizon} \consumptionproc[][\period] \leq \nominalratio \cdot \horizon.
\]
To ensure feasibility, we assume that there is a null action that always produces zero reward and consumes no resources, regardless of the request.

We extend this classical model as follows: in each round $\period$, the decision maker observes the current state $\stateproc[][\period]$ before observing the request.
The reward and resource consumption under action $\actproc[][\period]$ are determined by $\stateproc[][\period]$ and $\reqproc[][\period]$, as follows:
\[
  \rewardproc[][\period] = \reward(\stateproc[][\period], \reqproc[][\period], \actproc[][\period]), \qquad \consumptionproc[][\period] = \consumption(\stateproc[][\period], \reqproc[][\period], \actproc[][\period]).
\]
Meanwhile, action $\actproc[][\period]$ causes a Markov state transition from $\stateproc[][\period]$ to $\stateproc[][\period + 1]$ under the transition kernel
\[
  \stateproc[][\period + 1] \sim \transition(\cdot \mid \stateproc[][\period], \actproc[][\period]).
\]
We note that the null action produces zero reward and consumption for every state and request but can still induce a state transition.
All other components of the problem are identical to those of the classical online resource allocation problem.

Our goal is to develop algorithms that perform well in this online decision-making problem.

\subsection{Main contributions}
\label{subsec:main-results}

\paragraph{Benchmarks and regularity.}
To evaluate policies for our online resource allocation problem with a Markov state, we consider two closely related benchmarks.
The first is a transient fluid LP, which provides an upper bound on the total expected reward of any nonanticipating policy.
However, this benchmark is ``asymmetric'' because the state in the initial round is fixed, which makes the benchmark difficult to use directly to guide an online policy.
Therefore, we also consider a stationary fluid LP benchmark whose value is smaller than that of the transient fluid LP by at most a constant.
Like its counterpart in the original online resource allocation problem, this LP directly induces a randomized policy in each round by using normalized variables from its optimal solution as action probabilities.

Using the fluid LP benchmark means that regret on the order of $\sqrt{\horizon}$ is unavoidable in the worst case even for state-free online resource allocation. Nevertheless, this choice is necessary: the hindsight benchmark is strictly stronger in the state-free case but is unsuitable in the Markov-state setting.
Without a state, the number of arrivals of each request type determines the hindsight benchmark. With a state, the order in which requests arrive also matters because it affects how the state evolves. Incorporating this order, however, would make the benchmark unattainably large.

With a fluid benchmark, regularity conditions are crucial to algorithmic performance.
We consider three regularity conditions.
First, we impose the following condition throughout the paper.
\begin{itemize}
  \item The stationary LP has a unique optimum.
\end{itemize}
On the algorithmic side, assuming a unique optimum avoids technical complications. Without uniqueness, algorithms based on the stationary LP must specify which optimizer to use, especially to ensure consistency across solves.
To avoid these complications, we focus on the unique-optimum case.

The other two conditions are as follows:
\begin{itemize}
  \item The stationary LP has a nondegenerate optimum.
  \item The optimal induced kernel---the state transition probability matrix on the optimal target state class---is irreducible, and there is an entrance policy that guarantees entry into that class in a constant expected number of rounds and with constant expected resource consumption.
\end{itemize}

The first of these two conditions has already proved important in online resource allocation, separating $\bigO(1)$ regret from $\bigTheta(\sqrt{\horizon})$ regret (Proposition~\ref{prop:unique-degenerate-fluid-lower}).
The second condition is specific to the Markov-state setting.
It states that under the optimal fluid solution, all states that \emph{should be} reached \emph{can} be reached.
As we will show, this condition is both necessary and sufficient for the gap between the LP benchmark and the optimal policy to be sublinear (Proposition~\ref{prop:linear-gap-no-irreducibility}).

\paragraph{Frequent versus infrequent re-solving under a known request prior.}

In the known-request-prior case, we examine two algorithms: frequent re-solving and infrequent re-solving.
Variants of these algorithms have already been studied in the state-free setting, and we extend them to the Markov-state environment.
After entering the optimal target state class, frequent re-solving (FR) generally re-solves the stationary LP using the current average remaining inventory \emph{in each round} and uses the resulting randomized control in that round. By contrast, infrequent re-solving (IR) solves this stationary LP only at carefully scheduled rounds, for a total of $\bigO(\log \log \horizon)$ re-solves.
Previous work on these two methods~\citep{BW20,LWZ25,CLY24} suggests the following observations:
\begin{enumerate}
  \item[\textbf{O1.}] When the stationary LP has a unique nondegenerate optimum, FR has been shown to have $\bigO(1)$ regret. IR also appears to have $\bigO(1)$ regret, but FR performs better than IR\@.
  \item[\textbf{O2.}] When the stationary LP has a unique degenerate optimum, FR has been shown to have $\bigTheta(\sqrt{\horizon})$ regret, while IR also appears to have $\bigTheta(\sqrt{\horizon})$ regret.
  \item[\textbf{O3.}] When the same thresholding/argmax operation is applied to both heuristics, the variant of FR appears to perform better than the corresponding variant of IR\@.
\end{enumerate}

In this work, we revisit these observations for the more general problem with a Markov state. Table~\ref{tab:known-prior-landscape} presents our theoretical results:
\begin{enumerate}
  \item[\textbf{R1.}] When the stationary LP has a unique nondegenerate optimum, both FR and IR have $\bigO(1)$ regret, while IR can perform better than FR\@.
  \item[\textbf{R2.}] When the stationary LP has a unique degenerate optimum, the worst-case regret of IR is $\bigTheta(\sqrt{\horizon})$, while that of FR is $\bigTheta(\horizon)$.
  \item[\textbf{R3.}] Although the thresholded variant FRT can perform much better than FR, our numerical experiments show that it can still perform worse than IRT\@.
\end{enumerate}
Thus, the comparison between frequent and infrequent re-solving can reverse in the presence of a Markov state, as a concrete support to the conceptual result of \citet{Secomandi08}, which indicates that re-solving may not always lead to an improvement.

\begin{table}[t]
  \centering
  \caption{Known-prior regret relative to the transient fluid benchmark.}
  \label{tab:known-prior-landscape}
  \footnotesize
  \setlength{\tabcolsep}{3.2pt}
  \renewcommand{\arraystretch}{1.28}
  \begin{tabularx}{\textwidth}{
    >{\raggedright\arraybackslash}p{0.1\textwidth}
    *{3}{>{\centering\arraybackslash}X}
    }
    \toprule
    \multirow{2}{*}{\textbf{Policy}} & \multicolumn{2}{c}{\textbf{Optimal induced kernel irreducible}} & \textbf{Optimal induced kernel reducible} \\
    \cmidrule(lr){2-3} \cmidrule(lr){4-4}
    & \textbf{unique, non-degen.} & \textbf{unique, degen.} & \textbf{unique, non-degen.} \\
    \midrule
    FR  & $\bigO(1)$ & $\bigTheta(\horizon)$ & $\bigTheta(\horizon)$ \\
    IR  & $\bigO(1)$ & $\bigTheta(\sqrt{\horizon})$ & $\bigTheta(\horizon)$ \\
    Best policy & $\bigO(1)$ & $\bigTheta(\sqrt{\horizon})$ & $\bigTheta(\horizon)$ \\
    \bottomrule
  \end{tabularx}

  \vspace{0.35em}

  \parbox{\textwidth}{\emph{Note.} The lower bounds, and hence the $\bigTheta(\cdot)$ entries, are worst-case orders over the corresponding classes of instances.}
\end{table}

Why can IR perform better than FR in the presence of a Markov state, contrary to the pattern observed in the state-free setting?
Our counterexamples suggest the following mechanism.
A policy may move the system to a temporary state because it also prescribes a recovery action there. After the system has moved, however, frequent re-solving can produce a new fluid optimizer that assigns zero stationary probability to that state and \emph{omits the recovery action}.
Under infrequent re-solving, holding the policy fixed within an epoch preserves the sequence of decisions that makes the initial action worthwhile.
Thus, with a dynamically evolving state, retaining an epoch-wise stationary policy for some time can allow the fluid-LP objective to be achieved asymptotically, whereas changing the policy too frequently can be harmful.

\paragraph{Three-phase infrequent re-solving under an unknown request prior.}

\begin{table}[t]
  \centering
  \caption{Unknown-prior regret relative to the transient fluid benchmark.}
  \label{tab:unknown-prior-landscape}
  \footnotesize
  \setlength{\tabcolsep}{4pt}
  \renewcommand{\arraystretch}{1.3}
  \begin{tabularx}{\textwidth}{>{\raggedright\arraybackslash}p{0.13\textwidth}*{3}{>{\centering\arraybackslash}X}}
    \toprule
    \multirow{2}{*}{\textbf{Policy}} & \multicolumn{2}{c}{\textbf{Known optimal target state class}} & \textbf{Unknown optimal target state class} \\
    \cmidrule(lr){2-3} \cmidrule(lr){4-4}
    & \textbf{unique, non-degen.} & \textbf{unique, degen.} & \textbf{unique, degen.} \\
    \midrule
    UIR-UP & $\bigO(1)$ & $\bigTheta(\sqrt{\horizon})$ & --- \\
    Best policy & $\bigO(1)$ & $\bigTheta(\sqrt{\horizon})$ & $\bigTheta(\horizon)$ \\
    \bottomrule
  \end{tabularx}

  \vspace{0.35em}
  
  \parbox{\textwidth}{\emph{Note.} The lower bounds, and hence the $\bigTheta(\cdot)$ entries, are worst-case orders over the corresponding classes of instances.}
\end{table}

In the unknown-request-prior case, under the assumption that the optimal target state class is known, we propose a three-phase infrequent re-solving policy.
The three phases serve different purposes.
The first phase, with $\bigO(\log \log \log \horizon)$ re-solves, aims to correct inaccurate early prior estimates quickly when few request samples are available.
This phase lasts only $\bigO(\log^2 \horizon)$ rounds.
The second phase, which starts at the end of the first phase and ends at the midpoint of the horizon, balances learning error and inventory consumption.
For this phase, we design a new schedule that differs from those in previous work and guarantees $\bigO(\sqrt{\horizon})$ regret in the degenerate case relative to the fluid benchmark.
This schedule also has $\bigO(\log \log \horizon)$ re-solves.
The construction of this schedule may be of independent interest.
Informally, it satisfies
\[
  \sum_{\epoch = 0}^{\midSchedule - 1}
  \frac{\period[\epoch + 1] - \period[\epoch]}{\sqrt{\period[\epoch]}} = \bigO(\sqrt{\horizon}),
\]
where $\midSchedule = \bigO(\log \log \horizon)$ is the number of re-solves, $\period[0] = \bigTheta(\log^2 \horizon)$, and $\period[\midSchedule] = \bigTheta(\horizon)$.
The last phase mirrors the designs of \citet{BW20,LWZ25} by re-solving increasingly frequently near the end, when inventory errors become costly.

Our policy, which we call U-shaped IR under unknown priors (UIR-UP), has the following \emph{optimal} regret guarantees:
\begin{enumerate}
  \item When the stationary LP has a unique and nondegenerate optimum, it has $\bigO(1)$ regret.
  \item When the stationary LP has a unique and degenerate optimum, it has $\bigO(\sqrt{\horizon})$ regret.
\end{enumerate}

Knowledge of the optimal target state class is required by our UIR-UP policy and is necessary for achieving sublinear regret.
To show this, we consider two request distributions with different optimal target state classes that cannot be distinguished from observations sufficiently early.
Therefore, without knowing either the optimal target state class or the request prior, any policy incurs linear regret on at least one of these two instances.
All theoretical results for the unknown-request-prior case are presented in Table~\ref{tab:unknown-prior-landscape}.

In our numerical experiments, we compare four policies in the unknown-request-prior case: UIR-UP; FR-UP, a direct extension of FR to the unknown-prior setting; UIRT-UP; and FRT-UP\@.
The last two are the thresholded variants of UIR-UP and FR-UP, respectively.
Our empirical results are similar to those in the known-prior setting: FR-UP can still incur linear regret, and although thresholding prevents this trend, FRT-UP still performs worse than UIR-UP and UIRT-UP in most cases.

To further demonstrate the importance of re-solving frequency, we modify FR and FR-UP so that, after the first round, each policy independently chooses whether to re-solve the stationary LP in every round, with a constant probability.
Our results show that, when the optimum is nondegenerate, algorithmic performance improves as this re-solving probability decreases, regardless of whether the request prior is known.
However, this modification does not affect the regret order of FR or FR-UP, demonstrating the importance of the re-solving schedules used by IR and UIR-UP\@.

\subsection{Related work}
\label{subsec:related-works}

The extensive literature on online resource allocation under the small-item and long-horizon assumptions has developed along several major methodological lines.

One important line consists of primal-dual methods, motivated in part by the success of gradient-based optimization in machine learning.
These methods learn or adjust shadow prices and then use them to make online allocation decisions~\citep{AgrawalWangYe14,LiSunYe20,SunWangZhou20}.
These methods are often modular and can accommodate more general action spaces, including continuous ones~\citep{JiangLiZhang20,BalseiroLuMirrokni23,CastiglioniCelliKroer22}, as well as more restrictive feedback models, such as bandit feedback~\citep{HanZengWangXiangZhang22,SlivkinsZhouSankararamanFoster24}.
Many early and broadly applicable methods, however, incur an $\bigO(\sqrt{\horizon})$ gap from optimal performance.
More recent work obtains logarithmic or constant regret under different regularity assumptions~\citep{LiYe22,Bray24,MaCaoTsangXia25,GaoGeSunXueYe26,Gupta24,HeWeiXuYu25}.

Our work belongs to the line of re-solving methods.
These methods periodically recompute a deterministic fluid problem and implement the resulting control until the next update.
For example, \citet{ReimanWang08} show that a single re-solve can already yield $o(\sqrt{\horizon})$ regret.
A central question in this line of work concerns \emph{regularity conditions}.
For example, \citet{JK12,JK13} establish that frequent re-solving achieves constant regret under nondegeneracy assumptions.
Later work studies how re-solving algorithms behave without these assumptions and how they can be enhanced to perform well in their absence.
Regarding the first issue, \citet{CLY24} show that frequent re-solving still achieves $\tilde{\bigO}(\sqrt{\horizon})$ regret when nondegeneracy does not hold, and \citet{BalseiroBesbesPizarro24} study a wide range of regularity conditions that yield constant, logarithmic, or square-root regret.
Regarding the second issue, which is most closely related to our work, \citet{BW20,LWZ25} develop infrequent re-solving, which solves the LP according to a carefully designed schedule rather than in every round. They also study thresholding and argmax operations that post-process the LP solutions and substantially improve the performance of frequent and infrequent re-solving.
However, their results and observations suggest that, under the same post-processing rule (none, thresholding, or argmax), frequent re-solving still performs better than infrequent re-solving.
Our results show that this comparison can reverse in the presence of a Markov state, which also support the result of \citet{Secomandi08} -- this work shows abstractly that re-solving a mathematical-programming-based control algorithm need not improve performance when sequential consistency fails.
We mention that re-solving methods are also used for continuous request and action spaces~\citep{MaglarasMeissner06,Jasin14,Jasin15,ArlottoXie20,JiangZhang20,MiaoWang21,WangWang22,BesbesKanoriaKumar22} and for contextual bandits and decision-making with knapsack constraints~\citep{WuSrikantLiuJiang15,CDMK24} in machine learning.

A closely related recent line uses compensated coupling.
These methods compare an online trajectory with an offline or fluid benchmark by assigning a marginal value to the loss from each online decision~\citep{VB21,VBG21,JiangMaZhang25,BanerjeeFreund20,BanerjeeFreund25}.
Other relevant approaches include budget-ratio policies~\citep{ArlottoGurvich19,VeraArlottoGurvichLevin25}.
Compensated-coupling analyses yield $\bigO(1)$ or $\bigO(\log \horizon)$ regret, particularly for discrete request and action spaces. These settings substantially overlap with those studied using re-solving methods.

\paragraph{Markov dependence.}
Recent work incorporates Markov dependence into online resource allocation models in several ways.
Most extensions use an exogenous chain to model demand.
In \citet{FV07}, an observed demand mode evolves as a Markov chain and determines fare-product arrival probabilities; before each arrival, the firm chooses which feasible products to offer.
In \citet{Jiang26}, the observed state follows a known time-inhomogeneous Markov chain and determines the arriving customer's type.
\citet{LRTZ25} use a two-timescale model: an observed Markov chain moves once per stage, and its current state determines the distribution of the product requests arriving within that stage.
In all three models, allocation decisions deplete cumulative inventory but do not affect the demand-modulating chain.
Our requests are i.i.d.\ and observed before the decision; instead, the chosen action controls the state transition, and the resulting state affects later rewards and resource consumption.

The models of \citet{BrownZhang22} and \citet{BrownZhang25} are closer to ours because their actions also affect future states.
\citet{BrownZhang22} consider many project-specific controlled Markov states together with a common exogenous Markov signal that can affect rewards, transitions, resource use, and available capacity.
\citet{BrownZhang25} consider many independent, possibly heterogeneous, controlled Markov subproblems without that common signal.
In both models, the decision maker acts on all subproblems in parallel, and shared resource constraints couple those decisions within each period.
Our model instead has one persistent controlled state, one i.i.d.\ request per period, and cumulative inventory that can be allocated over the horizon.
Their large-system scaling increases the number of subproblems and the available per-period capacity, whereas our asymptotic parameter is the horizon.
Accordingly, their fluid state describes a cross-section of many simultaneous subproblems, while our stationary occupation-measure LP describes time allocation along one controlled trajectory.

%% file: sections/02_model.tex
\section{Model and benchmarks}
\label{sec:model}

We study an online decision-making problem.
There are $\horizon$ rounds and $\numres$ resources.
The initial resource inventory is $\inventory[][1] = \nominalratio \cdot \horizon$, where $\nominalratio \in \reals[++][\numres]$ is a constant vector independent of $\horizon$.
In each round $\period$, the decision maker first observes the current state $\stateproc[][\period] \in \states$ and a request $\reqproc[][\period] \in \requests$.
We assume that the state space $\states$ is finite and that requests $\reqproc[][1], \reqproc[][2], \ldots$ are i.i.d.\ with prior distribution
\[
  \prior = (\prior[\req])_{\req \in \requests},
  \qquad \prior[\req] > 0,
  \qquad \sum_{\req \in \requests}\prior[\req] = 1.
\]
For the first part of our results, we assume that this prior distribution is known before any decision is made; we then extend the analysis to the unknown-prior case in Section~\ref{sec:unknown-prior}.

Given these observations, the decision maker chooses an action $\actproc[][\period]$ from a finite action space $\actions$, producing reward
\[
  \rewardproc[][\period] = \reward(\stateproc[][\period], \reqproc[][\period], \actproc[][\period]) \in [0, \maxreward],
\]
and a resource-consumption vector
\[
  \consumptionproc[][\period] = \consumption(\stateproc[][\period], \reqproc[][\period], \actproc[][\period]) \in [0, \maxconsumption]^{\numres}.
\]
The deterministic reward and resource-consumption functions
\[
  \reward:\states \times \requests \times \actions\to\reals[+],
  \qquad
  \consumption:\states \times \requests \times \actions\to\reals[+][\numres]
\]
are known.
To ensure that the decision-making problem is feasible, there is a null action $\act[][0] \in \actions$ such that
\[
  \reward(\state, \req, \act[][0]) = 0,
  \quad
  \consumption(\state, \req, \act[][0]) = 0, \quad \forall \state \in \states, \req \in \requests,
\]
while every non-null action has a nonzero resource-consumption vector:
\[
  \act\ne \act[][0] \Longrightarrow \consumption(\state, \req, \act) \ne 0
  \quad\ \forall (\state, \req) \in \states \times \requests.
\]

The chosen action $\actproc[][\period]$ also induces a Markov state transition from $\stateproc[][\period]$ to $\stateproc[][\period + 1]$ according to the \emph{known} transition kernel
\[
  \stateproc[][\period + 1] \sim \transition(\cdot \mid \stateproc[][\period], \actproc[][\period]).
\]
In fact, we will show in Section~\ref{subsec:unknown-R-lower-bound} that, when the request prior is unknown, knowledge of the optimal target state class is necessary to achieve good performance relative to the benchmark introduced below.
The state transition is independent of the request. We assume that the initial state $\stateproc[][1]$ is fixed.

After the action in round $\period$, the remaining resource inventory is updated to
\[
  \inventory[][\period + 1] = \inventory[][\period] - \consumptionproc[][\period].
\]
The procedure stops after the first round in which any resource is depleted or after round $\horizon$, whichever comes first. Thus, the stopping time is
\begin{equation}
  \stoptime
  \coloneqq
  \min\left\{
  \inf\left\{\period \in [\horizon]:\min_{\res \in [\numres]} \inventory[\res][\period + 1] \le 0\right\},
  \horizon
  \right\}.
  \label{eq:stopping}
\end{equation}
Equivalently, the decision maker always chooses the null action after $\stoptime$.

Let $\filtration[][\period]$ be the history generated by all observations and randomizations strictly before the arrival of request $\reqproc[][\period]$, together with the current state $\stateproc[][\period]$ and inventory $\inventory[][\period]$. A \emph{nonanticipating} policy $\policy$ chooses action $\actproc[][\period]$ based only on $\filtration[][\period]$ and $\reqproc[][\period]$. Its value is the expected total reward, including the reward in the final round:
\[
  \valuefn[\horizon][\policy](\horizon\nominalratio, \state[][1])
  \coloneqq
  \expect^{\policy}[\totalreward[\horizon][\policy]],
  \qquad
  \totalreward[\horizon][\policy] \coloneqq \sum_{\period = 1}^{\stoptime}\rewardproc[][\period].
\]
The optimal nonanticipating value, that is, the optimal total expected reward among all nonanticipating policies, is
\begin{equation*}
  \valuefn[\horizon][*](\horizon\nominalratio, \state[][1])
  \coloneqq
  \sup_{\text{nonanticipating }\policy}\valuefn[\horizon][\policy](\horizon\nominalratio, \state[][1]).
\end{equation*}

\subsection{The transient fluid LP upper bound}

We formulate an upper-bound benchmark for a generic subproblem with $\subhorizon \ge 1$ remaining rounds, current inventory $\inventory \in \reals[+][\numres]$, and current state $\state[][0] \in \states$.

For $\state \in \states$, $\req \in \requests$, and $\act \in \actions$, let
\begin{itemize}[leftmargin=2.2em]
  \item $\nonterminalocc[\state\req\act] \ge 0$ be the expected number of \emph{nonterminal} rounds with state $\state$, request $\req$, and action $\act$;
  \item $\terminalocc[\state\req\act] \ge 0$ be the probability that the terminal round has state $\state$, request $\req$, and action $\act$;
  \item $\transientstateocc[\state] \ge 0$ be the expected number of active rounds spent in state $\state$, including the terminal round.
\end{itemize}

Define $\transientfluid[\subhorizon](\inventory, \state[][0])$ as the optimal value of the following LP:
\begin{equation}
  \boxed{
  \begin{aligned}
    \max_{\nonterminalocc, \terminalocc, \transientstateocc \ge 0}\quad
    &
    \sum_{\state \in \states}\sum_{\req \in \requests}\sum_{\act \in \actions}
    \reward(\state, \req, \act)\bigl(\nonterminalocc[\state\req\act] + \terminalocc[\state\req\act]\bigr)
    \\[0.3em]
    \text{subject to}\quad
    &
    \sum_{\act \in \actions}\bigl(\nonterminalocc[\state\req\act] + \terminalocc[\state\req\act]\bigr)
    = \prior[\req]\transientstateocc[\state],
    & & \state \in \states, \ \req \in \requests,
    \qquad\text{\rm (TF1)}
    \\
    &
    \transientstateocc[\state[][\prime]] = \indicator\{\state[][\prime] = \state[][0]\}
    +\sum_{\state \in \states}\sum_{\req \in \requests}\sum_{\act \in \actions}
    \transition(\state[][\prime] \mid \state, \act)\nonterminalocc[\state\req\act],
    & & \state[][\prime] \in \states,
    \qquad\text{\rm (TF2)}
    \\
    &
    \sum_{\state \in \states}\transientstateocc[\state] \le \subhorizon,
    \qquad\text{\rm (TF3)}
    \\
    &
    \sum_{\state \in \states}\sum_{\req \in \requests}\sum_{\act \in \actions}
    \consumption[\res](\state, \req, \act)\nonterminalocc[\state\req\act] \le \inventory[\res],
    & & \res \in [\numres].
    \qquad\text{\rm (TF4)}
  \end{aligned}
  \label{eq:transient-fluid}
  }
\end{equation}

The terminal variables $\terminalocc$ contribute reward without consuming fluid capacity. This relaxation accounts for the convention that only the resource consumption incurred strictly before the terminal action must fit within the initial inventory.

The following lemma shows that this LP gives an upper bound for every subproblem.

\begin{lemma}
  \label{lem:fluid-upper-bound}
  For every $\subhorizon \ge 1$, $\inventory \in \reals[+][\numres]$, and $\state[][0] \in \states$,
  \begin{equation*}
    \valuefn[\subhorizon][*](\inventory, \state[][0]) \le \transientfluid[\subhorizon](\inventory, \state[][0]).
  \end{equation*}
  In particular, we have $\valuefn[\horizon][*](\horizon\nominalratio, \state[][1]) \le \transientfluid[\horizon](\horizon\nominalratio, \state[][1])$.
\end{lemma}

Accordingly, we define the \emph{regret} of a nonanticipating policy $\policy$ as the gap between the value of the transient fluid LP and the policy's value:
\[\transientfluid[\horizon](\horizon\nominalratio, \state[][1]) - \valuefn[\horizon][\policy](\horizon\nominalratio, \state[][1]).\]

\subsection{The stationary LP}

Although the transient fluid LP gives an upper bound, it is difficult to use because of the asymmetry between the variables $\{\nonterminalocc[\state\req\act]\}$ and $\{\terminalocc[\state\req\act]\}$.
We next define the stationary LP for any per-period capacity vector $\capacityratio \in \reals[+][\numres]$.
This LP extends the fluid LP widely studied in the online resource allocation literature (e.g., \cite{JK12,Jasin15,BW20}) to the Markov-state setting.
In this LP, $\stateoccupation[\state]$ is the stationary probability of state $\state$, while $\occupation[\state\req\act]$ is the joint stationary probability of state $\state$, request $\req$, and action $\act$, or equivalently, the per-round stationary flow through that state-request-action triple.
\begin{equation}
  \boxed{
  \begin{aligned}
    \stationaryfluid(\capacityratio) \coloneqq
    \max_{\occupation, \stateoccupation \ge 0}\quad
    &
    \sum_{\state \in \states}\sum_{\req \in \requests}\sum_{\act \in \actions}
    \reward(\state, \req, \act)\occupation[\state\req\act]\\[0.3em]
    \text{subject to}\quad
    &
    \sum_{\act \in \actions}\occupation[\state\req\act] = \prior[\req]\stateoccupation[\state],
    & & \state \in \states, \ \req \in \requests,
    \qquad\text{\rm (SF1)}
    \\
    &
    \stateoccupation[\state[][\prime]] = \sum_{\state \in \states}\sum_{\req \in \requests}\sum_{\act \in \actions}
    \transition(\state[][\prime] \mid \state, \act)\occupation[\state\req\act],
    & & \state[][\prime] \in \states,
    \qquad\text{\rm (SF2)}
    \\
    &
    \sum_{\state \in \states}\stateoccupation[\state] = 1,
    \qquad\text{\rm (SF3)}
    \\
    &
    \sum_{\state \in \states}\sum_{\req \in \requests}\sum_{\act \in \actions}
    \consumption[\res](\state, \req, \act)\occupation[\state\req\act] \le \capacityratio[\res],
    & & \res \in [\numres].
    \qquad\text{\rm (SF4)}
  \end{aligned}
  }
  \label{eq:stationary-fluid}
\end{equation}

First, the LP above is always feasible: choose the null action and any stationary distribution of the transition matrix induced by that action.
Second, an optimal solution $(\occupation(\capacityratio), \stateoccupation(\capacityratio))$ naturally induces the following randomized policy at every state $\state$ for which $\stateoccupation[\state](\capacityratio) > 0$:
\begin{equation}
  \policy[\capacityratio](\act \mid \state, \req)
  \coloneqq
  \frac{\occupation[\state\req\act](\capacityratio)}{\prior[\req]\stateoccupation[\state](\capacityratio)}.
  \label{eq:lp-induced-policy}
\end{equation}
Constraint \textup{(SF1)} ensures that these probabilities sum to one. Furthermore, the induced Markov state transition matrix is
\begin{equation}
  \inducedkernel[\capacityratio](\state, \state[][\prime])
  \coloneqq
  \sum_{\req \in \requests}\prior[\req]\sum_{\act \in \actions}\policy[\capacityratio](\act \mid \state, \req)\transition(\state[][\prime] \mid \state, \act).
  \label{eq:induced-kernel}
\end{equation}
Constraint \textup{(SF2)} implies that $\stateoccupation(\capacityratio)$ is stationary for $\inducedkernel[\capacityratio]$.

However, this solution does not directly determine an action rule at a state with $\stateoccupation[\state](\capacityratio) = 0$. We therefore consider the natural extension
\begin{equation}
  \policy[\capacityratio](\act[][0] \mid \state, \req) = 1,
  \qquad \stateoccupation[\state](\capacityratio) = 0.
  \label{eq:null-completion}
\end{equation}
This extension always chooses the null action in this case. We call it the \emph{null-completion convention}. 
This is actually a very natural convention---when the process is in a state that should not be entered, the most natural choice is to remain inactive, consuming no resources and gaining no reward until the process is driven out of the state. Our algorithms below use this convention at states assigned zero stationary probability.

Finally, we show that $\transientfluid[\horizon](\horizon\nominalratio, \state[][1])$ exceeds $\horizon \stationaryfluid(\nominalratio)$ by at most a constant.
Therefore, up to an additive constant, we can use the more convenient $\horizon \stationaryfluid(\nominalratio)$ as a proxy for the transient fluid LP upper bound in the subsequent analysis, since our algorithms also rely on this homogeneous core LP\@.
\begin{lemma}
  \label{lem:boundary-comparison}
  For every $\horizon$,
  \begin{equation}
    \transientfluid[\horizon](\horizon\nominalratio, \state[][1]) \le \horizon\stationaryfluid(\nominalratio) + O(1).
    \label{eq:boundary-comparison}
  \end{equation}
\end{lemma}

\subsection{Why not the hindsight optimum?}

Previous work on online resource allocation has obtained universal constant-regret results by using the \emph{hindsight optimum} as the benchmark \citep{BW20,LWZ25}.
In the state-free setting, the hindsight optimum is the expected optimal reward attainable by a policy, possibly an anticipatory one, that knows the realization of all future requests.
This benchmark is naturally an upper bound on the value of any nonanticipating policy. In the state-free setting, it is also tighter than the stationary LP: the gap between $\horizon\stationaryfluid(\nominalratio)$ and the hindsight benchmark can be as large as $\bigTheta(\sqrt{\horizon})$ when $\stationaryfluid(\nominalratio)$ has a degenerate optimum, as discussed in Section~\ref{sec:regularity}.
As a direct consequence, adopting $\horizon\stationaryfluid(\nominalratio)$ as the benchmark inevitably produces an $\bigOmega(\sqrt{\horizon})$ regret lower bound in certain settings.

So why not use the hindsight optimum in our problem?
The main reason is that, in the state-free case, the order in which requests arrive does not affect the hindsight optimum as long as the number of arrivals of each request type is fixed. Thus, for each realization, the only decision variables are the numbers of times that each action is taken for each request type.
This property does not hold in our Markov-state setting: the order of requests and its interaction with the chosen actions affect the random state transitions and, consequently, resource consumption and reward collection.

Consequently, the randomness of the state transitions must be incorporated into the definition of the hindsight optimum, but the resulting benchmark is too strong.
\citet{CDMK24} show that such an ``optimum'' can exceed the value of the optimal nonanticipating policy by $\bigTheta(\horizon)$.
Therefore, we use the transient fluid LP as a conservative benchmark and the stationary LP as its analytically convenient proxy, and we regard finding a tighter benchmark as an important direction for future research.

%% file: sections/03_regularity_assumptions.tex
\section{Regularity assumptions}
\label{sec:regularity}

This section states the conditions that govern the benchmarks introduced above and the algorithms analyzed below.
These conditions concern the stationary LP at the nominal per-period capacity vector $\nominalratio$.

Throughout this work, we assume that the stationary LP has a unique optimum.

\begin{assumption}[Unique optimality]
  \label{ass:nominal-uniqueness}
  The stationary LP $\stationaryfluid(\nominalratio)$ has a unique optimal solution.
\end{assumption}

Our next condition concerns primal nondegeneracy in the fixed standard-form representation of the stationary LP\@.

\begin{assumption}[Nondegeneracy]
  \label{ass:nominal-nondegeneracy}
  The optimal solution of the stationary LP $\stationaryfluid(\nominalratio)$ is nondegenerate.
\end{assumption}

Together, Assumptions~\ref{ass:nominal-uniqueness} and~\ref{ass:nominal-nondegeneracy} imply local stability. Suppose they hold, and let the unique optimal solution be $(\occupation[][*], \stateoccupation[][*])$.
Define its resource slacks, which we call the nominal slacks, by
\[
  \resourceslack[\res][*] \coloneqq \nominalratio[\res] - \sum_{\state, \req, \act}\consumption[\res](\state, \req, \act)\occupation[\state\req\act][*],
  \qquad \res \in [\numres],
\]
and let
\begin{equation}
  \bindingset \coloneqq \{\res:\resourceslack[\res][*] = 0\},
  \qquad
  \recstates \coloneqq \{\state:\stateoccupation[\state][*] > 0\},
  \qquad
  \actsupport \coloneqq \{(\state, \req, \act):\occupation[\state\req\act][*] > 0\}.
  \label{eq:nominal-supports}
\end{equation}
Thus $\bindingset$ is the set of nominally binding resources, $\recstates$ is the \emph{optimal target state class} (equivalently, the active-state set of the nominal optimum), and $\actsupport$ is the set of active decisions.

To state the stability properties, we first define a one-sided neighborhood for any scalar $\neighborhoodradius > 0$:
\begin{equation*}
  \capacityneighborhood[\neighborhoodradius] \coloneqq
  \left\{
  \capacityratio \in \reals[+][\numres]:
  \begin{array}{ll}
    |\capacityratio[\res] - \nominalratio[\res]| < \neighborhoodradius, & \res \in \bindingset;\\
    \capacityratio[\res] > \nominalratio[\res] - \neighborhoodradius, & \res \notin \bindingset
  \end{array}
  \right\}.
\end{equation*}
The corresponding one-sided distance from a capacity ratio $\capacityratio$ to $\nominalratio$ is
\begin{equation}
  \capacitydistance(\capacityratio)
  \coloneqq \max\left\{
  \max_{\res \in \bindingset}|\capacityratio[\res] - \nominalratio[\res]|,
  \max_{\res \notin \bindingset}(\nominalratio[\res] - \capacityratio[\res])_+
  \right\}.
  \label{eq:one-sided-capacity-distance}
\end{equation}
Here and below, a maximum over an empty resource index set is interpreted as zero.
Equivalently, $\capacityratio \in \capacityneighborhood[\neighborhoodradius]$ if and only if $\capacitydistance(\capacityratio) < \neighborhoodradius$.
For the binding coordinates, define
\[
  \bindingbox[\neighborhoodradius] \coloneqq
  \left\{\bindingparameter \in \reals[][|\bindingset|]:
  |\bindingparameter[\res] - \nominalratio[\res]| \le \neighborhoodradius
  \text{ for every }\res \in \bindingset\right\}.
\]

We have the following result.
\begin{proposition}
  \label{prop:local-lp-stability}
  Under Assumptions~\ref{ass:nominal-uniqueness} and~\ref{ass:nominal-nondegeneracy}, there exist constants $\localradius, \positivitymargin > 0$, with $\localradius < \min_{\res \in [\numres]}\nominalratio[\res]$,
  such that, for every $\capacityratio \in \capacityneighborhood[3\localradius]$,
  \begin{enumerate}[label=(\roman*),leftmargin=2.5em]
    \item the unique optimal solution $(\occupation(\capacityratio), \stateoccupation(\capacityratio))$ of $\stationaryfluid(\capacityratio)$ has binding-resource set $\bindingset$, active-state set $\recstates$, and active-decision set $\actsupport$;
    \item every nonzero coordinate of $(\occupation(\capacityratio), \stateoccupation(\capacityratio))$ and every non-binding resource slack of $\stationaryfluid(\capacityratio)$ is at least $\positivitymargin$;
    \item every coordinate of $\occupation(\capacityratio)$ and $\stateoccupation(\capacityratio)$ is affine in $\capacityratio[\bindingset]$ and does not depend on $\capacityratio[\res]$ for $\res \notin \bindingset$.
  \end{enumerate}
\end{proposition}

Another important condition specific to the Markov-state setting is the following.

\begin{assumption}[Irreducibility of the optimal induced kernel]
  \label{ass:support-irreducibility}
  The \emph{optimal induced kernel}, namely, the state transition probability matrix on $\recstates$ induced by the optimal policy, is defined by
  \begin{equation*}
    \inducedkernel[][*](\state, \state[][\prime])
    \coloneqq \sum_{\req}\prior[\req]\sum_{\act}\policy[\nominalratio](\act \mid \state, \req)\transition(\state[][\prime] \mid \state, \act),
    \qquad \state, \state[][\prime] \in \recstates.
  \end{equation*}
  This matrix is irreducible on $\recstates$.
\end{assumption}

This assumption concerns movement within the optimal target state class $\recstates$: once the process is in $\recstates$, every state in the class communicates with every other under $\policy[\nominalratio]$.
We also have the following local properties.
\begin{proposition}
  \label{prop:support-regularity}
  Under Assumptions~\ref{ass:nominal-uniqueness},~\ref{ass:nominal-nondegeneracy}, and~\ref{ass:support-irreducibility}, there exists a constant $\localradius > 0$, with $\localradius < \min_{\res \in [\numres]}\nominalratio[\res]$, such that the following hold for every $\capacityratio \in \capacityneighborhood[3\localradius]$.
  \begin{enumerate}[label=(\roman*),leftmargin=2.5em]
    \item $\recstates$ is closed under $\inducedkernel[\capacityratio]$.
    \item Restricted to $\recstates$, $\stateoccupation(\capacityratio)$, $\policy[\capacityratio]$, and $\inducedkernel[\capacityratio]$, viewed as functions of $\capacityratio[\bindingset]$, are uniformly bounded and Lipschitz on the closed box $\bindingbox[3\localradius]$.
    \item $\inducedkernel[\capacityratio]$ is irreducible on $\recstates$.
  \end{enumerate}
\end{proposition}

The following result shows that, under nondegeneracy, the optimal target state class is the entire state space; hence, the optimal induced kernel is a transition matrix on the entire state space.
\begin{proposition}
  \label{prop:nondegenerate-full-state-support}
  Under Assumption~\ref{ass:nominal-nondegeneracy},
  \begin{equation*}
    \stateoccupation[\state][*] > 0, \,\forall\state \in \states;
    \quad \recstates = \states.
  \end{equation*}
\end{proposition}
Thus the initial state $\state[][1]$ belongs to $\recstates$ under nondegeneracy, so entrance is immediate.
However, under degeneracy, reachability of $\recstates$ is not automatic, so we impose it separately.

\begin{assumption}[Constant-cost entrance policy]
  \label{ass:constant-entrance}
  There exists a fixed nonanticipating \emph{entrance policy} $\entrancepolicy$, independent of $\horizon$, such that, for the corresponding infinite-resource process,
  \begin{equation*}
    \entrancelast \coloneqq \inf\{\period \ge 1:\stateproc[][\period] \in \recstates\}
  \end{equation*}
  is almost surely finite and
  \begin{equation}
    \expect[\entrancelast] = \bigO(1), \quad
    \expect\left[\sum_{\period = 1}^{\entrancelast - 1}\|\consumption(\stateproc[][\period], \reqproc[][\period], \actproc[][\period])\|_1\right]
    = \bigO(1).
    \label{eq:entrance-bounds}
  \end{equation}
  If $\state[][1] \in \recstates$, we take $\entrancelast = 1$ and the entrance phase is empty.
\end{assumption}

Assumption~\ref{ass:constant-entrance} is automatic under nondegeneracy by Proposition~\ref{prop:nondegenerate-full-state-support}: the entrance phase is empty.
In the degenerate case, we focus on \emph{entrance-augmented algorithms}.
They first execute the policy $\entrancepolicy$ specified by Assumption~\ref{ass:constant-entrance} until the beginning of the first round $\entrancelast$ for which $\stateproc[][\entrancelast] \in \recstates$. If the horizon ends or a resource is exhausted before such a round, the process terminates according to the original model.

%% file: sections/04_resolving_algorithms_and_guarantees.tex
\section{Frequent and infrequent re-solving under regularity}
\label{sec:algorithms}

In this section, we study two algorithms that re-solve the stationary LP at specified rounds using the average remaining inventory. As their names suggest, the main difference between them is their re-solving frequency.
Both algorithms have been studied extensively in the online resource allocation literature and can achieve constant-regret guarantees under \emph{regularity assumptions} (Assumptions~\ref{ass:nominal-uniqueness} and~\ref{ass:nominal-nondegeneracy})~\citep{BW20,CLY24}.
We extend these results to the Markov-state setting by further imposing the irreducibility assumption (Assumption~\ref{ass:support-irreducibility}).

\subsection{Algorithms}

We first specify how an LP optimizer is selected.  Although
$\stationaryfluid(\nominalratio)$ has a unique optimizer under
Assumption~\ref{ass:nominal-uniqueness}, the LPs evaluated at the perturbed
capacity ratios encountered by FR and IR may have several optimizers.  Fix
once and for all a deterministic measurable selector
\[
  \mathfrak S:\capacityratio\longmapsto
  \selectedsolution(\capacityratio)
  =\bigl(\occupation(\capacityratio),\stateoccupation(\capacityratio)\bigr)
\]
such that $\selectedsolution(\capacityratio)$ is an optimal solution of
$\stationaryfluid(\capacityratio)$ for every $\capacityratio$; for example,
one may use lexicographic tie-breaking with a fixed ordering of all LP
variables.  Both FR and IR use this same selector whenever they solve the
stationary LP\@. Thus the same LP input always yields the same selected optimizer. The word ``deterministic'' applies only to this tie-breaking
rule: after the optimizer is selected, the algorithms may still randomize
their actions according to the induced policy in
\eqref{eq:lp-induced-policy}.  
We impose no continuity requirement on
$\mathfrak S$.

The frequent re-solving (FR) algorithm is given in Algorithm~\ref{alg:frequent-resolving}. After entering the optimal target state class $\recstates$, it solves the stationary LP using the current average remaining inventory in every round.

\begin{algorithm}[t]
  \caption{Frequent re-solving ($\frequentalg$)}
  \label{alg:frequent-resolving}
  \begin{algorithmic}[1]
    \State Run $\entrancepolicy$ until successful entrance at the beginning of round $\entrancelast$, or until the process stops.
    \If{the process stops before successful entrance}
      \State \Return
    \EndIf
    \For{$\period = \entrancelast, \ldots, \horizon$}
      \State Set $\subhorizon[][\period] \gets \horizon - \period + 1$ and $\capacityratio[][\period] \gets \inventory[][\period] / \subhorizon[][\period]$.
      \State Solve $\stationaryfluid(\capacityratio[][\period])$ and select
      $(\occupation(\capacityratio[][\period]),\stateoccupation(\capacityratio[][\period]))
      \gets \mathfrak S(\capacityratio[][\period])$.
      \State Observe request $\reqproc[][\period]$ and the current state $\stateproc[][\period]$.
      \If{$\stateoccupation[\stateproc[][\period]](\capacityratio[][\period]) = 0$}
        \State Set $\actproc[][\period] \gets \act[][0]$.
      \Else
        \State Draw $\actproc[][\period] \sim \policy[\capacityratio[][\period]](\cdot \mid \stateproc[][\period], \reqproc[][\period])$ using \eqref{eq:lp-induced-policy}.
      \EndIf
      \State Execute $\actproc[][\period]$, earn reward $\rewardproc[][\period]$, and incur consumption $\consumptionproc[][\period]$.
      \State Set $\inventory[][\period + 1] \gets \inventory[][\period] - \consumptionproc[][\period]$ and generate $\stateproc[][\period + 1]$ according to $\transition(\cdot \mid \stateproc[][\period], \actproc[][\period])$.
      \If{$\min_{\res \in [\numres]}\inventory[\res][\period + 1] \le 0$}
        \State \Return
      \EndIf
    \EndFor
  \end{algorithmic}
\end{algorithm}

The infrequent re-solving algorithm (Algorithm~\ref{alg:infrequent-resolving}) reduces the number of LP solves by holding each selected policy fixed until the next scheduled update. The algorithm is parameterized by $\schedulepower \in (1 / 2, 1)$.
Upon successful entry at the beginning of round $\entrancelast$, let
\begin{equation}
  \epochhorizon[0] = \entrancerest \coloneqq \horizon - \entrancelast + 1.
  \label{eq:post-entry-h0}
\end{equation}
Define recursively
\begin{equation}
  \epochhorizon[\epoch + 1] \coloneqq \max\left\{1, \min\left\{\left\lceil \epochhorizon[\epoch][\schedulepower]\right\rceil,
  \epochhorizon[\epoch] - 1\right\}\right\}
  \quad\text{while }\epochhorizon[\epoch] > 1.
  \label{eq:horizon-schedule}
\end{equation}
Let $\numepochs$ be the first index with $\epochhorizon[\numepochs] = 1$, and set
\begin{equation}
  \period[\epoch] = \entrancelast + \epochhorizon[0] - \epochhorizon[\epoch],
  \qquad \epoch = 0, 1, \ldots, \numepochs.
  \label{eq:resolving-times}
\end{equation}
For $\epoch < \numepochs$, epoch $\epoch$ consists of rounds $\period[\epoch], \ldots, \period[\epoch + 1] - 1$ and has length
$\blocklength[\epoch] \coloneqq \epochhorizon[\epoch] - \epochhorizon[\epoch + 1]$.

The IR algorithm solves the stationary LP initially and re-solves it in rounds $\period[1], \ldots, \period[\numepochs]$; the solution associated with each epoch determines the policy used throughout that epoch.
Note that the algorithm \emph{does not re-solve} in round $\period[0] = \entrancelast$.
It is straightforward to verify that the IR algorithm performs $\bigO(\log\log \horizon)$ stationary-LP solves.

\begin{algorithm}[t]
  \caption{Infrequent re-solving ($\infrequentalg$)}
  \label{alg:infrequent-resolving}
  \begin{algorithmic}[1]
    \Require $\schedulepower \in (1 / 2, 1)$.
    \State Solve $\stationaryfluid(\nominalratio)$ once and select
    $(\occupation(\nominalratio),\stateoccupation(\nominalratio))
    \gets \mathfrak S(\nominalratio)$.
    \State Let $\policy[\nominalratio]$ be the induced policy, with the null-completion convention.
    \State Run $\entrancepolicy$ until successful entrance at the beginning of round $\entrancelast$, or until the process stops.
    \If{the process stops before successful entrance}
      \State \Return
    \EndIf
    \State Construct $\epochhorizon[0], \ldots, \epochhorizon[\numepochs]$ and $\period[0], \ldots, \period[\numepochs]$ using \eqref{eq:post-entry-h0}--\eqref{eq:resolving-times}. Set $\period[\numepochs + 1] \gets \horizon + 1$.
    \For{$\epoch = 0, \ldots, \numepochs$}
      \If{$\epoch = 0$}
        \State Set $\policy[][0] \gets \policy[\nominalratio]$. \Comment{No post-entry solve.}
      \Else
        \State Set $\capacityratio[][\epoch] \gets \inventory[][\period[\epoch]] / \epochhorizon[\epoch]$.
        \State Solve $\stationaryfluid(\capacityratio[][\epoch])$ and select
        $(\occupation(\capacityratio[][\epoch]),\stateoccupation(\capacityratio[][\epoch]))
        \gets \mathfrak S(\capacityratio[][\epoch])$.
        \State Set $\policy[][\epoch] \gets \policy[\capacityratio[][\epoch]]$ and apply the null-completion convention in \eqref{eq:null-completion}.
      \EndIf
      \For{$\period = \period[\epoch], \ldots, \period[\epoch + 1] - 1$}
        \State Observe request $\reqproc[][\period]$ and the current state $\stateproc[][\period]$.
        \State Draw $\actproc[][\period] \sim \policy[][\epoch](\cdot \mid \stateproc[][\period], \reqproc[][\period])$.
        \State Execute $\actproc[][\period]$, earn reward $\rewardproc[][\period]$, and incur consumption $\consumptionproc[][\period]$.
        \State Set $\inventory[][\period + 1] \gets \inventory[][\period] - \consumptionproc[][\period]$ and generate $\stateproc[][\period + 1]$ according to $\transition(\cdot \mid \stateproc[][\period], \actproc[][\period])$.
        \If{$\min_{\res \in [\numres]}\inventory[\res][\period + 1] \le 0$}
          \State \Return
        \EndIf
      \EndFor
    \EndFor
  \end{algorithmic}
\end{algorithm}

\subsection{Performance under regularity}

We now present constant-regret results for these two algorithms under regularity assumptions (Assumptions~\ref{ass:nominal-uniqueness},~\ref{ass:nominal-nondegeneracy}, and~\ref{ass:support-irreducibility}).
Note that under Assumption~\ref{ass:nominal-nondegeneracy}, entrance is immediate because $\recstates = \states$, so Assumption~\ref{ass:constant-entrance} is not required.

\begin{theorem}[Constant regret of frequent and infrequent re-solving under regularity]
  \label{thm:constant-regret}
  Suppose Assumptions~\ref{ass:nominal-uniqueness},~\ref{ass:nominal-nondegeneracy}, and~\ref{ass:support-irreducibility} hold.
  Then, for any fixed $\schedulepower \in (1 / 2, 1)$,
  \begin{equation*}
    \transientfluid[\horizon](\horizon\nominalratio, \state[][1]) - \valuefn[\horizon][\frequentalg](\horizon\nominalratio, \state[][1])
    = \bigO(1),
    \quad
    \transientfluid[\horizon](\horizon\nominalratio, \state[][1]) - \valuefn[\horizon][\infrequentalg](\horizon\nominalratio, \state[][1])
    = \bigO(1).
  \end{equation*}
  Consequently,
  \begin{equation*}
    \valuefn[\horizon][*](\horizon\nominalratio, \state[][1]) - \valuefn[\horizon][\frequentalg](\horizon\nominalratio, \state[][1])
    = \bigO(1),
    \quad
    \valuefn[\horizon][*](\horizon\nominalratio, \state[][1]) - \valuefn[\horizon][\infrequentalg](\horizon\nominalratio, \state[][1])
    = \bigO(1).
  \end{equation*}
\end{theorem}

%% file: sections/05_boundary_cases.tex
\section{Incomplete regularity}
\label{sec:boundary-cases}

In this section, we separately examine how primal nondegeneracy and irreducibility of the optimal induced kernel affect the algorithms' performance.

\subsection{A unique degenerate optimum}
\label{subsec:unique-degenerate}
\label{subsec:unique-degenerate-gmfr-counterexample}
\label{subsec:unique-degenerate-gimfr-sqrt}

\paragraph{Optimal policy: square-root lower bound.}
We start with the case in which the stationary LP $\stationaryfluid(\nominalratio)$ has a unique but degenerate optimum.
In this case, \citet{BW20} established an $\bigOmega(\sqrt{\horizon})$ gap between $\horizon\stationaryfluid(\nominalratio)$ and the optimal online value in a single-state setting.
In the single-state setting, $\horizon\stationaryfluid(\nominalratio)$ is exactly the transient fluid LP, so the lower bound directly applies.

\begin{proposition}[Single-state square-root lower bound under unique degeneracy, from \cite{BW20}]
  \label{prop:unique-degenerate-fluid-lower}
  There is a one-state, one-resource instance with a unique degenerate nominal stationary optimum such that
  \begin{equation*}
    \transientfluid[\horizon](\horizon\nominalratio, \state[][1])
    -\valuefn[\horizon][*](\horizon\nominalratio, \state[][1])
    = \horizon\stationaryfluid(\nominalratio)
    -\valuefn[\horizon][*](\horizon\nominalratio, \state[][1])
    = \bigTheta(\sqrt{\horizon}).
  \end{equation*}
\end{proposition}

\paragraph{IR: square-root upper bound.}

When the optimum is unique but degenerate, we can show that the IR algorithm achieves $\bigO(\sqrt{\horizon})$ regret, provided that the optimal induced kernel is irreducible and a constant-cost entrance policy exists.

\begin{theorem}[Square-root regret of IR under a unique nominal optimum]
  \label{thm:deg-gimfr-sqrt}
  Suppose Assumptions~\ref{ass:nominal-uniqueness},~\ref{ass:support-irreducibility},
  and~\ref{ass:constant-entrance} hold. Then, for any fixed
  $\schedulepower \in (1 / 2, 1)$,
  \begin{equation}
    \transientfluid[\horizon](\horizon\nominalratio, \state[][1]) - \valuefn[\horizon][\infrequentalg](\horizon\nominalratio, \state[][1])
    = \bigO(\sqrt{\horizon}).
    \label{eq:deg-gimfr-fluid-sqrt}
  \end{equation}
  Consequently,
  \begin{equation}
    \valuefn[\horizon][*](\horizon\nominalratio, \state[][1]) - \valuefn[\horizon][\infrequentalg](\horizon\nominalratio, \state[][1])
    = \bigO(\sqrt{\horizon}).
    \label{eq:deg-gimfr-optimal-sqrt}
  \end{equation}
\end{theorem}

Together with Proposition~\ref{prop:unique-degenerate-fluid-lower}, this result
establishes worst-case regret of order $\bigTheta(\sqrt{\horizon})$
for IR under Assumptions~\ref{ass:nominal-uniqueness},~\ref{ass:support-irreducibility},
and~\ref{ass:constant-entrance}.

\paragraph{FR: linear lower bound.}

For FR, however, the following result shows that the algorithm does not guarantee sublinear regret when the stationary LP has a unique but degenerate optimum, in contrast to the state-free setting.

\begin{proposition}
  \label{prop:unique-degenerate-gmfr-linear}
  There is a one-resource, one-request-type instance with a unique degenerate nominal stationary optimum that satisfies Assumptions~\ref{ass:support-irreducibility}
  and~\ref{ass:constant-entrance} such that
  \begin{equation}
    \valuefn[\horizon][*](\horizon\nominalratio, \hardstateA) - \valuefn[\horizon][\frequentalg](\horizon\nominalratio, \hardstateA) = \bigOmega(\horizon)
    \quad\text{and}\quad
    \transientfluid[\horizon](\horizon\nominalratio, \hardstateA) - \valuefn[\horizon][\frequentalg](\horizon\nominalratio, \hardstateA) = \bigOmega(\horizon).
    \label{eq:unique-degenerate-gmfr-linear}
  \end{equation}
\end{proposition}

\subsection{Failure of irreducibility of the optimal induced kernel}
\label{subsec:necessity-irreducibility}

In this subsection, we further show that Assumption~\ref{ass:support-irreducibility}, that is, irreducibility of the optimal induced kernel, is necessary for any algorithm to attain a universal sublinear-regret guarantee.
The intuition is that if the optimal induced kernel is reducible, then no algorithm can asymptotically implement the stationary solution: either the realized state-action frequencies do not match the stationary optimum, or some states in the optimal target state class cannot be reached.

\begin{proposition}
  \label{prop:linear-gap-no-irreducibility}
  There is an instance satisfying Assumptions~\ref{ass:nominal-uniqueness},~\ref{ass:nominal-nondegeneracy}, and~\ref{ass:constant-entrance}, but not Assumption~\ref{ass:support-irreducibility}, such that, for every $\horizon \geq 2$, every nonanticipating policy $\policy$ satisfies
  \begin{equation}
    \horizon\stationaryfluid(\nominalratio) - \valuefn[\horizon][\policy](\horizon\nominalratio, \state[][1]) \ge \frac{\horizon}{4}, \quad
    \transientfluid[\horizon](\horizon\nominalratio, \state[][1]) - \valuefn[\horizon][\policy](\horizon\nominalratio, \state[][1]) \ge \frac{\horizon}{4}.
    \label{eq:linear-stationary-gap}
  \end{equation}
\end{proposition}

%% file: sections/06_unknown_request_prior.tex
\section{Unknown request prior}
\label{sec:unknown-prior}

In this section, we extend our analysis to the incomplete-information setting in which the request distribution $\prior$ is unknown at the outset and must be learned online.
The resulting statistical estimation error affects control because the empirical prior enters the fluid LP; under degeneracy, a small perturbation can change the selected action support.
To make the request prior explicit, we write $\stationaryfluid[][\prior[][\prime]](\capacityratio)$ for the stationary LP in \eqref{eq:stationary-fluid} with $\prior$ replaced by $\prior[][\prime]$ and $\transientfluid[\subhorizon][\prior[][\prime]](\inventory, \state)$ for the corresponding transient LP in \eqref{eq:transient-fluid}.
In particular, $\stationaryfluid = \stationaryfluid[][\prior]$ and $\transientfluid[\subhorizon] = \transientfluid[\subhorizon][\prior]$.

This section gives two complementary results.
First, suppose that the optimal induced kernel satisfies Assumption~\ref{ass:support-irreducibility} and that both the
true optimal target state class $\recstates$ and an entrance policy satisfying Assumption~\ref{ass:constant-entrance} are supplied to the algorithm.
We propose an infrequent empirical re-solving policy that adapts automatically to degeneracy:
it has $\bigO(1)$ regret when the nominal optimum is unique and nondegenerate (Assumptions~\ref{ass:nominal-uniqueness} and~\ref{ass:nominal-nondegeneracy}) and $\bigO(\sqrt{\horizon})$ regret under uniqueness alone, which allows degeneracy (Assumption~\ref{ass:nominal-uniqueness}).
The algorithm also uses $\bigO(\log\log \horizon)$ LP solves in either case.
Second, if $\recstates$ itself is unknown, learning is impossible even under irreducibility of the optimal induced kernel (Assumption~\ref{ass:support-irreducibility}) and constant-cost entrance (Assumption~\ref{ass:constant-entrance}):
there are two priors satisfying both conditions and nominal uniqueness such that, for every horizon and every prior-uninformed policy, the maximum regret across the two instances is bounded below by a positive constant times $\horizon - 1$.
This lower bound highlights the necessity of knowing the optimal target state class $\recstates$ in our problem.

\subsection{Terminology and preliminaries}

We first introduce some terminology and notation for our algorithm. Throughout this subsection, we suppose that the optimal target state class $\recstates$ is known.
In fact, we will show that this condition is necessary for any policy under the unknown request prior to attain sublinear regret.

\paragraph{Safe policy.}
For every $\state \in \recstates$ and $\req \in \requests$, fix a \emph{safe action} $\act[][\mathrm{safe}](\state, \req)$ satisfying
\begin{equation*}
  \transition(\recstates \mid \state, \act[][\mathrm{safe}](\state, \req)) = 1.
\end{equation*}
Such an action exists: nominal request balance and closure of $\recstates$ imply that at least one nominally used action has this property.
Because $\transition$ and $\recstates$ are known, a safe action can be selected without knowing $\prior$.
Let $\policy[][\mathrm{safe}]$ denote the deterministic policy that selects $\act[][\mathrm{safe}](\state, \req)$.

\paragraph{Empirical estimation.}
Like FR and IR, the following algorithm first executes the entrance policy $\entrancepolicy$.
Suppose that the entrance policy succeeds at the beginning of round $\entrancelast$, and let $\entrancerest$ denote the number of remaining rounds.
A re-solving step taken at the beginning of round $\entrancelast + \completedrounds$ uses the following Laplace-smoothed estimate:
\begin{equation}
  \empiricalprior[][\completedrounds](\req)
  \coloneqq \frac{\requestcount[\req](\completedrounds) + 1}{\completedrounds + |\requests|}, \qquad \requestcount[\req](\completedrounds)
  \coloneqq
  \sum_{\period = \entrancelast}^{\entrancelast + \completedrounds - 1}
  \indicator\{\reqproc[][\period] = \req\},
  \qquad \completedrounds \ge 0.
  \label{eq:unknown-empirical-prior}
\end{equation}
In the above, $\requestcount[\req](\completedrounds)$ is the number of times that request $\req$ has arrived up to round $\entrancelast + \completedrounds$ (not including that round) since entrance.

\paragraph{Induced and core-projected policies.}
For a prior $\prior[][\prime]$ and capacity ratio $\capacityratio$, let
$\stationaryfluid[][\prior[][\prime], \recstates](\capacityratio)$ denote the stationary
LP in \eqref{eq:stationary-fluid} with $\prior$ replaced by $\prior[][\prime]$, subject
to the following additional restrictions:
\begin{equation*}
  \stateoccupation[\state] = 0,
  \qquad
  \occupation[\state\req\act] = 0,
  \qquad
  \state \notin \recstates.
\end{equation*}
If $\selectedsolution = (\occupation, \stateoccupation)$ solves a stationary LP with prior $\prior[][\prime]$, define its induced policy wherever the denominator is positive by
\begin{equation}
  \policy[\selectedsolution][\prior[][\prime]](\act \mid \state, \req)
  \coloneqq
  \frac{\occupation[\state\req\act]}
  {\prior[\req][\prime]\stateoccupation[\state]}
  = \frac{\occupation[\state\req\act]}
  {\sum_{\act[][\prime]}\occupation[\state\req\act[][\prime]]}.
  \label{eq:unknown-induced-policy}
\end{equation}

Next, define the set of actions that keep the known optimal target state class closed:
\begin{equation*}
  \actions[\recstates](\state)
  \coloneqq \{\act \in \actions:\transition(\recstates \mid \state, \act) = 1\},
  \qquad \state \in \recstates.
\end{equation*}
For any stationary LP solution $\selectedsolution = (\occupation, \stateoccupation)$, define its \emph{core projection}, whenever the denominator below is positive, by
\begin{equation}
  \coreprojection(\selectedsolution)(\act \mid \state, \req)
  \coloneqq
  \frac{\occupation[\state\req\act]\indicator\{\act \in \actions[\recstates](\state)\}}
  {\sum_{\act[][\prime] \in \actions[\recstates](\state)}\occupation[\state\req\act[][\prime]]},
  \qquad \state \in \recstates, \ \req \in \requests.
  \label{eq:unknown-core-projection}
\end{equation}
Thus, the projected policy can be implemented entirely within the supplied optimal target state class $\recstates$.

\paragraph{Three-phase re-solving schedule.}
The re-solving schedule is divided into three phases.
In the early phase, we fix any constant $\completedrounds[0] \ge 4$ and define
\[
  \protectlength
  \coloneqq \left\lceil(\log \entrancerest)^2\right\rceil.
\]
Starting with $\learningresolve[0] = \completedrounds[0]$, define recursively
\begin{equation}
  \learningresolve[\resolveindex + 1] \coloneqq
  \min\{\lceil \learningresolve[\resolveindex][2]\rceil, \protectlength\}
  \quad\text{while }\learningresolve[\resolveindex] < \protectlength.
  \label{eq:unknown-squaring-schedule}
\end{equation}
When $\completedrounds[0] \le \protectlength$, let $\earlySchedule$ be the first
index such that $\learningresolve[\earlySchedule] = \protectlength$.

In the middle phase, let
\[
  \scheduleendpoint \coloneqq
  \left\lfloor\frac{\entrancerest}{2}\right\rfloor.
\]
When $\learningresolve[0] < \protectlength < \scheduleendpoint$, set
\begin{equation}
  \midSchedule \coloneqq \min\left\{\scheduleendpoint - \protectlength,
  \max\left\{1, \left\lceil
  \log_2\log\frac{\scheduleendpoint}{\protectlength}
  \right\rceil + 2\right\}\right\},
  \label{eq:sparse-mid-count}
\end{equation}
and choose the integer re-solving times as an optimizer of
\begin{equation}
  \min_{\protectlength = \period[0] < \period[1] < \cdots < \period[\midSchedule] = \scheduleendpoint}
  \sum_{\epoch = 0}^{\midSchedule - 1}
  \frac{\period[\epoch + 1] - \period[\epoch]}{\sqrt{\period[\epoch]}},
  \label{eq:sparse-scheduling-cost}
\end{equation}
where the minimization is over integer partitions.

For the late phase, set
\begin{equation}
  \epochhorizon[0] = \entrancerest - \scheduleendpoint
  = \left\lceil\frac{\entrancerest}{2}\right\rceil,
  \qquad
  \epochhorizon[\resolveindex + 1] = \max\left\{1, \min\left\{
  \lceil \epochhorizon[\resolveindex][\schedulepower]\rceil,
  \epochhorizon[\resolveindex] - 1\right\}\right\},
  \label{eq:unknown-late-schedule}
\end{equation}
Apply this recursion while $\epochhorizon[\resolveindex] > 1$. Let $\lateSchedule$ be the first
index with $\epochhorizon[\lateSchedule] = 1$, and append the terminal endpoint
$\epochhorizon[\lateSchedule + 1] = 0$.

\subsection{U-shaped infrequent re-solving}
\label{subsec:universal-unknown}

Algorithm~\ref{alg:unknown-infrequent-resolving} presents our policy.
As described above, it executes three re-solving phases.
The first two phases allow the estimated prior to guide the decision rule early, and the final phase resembles the original infrequent re-solving scheme.
UIR-UP requires $\bigO(\log\log \horizon)$ LP solves.

\begin{algorithm}[t]
  \caption{U-shaped infrequent re-solving with an unknown prior ($\unknowninfrequentalg$)}
  \label{alg:unknown-infrequent-resolving}
  \small
  \begin{algorithmic}[1]
    \Require $\recstates$, $\completedrounds[0] \geq 4$, and $\schedulepower \in (1 / 2, 1)$.
    \State Run $\entrancepolicy$ until entrance occurs at the beginning of round $\entrancelast$, or until the process stops.
    \If{the process stops before entrance}
      \State \Return
    \EndIf
    \State Compute $\protectlength$ and $\scheduleendpoint$.
    \If{$\completedrounds[0] < \protectlength < \scheduleendpoint$ does not hold}
      \State Execute $\policy[][\mathrm{safe}]$ during rounds $\entrancelast, \ldots, \horizon$.
      \State \Return
    \EndIf
    \State Execute $\policy[][\mathrm{safe}]$ during rounds $\entrancelast, \ldots, \entrancelast + \completedrounds[0] - 1$. \Comment{Accumulate samples.}
    \State Construct $\learningresolve[0] < \cdots < \learningresolve[\earlySchedule] = \protectlength$ using \eqref{eq:unknown-squaring-schedule}. \Comment{Phase 1: infrequent re-solving.}
    \For{$\resolveindex = 0, \ldots, \earlySchedule - 1$}
      \State Set $\period \gets \entrancelast + \learningresolve[\resolveindex]$ and solve
        $\stationaryfluid[][\empiricalprior[][\learningresolve[\resolveindex]], \recstates](\inventory[][\period] / (\entrancerest - \learningresolve[\resolveindex]))$ and select $\selectedsolution[][\resolveindex]$ if the LP is feasible.
      \If{the LP is infeasible or $\stateoccupation[\stateproc[][\period]][\resolveindex] = 0$}
        \State Set $\policy[][\resolveindex] \gets \policy[][\mathrm{safe}]$.
      \Else
        \State Set $\policy[][\resolveindex] \gets
          \policy[\selectedsolution[][\resolveindex]][\empiricalprior[][\learningresolve[\resolveindex]]]$ using \eqref{eq:unknown-induced-policy}, and use $\policy[][\mathrm{safe}]$ wherever the denominator in that formula is zero.
      \EndIf
      \State Execute $\policy[][\resolveindex]$ during rounds
        $\entrancelast + \learningresolve[\resolveindex], \ldots,
        \entrancelast + \learningresolve[\resolveindex + 1] - 1$.
    \EndFor
    \State Choose $\protectlength = \period[0] < \cdots < \period[\midSchedule] = \scheduleendpoint$ using \eqref{eq:sparse-mid-count}--\eqref{eq:sparse-scheduling-cost}. \Comment{Phase 2: infrequent re-solving.}
    \For{$\epoch = 0, \ldots, \midSchedule - 1$}
      \State Set $\period \gets \entrancelast + \period[\epoch]$ and solve
        $\stationaryfluid[][\empiricalprior[][\period[\epoch]]](\inventory[][\period] / (\entrancerest - \period[\epoch]))$ and select $\selectedsolution[][\epoch]$.
      \If{$\coreprojection(\selectedsolution[][\epoch])$ is defined for every $(\state, \req) \in \recstates \times \requests$}
        \State Set $\policy[][\epoch] \gets \coreprojection(\selectedsolution[][\epoch])$.
      \Else
        \State Set $\policy[][\epoch] \gets \policy[][\mathrm{safe}]$.
      \EndIf
      \State Execute $\policy[][\epoch]$ during rounds
        $\entrancelast + \period[\epoch], \ldots, \entrancelast + \period[\epoch + 1] - 1$.
    \EndFor
    \State Construct $\epochhorizon[0], \ldots, \epochhorizon[\lateSchedule + 1]$ using \eqref{eq:unknown-late-schedule}. \Comment{Phase 3: infrequent re-solving.}
    \For{$\resolveindex = 0, \ldots, \lateSchedule$}
      \State Set $\period \gets \horizon + 1 - \epochhorizon[\resolveindex]$.
      \State Solve $\stationaryfluid[][\empiricalprior[][\entrancerest - \epochhorizon[\resolveindex]]](\inventory[][\period] / \epochhorizon[\resolveindex])$ and select $\selectedsolution[][\resolveindex]$.
      \If{$\coreprojection(\selectedsolution[][\resolveindex])$ is defined for every $(\state, \req) \in \recstates \times \requests$}
        \State Set $\policy[][\resolveindex] \gets \coreprojection(\selectedsolution[][\resolveindex])$.
      \Else
        \State Set $\policy[][\resolveindex] \gets \policy[][\mathrm{safe}]$.
      \EndIf
      \State Execute $\policy[][\resolveindex]$ during rounds
        $\horizon + 1 - \epochhorizon[\resolveindex], \ldots,
        \horizon - \epochhorizon[\resolveindex + 1]$.
    \EndFor
  \end{algorithmic}
\end{algorithm}

Our main result on the UIR-UP policy is the following.

\begin{theorem}[Order-optimal regret of UIR-UP]
  \label{thm:unknown-universal}
  Suppose Assumptions~\ref{ass:nominal-uniqueness},~\ref{ass:support-irreducibility},
  and~\ref{ass:constant-entrance} hold. Then UIR-UP satisfies
  \begin{equation}
    \transientfluid[\horizon](\horizon\nominalratio, \state[][1]) - \valuefn[\horizon][\unknowninfrequentalg](\horizon\nominalratio, \state[][1]) = \bigO(\sqrt{\horizon}).
    \label{eq:unknown-sqrt-result}
  \end{equation}
  If, in addition, Assumption~\ref{ass:nominal-nondegeneracy} holds, then UIR-UP satisfies
  \begin{equation}
    \transientfluid[\horizon](\horizon\nominalratio, \state[][1]) - \valuefn[\horizon][\unknowninfrequentalg](\horizon\nominalratio, \state[][1]) = \bigO(1).
    \label{eq:unknown-constant-result}
  \end{equation}
\end{theorem}

\subsection{Unknown optimal target state class \texorpdfstring{$\recstates$}{R*}: a linear information-theoretic regret lower bound}
\label{subsec:unknown-R-lower-bound}

Finally, the supplied optimal target state class and entrance policy play a substantive role in our problem.
If the optimal target state class $\recstates$ depends on the unknown prior and the process must irreversibly enter a class before enough requests can be observed, even sublinear regret is impossible.

\begin{proposition}
  \label{prop:unknown-R-linear}
  Fix any constant $\lowerreward \in (0, 1)$. There are two fixed request priors $\prior[][+]$ and $\prior[][-]$ such that, under each prior, the nominal stationary fluid LP has a unique degenerate optimum, Assumption~\ref{ass:support-irreducibility} holds, and Assumption~\ref{ass:constant-entrance} holds with a one-round entrance policy.
  For every integer horizon $\horizon \ge 1$ and every policy $\policy$ that knows neither the true prior nor its optimal target state class,
  \begin{equation*}
    \max_{\prior[][\prime] \in \{\prior[][+], \prior[][-]\}}
    \left\{\valuefn[\horizon][*, \prior[][\prime]](\horizon\nominalratio, \state[][1]) - \valuefn[\horizon][\policy, \prior[][\prime]](\horizon\nominalratio, \state[][1])\right\}
    \ge \frac{1 - \lowerreward}{8}(\horizon - 1).
  \end{equation*}
\end{proposition}

%% file: sections/07_numerical_experiments.tex
\section{Numerical experiments}
\label{sec:numerical-experiments}

In this section, we study how the algorithms analyzed in this work perform on a concrete problem instance as the capacity ratio of its first resource crosses a boundary at which the unique optimum becomes degenerate.
We also examine re-solving policies with thresholding modifications inspired by \citet{BW20,LWZ25}.
We start by introducing the algorithms we consider.

\subsection{Algorithms}

\begin{table}[t]
  \centering
  \small
  \caption{Algorithms used in the experiments.}
  \label{tab:numerical-algorithms}
  \begin{tabularx}{\textwidth}{@{}l c c X@{}}
    \toprule
    Algorithm & Prior & Re-solving & Description \\
    \midrule
    FR & known & $\bigTheta(\horizon)$ & Algorithm~\ref{alg:frequent-resolving}. \\
    IR & known & $\bigTheta(\log \log \horizon)$ & Algorithm~\ref{alg:infrequent-resolving}. \\
    FRT & known & $\bigTheta(\horizon)$ & FR with thresholding. \\
    IRT & known & $\bigTheta(\log \log \horizon)$ & IR with thresholding. \\
    \addlinespace
    FR-UP & unknown & $\bigTheta(\horizon)$ & FR with the request prior replaced by $\empiricalprior[][\completedrounds]$. \\
    UIR-UP & unknown & $\bigTheta(\log \log \horizon)$, three phases & Algorithm~\ref{alg:unknown-infrequent-resolving}. \\
    FRT-UP & unknown & $\bigTheta(\horizon)$ & FR-UP with thresholding. \\
    UIRT-UP & unknown & $\bigTheta(\log \log \horizon)$, three phases & UIR-UP with thresholding. \\
    \bottomrule
  \end{tabularx}
\end{table}
\Cref{tab:numerical-algorithms} summarizes the algorithms examined in this section.
We provide further details below.

\paragraph{Infrequent re-solving schedules.}
All infrequent re-solving algorithms use $\schedulepower = 0.7$; for the unknown-prior algorithms, this parameter applies in phase 3.

\paragraph{Empirical prior estimation.}
All algorithms in the unknown-prior setting use the Laplace-smoothed estimate after $\completedrounds$ requests have been observed:
\[
  \empiricalprior[][\completedrounds](\req) = \frac{\requestcount[\req](\completedrounds) + 1}{\completedrounds + |\requests|}.
\]

\paragraph{Sparse-schedule implementation.}
For UIR-UP and UIRT-UP, the sparse learning phase uses the explicit rounded construction in the proof of \cref{lem:sparse-balanced-schedule}.
Specifically, when $\scheduleendpoint / \protectlength \ge \exp(1)$, the implementation sets $\midSchedule = \lfloor\log_2\log(\scheduleendpoint / \protectlength)\rfloor + 2$, generates the real-valued times in \cref{eq:sparse-real-construction}, rounds them up, and removes repeated times.
The resulting schedule
$\protectlength = \period[0] < \cdots < \period[\midSchedule] = \scheduleendpoint$
satisfies
\[
  \sum_{\epoch = 0}^{\midSchedule - 1}
  \frac{\period[\epoch + 1] - \period[\epoch]}{\sqrt{\period[\epoch]}}
  \le 9\sqrt{\scheduleendpoint},
  \qquad
  \midSchedule = \bigO(\log\log\horizon).
\]
These are the sparse-phase scheduling properties used in the proof of \cref{thm:unknown-universal}; hence, this construction preserves the asymptotic guarantees for unthresholded UIR-UP\@.

\paragraph{Thresholding convention.}
At every re-solving step with $\subhorizon > 1$, FRT and FRT-UP set
$\empiricalthreshold = \subhorizon[][3(1 - 2\schedulepower) / 8] = \subhorizon[][-0.15]$.
For IRT and UIRT-UP, let
$\epochhorizon[\epoch]$ be the remaining horizon at the start of epoch $\epoch$, that is, at the $\epoch$-th re-solving step overall.
The threshold is fixed at
\[
  \empiricalthreshold[\epoch] = \epochhorizon[\epoch][3(1 - 2\schedulepower) / 8] = \epochhorizon[\epoch][-0.15]
\]
for all rounds until the next re-solving step; the policy is thresholded once at the start of the epoch and then frozen. Thresholding is disabled only for the last solve, when the remaining horizon is one.
For each state--request pair, if an action has probability greater than $1 - \empiricalthreshold$, the procedure deterministically selects an action with the largest probability;
otherwise, it sets probabilities below $\empiricalthreshold$ to zero and renormalizes the remaining probabilities.
If no mass remains, it again selects an action with the largest original probability.
Ties are broken in favor of the first action in the fixed action order. The suffix ``T'' in FRT, IRT, FRT-UP, and UIRT-UP denotes this thresholding operation.

\subsection{Environment}
\label{subsec:numerical-design}

\paragraph{Instance.}
In our experimental instance, there are three states $\states = \{\hardstateA, \hardstateB, \hardstateC\}$ and three request types with respective probabilities
\[
  \prior = \left(\frac12, \frac13, \frac16\right).
\]
There are two resources and three actions $\{0, \firstnonzeroact, \secondnonzeroact\}$, where $0$ is the null action.
The initial state is $\hardstateA$, and the per-period resource vector is
$\nominalratio = (\nominalratio[1], 1)$. Let
\[
  \requestbonus = \left(\frac12, \frac13, \frac14\right),
  \qquad
  \requestdiscount = \left(\frac1{24}, \frac1{12}, \frac18\right).
\]
For request $\req$, the deterministic next state, reward, and two-dimensional resource-consumption vector are given in \cref{tab:numerical-environment}.

\begin{table}[t]
  \centering
  \small
  \caption{Primitive data of the experimental environment. Each entry in the
  last column is $(\text{next state};\text{reward};\text{resource-consumption vector})$.}
  \label{tab:numerical-environment}
  \begin{tabularx}{0.5\textwidth}{lll}
    \toprule
    State & Action & Outcome for request $\req$ \\
    \midrule
    $\hardstateA$ & $0$ & $(\hardstateB;0;(0, 0))$ \\
    $\hardstateA$ & $\firstnonzeroact$ & $(\hardstateB;1;(3 / 2, 1 / 2))$ \\
    $\hardstateA$ & $\secondnonzeroact$ & $(\hardstateB;1 + \requestbonus[\req];(5 / 2, 3 / 2))$ \\
    \addlinespace
    $\hardstateB$ & $0$ & $(\hardstateA;0;(0, 0))$ \\
    $\hardstateB$ & $\firstnonzeroact$ & $(\hardstateA;1;(1 / 2, 3 / 2))$ \\
    $\hardstateB$ & $\secondnonzeroact$ & $(\hardstateC;1 - \requestdiscount[\req];(1 / 2, 3 / 2))$ \\
    \addlinespace
    $\hardstateC$ & $0$ & $(\hardstateC;0;(0, 0))$ \\
    $\hardstateC$ & $\firstnonzeroact$ & $(\hardstateA;3 / 4;(1 / 2, 1 / 2))$ \\
    $\hardstateC$ & $\secondnonzeroact$ & $(\hardstateC;0;(1, 1))$ \\
    \bottomrule
  \end{tabularx}
\end{table}

The only primitive varied across experimental instances is $\nominalratio[1]$.
In a neighborhood of $\nominalratio[1] = 1$, the optimal homogeneous core has state masses and value
\begin{equation*}
  (\stateoccupation[\hardstateA], \stateoccupation[\hardstateB], \stateoccupation[\hardstateC];\stationaryfluid(\nominalratio)) =
  \begin{cases}
    \bigl(\nominalratio[1] - \tfrac12, \nominalratio[1] - \tfrac12,
    2(1 - \nominalratio[1]);\,1 - \tfrac7{12}(1 - \nominalratio[1])\bigr),
    & \nominalratio[1] < 1, \\[2mm]
    \bigl(\tfrac12, \tfrac12, 0;\,1\bigr),
    & \nominalratio[1] \ge 1.
  \end{cases}
\end{equation*}
At $\nominalratio[1] = 1$, the optimizer is unique and degenerate, and the optimal target state class changes under an arbitrarily small downward perturbation.
Nearby values of $\nominalratio[1]$ above one also have unique degenerate optima with the same two-state optimal target state class.
At $\nominalratio[1] = 0.995$, the optimizer is nondegenerate, its optimal target state class is $\{\hardstateA, \hardstateB, \hardstateC\}$, and its optimal induced kernel is irreducible.
For the unknown-prior infrequent policies, the supplied optimal target state class is $\{\hardstateA, \hardstateB, \hardstateC\}$ when $\nominalratio[1] < 1$ and $\{\hardstateA, \hardstateB\}$ when $\nominalratio[1] \geq 1$.
In all cases, the entrance policy is empty.

\paragraph{Experimental settings.}
For every trial, we compute the empirical regret against the transient fluid LP benchmark,
\[
  \transientfluid[\horizon](\horizon\nominalratio, \state[][1]) - \sum_{\period = 1}^{\stoptime}\rewardproc[][\period],
\]
where $\stoptime$ is the stopping time in \eqref{eq:stopping}.
Reported bands are two-sided $95\%$ Student's $t$ confidence intervals for the mean regret.
The infrequent methods IR, IRT, UIR-UP, and UIRT-UP use $1{,}000$ independent trials per configuration, whereas the more expensive round-by-round methods FR, FRT,
FR-UP, and FRT-UP use $200$.

We conduct horizon sweeps over
\[
  \horizon \in \{500, 1000, 2000, 4000, 8000, 16000, 32000\}
\]
at both $\nominalratio[1] = 0.995$ and $\nominalratio[1] = 1$.

The capacity-ratio sweep fixes $\horizon = 2000$ and uses
\[
  \nominalratio[1] \in \{0.995, 0.997, 0.998, 0.999, 1, 1.001, 1.002, 1.003, 1.005\}.
\]

\subsection{Known request prior}
\label{subsec:numerical-known-prior}

\begin{figure}[t]
  \centering
  \begin{subfigure}[t]{0.49\textwidth}
    \centering
    \includegraphics[width=\linewidth]{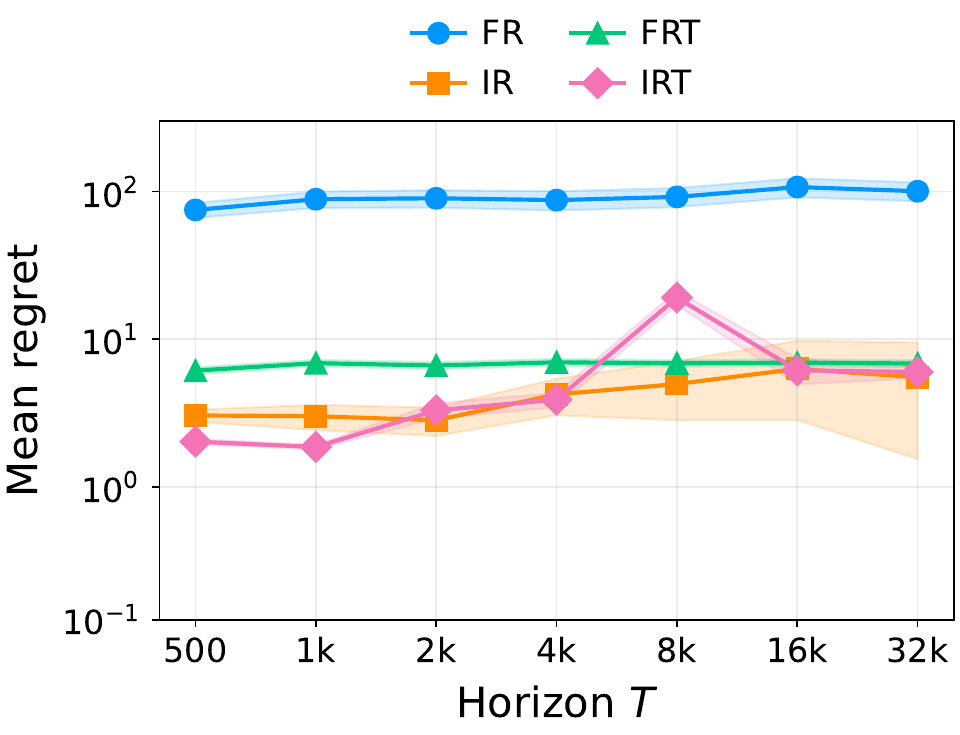}
    \caption{Unique nondegenerate optimum, $\nominalratio[1] = 0.995$.}
    \label{fig:numerical-nondegenerate-known}
  \end{subfigure}\hfill
  \begin{subfigure}[t]{0.49\textwidth}
    \centering
    \includegraphics[width=\linewidth]{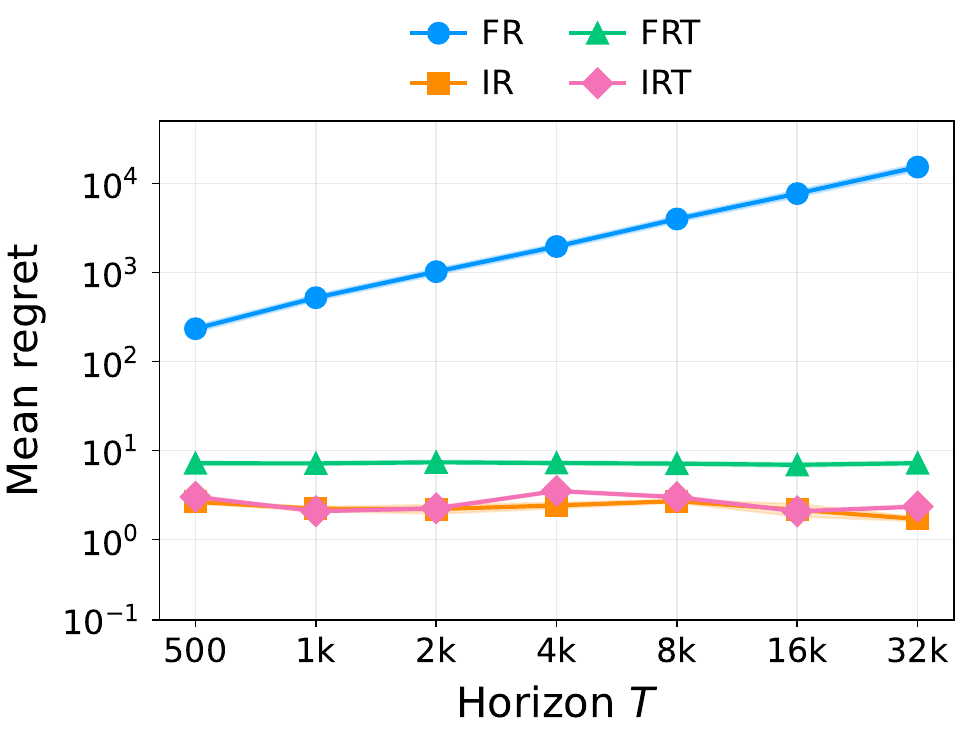}
    \caption{Unique degenerate optimum, $\nominalratio[1] = 1$.}
    \label{fig:numerical-degenerate-known}
  \end{subfigure}
  \caption{Mean regret against the transient fluid LP with a known request prior.}
  \label{fig:numerical-known-prior}
\end{figure}

Figure~\ref{fig:numerical-known-prior} compares the horizon sweeps at two nearby capacity ratios: $\nominalratio[1] = 0.995$, where the unique optimum is nondegenerate, and $\nominalratio[1] = 1$, where it is degenerate.
At the nondegenerate point, the mean regret of FR ranges from
$75.185$ to $107.435$ over the tested horizons, while that of FRT remains between $6.136$ and $6.982$ over the full sweep.
The behavior of unthresholded FR is consistent with Theorem~\ref{thm:constant-regret}, while thresholding further enhances performance, as reported by \citet{BW20} in the state-free setting.
At $\horizon = 32{,}000$, IR has mean regret $5.535$ with a $95\%$ confidence interval of $[1.542, 9.527]$, and IRT has mean regret $6.009$ with a $95\%$ confidence interval of $[5.445, 6.573]$.
In this case, the infrequent re-solving algorithms have a small advantage over the frequent re-solving algorithms.

This gap is much larger at the unique degenerate point, where the mean regret of FR increases from
$233.852$ at $\horizon = 500$ to $15{,}285.072$ at $\horizon = 32{,}000$.
Its ratio of mean regret to horizon stays between approximately $0.468$ and $0.521$ over the sweep, illustrating the linear-loss phenomenon in
Proposition~\ref{prop:unique-degenerate-gmfr-linear}.
Thresholding removes most of this loss: the mean regret of FRT remains between
$6.920$ and $7.391$.
Nevertheless, infrequent re-solving policies still perform better:
the mean regret of IR remains between $1.704$ and $2.706$. However, introducing thresholding appears to weaken IR: the mean regret of IRT ranges from $2.079$ to $3.518$.
These results suggest that re-solving too frequently may not be optimal in online resource allocation with a Markov state, contrary to the findings for the state-free problem reported by \citet{LWZ25}.

\subsection{Unknown request prior}
\label{subsec:numerical-unknown-prior}

\begin{figure}[t]
  \centering
  \begin{subfigure}[t]{0.49\textwidth}
    \centering
    \includegraphics[width=\linewidth]{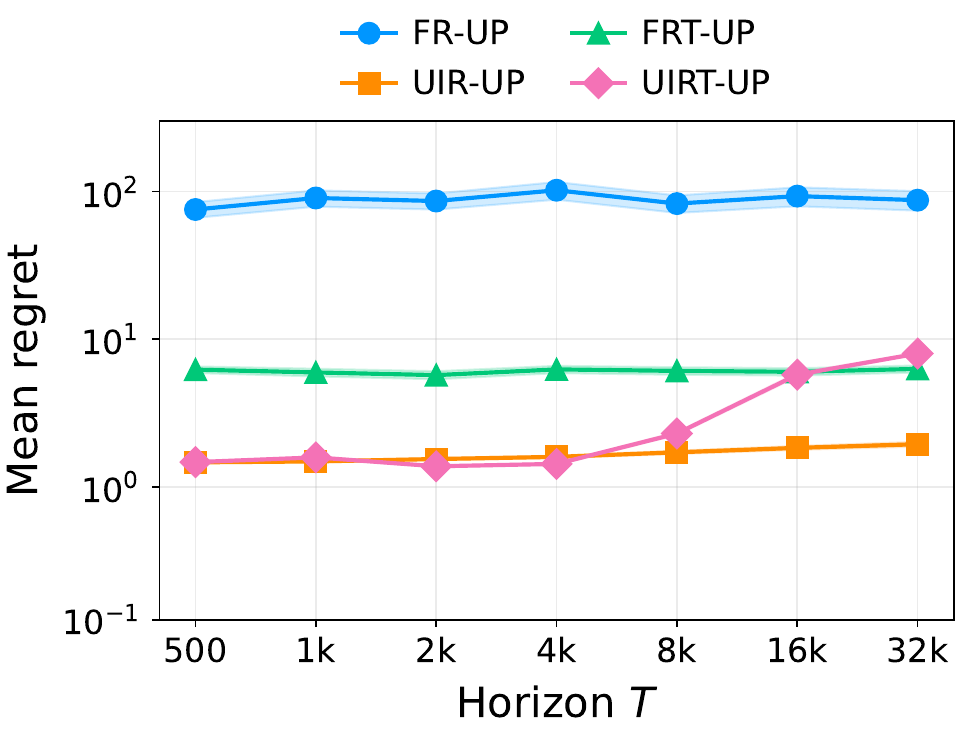}
    \caption{Unique nondegenerate optimum, $\nominalratio[1] = 0.995$.}
    \label{fig:numerical-nondegenerate-unknown}
  \end{subfigure}\hfill
  \begin{subfigure}[t]{0.49\textwidth}
    \centering
    \includegraphics[width=\linewidth]{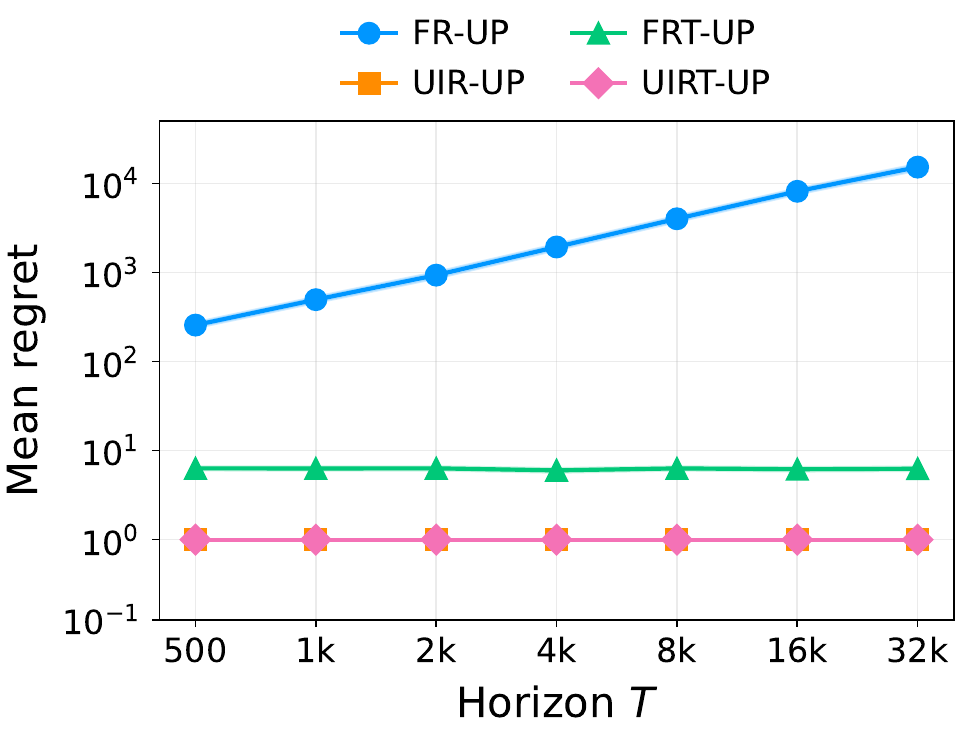}
    \caption{Unique degenerate optimum, $\nominalratio[1] = 1$.}
    \label{fig:numerical-degenerate-unknown}
  \end{subfigure}
  \caption{Mean regret against the transient fluid LP with an unknown request prior.}
  \label{fig:numerical-unknown-prior}
\end{figure}

\begin{table}[t]
  \centering
  \small
  \caption{Mean regret against the transient fluid LP at $\horizon = 32{,}000$.
  Parentheses give the number of trials.}
  \label{tab:numerical-endpoints}
  \begin{tabular}{@{}l rr rr@{}}
    \toprule
    & \multicolumn{2}{c}{Known prior} & \multicolumn{2}{c}{Unknown prior} \\
    \cmidrule(lr){2-3}\cmidrule(l){4-5}
    $\nominalratio[1]$ & FR (200) & IR (1,000) & FR-UP (200) & UIR-UP (1,000) \\
    \midrule
    $0.995$ & $100.725$ & $5.535$ & $87.568$ & $1.947$ \\
    $1$ & $15{,}285.072$ & $1.704$ & $15{,}266.652$ & $1.000$ \\
    \addlinespace
    & FRT (200) & IRT (1,000) & FRT-UP (200) & UIRT-UP (1,000) \\
    \midrule
    $0.995$ & $6.854$ & $6.009$ & $6.303$ & $8.001$ \\
    $1$ & $7.230$ & $2.351$ & $6.250$ & $1.000$ \\
    \bottomrule
  \end{tabular}
\end{table}

Figure~\ref{fig:numerical-unknown-prior} gives the corresponding horizon sweeps when the policies estimate the request prior.
At $\nominalratio[1] = 0.995$, the mean regret of FR-UP ranges from $75.704$ to $102.245$, while that of FRT-UP remains between $5.720$ and $6.303$ over the full sweep.
At $\horizon = 32{,}000$, UIR-UP has mean regret $1.947$ and UIRT-UP has mean regret $8.001$.
The regrets of FRT-UP, UIR-UP, and UIRT-UP remain far below that of FR-UP throughout the sweep.

At $\nominalratio[1] = 1$, FR-UP exhibits the same linear-loss pattern as FR, reaching a mean regret of $15{,}266.652$ at $\horizon = 32{,}000$ with a $95\%$ confidence interval of $[13{,}919.190, 16{,}614.114]$.
The mean regret of FRT-UP remains between $5.999$ and $6.320$ over the horizon sweep.
UIR-UP and UIRT-UP have zero empirical variance at each recorded horizon, suggesting that they consistently attain the optimal reward in this instance.

These observations match those in the known-prior case, implying that prior estimation is not the dominant difficulty.
Table~\ref{tab:numerical-endpoints} compares all eight policies at $\horizon = 32{,}000$.

\subsection{Perturbing the capacity ratio across the degeneracy boundary}
\label{subsec:numerical-capacity}

\begin{figure}[t]
  \centering
  \begin{subfigure}[t]{0.49\textwidth}
    \centering
    \includegraphics[width=\linewidth]{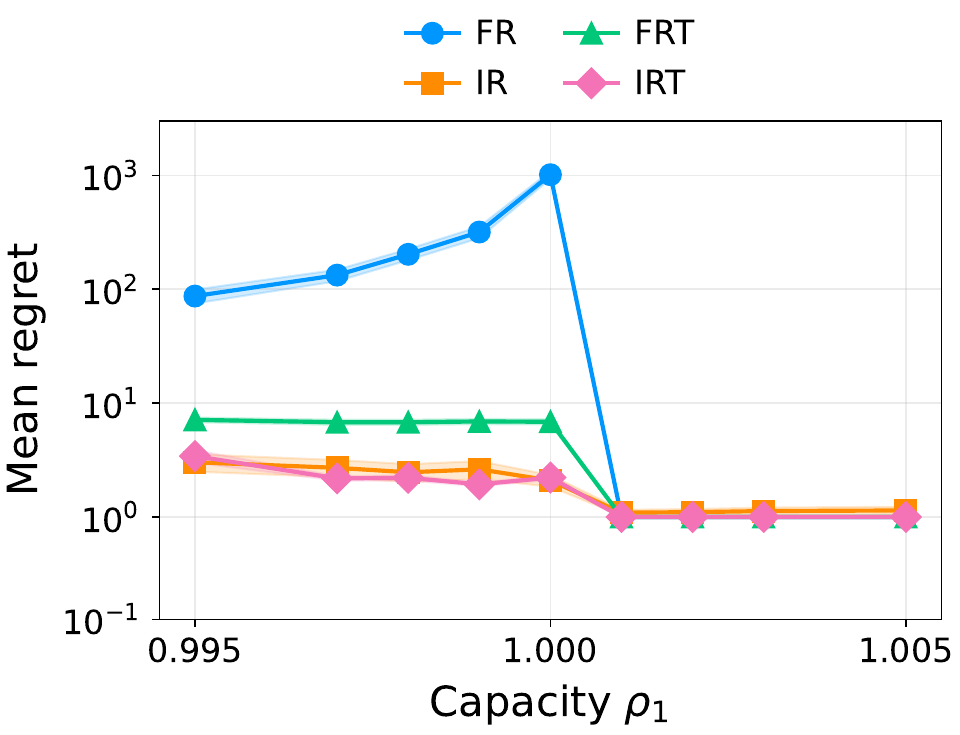}
    \caption{Known request prior.}
    \label{fig:numerical-capacity-known}
  \end{subfigure}\hfill
  \begin{subfigure}[t]{0.49\textwidth}
    \centering
    \includegraphics[width=\linewidth]{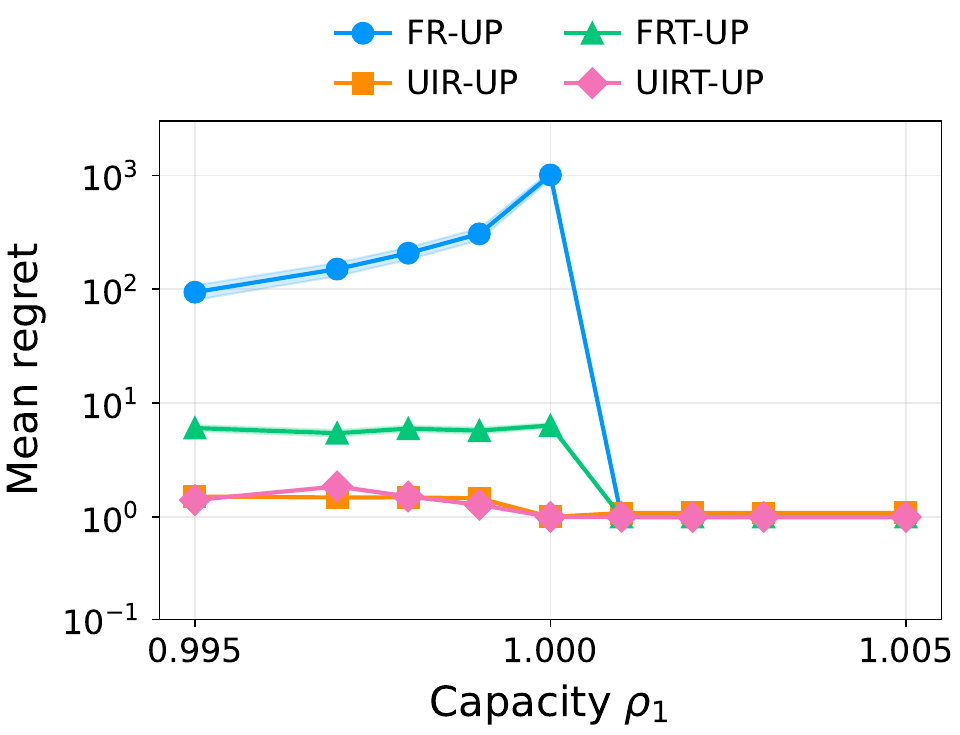}
    \caption{Unknown request prior.}
    \label{fig:numerical-capacity-unknown}
  \end{subfigure}
  \caption{Mean regret against the transient fluid LP, with $95\%$ confidence
  intervals, as a function of the first resource's capacity ratio $\nominalratio[1]$ at $\horizon = 2000$.
  All other model primitives are unchanged.}
  \label{fig:numerical-capacity}
\end{figure}

\begin{table}[t]
  \centering
  \small
  \caption{Mean regret against the transient fluid LP near the capacity-ratio
  boundary at $\horizon = 2000$.}
  \label{tab:numerical-boundary}
  \begin{tabular}{@{}c rrrr rrrr@{}}
    \toprule
    & \multicolumn{4}{c}{Known prior} & \multicolumn{4}{c}{Unknown prior} \\
    \cmidrule(lr){2-5}\cmidrule(l){6-9}
    $\nominalratio[1]$ & FR & IR & FRT & IRT & FR-UP & UIR-UP & FRT-UP & UIRT-UP \\
    \midrule
    $0.998$ & $203.055$ & $2.467$ & $6.794$ & $2.209$ & $207.775$ & $1.485$ & $5.971$ & $1.512$ \\
    $0.999$ & $318.835$ & $2.623$ & $6.901$ & $1.934$ & $307.075$ & $1.457$ & $5.734$ & $1.278$ \\
    $1.000$ & $1{,}016.932$ & $2.090$ & $6.860$ & $2.219$ & $1{,}011.322$ & $1.000$ & $6.360$ & $1.000$ \\
    $1.001$ & $1.000$ & $1.090$ & $1.000$ & $1.000$ & $1.000$ & $1.081$ & $1.000$ & $1.000$ \\
    $1.002$ & $1.000$ & $1.102$ & $1.000$ & $1.000$ & $1.000$ & $1.087$ & $1.000$ & $1.000$ \\
    \bottomrule
  \end{tabular}
\end{table}

\Cref{fig:numerical-capacity} reports the capacity-ratio sweep, which holds $\horizon = 2000$ fixed and varies $\nominalratio[1]$
across the boundary at $1$; the left panel shows the known-prior policies, and the right panel shows the unknown-prior policies.
This provides a direct numerical illustration of support instability under perturbations near the unique degenerate point.
For known-prior FR, mean regret rises from $87.255$ at $\nominalratio[1] = 0.995$ to $318.835$ at $0.999$, jumps to $1{,}016.932$ at $1$, and drops to $1.000$ at $1.001$.
FR-UP follows the same pattern: $307.075$ at $0.999$, $1{,}011.322$ at $1$, and $1.000$ at $1.001$.
A change of only $0.001$ in $\nominalratio[1]$
on either side of the boundary produces a large change in the observed performance
of unthresholded round-by-round re-solving.

Thresholding attenuates the boundary spike: at $\nominalratio[1] = 1$, FRT and FRT-UP have mean regrets of $6.860$ and $6.360$.
Their means stay between $5.441$ and $7.133$ across all tested capacity ratios at or below one and equal $1.000$ above one.
For policies using infrequent schedules, regret remains low when $\nominalratio[1] \in [0.995, 1.005]$.

\subsection{Lowering the re-solving frequency}
\label{subsec:numerical-constant-resolving-frequency}

\begin{figure}[t]
  \centering
  \begin{subfigure}[t]{0.49\textwidth}
    \centering
    \includegraphics[width=\linewidth]{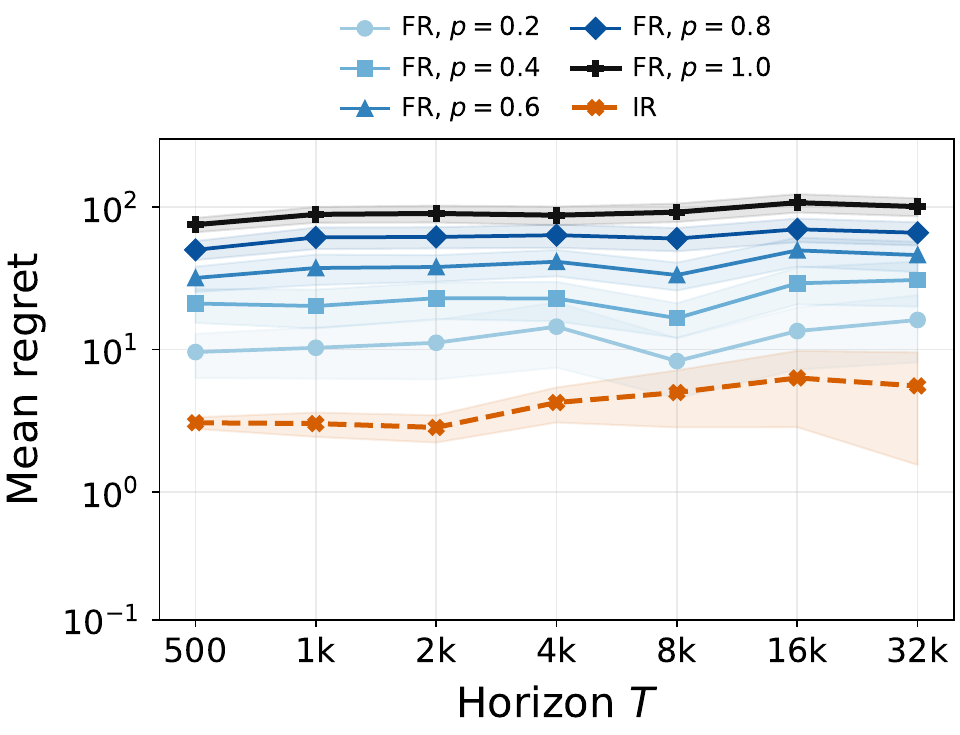}
    \caption{Known request prior, unique nondegenerate optimum, $\nominalratio[1] = 0.995$.}
    \label{fig:numerical-frequency-known-nondegenerate}
  \end{subfigure}\hfill
  \begin{subfigure}[t]{0.49\textwidth}
    \centering
    \includegraphics[width=\linewidth]{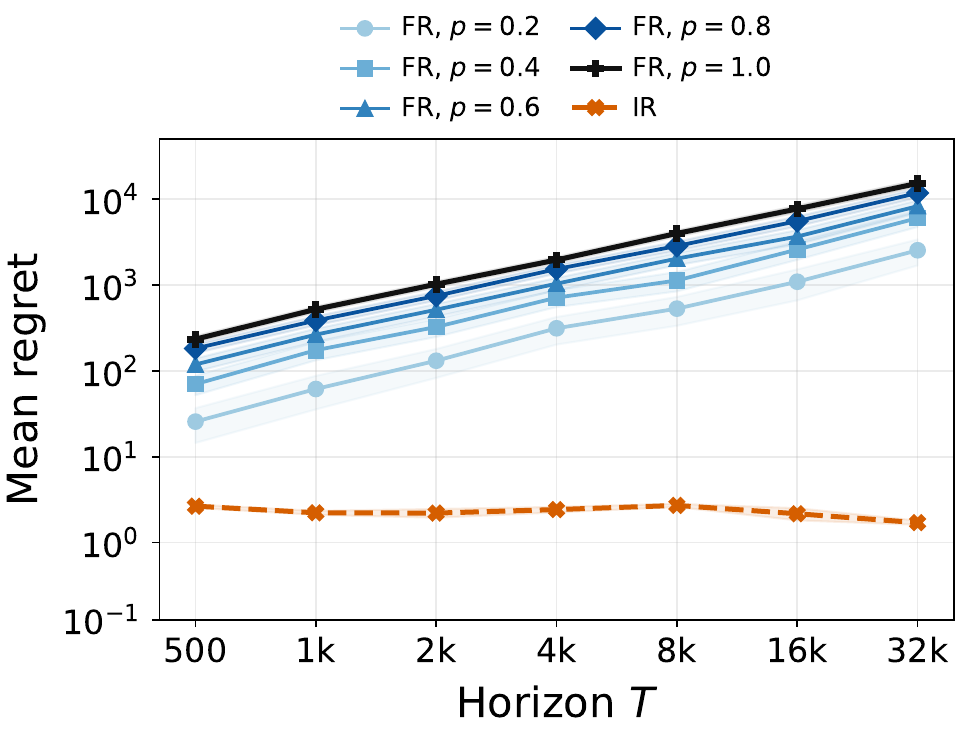}
    \caption{Known request prior, unique degenerate optimum, $\nominalratio[1] = 1$.}
    \label{fig:numerical-frequency-known-degenerate}
  \end{subfigure}
  \\
  \vspace{0.3em}

  \begin{subfigure}[t]{0.49\textwidth}
    \centering
    \includegraphics[width=\linewidth]{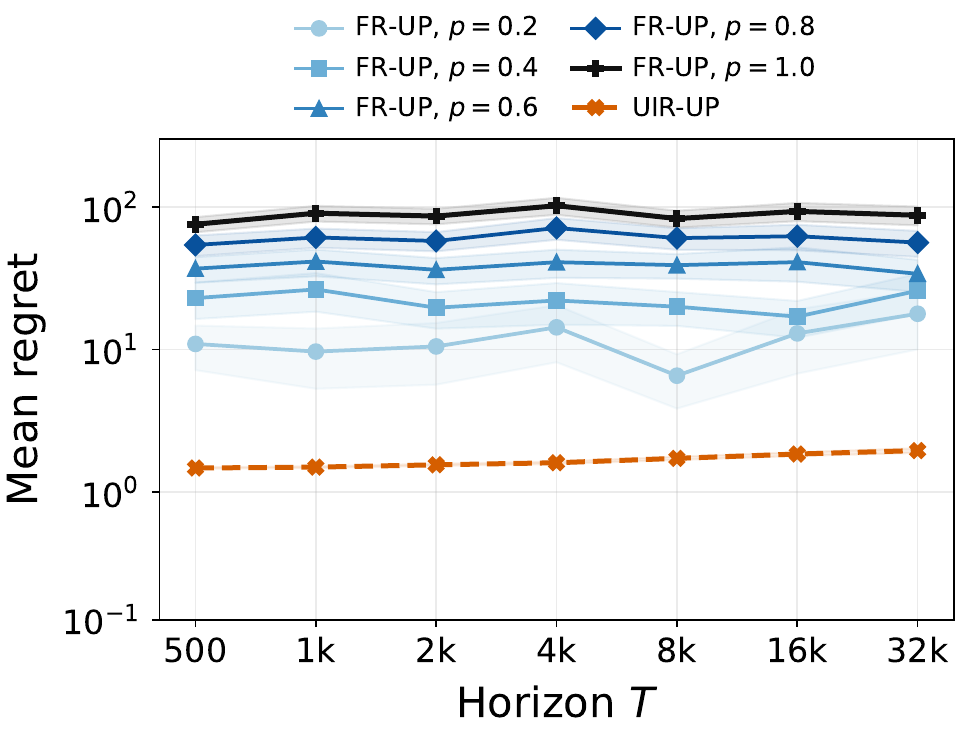}
    \caption{Unknown request prior, unique nondegenerate optimum, $\nominalratio[1] = 0.995$.}
    \label{fig:numerical-frequency-unknown-nondegenerate}
  \end{subfigure}\hfill
  \begin{subfigure}[t]{0.49\textwidth}
    \centering
    \includegraphics[width=\linewidth]{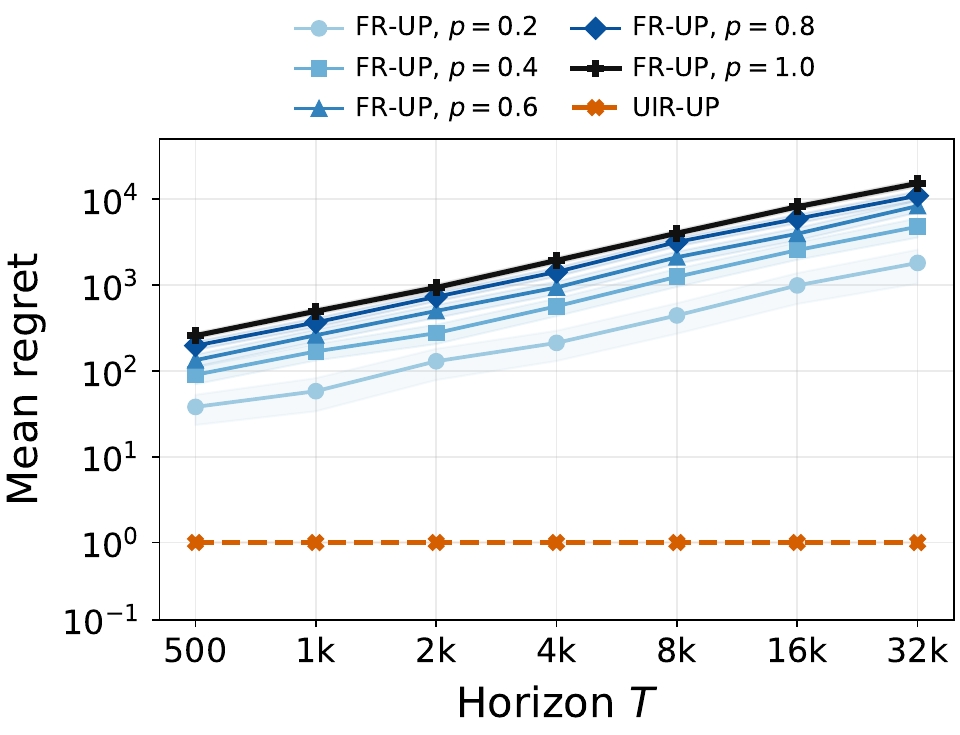}
    \caption{Unknown request prior, unique degenerate optimum, $\nominalratio[1] = 1$.}
    \label{fig:numerical-frequency-unknown-degenerate}
  \end{subfigure}
  \caption{Mean regret for constant-frequency re-solving.}
  \label{fig:numerical-frequency}
\end{figure}

To isolate the effect of re-solving frequency, we introduce FR-$p$ and FR-UP-$p$.
Each policy solves the LP in the first round and, before every later round, independently re-solves with probability $p$; otherwise, it retains the most recently computed stationary policy.
Thus $p = 1$ recovers FR or FR-UP pathwise.
We test $p \in \{0.2, 0.4, 0.6, 0.8\}$ using $200$ trials for each configuration and compare the resulting performance with the existing results for $p = 1$ and for the corresponding infrequent policy, IR or UIR-UP\@.

Figures~\ref{fig:numerical-frequency-known-nondegenerate} and~\ref{fig:numerical-frequency-known-degenerate} show a consistent ordering of the sample mean regrets by re-solving probability.
At the nondegenerate point, the mean regret at $\horizon = 32{,}000$ decreases from $100.725$ for $p = 1$ to $16.098$ for $p = 0.2$, while IR has mean regret $5.535$.
This ordering of the sample means holds at every tested horizon.
At the degenerate point, reducing $p$ lowers the loss substantially but does not prevent it from growing on the scale of the horizon. At $\horizon = 32{,}000$, the mean
regrets for $p = 0.2, 0.4, 0.6, 0.8, 1$ are, respectively, $2{,}540.116$, $6{,}021.859$, $8{,}299.999$, $11{,}809.823$, and $15{,}285.072$, whereas IR has mean regret $1.704$.

Figures~\ref{fig:numerical-frequency-unknown-nondegenerate} and~\ref{fig:numerical-frequency-unknown-degenerate} give the same comparison when the request prior is estimated.
The results follow the same pattern.
At the nondegenerate point, the sample mean regrets are ordered by $p$ at every tested horizon.
At $\horizon = 32{,}000$, their values are $17.806$, $25.826$, $33.965$, $56.300$, and $87.568$ for $p = 0.2, 0.4, 0.6, 0.8, 1$, respectively, compared with $1.947$ for UIR-UP\@.
At the degenerate point, the mean regrets at $\horizon = 32{,}000$ are $1{,}809.039$, $4{,}777.806$, $8{,}323.301$, $10{,}979.652$, and $15{,}266.652$ for the same five values of $p$, while UIR-UP has mean regret
$1.000$.

Together, these experiments show that using a lower constant re-solving probability can improve the performance of a re-solving algorithm, but this adjustment does not attain the performance order of IR and UIR-UP\@.
This observation supports the importance of a carefully designed LP-solving schedule for infrequent re-solving.

%% file: sections/08_conclusion.tex
\section{Conclusion}

In this work, we introduce a practically motivated extension of the well-studied online resource allocation problem in which a controlled state evolves according to action-dependent Markov transitions.
For this problem, we first develop a stationary LP that extends the fluid benchmark widely used in online resource allocation.
In the known-request-prior setting, we consider frequent re-solving policies, which re-solve the LP in every round, and infrequent re-solving policies, which re-solve it in $\bigO(\log \log \horizon)$ rounds.
When the stationary LP has a unique, primal-nondegenerate optimum and the optimal induced kernel is irreducible, both algorithms achieve $\bigO(1)$ regret.
However, without nondegeneracy, the infrequent re-solving policy exhibits greater robustness: it still achieves $\bigO(\sqrt{\horizon})$ regret, while the frequent re-solving algorithm can incur $\bigTheta(\horizon)$ regret.
This result contrasts sharply with results for the original online resource allocation problem showing that frequent LP re-solving is usually more robust.
The main underlying mechanism is that too many LP solves can drive the policy away from the stationary optimum and prevent it from returning.
In the unknown-prior setting, we further develop a three-phase infrequent re-solving schedule that guarantees order-optimal regret under different levels of regularity.
The resulting algorithm also solves the LP $\bigO(\log \log \horizon)$ times in total.
Our numerical studies further reveal the advantage of infrequent over frequent re-solving schemes for this problem.


%% file: appendices/01_model_proofs.tex
\section{Proofs for Section~\ref{sec:model}}
\label{app:model-proofs}

\subsection{Proof of Lemma~\ref{lem:fluid-upper-bound}}
\label{app:proof:lem:fluid-upper-bound}

Fix an arbitrary nonanticipating policy $\policy$ for the $\subhorizon$-round subproblem, and let $\stoptime \le \subhorizon$ be its stopping time. Define
\[
  \nonterminalocc[\state\req\act][\policy] \coloneqq \expect^{\policy}\left[\sum_{\period = 1}^{\stoptime - 1}
  \indicator\{\stateproc[][\period] = \state, \reqproc[][\period] = \req, \actproc[][\period] = \act\}\right],
\]
\[
  \terminalocc[\state\req\act][\policy] \coloneqq \prob[][\policy](\stateproc[][\stoptime] = \state, \reqproc[][\stoptime] = \req, \actproc[][\stoptime] = \act),
  \qquad
  \transientstateocc[\state][\policy] \coloneqq \expect^{\policy}\left[\sum_{\period = 1}^{\stoptime}\indicator\{\stateproc[][\period] = \state\}\right].
\]
We verify the four constraint families separately.

\noindent\textbf{Request balance.}
For each $\period$, the event $\{\period \le \stoptime, \stateproc[][\period] = \state\}$ is determined before $\reqproc[][\period]$ is drawn: it depends only on whether a resource was exhausted in an earlier round and on the current state. Therefore
\[
  \expect^{\policy}[\indicator\{\period \le \stoptime, \stateproc[][\period] = \state, \reqproc[][\period] = \req\}]
  = \prior[\req]\expect^{\policy}[\indicator\{\period \le \stoptime, \stateproc[][\period] = \state\}].
\]
Summing over $\period = 1, \ldots, \subhorizon$ gives
\[
  \sum_{\act}(\nonterminalocc[\state\req\act][\policy] + \terminalocc[\state\req\act][\policy]) = \prior[\req]\transientstateocc[\state][\policy],
\]
which is \textup{(TF1)}.

\noindent\textbf{State flow.}
Every visit to $\state[][\prime]$ is either the initial visit or follows a round that is strictly before the terminal round. Hence
\begin{align*}
  \transientstateocc[\state[][\prime]][\policy] & = \indicator\{\state[][\prime] = \state[][0]\}
  +\expect^{\policy}\left[\sum_{\period = 1}^{\stoptime - 1}\indicator\{\stateproc[][\period + 1] = \state[][\prime]\}\right]\\
  & = \indicator\{\state[][\prime] = \state[][0]\}
  +\sum_{\state, \req, \act}\transition(\state[][\prime] \mid \state, \act)\nonterminalocc[\state\req\act][\policy],
\end{align*}
where the second equality follows by conditioning on $(\stateproc[][\period], \reqproc[][\period], \actproc[][\period])$. This is \textup{(TF2)}.

\noindent\textbf{Horizon.}
By definition,
\[
  \sum_{\state}\transientstateocc[\state][\policy] = \expect^{\policy}[\stoptime] \le \subhorizon,
\]
which is \textup{(TF3)}.

\noindent\textbf{Resources before the terminal action.}
Pathwise,
\[
  \sum_{\period = 1}^{\stoptime - 1}\consumption[\res](\stateproc[][\period], \reqproc[][\period], \actproc[][\period]) = \inventory[\res] - \inventory[\res][\stoptime] \le \inventory[\res].
\]
The last inequality holds because no earlier round has exhausted resource $\res$; when $\stoptime = 1$, the left-hand sum is empty and the same inequality is immediate. Taking expectations gives \textup{(TF4)}.

Finally,
\[
  \sum_{\state, \req, \act}\reward(\state, \req, \act)(\nonterminalocc[\state\req\act][\policy] + \terminalocc[\state\req\act][\policy])
  = \expect^{\policy}\left[\sum_{\period = 1}^{\stoptime}\rewardproc[][\period]\right].
\]
Thus, for every nonanticipating policy, these expected state-request-action and state-visit counts form a feasible transient-LP solution with the same objective value. Maximizing over $\policy$ proves the result.

\subsection{Proof of Lemma~\ref{lem:boundary-comparison}}
\label{app:proof:lem:boundary-comparison}

Consider the dual of the stationary homogeneous core LP:
\begin{equation}
  \boxed{
  \begin{aligned}
    \min_{\resourceprice, \averagegain, \statebias, \statedual}\quad
    &
    \resourceprice[][\top]\capacityratio + \averagegain
    \\[0.3em]
    \text{subject to}\quad
    &
    \statedual[\state\req] \ge
    \reward(\state, \req, \act) - \resourceprice[][\top]\consumption(\state, \req, \act)
    +\sum_{\state[][\prime] \in \states}\transition(\state[][\prime] \mid \state, \act)\statebias[\state[][\prime]],
    & & \state, \req, \act,
    \qquad\text{\rm (SD1)}
    \\
    &
    \statebias[\state] + \averagegain \ge \sum_{\req \in \requests}\prior[\req]\statedual[\state\req],
    & & \state \in \states,
    \qquad\text{\rm (SD2)}
    \\
    &
    \resourceprice \ge 0.
  \end{aligned}
  }
  \label{eq:stationary-dual}
\end{equation}

The null action implies $\averagegain \ge 0$ for every feasible solution of the stationary dual. Indeed, \textup{(SD1)} with $\act = \act[][0]$, together with the fact that the null action has zero reward and zero consumption, gives
\[
  \statedual[\state\req] \ge \sum_{\state[][\prime] \in \states}
  \transition(\state[][\prime] \mid \state, \act[][0])\statebias[\state[][\prime]].
\]
Averaging this inequality over $\req$ and using \textup{(SD2)} yields
\[
  \statebias[\state] + \averagegain
  \ge \sum_{\state[][\prime] \in \states}
  \transition(\state[][\prime] \mid \state, \act[][0])\statebias[\state[][\prime]].
\]
Let $\nullstationarylaw$ be any stationary distribution of the null-action transition matrix; one exists because $\states$ is finite. Multiply the last inequality by $\nullstationarylaw[\state]$ and sum over $\state$. Stationarity cancels the state-bias terms and gives $\averagegain \ge 0$.

The dual of the transient LP \eqref{eq:transient-fluid} is
\begin{equation}
  \boxed{
  \begin{aligned}
    \min_{\resourceprice, \averagegain, \statebias, \statedual}\quad
    &
    \resourceprice[][\top]\inventory + \averagegain \subhorizon + \statebias[\state[][0]]\\
    \text{subject to}\quad
    &
    \statedual[\state\req] \ge \reward(\state, \req, \act),
    & & \state, \req, \act,
    \qquad\text{\rm (TD1)}
    \\
    &
    \statedual[\state\req] \ge \reward(\state, \req, \act) - \resourceprice[][\top]\consumption(\state, \req, \act)
    +\sum_{\state[][\prime]}\transition(\state[][\prime] \mid \state, \act)\statebias[\state[][\prime]],
    & & \state, \req, \act,
    \qquad\text{\rm (TD2)}
    \\
    &
    \statebias[\state] + \averagegain \ge \sum_{\req}\prior[\req]\statedual[\state\req],
    & & \state,
    \qquad\text{\rm (TD3)}
    \\
    &
    \resourceprice \ge 0, \qquad\averagegain \ge 0.
  \end{aligned}
  }
  \label{eq:transient-dual}
\end{equation}

Let $(\resourceprice[][*], \averagegain[][*], \statebias[][*], \statedual[][*])$ be an optimal solution of \cref{eq:stationary-dual} at $\capacityratio = \nominalratio$, normalized by $\statebias[\state[][1]][*] = 0$.
Define
\begin{equation*}
  \boundaryconstant
  \coloneqq
  \max_{\state, \req, \act}\bigl(\reward(\state, \req, \act) - \statedual[\state\req][*]\bigr)_+.
\end{equation*}
Increase every component of $\statedual[][*]$ and $\statebias[][*]$ by $\boundaryconstant$, and leave $\resourceprice[][*], \averagegain[][*]$ unchanged. By definition of $\boundaryconstant$, \textup{(TD1)} holds. Since every row of $\transition(\cdot \mid \state, \act)$ sums to one, adding the same constant to both $\statedual$ and $\statebias$ preserves \textup{(TD2)}. Since $\sum_{\req}\prior[\req] = 1$, it also preserves \textup{(TD3)}. Moreover, $\averagegain[][*] \ge 0$. Hence this shifted solution is feasible for \cref{eq:transient-dual}, and its objective at $\inventory = \horizon\nominalratio, \subhorizon = \horizon, \state[][0] = \state[][1]$ is
\[
  \horizon\bigl((\resourceprice[][*])^\top\nominalratio + \averagegain[][*]\bigr)
  +\statebias[\state[][1]][*] + \boundaryconstant
  = \horizon\stationaryfluid(\nominalratio) + \boundaryconstant.
\]
Weak duality proves the lemma.

%% file: appendices/02_regularity_proofs.tex
\section{Proofs for Section~\ref{sec:regularity}}
\label{app:regularity-proofs}

\subsection{Proof of Proposition~\ref{prop:local-lp-stability}}
\label{app:proof:prop:local-lp-stability}

\begin{proof}
  Consider an LP in the standard maximization form:
  \[
    \max\left\{
    \objectivevector[][\top]\selectedsolution:
    \constraintmatrix\selectedsolution = \rhsvector, \
    \selectedsolution \ge 0
    \right\},
  \]
  and a nonsingular basis $\basis$ of $\constraintmatrix$. Let $\nonbasiccols$ be the columns of $\constraintmatrix$ outside $\basis$, and define the reduced costs by
  \[
    \reducedcost[\nonbasiccols] \coloneqq \objectivevector[\nonbasiccols] - \objectivevector[\basis][\top]\basisinverse\nonbasiccols.
  \]
  We first recall an important result from \cite{bertsimas1997introduction}:
  \begin{proposition}[From \cite{bertsimas1997introduction}]
    \label{prop:reduced-costs}
    The following hold:
    \begin{itemize}
      \item If $\selectedsolution$ is the unique and nondegenerate optimal solution and $\basis$ is its basis, then $\reducedcost[\nonbasiccols] < 0$.
      \item If $\selectedsolution$ is a feasible solution with basis $\basis$, $\selectedsolution[\basis] > 0$, and $\reducedcost[\nonbasiccols] < 0$, then $\selectedsolution$ is the unique and nondegenerate optimal solution.
    \end{itemize}
  \end{proposition}

  Using Proposition~\ref{prop:reduced-costs}, we prove our results as follows.
  First, add a slack $\resourceslack[\res] \ge 0$ to each resource constraint in \textup{(SF4)}, retain every resource row, and delete a fixed set of redundant rows among \textup{(SF1)--(SF3)}.
  Every resource row can be retained because the column of $\resourceslack[\res]$ is nonzero only in that row.
  With
  \[
    \selectedsolution = (\occupation, \stateoccupation, \resourceslack),
  \]
  the resulting full-row-rank standard form is
  \begin{equation}
    \max\left\{
    \objectivevector[][\top]\selectedsolution:
    \constraintmatrix\selectedsolution = \rhsvector(\capacityratio), \
    \selectedsolution \ge 0
    \right\}.
    \label{eq:standard-parametric-lp}
  \end{equation}

  Notice that only the resource entries of the right-hand side depend on $\capacityratio$, and they are affine in $\capacityratio$. In fact, if $\unitvector[\res]$ is the unit vector corresponding to resource row $\res$, then
  \[
    \rhsvector(\capacityratio)
    = \rhsvector(\nominalratio)
    +\sum_{\res \in [\numres]}
    (\capacityratio[\res] - \nominalratio[\res])
    \unitvector[\res].
  \]

  Now let $\basis$ be the basis for the optimal solution $\selectedsolution[][*](\nominalratio)$ under $\capacityratio = \nominalratio$, and let $\nonbasiccols$ be the remaining columns of $\constraintmatrix$. Therefore,
  \[
    \selectedsolution[\basis][*](\nominalratio) = \basisinverse \rhsvector(\nominalratio) > 0, \quad \selectedsolution[\nonbasiccols][*](\nominalratio) = 0.
  \]
  Also, Proposition~\ref{prop:reduced-costs} and the assumed uniqueness and nondegeneracy give $\reducedcost[\nonbasiccols] < 0$.

  Define $\selectedsolution(\capacityratio)$ by
  \[
    \selectedsolution[\basis](\capacityratio) \coloneqq \basisinverse \rhsvector(\capacityratio), \quad \selectedsolution[\nonbasiccols](\capacityratio) \coloneqq 0.
  \]
  Since $\basisinverse \rhsvector(\nominalratio) > 0$, there exist sufficiently small constants $\localradius, \positivitymargin > 0$ such that $\selectedsolution[\basis](\capacityratio) \geq \positivitymargin$ whenever $\capacityratio \in \capacityneighborhood[3\localradius]$. By Proposition~\ref{prop:reduced-costs}, $\selectedsolution(\capacityratio)$ is the unique and nondegenerate optimal solution to \eqref{eq:standard-parametric-lp} and has the same basis $\basis$.

  Further, for $\res \notin \bindingset$, the column of the basic slack
  $\resourceslack[\res]$ is $\unitvector[\resourceslack[\res]]$. Hence
  \[
    \basisinverse\unitvector[\res] = \unitvector[\resourceslack[\res]],
  \]
  and
  \begin{align*}
    \selectedsolution[\basis](\capacityratio)
    & = \basisinverse \rhsvector(\capacityratio) \\
    & = \basisinverse \bigl(\rhsvector(\nominalratio)
    +\sum_{\res \in [\numres]}
    (\capacityratio[\res] - \nominalratio[\res])
    \unitvector[\res]\bigr) \\
    & = \selectedsolution[\basis][*](\nominalratio) + \sum_{\res \in \bindingset}
    (\capacityratio[\res] - \nominalratio[\res]) \basisinverse \unitvector[\res] + \sum_{\res \notin \bindingset}
    (\capacityratio[\res] - \nominalratio[\res]) \unitvector[\resourceslack[\res]].
  \end{align*}
  Therefore, the coordinates of $\selectedsolution[\basis](\capacityratio)$ that do not correspond to resource slacks (i.e., those corresponding to states and decisions) do not depend on $\capacityratio[\res]$ for $\res \notin \bindingset$.
\end{proof}

\subsection{Proof of Proposition~\ref{prop:support-regularity}}
\label{app:proof:prop:support-regularity}

\begin{proof}
  \noindent\textbf{Closure.}
  Fix $\state[][\prime] \notin \recstates$.  Local support stability gives $\stateoccupation[\state[][\prime]](\capacityratio) = 0$.  The state-flow equality is
  \[
    0 = \stateoccupation[\state[][\prime]](\capacityratio) = \sum_{\state, \req, \act}\transition(\state[][\prime] \mid \state, \act)\occupation[\state\req\act](\capacityratio).
  \]
  Every summand is nonnegative.  Therefore, whenever $\occupation[\state\req\act](\capacityratio) > 0$, one must have $\transition(\state[][\prime] \mid \state, \act) = 0$.  Averaging these transition probabilities under $\policy[\capacityratio]$ shows $\inducedkernel[\capacityratio](\state, \state[][\prime]) = 0$ for every $\state \in \recstates$.

  \noindent\textbf{Lipschitz dependence.}
  By Proposition~\ref{prop:local-lp-stability}, the fixed-basis stationary LP solution is affine in $\capacityratio[\bindingset]$.  On $\recstates$, every $\stateoccupation[\state](\capacityratio)$ is at least some positive constant, so division in \eqref{eq:lp-induced-policy} preserves Lipschitz continuity.  \eqref{eq:induced-kernel} then gives the same property for $\inducedkernel[\capacityratio]$.  Uniform boundedness is immediate.

  \noindent\textbf{Irreducibility.}
  Finally, irreducibility within $\recstates$ follows by observing that $\inducedkernel[][*](\state, \state[][\prime]) > 0$ if and only if $\inducedkernel[\capacityratio](\state, \state[][\prime]) > 0$.
\end{proof}

\subsection{Proof of Proposition~\ref{prop:nondegenerate-full-state-support}}
\label{app:proof:prop:nondegenerate-full-state-support}

\begin{proof}
  Consider the fixed standard form \eqref{eq:standard-parametric-lp}, which retains all variables $(\occupation, \stateoccupation, \resourceslack)$ and deletes only redundant equality rows.
  Let $\basis$ be the nominal basis matrix.  It is nonsingular, and Assumption~\ref{ass:nominal-nondegeneracy} ensures that every basic variable is strictly positive.

  Suppose that $\stateoccupation[\state][*] = 0$ for some $\state \in \states$.  By \textup{(SF1)} in \eqref{eq:stationary-fluid},
  \[
    \sum_{\act \in \actions}\occupation[\state\req\act][*] = \prior[\req]\stateoccupation[\state][*] = 0
    \qquad\text{for every }\req \in \requests.
  \]
  Nonnegativity implies $\occupation[\state\req\act][*] = 0$ for every $\req \in \requests$ and $\act \in \actions$.  Thus $\stateoccupation[\state]$ and all coordinates $\occupation[\state\req\act]$ indexed by this state are nonbasic.

  Fix any $\req \in \requests$ and consider the coefficient row of the equality
  \[
    \sum_{\act \in \actions}\occupation[\state\req\act] - \prior[\req]\stateoccupation[\state] = 0.
  \]
  This row is nonzero, but all its nonzero coefficients belong to nonbasic columns, so its restriction to the basis columns is zero.
  Thus the basis matrix contains an all-zero row, contradicting its nonsingularity.
  Hence $\stateoccupation[\state][*] > 0$ for every $\state \in \states$, and \eqref{eq:nominal-supports} gives $\recstates = \states$.
\end{proof}

%% file: appendices/03_algorithm_proofs.tex

\subsection{Analytical setup}
\label{app:algorithm-proofs}

This appendix supplies the stopped comparisons and supporting analysis for the algorithms in Section~\ref{sec:algorithms}.

To start with, we mention that the null-completion convention is never invoked along a regular path in the proof of Theorem~\ref{thm:constant-regret}.
Nevertheless, it is essential to the unique-degenerate lower-bound construction in Proposition~\ref{prop:unique-degenerate-gmfr-linear} for frequent re-solving.

For the upper-bound proof, choose an integer cutoff $\cutrest$ satisfying
\begin{equation}
  \cutrest \coloneqq \left\lceil 1 +
  \max\left\{
  \frac{\maxconsumption}{\min_{\res}(\nominalratio[\res] - \localradius)},
  \frac{\maxconsumption + \max_{\res}\nominalratio[\res]}{\localradius} + 1
  \right\} \right\rceil.
  \label{eq:cutoff}
\end{equation}

\paragraph{Analytical-check convention.}
Throughout the appendices, a \emph{check} means a history-measurable test that
can terminate a proof-only comparison.  We use five distinct types.  
\begin{enumerate}
  \item The
\emph{locality check} is applied at a resolving time, before an action, and
tests membership of the current resolving parameter in the neighborhood
specified at that use. 
  \item The \emph{target-state check}, also applied at a
resolving time before an action, tests whether the current state belongs to
$\recstates$.  
  \item The \emph{resource near-depletion check}
is applied before each contemplated action and fires when
$\min_{\res\in[\numres]}\inventory[\res][\period]\le\maxconsumption$.
  \item The
\emph{epoch-movement check} is applied after a reached epoch endpoint and
before the checks for the next epoch; it compares the old-to-new capacity
movement with the buffer specified there. 
  \item Finally, the \emph{prior-error check}, used only with an unknown request prior, is
applied at the start of an empirical epoch before an action and compares the
empirical-prior error with the tolerance specified there. 
\end{enumerate} 
Each use below states the relevant
neighborhood, tolerance, or buffer explicitly.  

On a successful entrance, call the entrance \emph{regular} when
$\entrancecapacityratio\in\capacityneighborhood[\localradius]$ and
\emph{irregular} otherwise.  Define $\gfrequentalg$ by coupling it to FR until
the first post-entry round at which the pre-action locality check
$\capacityratio[][\period]\in\capacityneighborhood[\localradius]$ fails or the
deterministic cutoff $\subhorizon[][\period]\le\cutrest$ is reached, and using
the null action thereafter; it also truncates immediately after an irregular
entrance.
The comparison then satisfies the pathwise inequality
\begin{equation}
  \totalreward[\horizon][\frequentalg] \ge \totalreward[\horizon][\gfrequentalg].
  \label{eq:fr-dominates-guarded-fr}
\end{equation}
Thus, every upper bound proved for the guarded comparison transfers directly to FR\@.

Before its analytical truncation, \cref{eq:cutoff} implies $\inventory[\res][\period] > \maxconsumption$ for every resource. 
Upon successful entry, $\stateproc[][\entrancelast] \in \recstates$, and Proposition~\ref{prop:support-regularity} ensures that the optimal target state class $\recstates$ is closed under every locally induced kernel. 
Hence the denominator in \cref{eq:lp-induced-policy} is uniformly positive along the guarded comparison.
The second term in \cref{eq:cutoff} also provides a one-step buffer: if $\capacityratio[][\period] \in \capacityneighborhood[\localradius]$ and $\subhorizon[][\period] > \cutrest$, then $\capacityratio[][\period + 1] \in \capacityneighborhood[2\localradius]$.

Similarly, define $\ginfrequentalg$ by coupling it to IR and switching permanently to the null action after an irregular entrance. After a regular entrance, the switch occurs at the first resolving round at which the pre-action locality check
$\capacityratio[][\period]\in\capacityneighborhood[\localradius]$ fails, the
first round at which the pre-action resource near-depletion check
$\min_{\res \in [\numres]}\inventory[\res][\period]\le\maxconsumption$ fires,
or the first round at which the deterministic cutoff
$\subhorizon[][\period]\le\cutrest$ is reached.
The comparison also pathwise satisfies
\begin{equation}
  \totalreward[\horizon][\infrequentalg] \ge \totalreward[\horizon][\ginfrequentalg].
  \label{eq:ir-dominates-guarded-ir}
\end{equation}

%% file: appendices/04_common_ingredients_proofs.tex

\subsection{Common analytical ingredients and proofs}
\label{app:common-proofs}
\label{sec:common-analysis}

\subsubsection{The dual Bellman potential}
\label{app:proof:lem:potential-identity}
\label{app:proof:lem:entrance-bookkeeping}

Choose an arbitrary optimal dual solution of $\stationaryfluid(\nominalratio)$,
\begin{equation*}
  (\resourceprice[][*], \averagegain[][*], \statebias[][*], \statedual[][*]),
\end{equation*}
and use the dual shift invariance to normalize $\statebias[\state[][1]][*] = 0$.
Complementary slackness gives
\begin{align}
  \resourceprice[\res][*] & = 0,
  & & \res \notin \bindingset,
  \notag\\
  \statedual[\state\req][*] & = \reward(\state, \req, \act) - (\resourceprice[][*])^\top \consumption(\state, \req, \act)
  +\sum_{\state[][\prime]}\transition(\state[][\prime] \mid \state, \act)\statebias[\state[][\prime]][*],
  & & (\state, \req, \act) \in \actsupport,
  \label{eq:nominal-cs-action}\\
  \statebias[\state][*] + \averagegain[][*] & = \sum_{\req}\prior[\req]\statedual[\state\req][*],
  & & \state \in \recstates.
  \label{eq:nominal-cs-state}
\end{align}

For every action with $\policy[\capacityratio[][\prime]](\act \mid \state, \req) > 0$, one has $\occupation[\state\req\act](\capacityratio[][\prime]) > 0$. Proposition~\ref{prop:local-lp-stability} therefore implies $(\state, \req, \act) \in \actsupport$. Nominal complementary slackness \eqref{eq:nominal-cs-action} gives
\[
  \reward(\state, \req, \act) - (\resourceprice[][*])^\top \consumption(\state, \req, \act)
  +\sum_{\state[][\prime]}\transition(\state[][\prime] \mid \state, \act)\statebias[\state[][\prime]][*] = \statedual[\state\req][*].
\]
Average this equality over $\actproc \sim \policy[\capacityratio[][\prime]](\cdot \mid \state, \req)$. The right-hand side remains $\statedual[\state\req][*]$ because it does not depend on $\act$. Then average over $\reqproc \sim \prior$. Since $\state \in \recstates$, \eqref{eq:nominal-cs-state} gives
\begin{equation}
  \expect\left[
  \rewardproc - (\resourceprice[][*])^\top \consumption(\state, \reqproc, \actproc) + \statebias[\stateproc[][\prime]][*] \mid \stateproc = \state
  \right]
  = \sum_{\req}\prior[\req]\statedual[\state\req][*] = \statebias[\state][*] + \averagegain[][*].
  \label{eq:action-tightness-used}
\end{equation}

Now define
\begin{equation*}
  \potential(\subhorizon, \inventory, \state) \coloneqq (\resourceprice[][*])^\top \inventory + \averagegain[][*]\subhorizon + \statebias[\state][*].
\end{equation*}
We directly obtain the following lemma.

\begin{lemma}
  \label{lem:potential-identity}
  Fix $\capacityratio[][\prime] \in \capacityneighborhood[3\localradius]$ and $\state \in \recstates$. Suppose that one round uses the policy $\policy[\capacityratio[][\prime]]$, regardless of the current inventory ratio. Then, for every $\subhorizon \ge 1$ and $\inventory \in \reals[][\numres]$,
  \begin{equation}
    \expect\left[
    \rewardproc + \potential(\subhorizon - 1, \inventory - \consumption(\state, \reqproc, \actproc), \stateproc[][\prime])
    \mid \stateproc = \state, \inventory, \subhorizon
    \right]
    = \potential(\subhorizon, \inventory, \state),
    \label{eq:one-step-potential}
  \end{equation}
  where $\reqproc \sim \prior$, $\actproc \sim \policy[\capacityratio[][\prime]](\cdot \mid \state, \reqproc)$, and $\stateproc[][\prime] \sim \transition(\cdot \mid \state, \actproc)$.
\end{lemma}

\begin{proof}
  Adding $(\resourceprice[][*])^\top \inventory + \averagegain[][*](\subhorizon - 1)$ to both sides of \eqref{eq:action-tightness-used} gives \eqref{eq:one-step-potential}. 
\end{proof}

At the initial point, strong duality and $\statebias[\state[][1]][*] = 0$ give
\begin{equation}
  \potential(\horizon, \horizon\nominalratio, \state[][1])
  = \horizon\bigl((\resourceprice[][*])^\top\nominalratio + \averagegain[][*]\bigr)
  = \horizon\stationaryfluid(\nominalratio).
  \label{eq:initial-potential}
\end{equation}
For any history on which entrance succeeds, recall that $\entrancelast$ is the first round $\period$ such that $\stateproc[][\period] \in \recstates$, and define
\begin{gather*}
  \entranceduration \coloneqq \entrancelast - 1, \quad \entrancerest \coloneqq \horizon - \entranceduration; \\
  \entranceconsumption \coloneqq \sum_{\period = 1}^{\entranceduration}\consumption(\stateproc[][\period], \reqproc[][\period], \actproc[][\period]), \quad 
  \entranceinventory \coloneqq \inventory[][1] - \entranceconsumption, \quad 
  \entrancereward \coloneqq \sum_{\period = 1}^{\entranceduration}\reward(\stateproc[][\period], \reqproc[][\period], \actproc[][\period]); \\
  \entrancestate \coloneqq \stateproc[][\entrancelast], \quad
  \entranceinventorydeviation
  \coloneqq \entranceinventory - \entrancerest\nominalratio.
\end{gather*}

For any inventory deviation $\inventoryimbalance \in \reals[][\numres]$, define the one-sided norm
\begin{equation}
  \inventorydeviationnorm(\inventoryimbalance)
  \coloneqq \sum_{\res \in \bindingset}|\inventoryimbalance[\res]|
  +\sum_{\res \notin \bindingset}(-\inventoryimbalance[\res])_+.
  \label{eq:one-sided-inventory-deviation}
\end{equation}
The next lemma shows that the entrance imbalance affects every performance bound by at most a constant.

\begin{lemma}
  \label{lem:entrance-bookkeeping}
  Under Assumption~\ref{ass:constant-entrance}, 
  \begin{itemize}
    \item let $\event[\horizon]$ be the event that $\entranceduration < \horizon$ and $\consumptionproc[\res][\entrance] < \horizon\nominalratio[\res]$ for every resource. Then
    \begin{align}
      \horizon\,\prob(\event[\horizon][\mathrm c]) & = \bigO(1), \label{eq:entrance-failure-prob}\\
      \expect\bigl[\|\entranceinventorydeviation\|_1\indicator\{\event[\horizon]\}\bigr] & = \bigO(1), \quad \expect[\inventorydeviationnorm(\entranceinventorydeviation)\indicator\{\event[\horizon]\}] = \bigO(1); \label{eq:entrance-imbalance-bound}
    \end{align}
    \item on $\event[\horizon]$,
    \begin{equation}
      \begin{aligned}
        &\mathrel{\phantom{=}}\horizon\stationaryfluid(\nominalratio) - \entrancereward - \potential(\entrancerest, \entranceinventory, \entrancestate) \\
        &\le (\|\resourceprice[][*]\|_\infty + 1)\|\entranceconsumption\|_1
        +(\|\nominalratio\|_1 + |\averagegain[][*]|)\entranceduration
        +\max_{\state}|\statebias[\state][*]|,
      \end{aligned}
      \label{eq:entrance-potential-cost}
    \end{equation}
    and the expectation of the right-hand side is $\bigO(1)$;
    \item for every fixed $\neighborhoodradius > 0$,
    \begin{equation}
      \expect\left[\entrancerest\indicator\{\event[\horizon],
      \entranceinventory / \entrancerest \notin \capacityneighborhood[\neighborhoodradius]\}\right]
      \le \frac1\neighborhoodradius\expect[\inventorydeviationnorm(\entranceinventorydeviation)\indicator\{\event[\horizon]\}]
      = \bigO(1).
      \label{eq:bad-entrance-weighted-horizon}
    \end{equation}
  \end{itemize}
\end{lemma}

\begin{proof}
  Markov's inequality and Assumption~\ref{ass:constant-entrance} give
  \[
    \prob(\entranceduration \ge \horizon) \le \frac{\expect \entranceduration}\horizon,
    \qquad
    \prob(\consumptionproc[\res][\entrance] \ge \horizon\nominalratio[\res])
    \le \frac{\expect \consumptionproc[\res][\entrance]}{\horizon\nominalratio[\res]}.
  \]
  A union bound proves \eqref{eq:entrance-failure-prob}. 
  Interpret the entrance quantities on the infinite-resource trajectory from
  Assumption~\ref{ass:constant-entrance}. On $\event[\horizon]$, we know that
  \[
    \entranceinventorydeviation
    = (\horizon\nominalratio - \entranceconsumption)
    - (\horizon - \entranceduration)\nominalratio
    = \entranceduration\,\nominalratio - \entranceconsumption.
  \]
  Consequently, using componentwise nonnegativity of consumption,
  \begin{align*}
    \expect\bigl[\|\entranceinventorydeviation\|_1
      \indicator\{\event[\horizon]\}\bigr]
    &\le \|\nominalratio\|_1\expect[\entranceduration]
      + \expect[\|\entranceconsumption\|_1] \\
    &= \|\nominalratio\|_1\expect[\entrancelast - 1]
      + \expect\left[\sum_{\period = 1}^{\entrancelast - 1}
        \|\consumption(\stateproc[][\period], \reqproc[][\period],
        \actproc[][\period])\|_1\right]
    = \bigO(1),
  \end{align*}
  where the indicator was dropped when obtaining the first upper bound and
  the final equality follows from \eqref{eq:entrance-bounds}. Moreover,
  \[
    \inventorydeviationnorm(\inventoryimbalance)
    \le \|\inventoryimbalance\|_1,
    \qquad \inventoryimbalance \in \reals[][\numres],
  \]
  directly from \eqref{eq:one-sided-inventory-deviation}. This proves
  \eqref{eq:entrance-imbalance-bound}. Direct substitution into the definition
  of $\potential$ gives
  \[
    \horizon\stationaryfluid(\nominalratio) - \entrancereward - \potential(\entrancerest, \entranceinventory, \entrancestate)
    = (\resourceprice[][*])^\top \entranceconsumption + \averagegain[][*]\entranceduration
    -\statebias[\entrancestate][*] - \entrancereward.
  \]
  This expression is bounded above by the right-hand side of \eqref{eq:entrance-potential-cost}, whose expectation is finite because rewards are nonnegative and the entrance bounds hold. Finally, if $\entranceinventory / \entrancerest \notin \capacityneighborhood[\neighborhoodradius]$, then either $|\inventorydeviation[\res][\entrance]| \ge \neighborhoodradius\entrancerest$ for some $\res \in \bindingset$ or $-\inventorydeviation[\res][\entrance] \ge \neighborhoodradius\entrancerest$ for some $\res \notin \bindingset$. Hence $\neighborhoodradius\entrancerest \le \inventorydeviationnorm(\entranceinventorydeviation)$ on that event, proving \eqref{eq:bad-entrance-weighted-horizon}.
\end{proof}

\subsection{Common resource averages and Poisson equations}
\label{app:proof:lem:poisson-bounds}

For each resource $\res$, define
\begin{equation*}
  \stateconsumption[\res](\capacityratio, \state)
  \coloneqq \sum_{\req}\prior[\req]\sum_{\act}\policy[\capacityratio](\act \mid \state, \req)\consumption[\res](\state, \req, \act),
\end{equation*}
and
\begin{equation}
  \stationaryconsumption[\res](\capacityratio)
  \coloneqq \sum_{\state}\stateoccupation[\state](\capacityratio)\stateconsumption[\res](\capacityratio, \state)
  = \sum_{\state, \req, \act}\consumption[\res](\state, \req, \act)\occupation[\state\req\act](\capacityratio).
  \label{eq:stationary-consumption}
\end{equation}
Proposition~\ref{prop:local-lp-stability} implies that when $\capacityratio \in \capacityneighborhood[3\localradius]$,
\begin{equation}
  \stationaryconsumption[\res](\capacityratio) = \capacityratio[\res]\quad(\res \in \bindingset),
  \qquad
  \stationaryconsumption[\res](\capacityratio) < \capacityratio[\res] - \positivitymargin\quad(\res \notin \bindingset).
  \label{eq:resource-drift}
\end{equation}

Now let $\inducedkernel[\capacityratio]$ and
$\stateoccupation[\capacityratio]$ denote the kernel and stationary
distribution on $\recstates$ corresponding to those binding coordinates, and
set
\[
  \stationaryprojector[\capacityratio]
  \coloneqq \onesvector\stateoccupation[\capacityratio][\top].
\]
These objects are independent of the inactive capacity coordinates by
Proposition~\ref{prop:local-lp-stability}, and
$\inducedkernel[\capacityratio]$ is irreducible by
Proposition~\ref{prop:support-regularity}. We claim that
\[
  \identitymatrix-\inducedkernel[\capacityratio]
  +\stationaryprojector[\capacityratio]
\]
is nonsingular. Indeed, suppose that $\genericvalue\in
\reals[][|\recstates|]$ belongs to its null space. Premultiplying by
$\stateoccupation[\capacityratio][\top]$ and using stationarity $\stateoccupation[\capacityratio][\top]\inducedkernel[\capacityratio] = \stateoccupation[\capacityratio][\top]$ and
$\stateoccupation[\capacityratio][\top]\onesvector = 1$ gives
\[
  0
  =\stateoccupation[\capacityratio][\top]
  (\identitymatrix-\inducedkernel[\capacityratio]
  +\stationaryprojector[\capacityratio])\genericvalue
  =\stateoccupation[\capacityratio][\top]\genericvalue.
\]
Hence $\stationaryprojector[\capacityratio]\genericvalue = \onesvector\stateoccupation[\capacityratio][\top]\genericvalue = 0$, and the null-space equation reduces to $(\identitymatrix - \inducedkernel[\capacityratio])\genericvalue = 0$.
For a finite irreducible stochastic matrix, the right eigenspace associated with eigenvalue one is $\operatorname{span}\{\onesvector\}$. 
Thus $\genericvalue=c\onesvector$ for some scalar $c$, and the centering equality above gives $c=0$. The null space is therefore trivial, so the matrix is nonsingular. 
Define
\[
  \fundamentalmatrix[\capacityratio]
  \coloneqq
  (\identitymatrix-\inducedkernel[\capacityratio]
  +\stationaryprojector[\capacityratio])^{-1}.
\]
More generally, if $\stateoccupation[\capacityratio][\top]\genericvalue=0$ and $\genericvalue[][\prime]=\fundamentalmatrix[\capacityratio]\genericvalue$, then premultiplying $(\identitymatrix-\inducedkernel[\capacityratio] +\stationaryprojector[\capacityratio])\genericvalue[][\prime]=\genericvalue$ by $\stateoccupation[\capacityratio][\top]$ gives $\stationaryprojector[\capacityratio]\genericvalue[][\prime]=0$
and
$(\identitymatrix-\inducedkernel[\capacityratio])\genericvalue[][\prime]=\genericvalue$. 
Conversely, every centered solution of this last equation solves the invertible system with the stationary projector added, so the centered solution is unique.

For $\capacityratio \in \capacityneighborhood[3\localradius]$, the vector
$\stateconsumption[\res](\capacityratio,\cdot)
-\stationaryconsumption[\res](\capacityratio)\onesvector$ has zero
$\stateoccupation(\capacityratio)$-mean by
\eqref{eq:stationary-consumption}. Consequently,
$\fundamentalmatrix[\capacityratio[\bindingset]]$ maps this vector to the
unique centered solution of the Poisson equation. We denote that solution by
$\poissoncorrector[\res](\capacityratio, \cdot)$; equivalently, it is the
unique solution on $\recstates$ of
\begin{align*}
  \poissoncorrector[\res](\capacityratio, \state) - \sum_{\state[][\prime] \in \recstates}\inducedkernel[\capacityratio](\state, \state[][\prime])\poissoncorrector[\res](\capacityratio, \state[][\prime])
  & = \stateconsumption[\res](\capacityratio, \state) - \stationaryconsumption[\res](\capacityratio),
\\
  \sum_{\state \in \recstates}\stateoccupation[\state](\capacityratio)\poissoncorrector[\res](\capacityratio, \state) & = 0.
\end{align*}

\begin{lemma}
  \label{lem:poisson-bounds}
  There are finite constants $\poissonbound, \poissonlip$ such that, for every resource $\res$, every state $\state \in \recstates$, and all $\capacityratio, \capacityratio[][\prime] \in \capacityneighborhood[2\localradius]$,
  \begin{equation*}
    |\poissoncorrector[\res](\capacityratio, \state)| \le \poissonbound,
    \qquad
    |\poissoncorrector[\res](\capacityratio, \state) - \poissoncorrector[\res](\capacityratio[][\prime], \state)|
    \le \poissonlip\|\capacityratio[\bindingset] - \capacityratio[\bindingset][\prime]\|_\infty.
  \end{equation*}
\end{lemma}

\begin{proof}
  Only the binding coordinates affect $\stateoccupation$, $\policy$, $\inducedkernel$, $\stateconsumption[\res]$, and $\stationaryconsumption[\res]$. Consider the compact box $\bindingbox[2\localradius]$ defined in Section~\ref{sec:regularity}.
  It lies strictly inside the basis-stability region $|\capacityratio[\res] - \nominalratio[\res]| < 3\localradius$.

  For $\capacityratio \in \bindingbox[2\localradius]$, consider the corresponding restrictions to $\recstates$ and the fundamental matrix defined above. 
  Proposition~\ref{prop:support-regularity} ensures irreducibility throughout the box. 
  The matrix depends continuously on $\capacityratio$, so compactness gives
  \begin{equation}
    \sup_{\capacityratio \in \bindingbox[2\localradius]}\|\fundamentalmatrix[\capacityratio]\| < \infty.
    \label{eq:uniform-fundamental-bound}
  \end{equation}
  The centered solution is
  \[
    \poissoncorrector[\res](\capacityratio) = \fundamentalmatrix[\capacityratio]\bigl(\stateconsumption[\res](\capacityratio) - \stationaryconsumption[\res](\capacityratio)\onesvector\bigr).
  \]
  The right-hand side is uniformly bounded, proving the first claim.

  For the second claim, the fixed-basis formulas make $\stateoccupation$, $\policy$, $\inducedkernel$, $\stateconsumption[\res]$, and $\stationaryconsumption[\res]$ Lipschitz on the box. 
  The inverse identity
  \[
    \fundamentalmatrix[\capacityratio] - \fundamentalmatrix[\capacityratio[][\prime]] = \fundamentalmatrix[\capacityratio]\bigl[(\inducedkernel[\capacityratio] - \inducedkernel[\capacityratio[][\prime]]) - (\stationaryprojector[\capacityratio] - \stationaryprojector[\capacityratio[][\prime]])\bigr]\fundamentalmatrix[\capacityratio[][\prime]]\]
  combined with \eqref{eq:uniform-fundamental-bound} makes $\fundamentalmatrix[\capacityratio]$ Lipschitz as well. 
  Substitution into the formula for $\poissoncorrector[\res]$ proves the stated bound. 
  Since inactive coordinates do not affect these objects, the result holds on the full one-sided neighborhood.
\end{proof}

%% file: appendices/05_frequent_resolving_proofs.tex

\subsection{Frequent re-solving analysis and proof}
\label{app:frequent-proofs}
\label{sec:frequent-analysis}

Throughout this section, we analyze the policy $\gfrequentalg$ defined in Appendix~\ref{app:algorithm-proofs}.
The implemented FR policy dominates this stopped policy pathwise by \eqref{eq:fr-dominates-guarded-fr}, so the resulting regret upper bound transfers to FR\@.

\subsubsection{Time-varying martingale decomposition and exact ratio recursion}
\label{app:proof:lem:exact-ratio-recursion}

We call a round \emph{regular} if it occurs before analytical truncation, its state belongs to $\recstates$, its remaining horizon exceeds $\cutrest$, and the remaining average inventory of the round belongs to $\capacityneighborhood[\localradius]$.
In a regular FR round, write
\[
  \consumptionproc[][\period] \coloneqq \consumption(\stateproc[][\period], \reqproc[][\period], \actproc[][\period]),
  \qquad
  \consumptionproc[\res][\period] \coloneqq \consumption[\res](\stateproc[][\period], \reqproc[][\period], \actproc[][\period]).
\]
Define the martingale difference
\begin{align*}
  \martdiff[\res][\period + 1] \coloneqq {} & \consumptionproc[\res][\period] - \stateconsumption[\res](\capacityratio[][\period], \stateproc[][\period])
  +\poissoncorrector[\res](\capacityratio[][\period], \stateproc[][\period + 1])
  \nonumber\\
  & -\sum_{\state[][\prime]}\inducedkernel[\capacityratio[][\period]](\stateproc[][\period], \state[][\prime])\poissoncorrector[\res](\capacityratio[][\period], \state[][\prime]).
\end{align*}
We thus have by definition that
\begin{equation*}
  \expect[\martdiff[\res][\period + 1] \mid \filtration[][\period]] = 0.
\end{equation*}
Moreover, $|\martdiff[\res][\period + 1]| \le \maxconsumption + 2\poissonbound = \bigO(1)$. The Poisson equation gives the exact identity
\begin{equation}
  \consumptionproc[\res][\period] - \stationaryconsumption[\res](\capacityratio[][\period])
  = \martdiff[\res][\period + 1] + \poissoncorrector[\res](\capacityratio[][\period], \stateproc[][\period]) - \poissoncorrector[\res](\capacityratio[][\period], \stateproc[][\period + 1]).
  \label{eq:cost-decomposition}
\end{equation}

Since $\inventory[][\period + 1] = \inventory[][\period] - \consumptionproc[][\period]$ and $\subhorizon[][\period + 1] = \subhorizon[][\period] - 1$,
\begin{equation}
  \capacityratio[\res][\period + 1] - \capacityratio[\res][\period] = -\frac{\consumptionproc[\res][\period] - \capacityratio[\res][\period]}{\subhorizon[][\period] - 1}.
  \label{eq:ratio-recursion}
\end{equation}
Define
\begin{equation*}
  \correcteddeviation[\res][\period] \coloneqq \capacityratio[\res][\period] - \nominalratio[\res] - \frac{\poissoncorrector[\res](\capacityratio[][\period], \stateproc[][\period])}{\subhorizon[][\period]}.
\end{equation*}

\begin{lemma}
  \label{lem:exact-ratio-recursion}
  When $\capacityratio[][\period] \in \capacityneighborhood[\localradius]$, $\subhorizon[][\period] > \cutrest$, $\capacityratio[][\period + 1] \in \capacityneighborhood[2\localradius]$. Further, for every resource $\res$,
  \begin{equation}
    \correcteddeviation[\res][\period + 1] - \correcteddeviation[\res][\period] = -\frac{\martdiff[\res][\period + 1]}{\subhorizon[][\period] - 1}
    -\frac{\stationaryconsumption[\res](\capacityratio[][\period]) - \capacityratio[\res][\period]}{\subhorizon[][\period] - 1}
    +\ratioremainder[\res][\period + 1],
    \label{eq:corrected-recursion}
  \end{equation}
  where
  \begin{equation}
    \ratioremainder[\res][\period + 1] = \frac{\poissoncorrector[\res](\capacityratio[][\period], \stateproc[][\period + 1]) - \poissoncorrector[\res](\capacityratio[][\period + 1], \stateproc[][\period + 1])}{\subhorizon[][\period] - 1}
    -\frac{\poissoncorrector[\res](\capacityratio[][\period], \stateproc[][\period])}{\subhorizon[][\period](\subhorizon[][\period] - 1)},
    \label{eq:exact-remainder}
  \end{equation}
  \begin{equation}
    |\ratioremainder[\res][\period + 1]| = \bigO\!\left((\subhorizon[][\period] - 1)^{-2}\right).
    \label{eq:remainder-bound}
  \end{equation}
\end{lemma}

\begin{proof}
  For $\res \in \bindingset$ and $\capacityratio[][\period] \in \capacityneighborhood[\localradius]$,
  \[
    |\capacityratio[\res][\period + 1] - \capacityratio[\res][\period]|
    \le \frac{\consumptionproc[\res][\period] + \capacityratio[\res][\period]}{\subhorizon[][\period] - 1}
    \le \frac{\maxconsumption + \nominalratio[\res] + \localradius}{\subhorizon[][\period] - 1}
    \le \frac{\maxconsumption + \max_{\res}\nominalratio[\res] + \localradius}{\subhorizon[][\period] - 1} < \localradius.
  \]
  For $\res \notin \bindingset$, only the lower boundary matters and
  \[
    \capacityratio[\res][\period + 1] \ge \capacityratio[\res][\period] - \frac{\consumptionproc[\res][\period]}{\subhorizon[][\period] - 1}
    > \nominalratio[\res] - \localradius - \frac{\maxconsumption}{\subhorizon[][\period] - 1} > \nominalratio[\res] - 2\localradius.
  \]
  Thus $\capacityratio[][\period + 1] \in \capacityneighborhood[2\localradius]$.

  Substitute \eqref{eq:cost-decomposition} into \eqref{eq:ratio-recursion}, and then subtract the two Poisson-correction terms in the definition of $\correcteddeviation$. Direct cancellation gives \eqref{eq:corrected-recursion} and \eqref{eq:exact-remainder}.

  The preceding bound and Lemma~\ref{lem:poisson-bounds} imply
  \[
    |\poissoncorrector[\res](\capacityratio[][\period], \stateproc[][\period + 1]) - \poissoncorrector[\res](\capacityratio[][\period + 1], \stateproc[][\period + 1])|
    \le \frac{\poissonlip(\maxconsumption + \max_{\res}\nominalratio[\res] + \localradius)}{\subhorizon[][\period] - 1}.
  \]
  Also $|\poissoncorrector[\res](\capacityratio[][\period], \stateproc[][\period])| \le \poissonbound$. Therefore \eqref{eq:exact-remainder} gives
  \[
    |\ratioremainder[\res][\period + 1]|
    \le \frac{\poissonlip(\maxconsumption + \max_{\res}\nominalratio[\res] + \localradius)}{(\subhorizon[][\period] - 1)^2}
    +\frac{\poissonbound}{\subhorizon[][\period](\subhorizon[][\period] - 1)}
    = \bigO\!\left((\subhorizon[][\period] - 1)^{-2}\right).
  \]
\end{proof}

\subsubsection{An exponential exit bound}
\label{app:proof:lem:exit-bound}

We use the following standard maximal form of Freedman's inequality.

\begin{fact}[Freedman's inequality, from \citet{Freedman75}]
  \label{fact:freedman}
  Let $(\martingale[][\martindex], \filtration[][\martindex])$ be a martingale with $\martingale[][0] = 0$, increments $|\martingale[][\martindex] - \martingale[][\martindex - 1]| \le \incrementbound$, and accumulated quadratic variation
  \[
    \quadvariation[][\martindex] \coloneqq \sum_{\prefixstep = 1}^{\martindex}\expect[(\martingale[][\prefixstep] - \martingale[][\prefixstep - 1])^2 \mid \filtration[][\prefixstep - 1]].
  \]
  Then, for $\deviationlevel, \variancebudget > 0$,
  \[
    \prob\left(\max_{\prefixstep \le \martindex}\martingale[][\prefixstep] \ge \deviationlevel, \ \quadvariation[][\martindex] \le \variancebudget\right)
    \le \exp\left(-\frac{\deviationlevel[][2]}{2(\variancebudget + \incrementbound\deviationlevel / 3)}\right).
  \]
  Applying the result to $\martingale$ and $-\martingale$ gives the corresponding two-sided bound with a factor $2$.
\end{fact}

Fix a post-entry stopping time $\period[0]$ with $\stateproc[][\period[0]] \in \recstates$ and define
\begin{equation*}
  \analyticalstop[N](\period[0]) \coloneqq \inf\{\period \ge \period[0]:\subhorizon[][\period] > \cutrest\text{ and }\capacityratio[][\period] \notin \capacityneighborhood[\localradius]\},
\end{equation*}
with the convention $\inf\varnothing = \infty$. Under the hypotheses of Lemma~\ref{lem:exit-bound}, the stopped quantities in its proof use only regular rounds before $\analyticalstop[N](\period[0])$ and above the fixed cutoff, so every Poisson term is well defined by the one-step buffer.

\begin{lemma}
  \label{lem:exit-bound}
  Conditional on any history for which $\gfrequentalg$ has not been analytically truncated before $\period[0]$ and
  \begin{equation*}
    \stateproc[][\period[0]] \in \recstates, \qquad \subhorizon[][\period[0]] > \cutrest, \qquad
    \capacityratio[][\period[0]] \in \capacityneighborhood[\localradius / 4],
  \end{equation*}
  the following bound holds conditionally on $\filtration[][\period[0]]$ for every integer $\epochhorizon \ge \cutrest$:
  \begin{equation}
    \prob(\analyticalstop[N](\period[0]) < \infty, \ \subhorizon[][\analyticalstop[N](\period[0])] \ge \epochhorizon \mid \filtration[][\period[0]])
    \le \exp(-\bigOmega(\epochhorizon)).
    \label{eq:exit-tail}
  \end{equation}
\end{lemma}

\begin{proof}
  Condition on a history in $\filtration[][\period[0]]$ satisfying the stated hypotheses. The guarded comparison follows FR from $\period[0]$ until its first subsequent neighborhood exit or fixed cutoff, so Lemma~\ref{lem:exact-ratio-recursion} applies to every increment used below. Fix a resource $\res$ and an integer $\epochhorizon \ge \cutrest$. If $\epochhorizon > \subhorizon[][\period[0]]$, the event in \eqref{eq:exit-tail} is empty, so assume $\epochhorizon \le \subhorizon[][\period[0]]$. Rounds whose increments can affect a state with remaining horizon at least $\epochhorizon$ satisfy $\subhorizon[][\period] - 1 \ge \epochhorizon$. Define the stopped martingale, for $\martindex = \period[0] - 1, \period[0], \ldots, \horizon - \epochhorizon$,
  \begin{equation*}
    \martingale[\res][\martindex, (\epochhorizon)] \coloneqq \sum_{\period = \period[0]}^{\martindex \wedge (\analyticalstop[N](\period[0]) - 1)}\frac{\martdiff[\res][\period + 1]}{\subhorizon[][\period] - 1}.
  \end{equation*}
  With respect to the shifted filtration $(\filtration[][\martindex + 1])$, this is a martingale because $\indicator\{\period[0] \le \period < \analyticalstop[N](\period[0])\}$ is $\filtration[][\period]$-measurable. Its increments are $\bigO(\epochhorizon[][-1])$, and its accumulated quadratic variation is $\bigO(\epochhorizon[][-1])$ by computation.
  By \eqref{eq:remainder-bound}, the cumulative remainder over the same rounds satisfies
  \begin{equation}
    \sum_{\period = \period[0]}^{(\horizon - \epochhorizon) \wedge (\analyticalstop[N](\period[0]) - 1)}|\ratioremainder[\res][\period + 1]|
    = \bigO(\epochhorizon[][-1]).
    \label{eq:stopped-remainder}
  \end{equation}
  Finally,
  \begin{equation}
    |\correcteddeviation[\res][\period] - (\capacityratio[\res][\period] - \nominalratio[\res])| \le \frac \poissonbound{\subhorizon[][\period]}.
    \label{eq:W-beta-distance}
  \end{equation}

  For all sufficiently large $\epochhorizon$, the Poisson correction, which is at most $2\poissonbound / \epochhorizon$, and the $\bigO(\epochhorizon[][-1])$ accumulated remainder sum to at most $\localradius / 4$. We first treat such $\epochhorizon$.

  \noindent\textbf{Binding resource.}
  For $\res \in \bindingset$, \eqref{eq:resource-drift} makes the second term in \eqref{eq:corrected-recursion} zero. If this coordinate exits while $\subhorizon[][\analyticalstop[N](\period[0])] \ge \epochhorizon$, then
  \[
    |\capacityratio[\res][\analyticalstop[N](\period[0])] - \nominalratio[\res]| \ge \localradius.
  \]
  Summing \eqref{eq:corrected-recursion} from $\period[0]$ to $\analyticalstop[N](\period[0]) - 1$ and then using \eqref{eq:W-beta-distance} and \eqref{eq:stopped-remainder} together with the regular-start condition gives
  \[
    |\correcteddeviation[\res][\period[0]]| \le \frac\localradius4 + \frac \poissonbound{\subhorizon[][\period[0]]}
  \]
  for binding coordinates. For all sufficiently large $\epochhorizon$, the initial offset, terminal Poisson correction, and accumulated remainder together use at most $\localradius / 2$ of the available $\localradius$ margin. Hence an exit with remaining horizon at least $\epochhorizon$ implies
  \[
    \max_{\martindex \le \horizon - \epochhorizon}|\martingale[\res][\martindex, (\epochhorizon)]| \ge \frac\localradius2.
  \]
  Fact~\ref{fact:freedman}, with $\deviationlevel = \localradius / 2$ and the preceding increment and variation bounds, gives
  \begin{equation*}
    \prob(\text{binding coordinate $\res$ exits while }\subhorizon \ge \epochhorizon)
    \le \exp(-\bigOmega(\epochhorizon)).
  \end{equation*}

  \noindent\textbf{Non-binding resource.}
  For $\res \notin \bindingset$, only a lower-bound exit is possible. By \eqref{eq:resource-drift},
  \[
    -\frac{\stationaryconsumption[\res](\capacityratio[][\period]) - \capacityratio[\res][\period]}{\subhorizon[][\period] - 1}
    \ge \frac\positivitymargin{\subhorizon[][\period] - 1} \ge 0.
  \]
  Thus, after summing \eqref{eq:corrected-recursion}, dropping this nonnegative drift can only decrease the lower bound on $\correcteddeviation[\res][\period]$:
  \[
    \correcteddeviation[\res][\period] \ge \correcteddeviation[\res][\period[0]] - \martingale[\res][\period - 1, (\epochhorizon)] - \bigO(\epochhorizon[][-1]).
  \]
  For all sufficiently large $\epochhorizon$, if $\capacityratio[\res][\period] \le \nominalratio[\res] - \localradius$ and $\subhorizon[][\period] \ge \epochhorizon$, then \eqref{eq:W-beta-distance} implies
  \[
    \martingale[\res][\period - 1, (\epochhorizon)] \ge \frac\localradius2.
  \]
  The one-sided Freedman bound gives
  \begin{equation*}
    \prob(\text{non-binding coordinate $\res$ exits while }\subhorizon \ge \epochhorizon)
    \le \exp(-\bigOmega(\epochhorizon)).
  \end{equation*}

  A union bound over the finite resource set proves \eqref{eq:exit-tail} for all sufficiently large $\epochhorizon$. The finitely many remaining integers $\epochhorizon \ge \cutrest$ are absorbed into the multiplicative constant implicit in \eqref{eq:exit-tail}.
\end{proof}

\subsubsection{Finishing the proof for frequent re-solving}
\label{app:proof:thm:constant-regret}

By Proposition~\ref{prop:nondegenerate-full-state-support}, $\state[][1] \in \recstates$ and the entrance phase is empty. Let
\begin{equation*}
  \analyticalstop \coloneqq \inf\{\period \ge 1:\subhorizon[][\period] \le \cutrest\text{ or }\capacityratio[][\period] \notin \capacityneighborhood[\localradius]\}
\end{equation*}
be the first round in which the guarded comparison is truncated. Before $\analyticalstop$, the current parameter lies in $\capacityneighborhood[\localradius]$, and Proposition~\ref{prop:support-regularity} keeps the state in $\recstates$. Consequently, Lemma~\ref{lem:potential-identity} makes the reward-plus-potential increments martingale differences up to $\analyticalstop$. Conditional finite-horizon telescoping from $1$ gives
\begin{equation*}
  \potential(\subhorizon[][1], \inventory[][1], \stateproc[][1]) - \valuefn[\horizon][\gfrequentalg] \le \expect\left[\potential(\subhorizon[][\analyticalstop], \inventory[][\analyticalstop], \stateproc[][\analyticalstop])^+ \mid \filtration[][\entrancelast]\right].
\end{equation*}

Whether the cutoff or a neighborhood exit is reached, the preceding regular round and the one-step buffer give $\capacityratio[][\analyticalstop] \in \capacityneighborhood[2\localradius]$. Thus the inventory is at most $\subhorizon[][\analyticalstop](\nominalratio[\res] + 2\localradius)$ for every $\res \in \bindingset$, while $\resourceprice[\res][*] = 0$ for $\res \notin \bindingset$. The definition of $\potential$ therefore gives
\[
  \potential(\subhorizon[][\analyticalstop], \inventory[][\analyticalstop], \stateproc[][\analyticalstop])^+ \le \genericconstant(1 + \subhorizon[][\analyticalstop]).
\]
At the cutoff, $\subhorizon[][\analyticalstop] \le \cutrest$, so this is a deterministic constant. On a neighborhood exit, the tail-sum formula and Lemma~\ref{lem:exit-bound} give the required integrable bound: conditionally on $\filtration[][1]$,
\begin{align*}
  \expect[\subhorizon[][\analyticalstop]\indicator\{\text{neighborhood exit}\} \mid \filtration[][1]]
  & \le \cutrest + \sum_{\epochhorizon = \cutrest + 1}^{\infty}\exp(-\bigOmega(\epochhorizon)) = \bigO(1),
\end{align*}
uniformly over all regular entry histories. Consequently,
\begin{equation*}
  \expect\left[\potential(\subhorizon[][\analyticalstop], \inventory[][\analyticalstop], \stateproc[][\analyticalstop])^+\right] = \bigO(1).
\end{equation*}
Hence
\[
  \horizon\stationaryfluid(\nominalratio) - \valuefn[\horizon][\gfrequentalg] = \bigO(1).
\]

%% file: appendices/06_infrequent_resolving_proofs.tex

\subsection{Infrequent re-solving analysis and proof}
\label{app:infrequent-proofs}
\label{sec:infrequent-analysis}

This section analyzes $\ginfrequentalg$ from Appendix~\ref{app:algorithm-proofs}. 
By \eqref{eq:ir-dominates-guarded-ir}, the implemented guard-free IR policy earns at least as much reward on every coupled sample path, so all upper bounds transfer to IR\@.

\subsubsection{Uniform concentration within a frozen epoch}
\label{app:proof:lem:frozen-concentration}

We use the following maximal bounded-increment inequality.

\begin{fact}[Maximal Azuma--Hoeffding inequality, from \citet{Roch24}]
  \label{fact:maximal-azuma}
  If $(\martingale[][\prefixidx])_{\prefixidx = 0}^{\blocklength}$ is a martingale with $\martingale[][0] = 0$ and $|\martingale[][\prefixidx] - \martingale[][\prefixidx - 1]| \le \incrementbound$ almost surely, then
  \[
    \prob\left(\max_{0 \le \prefixidx \le \blocklength}|\martingale[][\prefixidx]| \ge \deviationlevel\right)
    \le 2\exp\left(-\frac{\deviationlevel[][2]}{2\blocklength\incrementbound[][2]}\right).
  \]
\end{fact}

Fix $\capacityratio \in \capacityneighborhood[2\localradius]$ and run the stationary randomized policy $\policy[\capacityratio]$ for $\blocklength$ rounds from an arbitrary initial state in $\recstates$. 
We have the following lemma. 
Here $\prob$ denotes probability under this frozen policy; its stationary mean consumption of resource $\res$ is the already defined quantity $\stationaryconsumption[\res](\capacityratio)$.

\begin{lemma}
  \label{lem:frozen-concentration}
  For every resource $\res$ and $\deviationlevel \ge \max\{4\poissonbound, 1\}$,
  \begin{equation*}
    \prob\left(
    \max_{1 \le \prefixidx \le \blocklength}
    \left|
    \sum_{\prefixstep = 1}^{\prefixidx}(\consumptionproc[\res][\prefixstep] - \stationaryconsumption[\res](\capacityratio))
    \right| \ge \deviationlevel
    \right)
    \le 2\exp\left(-\frac{\deviationlevel[][2]}{8\blocklength(\maxconsumption + 2\poissonbound)^2}\right).
  \end{equation*}
\end{lemma}

\begin{proof}
  Use the Poisson solution $\poissoncorrector[\res](\capacityratio, \cdot)$. 
  Under the frozen policy, similarly define
  \[
    \martdiff[\res][\prefixstep + 1] \coloneqq \consumptionproc[\res][\prefixstep] - \stateconsumption[\res](\capacityratio, \stateproc[][\prefixstep])
    +\poissoncorrector[\res](\capacityratio, \stateproc[][\prefixstep + 1])
    -\sum_{\state[][\prime]}\inducedkernel[\capacityratio](\stateproc[][\prefixstep], \state[][\prime])\poissoncorrector[\res](\capacityratio, \state[][\prime]).
  \]
  The conditional law of the request, action, and next state gives
  \[
    \expect[\martdiff[\res][\prefixstep + 1] \mid \filtration[][\prefixstep]] = 0,
    \qquad
    |\martdiff[\res][\prefixstep + 1]| \le \maxconsumption + 2\poissonbound.
  \]
  The Poisson equation yields
  \begin{equation}
    \consumptionproc[\res][\prefixstep] - \stationaryconsumption[\res](\capacityratio)
    = \martdiff[\res][\prefixstep + 1] + \poissoncorrector[\res](\capacityratio, \stateproc[][\prefixstep]) - \poissoncorrector[\res](\capacityratio, \stateproc[][\prefixstep + 1]).
    \label{eq:frozen-martingale-decomposition}
  \end{equation}
  Summing to a prefix $\prefixidx$ gives
  \[
    \sum_{\prefixstep = 1}^{\prefixidx}(\consumptionproc[\res][\prefixstep] - \stationaryconsumption[\res](\capacityratio))
    = \sum_{\prefixstep = 1}^{\prefixidx}\martdiff[\res][\prefixstep + 1] + \poissoncorrector[\res](\capacityratio, \stateproc[][1]) - \poissoncorrector[\res](\capacityratio, \stateproc[][\prefixidx + 1]).
  \]
  The last two terms have total absolute value at most $2\poissonbound$. Therefore, for $\deviationlevel \ge 4\poissonbound$, a centered-cost deviation of size $\deviationlevel$ implies a martingale deviation of size at least $\deviationlevel / 2$. Fact~\ref{fact:maximal-azuma} gives
  \[
    \prob\left(
    \max_{\prefixidx \le \blocklength}\left|\sum_{\prefixstep = 1}^{\prefixidx}(\consumptionproc[\res][\prefixstep] - \stationaryconsumption[\res](\capacityratio))\right| \ge \deviationlevel
    \right)
    \le 2\exp\left(-\frac{\deviationlevel[][2]}{8\blocklength(\maxconsumption + 2\poissonbound)^2}\right).
  \]
\end{proof}

\subsubsection{Stability of infrequent resolving points}
\label{app:proof:lem:one-epoch-bounds}
\label{app:proof:lem:protected-first-epoch}
\label{app:proof:lem:cumulative-stability}

Choose $\concentrationpower$ such that
\begin{equation}
  0 < \concentrationpower < 1 - \frac1{2\schedulepower},
  \qquad
  \concentrationexponent \coloneqq 2\schedulepower(1 - \concentrationpower) - 1 > 0.
  \label{eq:zeta-kappa}
\end{equation}
For every epoch with $\epochhorizon[\epoch + 1] > \cutrest$, define
\begin{equation*}
  \epochbuffer[\epoch] \coloneqq \epochhorizon[\epoch + 1][-\concentrationpower].
\end{equation*}

In the following, we temporarily write $\capacityratio[][\period[\epoch]]$ as $\capacityratio[][\epoch]$ for brevity.
Notice that the entrance policy is not invoked when the stationary LP has a unique nondegenerate optimum according to Proposition~\ref{prop:nondegenerate-full-state-support}.

The IR epoch analysis uses two of the checks in the convention of
Appendix~\ref{app:algorithm-proofs}.  The \emph{resource near-depletion check} is the
pre-action test
\begin{equation*}
  \min_{\res\in[\numres]}\inventory[\res][\period]\le\maxconsumption.
\end{equation*}
The guarded comparison stops before taking the next action when this check
fires.  It is not the actual depletion rule \eqref{eq:stopping}: since every
coordinate of one-period consumption is at most $\maxconsumption$, actual
depletion cannot occur before the near-depletion check fires.

For an epoch $\epoch$ satisfying the hypotheses below, condition on
$\filtration[][\period[\epoch]]$ and couple the guarded comparison to a
\emph{full frozen continuation} that uses
$\policy[\capacityratio[][\epoch]]$ for all $\blocklength[\epoch]$ scheduled
rounds.  Use the same requests, action randomizations, and transition
randomizations until the resource near-depletion check fires.  Denote the
hypothetical costs by $\hypotheticalcost[][\period]$ and define the hypothetical endpoint
ratio by
\begin{equation}
  \hypotheticalratio[][\epoch+1]
  \coloneqq \capacityratio[][\epoch]
  +\frac{\blocklength[\epoch]\capacityratio[][\epoch]
    -\sum_{\period=\period[\epoch]}^{\period[\epoch+1]-1}
      \hypotheticalcost[][\period]}
    {\epochhorizon[\epoch+1]}.
  \label{eq:epoch-ratio-recursion}
\end{equation}
At the hypothetical endpoint, the \emph{epoch-movement check} fails if
\begin{equation*}
  \left|\hypotheticalratio[\res][\epoch+1]
    -\capacityratio[\res][\epoch]\right|>\epochbuffer[\epoch]
  \quad\text{for some }\res\in\bindingset,
  \quad\text{or}\quad
  \hypotheticalratio[\res][\epoch+1]
    <\capacityratio[\res][\epoch]-\epochbuffer[\epoch]
  \quad\text{for some }\res\notin\bindingset.
\end{equation*}
Let $\event[\epoch]$ be the event that either the epoch-movement check fails or
the resource near-depletion check fires during epoch $\epoch$.  Thus the two
checks are distinct: the first controls movement of the endpoint ratio, while
the second prevents resource depletion within an epoch.

\begin{lemma}
  \label{lem:one-epoch-bounds}
  Suppose the guarded comparison reaches $\period[\epoch]$ without truncation,
  uses the frozen policy $\policy[\capacityratio[][\epoch]]$ during epoch
  $\epoch$, and satisfies
  $\capacityratio[][\epoch]\in\capacityneighborhood[\localradius]$ and
  $\epochhorizon[\epoch+1]>\cutrest$.  Then
  \[
    \prob(\event[\epoch]\mid\filtration[][\period[\epoch]])
    \le
    \exp(-\bigOmega(\epochhorizon[\epoch][\concentrationexponent])).
  \]
\end{lemma}

\begin{proof}
  Condition on $\filtration[][\period[\epoch]]$ and use the full frozen
  continuation defined above.  Until the resource near-depletion check fires,
  the actual and hypothetical trajectories agree.  Hence, if neither component
  of $\event[\epoch]$ occurs, the guarded comparison reaches the endpoint and
  its endpoint ratio equals $\hypotheticalratio[][\epoch+1]$.

  \noindent\textbf{Binding coordinates.}
  For $\res \in \bindingset$, $\stationaryconsumption[\res](\capacityratio[][\epoch]) = \capacityratio[\res][\epoch]$. Hence
  \[
    \hypotheticalratio[\res][\epoch + 1] - \capacityratio[\res][\epoch] = -\frac1{\epochhorizon[\epoch + 1]}
    \sum_{\period = \period[\epoch]}^{\period[\epoch + 1] - 1}
    (\hypotheticalcost[\res][\period] - \stationaryconsumption[\res](\capacityratio[][\epoch])).
  \]
  Applying Lemma~\ref{lem:frozen-concentration} with
  \[
    \deviationlevel = \epochbuffer[\epoch]\epochhorizon[\epoch + 1] = \epochhorizon[\epoch + 1][1 - \concentrationpower]\]
  gives
  \[
    \prob\left(
    |\hypotheticalratio[\res][\epoch + 1] - \capacityratio[\res][\epoch]| > \epochbuffer[\epoch] \mid \filtration[][\period[\epoch]]\right)
    \le \exp\left(
    -\bigOmega\!\left(\frac{\epochhorizon[\epoch + 1][2(1 - \concentrationpower)]}{\blocklength[\epoch]}\right)\right).
  \]
  Because $\epochhorizon[\epoch + 1] \ge \epochhorizon[\epoch][\schedulepower]$ and $\blocklength[\epoch] \le \epochhorizon[\epoch]$, this is $\exp(-\bigOmega(\epochhorizon[\epoch][\concentrationexponent]))$.

  \noindent\textbf{Non-binding coordinates.}
  For $\res \notin \bindingset$, $\capacityratio[\res][\epoch] - \stationaryconsumption[\res](\capacityratio[][\epoch]) \ge \positivitymargin$. \eqref{eq:epoch-ratio-recursion} shows that
  \[
    \hypotheticalratio[\res][\epoch + 1] < \capacityratio[\res][\epoch] - \epochbuffer[\epoch]\]
  requires
  \[
    \sum_{\period = \period[\epoch]}^{\period[\epoch + 1] - 1}
    (\hypotheticalcost[\res][\period] - \stationaryconsumption[\res](\capacityratio[][\epoch]))
    > \epochbuffer[\epoch]\epochhorizon[\epoch + 1] + \positivitymargin \blocklength[\epoch].
  \]
  This threshold is at least $\epochbuffer[\epoch]\epochhorizon[\epoch + 1]$, so the same exponential estimate applies.

  \noindent\textbf{Resource near-depletion.}
  Let
  \begin{equation*}
    \minimumcapacityratio \coloneqq \min_{\res}(\nominalratio[\res] - \localradius) > 0.
  \end{equation*}
  After a hypothetical prefix of length $\prefixidx \le \blocklength[\epoch]$, the deterministic inventory reference centered at the stationary mean satisfies
  \begin{align*}
    \epochhorizon[\epoch]\capacityratio[\res][\epoch] - \prefixidx\stationaryconsumption[\res](\capacityratio[][\epoch])
    & \ge (\epochhorizon[\epoch] - \prefixidx)\capacityratio[\res][\epoch]\\
    & \ge \epochhorizon[\epoch + 1]\minimumcapacityratio.
  \end{align*}
  The first inequality uses $\stationaryconsumption[\res](\capacityratio[][\epoch]) \le \capacityratio[\res][\epoch]$, with equality for binding resources and strict slack for non-binding resources. For each fixed prefix length, the conditional expected inventory of the hypothetical frozen-policy process can differ from this reference by at most $2\poissonbound$ because of the Poisson boundary terms in \eqref{eq:frozen-martingale-decomposition}. The pathwise near-depletion comparison below uses the stationary-centered cost sum, whose concentration bound already includes those terms. If the resource near-depletion check fires after $\prefixidx$ preceding actions, the coupled hypothetical inventory at that prefix is at most $\maxconsumption$. Therefore
  \[
  \sum_{\prefixstep = 0}^{\prefixidx-1}
    (\hypotheticalcost[\res][\period[\epoch]+\prefixstep]
      - \stationaryconsumption[\res](\capacityratio[][\epoch]))
    \ge \epochhorizon[\epoch + 1]\minimumcapacityratio - \maxconsumption
  \]
  for some resource $\res$. Enlarge $\cutrest$ so that this threshold is at least $(\minimumcapacityratio / 2)\epochhorizon[\epoch + 1]$ and at least $\max\{4\poissonbound, 1\}$. Lemma~\ref{lem:frozen-concentration} then gives
  \[
    \prob(\text{resource near-depletion in epoch }\epoch \mid \filtration[][\period[\epoch]])
    \le \genericconstant\exp\left(-\thresholdfraction\frac{\epochhorizon[\epoch + 1][2]}{\blocklength[\epoch]}\right)
    \le \genericconstant \exp(-\thresholdfraction \epochhorizon[\epoch][2\schedulepower - 1]).
  \]
  Since $\concentrationexponent < 2\schedulepower - 1$, a union bound over the
  two components of $\event[\epoch]$ and the finitely many resources proves the
  lemma.
\end{proof}

In the following lemma, with minor abuse of notation, we use $\numepochs$ to denote the first schedule index with $\epochhorizon[\numepochs] \le \cutrest$.
Recall that it originally refers to the index of the final re-solving epoch of the non-truncated IR\@.

\begin{lemma}
  \label{lem:cumulative-stability}
  The cutoff $\cutrest$ can be enlarged so that, for every $\horizon$,
  \begin{equation}
    \sum_{\epoch = 0}^{\numepochs - 2}\epochbuffer[\epoch] < \frac\localradius2,
    \label{eq:sum-movements}
  \end{equation}
  with an empty sum when $\numepochs \le 1$. More generally, fix any $\epoch[0] \in \{0, \ldots, \numepochs - 1\}$. If $\capacityratio[][\epoch[0]] \in \capacityneighborhood[\localradius / 4]$ and none of the events $\event[\epoch]$ defined before Lemma~\ref{lem:one-epoch-bounds} occurs for $\epoch = \epoch[0], \ldots, \numepochs - 2$, then every resolving ratio $\capacityratio[][\epoch]$, $\epoch[0] \le \epoch \le \numepochs - 1$, lies in $\capacityneighborhood[\localradius]$. 
\end{lemma}

\begin{proof}
  Choose any number $\schedulepower[][\prime]$ with $\schedulepower < \schedulepower[][\prime] < 1$. Enlarge $\cutrest$ so that
  \begin{equation}
    \lceil \epochhorizon[][\schedulepower]\rceil \le \epochhorizon[][\schedulepower[][\prime]]\qquad(\epochhorizon > \cutrest).
    \label{eq:alpha-prime-schedule}
  \end{equation}
  If $\numepochs \le 1$, \eqref{eq:sum-movements} is an empty sum and the propagation assertion is immediate. For the movement-sum argument, assume $\numepochs \ge 2$.
  For $\epoch \le \numepochs - 2$, both $\epochhorizon[\epoch]$ and $\epochhorizon[\epoch + 1]$ exceed $\cutrest$, and \eqref{eq:alpha-prime-schedule} implies
  \[
    \epochhorizon[\epoch][-\concentrationpower] \le \epochhorizon[\epoch + 1][-\concentrationpower / \schedulepower[][\prime]].
  \]
  Let $\seriesratio \coloneqq \epochhorizon[\numepochs - 1][-\concentrationpower] \le \cutrest[][-\concentrationpower] < 1$. Reading the sequence backward gives
  \[
    \epochhorizon[\numepochs - 2][-\concentrationpower] \le \seriesratio[][1 / \schedulepower[][\prime]],
    \quad
    \epochhorizon[\numepochs - 3][-\concentrationpower] \le \seriesratio[][1 / (\schedulepower[][\prime])^2],
    \quad\ldots.
  \]
  The ratio of each successive term in this backward list to the preceding one is at most
  $\seriesratio[][1 / \schedulepower[][\prime] - 1] \le
  \cutrest[][-\concentrationpower(1 / \schedulepower[][\prime] - 1)]$.  Enlarge
  $\cutrest$ so that this bound is at most $1 / 2$. Then
  \begin{equation*}
    \sum_{\epoch = 0}^{\numepochs - 2}\epochbuffer[\epoch] = \sum_{\prefixstep = 1}^{\numepochs - 1}\epochhorizon[\prefixstep][-\concentrationpower] \le 2\seriesratio
    \le 2\cutrest[][-\concentrationpower].
  \end{equation*}
  A final enlargement makes this smaller than $\localradius / 2$.

  If $\capacityratio[][\epoch[0]] \in \capacityneighborhood[\localradius / 4]$, then, on the complement of all bad events, every binding coordinate changes by at most $\epochbuffer[\epoch]$ in absolute value during epoch $\epoch$, and every non-binding coordinate decreases by at most $\epochbuffer[\epoch]$. The tail sum $\sum_{\epoch = \epoch[0]}^{\numepochs - 2}\epochbuffer[\epoch]$ is bounded by the full sum in \eqref{eq:sum-movements}. Induction therefore gives
  \[
    |\capacityratio[\res][\epoch] - \nominalratio[\res]| < 3\localradius / 4\quad(\res \in \bindingset),
    \qquad
    \capacityratio[\res][\epoch] > \nominalratio[\res] - 3\localradius / 4\quad(\res \notin \bindingset)
  \]
  for $\epoch[0] \le \epoch \le \numepochs - 1$. Hence all these average capacities lie in $\capacityneighborhood[\localradius]$.
\end{proof}

\subsubsection{Finishing the proof for infrequent re-solving}
\label{app:proof:thm:infrequent-constant-regret}

\begin{proof}
  By Proposition~\ref{prop:nondegenerate-full-state-support},
  $\state[][1] \in \recstates$ and the entrance phase is empty. 
  First handle $\epochhorizon[1] \le \cutrest$. By the clipped recursion, this implies $\epochhorizon[0] \le \max\{\cutrest + 1, \cutrest[][1 / \schedulepower]\}$, a fixed constant. The inventory at entry is at most $\epochhorizon[0](\nominalratio[\res] + \localradius / 4)$ for every $\res \in \bindingset$ and can only decrease. Since $\resourceprice[\res][*] = 0$ for $\res \notin \bindingset$, the definition of $\potential$ therefore bounds every later positive potential by $\genericconstant(1 + \epochhorizon[0]) = \bigO(1)$, so nonnegative continuation rewards suffice for this case. We may therefore assume $\epochhorizon[1] > \cutrest$, and hence $\numepochs \ge 2$.

  At $\period[0]=1$, the epoch-$0$ ratio is $\inventory[][1]/\epochhorizon[0]=\nominalratio$, and epoch $0$ uses the corresponding frozen policy $\policy[\nominalratio]$. Thus Lemma~\ref{lem:one-epoch-bounds} applies also to epoch $0$. For $\epoch = 0, \ldots, \numepochs - 2$, use the bad event
  $\event[\epoch]$ defined before Lemma~\ref{lem:one-epoch-bounds}, and define
  \[
    \firstbadepoch \coloneqq \inf\{\epoch \in \{0, \ldots, \numepochs - 2\}:\event[\epoch]\text{ occurs}\},
  \]
  with $\firstbadepoch = \infty$ if no such event occurs.  Let $\analyticalstop$ be the round in which the guarded comparison is first truncated by the locality, resource near-depletion, or cutoff rule specified in Appendix~\ref{app:algorithm-proofs}. Every pre-$\analyticalstop$ action in epoch $0$ uses $\policy[\nominalratio]$, and every later pre-$\analyticalstop$ action uses a frozen policy $\policy[\capacityratio[][\epoch]]$ with parameter in $\capacityneighborhood[\localradius]$; Proposition~\ref{prop:support-regularity} keeps the state in $\recstates$. Consequently, Lemma~\ref{lem:potential-identity} makes the reward-plus-potential increments martingale differences up to $\analyticalstop$. Finite-horizon telescoping from round $1$, followed by the zero-reward null actions after truncation and deletion of the negative part of the terminal potential, gives
  \begin{equation*}
    \potential(\subhorizon[][1], \inventory[][1], \stateproc[][1]) - \valuefn[\horizon][\ginfrequentalg] \le \expect\left[\potential(\subhorizon[][\analyticalstop], \inventory[][\analyticalstop], \stateproc[][\analyticalstop])^+\right].
  \end{equation*}

  On $\{\firstbadepoch = \infty\}$, Lemma~\ref{lem:cumulative-stability} with $\epoch[0] = 0$ shows that all subsequent pre-cutoff resolving ratios remain stable and the last epoch before the proof cutoff begins with $\epochhorizon[\numepochs - 1] \le \cutrest[][1 / \schedulepower]$. At the start of that last epoch, the inventory is at most $\epochhorizon[\numepochs - 1](\nominalratio[\res] + \localradius)$ for every $\res \in \bindingset$ and can only decrease. Since $\resourceprice[\res][*] = 0$ for $\res \notin \bindingset$, the definition of $\potential$ and the displayed bound on $\epochhorizon[\numepochs - 1]$ give a deterministic constant bound on the residual at truncation. On $\{\firstbadepoch = \epoch\}$, no earlier bad event occurred, so $\capacityratio[][\epoch] \in \capacityneighborhood[\localradius]$, and the same argument from the regular start of epoch $\epoch$ gives
  \[
    \potential(\subhorizon[][\analyticalstop], \inventory[][\analyticalstop], \stateproc[][\analyticalstop])^+ \le \genericconstant(1 + \epochhorizon[\epoch]).
  \]
  Applying Lemma~\ref{lem:one-epoch-bounds} conditionally on $\filtration[][\period[\epoch]]$ and then using the tower property gives
  \[
    \prob(\firstbadepoch = \epoch)
    \le \exp\bigl(-\bigOmega(\epochhorizon[\epoch][\concentrationexponent])\bigr).
  \]
  Therefore,
  \[
    \expect[\potential(\subhorizon[][\analyticalstop], \inventory[][\analyticalstop], \stateproc[][\analyticalstop])^+]
    \le \genericconstant + \genericconstant\sum_{\epoch = 0}^{\numepochs - 2}(1 + \epochhorizon[\epoch])\exp(-\thresholdfraction \epochhorizon[\epoch][\concentrationexponent]).
  \]
  The last sum is uniformly bounded in $\horizon$.  Indeed, before the cutoff the sequence $(\epochhorizon[\epoch])$ is a strictly decreasing sequence of distinct integers larger than $\cutrest$.  Hence
  \begin{equation*}
    \sum_{\epoch = 0}^{\numepochs - 2}(1 + \epochhorizon[\epoch])\exp(-\thresholdfraction \epochhorizon[\epoch][\concentrationexponent]) \le
    \sum_{\epochhorizon = \cutrest + 1}^{\infty}(1 + \epochhorizon)\exp(-\thresholdfraction \epochhorizon[][\concentrationexponent]) < \infty,
  \end{equation*}
  where the right-hand side depends only on the fixed instance constants, $\cutrest$, and $\concentrationexponent > 0$. Thus
  \[
    \expect[\potential(\subhorizon[][\analyticalstop], \inventory[][\analyticalstop], \stateproc[][\analyticalstop])^+] = \bigO(1).
  \]
  
  Hence
  \[
    \horizon\stationaryfluid(\nominalratio) - \valuefn[\horizon][\ginfrequentalg] = \bigO(1).
  \]
\end{proof}

%% file: appendices/07_1_boundary_cases_proofs.tex
\section{Incomplete regularity analysis and proofs}
\label{app:boundary-proofs}

\subsection{Proof of Theorem~\ref{thm:deg-gimfr-sqrt}}
\label{app:proof:thm:deg-gimfr-sqrt}

Degeneracy creates two difficulties that were absent from the constant-regret proof.  
First, a nearby LP may have several optimizers, so there need not be a locally fixed basis.  
Second, a nearby optimizer may assign positive stationary state probability to states outside the optimal target state class.  
The analysis below handles these difficulties separately: uniqueness gives a selector-uniform Lipschitz bound on every nearby optimizer, and stationary domination converts the small total stationary probability outside the optimal target state class into a small finite-time probability of visiting those states.

Use the unique-optimum notation $\selectedsolution[][*], \bindingset, \recstates, \actsupport, \resourceslack[][*]$, and $\capacitydistance$ introduced in \eqref{eq:nominal-supports} and \eqref{eq:one-sided-capacity-distance}, together with the one-sided inventory-deviation norm $\inventorydeviationnorm$ defined in \eqref{eq:one-sided-inventory-deviation}. 

The proof of Theorem~\ref{thm:deg-gimfr-sqrt} uses the following lemmas. 
Lemma~\ref{lem:deg-gimfr-polyhedral-sensitivity} is the only step that uses nominal uniqueness; it replaces local basis stability by an upper-Lipschitz estimate. 
Lemma~\ref{lem:deg-gimfr-outside-domination} shows that states outside the optimal target state class remain rare throughout a frozen epoch. 
Lemma~\ref{lem:deg-gimfr-frozen-additive} then gives a maximal additive-functional estimate without requiring the perturbed stationary state distribution to be supported on the optimal target state class. 
Finally, Lemma~\ref{lem:deg-gimfr-epoch-control} propagates inventory deviations through a fixed number of epochs, by which time the remaining horizon is $\bigO(\sqrt {\horizon})$. 
The theorem then follows from one stopped dual-potential identity.

For brevity, we define $\outsidestates \coloneqq \states \setminus \recstates$ in this proof.

\subsubsection{Lipschitz continuity of optimizers}
\begin{lemma}
  \label{lem:deg-gimfr-polyhedral-sensitivity}
  There is a radius $\localradius[0] > 0$ such that, whenever $\capacitydistance(\capacityratio) \le \localradius[0]$, every optimizer $\selectedsolution(\capacityratio)$ of $\stationaryfluid(\capacityratio)$ satisfies
  \begin{equation}
    \|\selectedsolution(\capacityratio) - \selectedsolution[][*]\|_1
    = \bigO(\capacitydistance(\capacityratio)).
    \label{eq:deg-gimfr-optimizer-lipschitz}
  \end{equation}
  In particular,
  there is $\positivitymargin > 0$ such that every selected local optimizer satisfies
  \begin{align}
    \stateoccupation[\state](\capacityratio) & \ge \positivitymargin,
    & & \state \in \recstates,
    \label{eq:deg-gimfr-core-positive}\\
    \sum_{\state \in \outsidestates}\stateoccupation[\state](\capacityratio) & = \bigO(\capacitydistance(\capacityratio)),
    \label{eq:deg-gimfr-outside-mass}\\
    \|\policy[\capacityratio](\cdot \mid \state, \req) - \policy[\nominalratio](\cdot \mid \state, \req)\|_1
    & = \bigO(\capacitydistance(\capacityratio)),
    & & \state \in \recstates, \ \req \in \requests.
    \label{eq:deg-gimfr-policy-close}
  \end{align}
\end{lemma}

\begin{proof}
  For a stationary flow $\selectedsolution=(\occupation,\stateoccupation)$, with minor abuse of notation, abbreviate its reward and its rate of using resource $\res$ by
  \[
    \stationaryreward(\selectedsolution)
    \coloneqq \sum_{\state,\req,\act}
    \reward(\state,\req,\act)\occupation[\state\req\act],
    \qquad
    \stationaryconsumption[\res](\selectedsolution)
    \coloneqq \sum_{\state,\req,\act}
    \consumption[\res](\state,\req,\act)\occupation[\state\req\act].
  \]
  Also set $\stationaryreward[][*]\coloneqq\stationaryreward(\selectedsolution[][*])$.

  \noindent\textbf{Step 1: one fixed comparison polytope.}
  Let $\bindingpolytope$ be the fixed polytope of nonnegative stationary flows satisfying \textup{(SF1)--(SF3)} and only the nominally binding resource constraints
  \[
    \stationaryconsumption[\res](\selectedsolution) \le \nominalratio[\res],
    \qquad \res \in \bindingset;
  \]
  the inactive resource constraints are omitted.  This polytope is compact: \textup{(SF1)--(SF3)} imply
  $\sum_{\state,\req,\act}\occupation[\state\req\act]
  =\sum_{\state}\stateoccupation[\state]=1$.

  We first show that $\selectedsolution[][*]$ is the unique reward maximizer over $\bindingpolytope$.  Take any $\comparatorsolution\in\bindingpolytope$ with $\comparatorsolution\ne\selectedsolution[][*]$.  Since there are finitely many inactive resources and each has nominal slack $\resourceslack[\res][*]>0$, one can choose a single $0<\perturbation<1$ so small that
  \begin{equation}
    \perturbation
    \bigl(\stationaryconsumption[\res](\comparatorsolution)
      -\stationaryconsumption[\res](\selectedsolution[][*])\bigr)_+
    \le \resourceslack[\res][*],
    \qquad \res\notin\bindingset.
    \label{eq:deg-gimfr-mixture-slack}
  \end{equation}
  The mixture
  \[
    \selectedsolution[\perturbation]
    \coloneqq(1-\perturbation)\selectedsolution[][*]
      +\perturbation\comparatorsolution
  \]
  is then feasible for the nominal LP\@.  The binding constraints hold by convexity because both endpoints satisfy them, while for $\res\notin\bindingset$,
  \[
    \stationaryconsumption[\res](\selectedsolution[\perturbation])
    \le \stationaryconsumption[\res](\selectedsolution[][*])
       +\perturbation
        \bigl(\stationaryconsumption[\res](\comparatorsolution)
          -\stationaryconsumption[\res](\selectedsolution[][*])\bigr)_+
    \le \nominalratio[\res]
  \]
  by \eqref{eq:deg-gimfr-mixture-slack}.  Consequently,
  $\stationaryreward(\comparatorsolution)>\stationaryreward(\selectedsolution[][*])$ would produce a strictly better nominal feasible solution, and equality would produce the distinct nominal optimizer $\selectedsolution[\perturbation]$.  Both contradict nominal uniqueness.  Thus
  \[
    \stationaryreward(\comparatorsolution)<\stationaryreward(\selectedsolution[][*])
    \qquad
    \text{for every }\comparatorsolution\in\bindingpolytope
    \setminus\{\selectedsolution[][*]\}.
  \]

  A unique maximizer of a linear objective over a polytope is a vertex.  Because $\bindingpolytope$ has finitely many vertices, the strict separation above can be made quantitative.  If $\bindingpolytope\ne\{\selectedsolution[][*]\}$, define
  \[
    \sharpnessconstant
    \coloneqq
    \min_{\substack{\vertex\in\operatorname{vert}(\bindingpolytope)\\
                     \vertex\ne\selectedsolution[][*]}}
    \frac{\stationaryreward[][*]-\stationaryreward(\vertex)}
         {\|\vertex-\selectedsolution[][*]\|_1}>0.
  \]
  If $\comparatorsolution=\sum_{\vertex}\lambda_{\vertex}\vertex\in\bindingpolytope$ is any vertex decomposition, then
  \begin{align}
    \stationaryreward[][*]-\stationaryreward(\comparatorsolution)
    &=\sum_{\vertex}\lambda_{\vertex}
      \bigl(\stationaryreward[][*]-\stationaryreward(\vertex)\bigr)  \notag\\
    &\ge \sharpnessconstant
      \sum_{\vertex}\lambda_{\vertex}
      \|\vertex-\selectedsolution[][*]\|_1  \notag\\
    &\ge \sharpnessconstant
      \|\comparatorsolution-\selectedsolution[][*]\|_1.
    \label{eq:deg-gimfr-weak-sharp}
  \end{align}

  \noindent\textbf{Step 2: move a nearby optimizer into the fixed polytope.}
  Set
  \[
    d\coloneqq\capacitydistance(\capacityratio),
    \qquad
    \underline\rho\coloneqq\min_{\res\in[\numres]}\nominalratio[\res]>0,
    \qquad
    \mixtureweight\coloneqq d/\underline\rho.
  \]
  We will shrink $\localradius[0]$ a few times.  First choose it so that
  $\localradius[0]\le\underline\rho/2$ and
  $\localradius[0]\le\resourceslack[\res][*]/2$ for every
  $\res\notin\bindingset$, and suppose $d\le\localradius[0]$.
  The second condition is vacuous if every resource is binding.
  To make the comparison flow explicit, let $\nullstationarylaw$ be any stationary distribution of the finite-state transition matrix induced by the null action, and define $\nullsolution=(\occupation[][0],\stateoccupation[][0])$ by
  \[
    \stateoccupation[\state][0]=\nullstationarylaw[\state],
    \qquad
    \occupation[\state\req\act][0]
    =\prior[\req]\nullstationarylaw[\state]
      \indicator\{\act=\act[][0]\}.
  \]
  Then $\nullsolution$ satisfies \textup{(SF1)--(SF3)} and has zero reward and zero resource use.

  First mix the nominal optimizer with $\nullsolution$:
  \[
    \widetilde{\selectedsolution}
    \coloneqq (1-\mixtureweight)\selectedsolution[][*]
    +\mixtureweight\nullsolution.
  \]
  This flow is feasible at $\capacityratio$.  For a binding resource,
  \[
    \stationaryconsumption[\res](\widetilde{\selectedsolution})
    =(1-\mixtureweight)\nominalratio[\res]
    \le \nominalratio[\res]-d
    \le \capacityratio[\res],
  \]
  because $\mixtureweight\nominalratio[\res]\ge d$.  For an inactive resource,
  \[
    \stationaryconsumption[\res](\widetilde{\selectedsolution})
    \le \nominalratio[\res]-\resourceslack[\res][*]
    \le \nominalratio[\res]-d
    \le \capacityratio[\res].
  \]
  Thus every optimizer $\selectedsolution(\capacityratio)$ satisfies
  \begin{equation}
    \stationaryreward(\selectedsolution(\capacityratio))
    \ge \stationaryreward(\widetilde{\selectedsolution})
    =(1-\mixtureweight)\stationaryreward[][*].
    \label{eq:deg-gimfr-perturbed-value-lower}
  \end{equation}

  Now apply the same mixture to an arbitrary perturbed optimizer:
  \[
    \widehat{\selectedsolution}
    \coloneqq (1-\mixtureweight)\selectedsolution(\capacityratio)
    +\mixtureweight\nullsolution.
  \]
  For every $\res\in\bindingset$, feasibility at $\capacityratio$ and the definition of $d$ give
  \[
    \stationaryconsumption[\res](\widehat{\selectedsolution})
    \le (1-\mixtureweight)(\nominalratio[\res]+d)
    =\nominalratio[\res]
      +(d-\mixtureweight\nominalratio[\res])-\mixtureweight d
    \le \nominalratio[\res],
  \]
  where the last inequality again uses
  $\mixtureweight\nominalratio[\res]\ge d$.
  Hence $\widehat{\selectedsolution}\in\bindingpolytope$.  By \eqref{eq:deg-gimfr-perturbed-value-lower},
  \[
    \stationaryreward(\widehat{\selectedsolution})
    =(1-\mixtureweight)\stationaryreward(\selectedsolution(\capacityratio))
    \ge(1-\mixtureweight)^2\stationaryreward[][*].
  \]
  Since $\widehat{\selectedsolution}\in\bindingpolytope$ and $\selectedsolution[][*]$ maximizes reward there,
  \[
    0
    \le \stationaryreward[][*]-\stationaryreward(\widehat{\selectedsolution})
    \le \bigl(1-(1-\mixtureweight)^2\bigr)
      \stationaryreward[][*]
    \le 2\mixtureweight\stationaryreward[][*].
  \]
  Therefore, if $\bindingpolytope\ne\{\selectedsolution[][*]\}$, the sharpness bound \eqref{eq:deg-gimfr-weak-sharp} yields
  \begin{equation}
    \|\widehat{\selectedsolution}-\selectedsolution[][*]\|_1
    \le \frac{2\stationaryreward[][*]}{\sharpnessconstant}
    \mixtureweight.
    \label{eq:deg-gimfr-contracted-close}
  \end{equation}
  If $\bindingpolytope=\{\selectedsolution[][*]\}$, the left-hand side is instead zero.

  Finally, every stationary flow $(\occupation,\stateoccupation)$ has
  $\|\occupation\|_1=\|\stateoccupation\|_1=1$.  Thus, for the concatenated $\ell_1$-norm used in the lemma, every stationary flow has norm two, and
  \[
    \|\selectedsolution(\capacityratio)-\widehat{\selectedsolution}\|_1
    =\mixtureweight
      \|\selectedsolution(\capacityratio)-\nullsolution\|_1
    \le4\mixtureweight.
  \]
  Combining this estimate with \eqref{eq:deg-gimfr-contracted-close} gives, in the nonsingleton case,
  \[
    \|\selectedsolution(\capacityratio)-\selectedsolution[][*]\|_1
    \le
    \left(4+\frac{2\stationaryreward[][*]}{\sharpnessconstant}\right)
    \frac{d}{\underline\rho}.
  \]
  In the singleton case the same bound holds with the second term in parentheses omitted.  Thus there is a constant $K<\infty$, depending only on the nominal problem, such that
  \begin{equation}
    \|\selectedsolution(\capacityratio)-\selectedsolution[][*]\|_1
    \le Kd.
    \label{eq:deg-gimfr-explicit-lipschitz}
  \end{equation}
  This proves \eqref{eq:deg-gimfr-optimizer-lipschitz}, uniformly over every choice of the perturbed optimizer.

  \noindent\textbf{Step 3: state masses and policies.}
  The set $\recstates$ is nonempty because the nominal state masses sum to one.  Define
  \[
    q_{\min}^*\coloneqq\min_{\state\in\recstates}
      \stateoccupation[\state][*]>0,
    \qquad
    \positivitymargin\coloneqq q_{\min}^*/2.
  \]
  Shrink $\localradius[0]$ once more so that
  $K\localradius[0]\le q_{\min}^*/2$.  Then, for every
  $\state\in\recstates$, \eqref{eq:deg-gimfr-explicit-lipschitz} gives
  \[
    \stateoccupation[\state](\capacityratio)
    \ge \stateoccupation[\state][*]-Kd
    \ge q_{\min}^*/2
    =\positivitymargin,
  \]
  proving \eqref{eq:deg-gimfr-core-positive}.  On the other hand,
  $\stateoccupation[\state][*]=0$ for $\state\in\outsidestates$, so
  \[
    \sum_{\state\in\outsidestates}
      \stateoccupation[\state](\capacityratio)
    \le \|\stateoccupation(\capacityratio)-\stateoccupation[][*]\|_1
    \le Kd,
  \]
  which proves \eqref{eq:deg-gimfr-outside-mass}.

  It remains to compare the induced policies.  Set
  $p_{\min}\coloneqq\min_{\req\in\requests}\prior[\req]>0$.
  Fix $\state\in\recstates$ and $\req\in\requests$, and abbreviate, for the following calculation,
  \[
    z_{\act}=\occupation[\state\req\act](\capacityratio),
    \quad z_{\act}^*=\occupation[\state\req\act][*],
    \quad q=\stateoccupation[\state](\capacityratio),
    \quad q^*=\stateoccupation[\state][*],
    \quad p=\prior[\req].
  \]
  Since $q\ge\positivitymargin$ and
  $\sum_{\act}z_{\act}^*=pq^*$ by \textup{(SF1)},
  \begin{align*}
    \|\policy[\capacityratio](\cdot\mid\state,\req)
      -\policy[\nominalratio](\cdot\mid\state,\req)\|_1
    &=\sum_{\act}\left|\frac{z_{\act}}{pq}
      -\frac{z_{\act}^*}{pq^*}\right| \\
    &\le \frac{\sum_{\act}|z_{\act}-z_{\act}^*|}{pq}
      +\frac{\sum_{\act}z_{\act}^*}{p}
        \left|\frac1q-\frac1{q^*}\right| \\
    &=\frac{\sum_{\act}|z_{\act}-z_{\act}^*|}{pq}
      +\frac{|q-q^*|}{q} \\
    &\le \left(\frac1{p_{\min}\positivitymargin}
      +\frac1{\positivitymargin}\right)Kd.
  \end{align*}
  This is \eqref{eq:deg-gimfr-policy-close}.  Every constant above depends only on the fixed nominal problem, so all bounds are uniform over every optimizer.
\end{proof}

\subsubsection{Rare opportunity to escape the optimal target state class}

\begin{lemma}
  \label{lem:deg-gimfr-outside-domination}
  Fix $\capacityratio$ with $\capacitydistance(\capacityratio) \le \localradius[0]$, and run the frozen selected policy $\policy[\capacityratio]$. If the initial state distribution $\initiallaw$ is supported on $\recstates$, then for every $\poststep \ge 0$,
  \begin{equation*}
    \prob(\stateproc[][\poststep] \in \outsidestates) = \bigO(\capacitydistance(\capacityratio)).
  \end{equation*}
  Consequently, for every $\blocklength \ge 1$,
  \begin{equation}
    \expect\left[\sum_{\poststep = 0}^{\blocklength - 1}\indicator\{\stateproc[][\poststep] \in \outsidestates\}\right]
    = \bigO(\blocklength\capacitydistance(\capacityratio)).
    \label{eq:deg-gimfr-outside-visits}
  \end{equation}
\end{lemma}

\begin{proof}
  Let $\inducedkernel[\capacityratio]$ be the kernel induced by the frozen policy and let $\stateoccupation(\capacityratio)$ be its stationary state-probability vector supplied by the selected fluid optimizer.  Since $\initiallaw$ is supported on $\recstates$ and $\stateoccupation[\state](\capacityratio) \ge \positivitymargin$ there,
  \[
    \initiallaw \le \positivitymargin[][-1]\stateoccupation(\capacityratio).
  \]
  This is a componentwise inequality of nonnegative row vectors.  Multiplication by the nonnegative matrix $\inducedkernel[\capacityratio][\poststep]$ preserves it, and stationarity gives $\stateoccupation(\capacityratio)\inducedkernel[\capacityratio][\poststep] = \stateoccupation(\capacityratio)$.  Therefore
  \[
    \prob(\stateproc[][\poststep] \in \outsidestates)
    = \initiallaw \inducedkernel[\capacityratio][\poststep](\outsidestates)
    \le \positivitymargin[][-1]\stateoccupation(\capacityratio)(\outsidestates)
    = \bigO(\capacitydistance(\capacityratio))
  \]
  by \eqref{eq:deg-gimfr-outside-mass}.  Summing over $\poststep$ proves the second claim.
\end{proof}

\subsubsection{Bounding the accumulated gap}

\begin{lemma}
  \label{lem:deg-gimfr-frozen-additive}
  Let $\filtration[][\poststep]$ contain the full history before the request in post-entry round $\poststep$.  Let $\oneperiodquantity[][\poststep]$ be a uniformly bounded, $\filtration[][\poststep + 1]$-measurable one-period quantity under a frozen policy $\policy[\capacityratio]$.  Assume there is a state function $\conditionalmean[\capacityratio]$ such that, for every $\poststep$,
  \[
    \expect[\oneperiodquantity[][\poststep] \mid \filtration[][\poststep]]
    = \conditionalmean[\capacityratio](\stateproc[][\poststep]).
  \]
  In particular, current-round reward and resource consumption satisfy this full-history conditional-mean condition.
  Suppose that there is a bounded nominal function $\conditionalmean[][*]$ on $\recstates$ such that, uniformly over $\state \in \recstates$,
  \begin{equation}
    |\conditionalmean[\capacityratio](\state) - \conditionalmean[][*](\state)| = \bigO(\capacitydistance(\capacityratio)).
    \label{eq:deg-gimfr-conditional-mean-close}
  \end{equation}
  Set $\nominalmean \coloneqq \sum_{\state \in \recstates}\stateoccupation[\state][*]\conditionalmean[][*](\state)$.  If the epoch starts in $\recstates$, then
  \begin{equation}
    \expect\left[
    \max_{1 \le \poststep \le \blocklength}
    \left|\sum_{\prefixstep = 0}^{\poststep - 1}(\oneperiodquantity[][\prefixstep] - \nominalmean)\right|
    \right]
    = \bigO\left(\sqrt {\blocklength} + \blocklength\capacitydistance(\capacityratio) + 1\right).
    \label{eq:deg-gimfr-frozen-maximal}
  \end{equation}
\end{lemma}

\begin{proof}
  If $\capacitydistance(\capacityratio) > \localradius[0]$, uniform boundedness of $\oneperiodquantity[][\poststep]$ and $\nominalmean$ bounds the left-hand side of \eqref{eq:deg-gimfr-frozen-maximal} by $\bigO(\blocklength) = \bigO(\blocklength\capacitydistance(\capacityratio))$. It therefore suffices to consider $\capacitydistance(\capacityratio) \le \localradius[0]$, where the local estimates below apply.

  \noindent\textbf{Step 1: use a corrector for the nominal chain.}
  Assumption~\ref{ass:support-irreducibility} makes the optimal induced kernel $\inducedkernel[][*]$ irreducible on $\recstates$.  Therefore the centered Poisson equation
  \begin{equation*}
    (\identitymatrix - \inducedkernel[][*])\poissoncorrector = \conditionalmean[][*] - \nominalmean\onesvector,
    \qquad (\stateoccupation[][*])^\top \poissoncorrector = 0
  \end{equation*}
  has a bounded solution.  Extend $\poissoncorrector$ by zero on $\outsidestates$.

  \noindent\textbf{Step 2: isolate a bounded-increment martingale.}
  Define
  \begin{equation*}
    \martingale[][\prefixstep + 1] \coloneqq \oneperiodquantity[][\prefixstep] - \conditionalmean[\capacityratio](\stateproc[][\prefixstep])
    +\poissoncorrector(\stateproc[][\prefixstep + 1]) - (\inducedkernel[\capacityratio]\poissoncorrector)(\stateproc[][\prefixstep]).
  \end{equation*}
  Then $\expect[\martingale[][\prefixstep + 1] \mid \filtration[][\prefixstep]] = 0$, and the increments are uniformly bounded. Doob's square-integrable maximal inequality gives
  \begin{equation}
    \expect\max_{\poststep \le \blocklength}\left|\sum_{\prefixstep = 0}^{\poststep - 1}\martingale[][\prefixstep + 1]\right|
    = \bigO(\sqrt {\blocklength}).
    \label{eq:deg-gimfr-additive-martingale-bound}
  \end{equation}

  \noindent\textbf{Step 3: bound the remainder.}
  Set
  \[
    \approxerror[\capacityratio](\state) \coloneqq
    \conditionalmean[\capacityratio](\state) - \nominalmean - \poissoncorrector(\state) + (\inducedkernel[\capacityratio]\poissoncorrector)(\state).
  \]
  For $\state \in \recstates$, the nominal Poisson equation gives
  \[
    \approxerror[\capacityratio](\state)
    = [\conditionalmean[\capacityratio](\state) - \conditionalmean[][*](\state)]
    +[(\inducedkernel[\capacityratio] - \inducedkernel[][*])\poissoncorrector](\state).
  \]
  The first term is $\bigO(\capacitydistance(\capacityratio))$ by \eqref{eq:deg-gimfr-conditional-mean-close}; the second is of the same order by \eqref{eq:deg-gimfr-policy-close}.  Hence
  \[
    |\approxerror[\capacityratio](\state)| = \bigO(\capacitydistance(\capacityratio)),
    \qquad \state \in \recstates.
  \]
  On $\outsidestates$, boundedness of $\oneperiodquantity[][\poststep]$ and $\poissoncorrector$ gives $|\approxerror[\capacityratio](\state)| = \bigO(1)$.  The exact pathwise decomposition is
  \begin{align*}
    \sum_{\prefixstep = 0}^{\poststep - 1}(\oneperiodquantity[][\prefixstep] - \nominalmean)
    = {} & \sum_{\prefixstep = 0}^{\poststep - 1}\martingale[][\prefixstep + 1] + \poissoncorrector(\stateproc[][0]) - \poissoncorrector(\stateproc[][\poststep])\\
    & +\sum_{\prefixstep = 0}^{\poststep - 1}\approxerror[\capacityratio](\stateproc[][\prefixstep]).
  \end{align*}
  Consequently,
  \[
    \begin{aligned}
      \max_{\poststep \le \blocklength}\left|\sum_{\prefixstep = 0}^{\poststep - 1}(\oneperiodquantity[][\prefixstep] - \nominalmean)\right|
      \le {} &
      \max_{\poststep \le \blocklength}\left|\sum_{\prefixstep = 0}^{\poststep - 1}\martingale[][\prefixstep + 1]\right|
      +2\|\poissoncorrector\|_\infty\\
      & +\bigO\bigl(\blocklength\capacitydistance(\capacityratio)\bigr)
      +\bigO\left(\sum_{\prefixstep = 0}^{\blocklength - 1}\indicator\{\stateproc[][\prefixstep] \in \outsidestates\}\right).
    \end{aligned}
  \]
  Taking expectations and applying \eqref{eq:deg-gimfr-outside-visits} and \eqref{eq:deg-gimfr-additive-martingale-bound} proves \eqref{eq:deg-gimfr-frozen-maximal}.
\end{proof}

\subsubsection{Stopping control}

Before proceeding to the next lemma, we make the following definitions.
The successful-entrance
event is
\begin{equation*}
  \event[\horizon]
  \coloneqq
  \{\entranceduration<\horizon\}
  \cap\bigcap_{\res\in[\numres]}
    \{\entranceconsumption[\res]
      <\horizon\nominalratio[\res]\}.
\end{equation*}
Thus, on $\event[\horizon]$, at least one round and a strictly positive amount
of every resource remain at entrance.  On this event set
\begin{equation*}
  \period[0]\coloneqq\entrancelast,
  \qquad
  \epochhorizon[0]\coloneqq\entrancerest,
  \qquad
  \capacityratio[][0]
    \coloneqq\frac{\entranceinventory}{\epochhorizon[0]}.
\end{equation*}
The value of $\capacityratio[][0]$ off $\event[\horizon]$ is immaterial; set
it equal to $\nominalratio$ there.  On $\event[\horizon]$, restart the clipped
power schedule \eqref{eq:horizon-schedule} from $\epochhorizon[0]$.
With minor abuse of notation, let $\numepochs$ be the first
schedule index for which $\epochhorizon[\numepochs]\le\cutrest$.
Set
\begin{equation*}
  \sqrtepoch[\max]
    \coloneqq\min\{\epoch\ge1:\schedulepower[][\epoch]\le1/2\},
  \qquad
  \sqrtepoch
    \coloneqq\min\{\numepochs,\sqrtepoch[\max]\}.
\end{equation*}
The restarted schedule and $\sqrtepoch$ are determined by the entrance
history, whereas $\sqrtepoch[\max]$ is a deterministic constant.

Define the regular-entry event by
\begin{equation*}
  \goodevent[0]
  \coloneqq\event[\horizon]
    \cap\{\capacitydistance(\capacityratio[][0])<\localradius[0]\}.
\end{equation*}

On $\goodevent[0]$, let $\analyticalstop$ be the first of the following times:
\begin{enumerate}[label=\textup{(\alph*)},leftmargin=2.2em]
  \item a resolving time $\period[\epoch]$, $1\le\epoch<\sqrtepoch$, at
  which either the pre-action locality check
  $\capacitydistance(\capacityratio[][\epoch])<\localradius[0]$ or the
  pre-action target-state check
  $\stateproc[][\period[\epoch]]\in\recstates$ fails;
  \item the first pre-action round $\period$ in an epoch
  $\epoch<\sqrtepoch$ at which the resource near-depletion check
  $\min_{\res\in[\numres]}\inventory[\res][\period]
  \le\maxconsumption$ fires;
  \item the scheduled cutoff $\period[\sqrtepoch]$, which is fixed at entry.
\end{enumerate}

For each fixed $0\le\epoch\le\sqrtepoch[\max]$, define
$\reachedevent[\epoch]$ as follows.  On
$\goodevent[0]\cap\{\epoch\le\sqrtepoch\}$, it is the event
$\{\analyticalstop\ge\period[\epoch]\}$; off that event, set
$\indicator\{\reachedevent[\epoch]\}=0$.  Thus
$\reachedevent[\epoch]$ means that the analysis reaches
$\period[\epoch]$ before applying the checks there.

For every epoch $\epoch<\sqrtepoch$ that begins with a local parameter and a
state in $\recstates$, couple the actual path to its \emph{full frozen
continuation}: this continuation uses the policy assigned to epoch $\epoch$
for all $\blocklength[\epoch]$ scheduled rounds, even if the resource near-depletion check would
stop the actual path.  Let $\consumptionproc[\epoch,\res](\prefixstep)$ be its
cumulative consumption of resource $\res$ in the first $\prefixstep$ rounds,
write $\consumptionproc[\epoch](\prefixstep)$ for the vector of these totals,
and let $\stateproc[][\period[\epoch]+\prefixstep,\mathrm{hyp}]$ be its state
after those rounds.  The actual and hypothetical paths agree until resource near-depletion
activation.  At the hypothetical endpoint define
\begin{equation}
  \inventorydeviation[][\epoch+1,\mathrm{hyp}]
  \coloneqq
    \inventory[][\period[\epoch]]
    -\consumptionproc[\epoch](\blocklength[\epoch])
    -\epochhorizon[\epoch+1]\nominalratio,
  \qquad
  \onesideddeviation[\epoch+1][\mathrm{hyp}]
  \coloneqq
    \inventorydeviationnorm(
      \inventorydeviation[][\epoch+1,\mathrm{hyp}]).
  \label{eq:deg-gimfr-hypothetical-endpoint}
\end{equation}
All full-epoch conditional estimates below concern this continuation and do
not condition on the actual path surviving the epoch.

At entrance, define on all of $\event[\horizon]$
\begin{equation*}
  \inventorydeviation[][0]
    \coloneqq\entranceinventory-\epochhorizon[0]\nominalratio
    =\entranceinventorydeviation,
  \qquad
  \onesideddeviation[0]
    \coloneqq\inventorydeviationnorm(\inventorydeviation[][0]),
\end{equation*}
and set both quantities to zero off $\event[\horizon]$.  For
$1\le\epoch\le\sqrtepoch[\max]$, define
\begin{equation*}
  \inventorydeviation[][\epoch]
  \coloneqq
  \begin{cases}
    \inventory[][\period[\epoch]]
      -\epochhorizon[\epoch]\nominalratio,
      &\text{on }\reachedevent[\epoch],\\
    0,&\text{otherwise},
  \end{cases}
  \qquad
  \onesideddeviation[\epoch]
    \coloneqq\inventorydeviationnorm(\inventorydeviation[][\epoch]).
\end{equation*}
Whenever the corresponding ratio and deviation are defined---on
$\event[\horizon]$ for $\epoch=0$ and on $\reachedevent[\epoch]$ for
$\epoch\ge1$---we have
\begin{equation}
  \epochhorizon[\epoch]
    \capacitydistance(\capacityratio[][\epoch])
  \le\onesideddeviation[\epoch].
  \label{eq:deg-gimfr-distance-deviation-relation}
\end{equation}
Indeed, after multiplying by $\epochhorizon[\epoch]$, the distance on the
left is the maximum of the relevant coordinate deviations, whereas the norm
on the right is their sum.  Any other quantity multiplied by a reached-event
indicator is assigned zero off that event, and an empty epoch sum is zero.

\begin{lemma}
  \label{lem:deg-gimfr-epoch-control}
  For every fixed $0\le\epoch\le\sqrtepoch[\max]$,
  \begin{equation}
    \expect[\onesideddeviation[\epoch]\indicator\{\reachedevent[\epoch]\}]
    = \bigO(\sqrt {\horizon}).
    \label{eq:deg-gimfr-Dk-sqrt}
  \end{equation}
  Moreover,
  \begin{align}
    \expect\left[\sum_{\epoch = 0}^{\sqrtepoch - 1}\onesideddeviation[\epoch]\indicator\{\reachedevent[\epoch], \ \capacitydistance(\capacityratio[][\epoch]) < \localradius[0], \
    \stateproc[][\period[\epoch]] \in \recstates\}\right]
    & = \bigO(\sqrt {\horizon}),
    \label{eq:deg-gimfr-sum-Dk}\\
    \expect\left[
    \left(\subhorizon[][\analyticalstop] + \sum_{\res \in \bindingset}\inventory[\res][\analyticalstop]\right)
    \indicator\{\goodevent[0]\}
    \right]
    & = \bigO(\sqrt {\horizon}).
    \label{eq:deg-gimfr-terminal-envelope}
  \end{align}
  In particular, after multiplication by $\indicator\{\goodevent[0]\}$, both the remaining horizon and the total binding-resource inventory at the analytical stop have expectations of order $\bigO(\sqrt {\horizon})$.
\end{lemma}

\begin{proof}
  If $\sqrtepoch = 0$, then $\analyticalstop = \entrancelast$ and $\epochhorizon[0] \le \cutrest$. On $\goodevent[0]$, each binding inventory is at most $(\nominalratio[\res] + \localradius[0])\epochhorizon[0]$, so the terminal envelope on these histories is $\bigO(1)$ and the epoch sum is empty. Moreover, $\expect[\onesideddeviation[0]\indicator\{\goodevent[0]\}] = \bigO(1)$ by Lemma~\ref{lem:entrance-bookkeeping}. In what follows, we therefore consider epochs only on histories with $\sqrtepoch \ge 1$; all sums and recursions include the corresponding reach indicators.

  \noindent\textbf{Step 1: the schedule reaches the square-root horizon after a fixed number of epochs.}
  At every schedule index $\epoch \le \sqrtepoch$, the ceiling recursion gives
  \[
    \epochhorizon[\epoch] = \bigO(\horizon[][\schedulepower[][\epoch]] + 1).
  \]
  If $\sqrtepoch=\numepochs$, then
  $\epochhorizon[\sqrtepoch]
  =\epochhorizon[\numepochs]\le\cutrest$.
  Otherwise $\schedulepower[][\sqrtepoch]\le1/2$. In either case,
  \begin{equation}
    \epochhorizon[\sqrtepoch] = \bigO(\sqrt {\horizon}).
    \label{eq:deg-gimfr-hJ}
  \end{equation}
  Moreover, $\sqrtepoch$ is bounded by a constant depending only on $\schedulepower$, even when the fixed cutoff is reached first.

  \noindent\textbf{Step 2: maximal resource deviations within one epoch.}
  Epoch $0$ is protected and uses the nominal policy $\policy[\nominalratio]$, not the selector at $\capacityratio[][0]$. The nominal Poisson decomposition gives, conditional on the entrance history,
  \begin{equation}
    \expect\left[
    \max_{\poststep \le \blocklength[0]}|\consumptionproc[0, \res](\poststep) - \poststep\stationaryconsumption[\res](\nominalratio)|
    \mid \filtration[][\period[0]]\right]
    = \bigO(\sqrt{\epochhorizon[0]} + 1).
    \label{eq:deg-gimfr-protected-resource-maximal}
  \end{equation}
  For $\epoch \ge 1$, suppose the analysis has reached $\period[\epoch]$ and passed the checks at $\period[\epoch]$. Lemma~\ref{lem:deg-gimfr-frozen-additive}, applied to the one-period consumption of resource $\res$, gives
  \begin{equation}
    \expect\left[
    \max_{\poststep \le \blocklength[\epoch]}|\consumptionproc[\epoch, \res](\poststep) - \poststep\stationaryconsumption[\res](\nominalratio)|
    \mid \filtration[][\period[\epoch]]\right]
    = \bigO(\sqrt{\epochhorizon[\epoch]} + \onesideddeviation[\epoch] + 1).
    \label{eq:deg-gimfr-resource-maximal}
  \end{equation}
  Indeed, the conditional mean on $\recstates$ differs from the nominal mean by $\bigO(\capacitydistance(\capacityratio[][\epoch]))$ by \eqref{eq:deg-gimfr-policy-close}, and
  $\blocklength[\epoch]\capacitydistance(\capacityratio[][\epoch]) \le \epochhorizon[\epoch]\capacitydistance(\capacityratio[][\epoch]) \le \onesideddeviation[\epoch]$.

  \noindent\textbf{Step 3: propagate the one-sided inventory deviation.}
  For $\res \in \bindingset$, $\stationaryconsumption[\res](\nominalratio) = \nominalratio[\res]$, and at a hypothetical full-epoch endpoint,
  \[
    \inventorydeviation[\res][\epoch + 1, \mathrm{hyp}] = \inventorydeviation[\res][\epoch] - [\consumptionproc[\epoch, \res](\blocklength[\epoch]) - \blocklength[\epoch]\nominalratio[\res]].
  \]
  For $\res \notin \bindingset$, $\stationaryconsumption[\res](\nominalratio) = \nominalratio[\res] - \resourceslack[\res][*]$, and
  \[
    \inventorydeviation[\res][\epoch + 1, \mathrm{hyp}] = \inventorydeviation[\res][\epoch] + \blocklength[\epoch]\resourceslack[\res][*] - [\consumptionproc[\epoch, \res](\blocklength[\epoch]) - \blocklength[\epoch]\stationaryconsumption[\res](\nominalratio)].
  \]
  The favorable drift $\blocklength[\epoch]\resourceslack[\res][*]$ can only reduce the negative part, so
  \[
    (-\inventorydeviation[\res][\epoch + 1, \mathrm{hyp}])_+
    \le (-\inventorydeviation[\res][\epoch])_+
    +|\consumptionproc[\epoch, \res](\blocklength[\epoch]) - \blocklength[\epoch]\stationaryconsumption[\res](\nominalratio)|.
  \]
  Taking the one-sided norm specified in
  \eqref{eq:deg-gimfr-hypothetical-endpoint}, summing over resources, and using
  \eqref{eq:deg-gimfr-protected-resource-maximal} in epoch $0$ gives
  \begin{equation*}
    \expect[\onesideddeviation[1][\mathrm{hyp}] \mid \filtration[][\period[0]]]
    = \bigO(\onesideddeviation[0] + \sqrt{\epochhorizon[0]} + 1).
  \end{equation*}
  For $1 \le \epoch < \sqrtepoch$, \eqref{eq:deg-gimfr-resource-maximal} similarly gives
  \begin{equation*}
    \expect[\onesideddeviation[\epoch + 1][\mathrm{hyp}] \mid \filtration[][\period[\epoch]]]
    = \bigO(\onesideddeviation[\epoch] + \sqrt{\epochhorizon[\epoch]} + 1).
  \end{equation*}

  On $\goodevent[0]$,
  \[
    \inventorydeviation[][0] = \entranceinventorydeviation = \entranceduration\nominalratio - \entranceconsumption,
  \]
  so Lemma~\ref{lem:entrance-bookkeeping} gives
  \[
    \expect[\onesideddeviation[0]\indicator\{\goodevent[0]\}] = \bigO(1).
  \]
  If the analysis reaches $\period[\epoch + 1]$, then epoch $\epoch$ did not stop for resource near-depletion, so the actual and hypothetical endpoint inventories coincide. Consequently,
  \[
    \begin{aligned}
      \expect[\onesideddeviation[\epoch + 1]\indicator\{\reachedevent[\epoch + 1]\}]
      & \le
      \expect[\onesideddeviation[\epoch + 1][\mathrm{hyp}]\indicator\{\reachedevent[\epoch], \epoch < \sqrtepoch, \capacitydistance(\capacityratio[][\epoch]) < \localradius[0],
      \stateproc[][\period[\epoch]] \in \recstates\}]\\
      & \le
      \bigO\left(\expect[\onesideddeviation[\epoch]\indicator\{\reachedevent[\epoch]\}]
      +\sqrt{\horizon} + 1\right).
    \end{aligned}
  \]
  Here the start-event indicator is measurable at $\period[\epoch]$, and $\sqrt{\epochhorizon[\epoch]} \le \sqrt{\horizon}$ on that event. Iterating over the deterministic index range $0 \le \epoch < \sqrtepoch[\max]$, with the reach indicators equal to zero after the realized cutoff, proves
  \eqref{eq:deg-gimfr-Dk-sqrt}. Summing over the same fixed set of indices proves
  \eqref{eq:deg-gimfr-sum-Dk}.

  \noindent\textbf{Step 4: charge the analytical stopping events.}
  At every reached resolving point,
  \[
    \epochhorizon[\epoch]\capacitydistance(\capacityratio[][\epoch]) \le \onesideddeviation[\epoch].
  \]
  Therefore
  \begin{equation*}
    \epochhorizon[\epoch]\indicator\{\reachedevent[\epoch], \capacitydistance(\capacityratio[][\epoch]) \ge \localradius[0]\}
    \le \localradius[0][-1]\onesideddeviation[\epoch]\indicator\{\reachedevent[\epoch]\},
  \end{equation*}
  whose expectation is $\bigO(\sqrt {\horizon})$ by \eqref{eq:deg-gimfr-Dk-sqrt}.

  Closure of the optimal target state class under the optimal induced kernel gives
  \[
    \prob(\stateproc[][\period[1]] \notin \recstates, \reachedevent[1]) = 0.
  \]
  For $2 \le \epoch < \sqrtepoch$, Lemma~\ref{lem:deg-gimfr-outside-domination}, applied to the preceding frozen epoch, gives
  \[
    \prob(\stateproc[][\period[\epoch], \mathrm{hyp}] \notin \recstates \mid \filtration[][\period[\epoch - 1]])
    = \bigO(\capacitydistance(\capacityratio[][\epoch - 1]))
  \]
  on every reached regular start of epoch $\epoch - 1$ with $\epoch - 1 < \sqrtepoch$. On $\reachedevent[\epoch]$, the actual and hypothetical endpoints coincide; this inclusion, rather than conditioning on survival, transfers the estimate to the reached endpoint. Since $\epochhorizon[\epoch] \le \epochhorizon[\epoch - 1]$,
  \[
    \begin{aligned}
      \expect\!\left[
      \epochhorizon[\epoch]\indicator\{\reachedevent[\epoch], \stateproc[][\period[\epoch]] \notin \recstates\}
      \right]
      & \le
      \bigO\left(\expect\!\left[
      \epochhorizon[\epoch - 1]\capacitydistance(\capacityratio[][\epoch - 1])
      \indicator\{\reachedevent[\epoch - 1], \epoch - 1 < \sqrtepoch, \capacitydistance(\capacityratio[][\epoch - 1]) < \localradius[0],
      \stateproc[][\period[\epoch - 1]] \in \recstates\}
      \right]\right)\\
      & = \bigO\left(\expect[\onesideddeviation[\epoch - 1]\indicator\{\reachedevent[\epoch - 1]\}]\right)
      = \bigO(\sqrt {\horizon}).
    \end{aligned}
  \]

  Now suppose resource near-depletion activates after $\poststep$ completed actions of a regular
  epoch $\epoch<\sqrtepoch$.  There are then
  $\epochhorizon[\epoch]-\poststep$ rounds remaining. For
  $\res\in\bindingset$,
  \[
    \inventory[\res][\period[\epoch] + \poststep]
    = (\epochhorizon[\epoch]-\poststep)\nominalratio[\res]
      + \inventorydeviation[\res][\epoch]
      - [\consumptionproc[\epoch, \res](\poststep)
        - \poststep\nominalratio[\res]],
  \]
  and for $\res \notin \bindingset$,
  \[
    \inventory[\res][\period[\epoch] + \poststep]
    = (\epochhorizon[\epoch]-\poststep)\nominalratio[\res]
      + \inventorydeviation[\res][\epoch]
      + \poststep\resourceslack[\res][*]
      - [\consumptionproc[\epoch, \res](\poststep)
        - \poststep\stationaryconsumption[\res](\nominalratio)].
  \]
  If some inventory is at most $\maxconsumption$, then in either case
  \begin{equation}
    (\epochhorizon[\epoch]-\poststep)\min_{\res}\nominalratio[\res]
    \le \maxconsumption + \onesideddeviation[\epoch] +
    \sum_{\res}\max_{\prefixstep \le \blocklength[\epoch]}|\consumptionproc[\epoch, \res](\prefixstep) - \prefixstep\stationaryconsumption[\res](\nominalratio)|.
    \label{eq:deg-gimfr-resource-near-depletion-pathwise}
  \end{equation}
  Multiply \eqref{eq:deg-gimfr-resource-near-depletion-pathwise} by the event that resource near-depletion first activates in this epoch. This event is contained in its reached regular-start event; because the right-hand side is nonnegative, it can be bounded using that start event before taking conditional expectations. For $\epoch = 0$, the maximal terms are controlled by
  \eqref{eq:deg-gimfr-protected-resource-maximal}; for $\epoch \ge 1$ they are controlled by \eqref{eq:deg-gimfr-resource-maximal}. Together with \eqref{eq:deg-gimfr-Dk-sqrt}, this bounds the expected remaining horizon at a resource near-depletion stop by $\bigO(\sqrt {\horizon})$. Summing over at most $\sqrtepoch[\max]$ epochs does not change the order. If the analysis instead reaches $\period[\sqrtepoch]$, \eqref{eq:deg-gimfr-hJ} gives $\subhorizon[][\analyticalstop] = \epochhorizon[\sqrtepoch] = \bigO(\sqrt {\horizon})$. Combining the ratio, outside-state, resource near-depletion, and scheduled-cutoff bounds yields
  \[
    \expect[\subhorizon[][\analyticalstop]\indicator\{\goodevent[0]\}] = \bigO(\sqrt {\horizon}).
  \]

  It remains to control binding inventories at $\analyticalstop$. At a resolving stop $\period[\epoch]$,
  \[
    \inventory[\res][\period[\epoch]] = \epochhorizon[\epoch]\nominalratio[\res] + \inventorydeviation[\res][\epoch] \le \epochhorizon[\epoch]\nominalratio[\res] + \onesideddeviation[\epoch],
    \qquad \res \in \bindingset.
  \]
  The preceding event-specific bounds control $\epochhorizon[\epoch]$, while
  \eqref{eq:deg-gimfr-Dk-sqrt} controls $\onesideddeviation[\epoch]$. At a resource near-depletion stop in epoch $\epoch$,
  \[
    \inventory[\res][\analyticalstop] \le \subhorizon[][\analyticalstop]\nominalratio[\res] + |\inventorydeviation[\res][\epoch]|
    +\max_{\prefixstep \le \blocklength[\epoch]}|\consumptionproc[\epoch, \res](\prefixstep) - \prefixstep\nominalratio[\res]|,
    \qquad \res \in \bindingset.
  \]
  The resource near-depletion estimate, \eqref{eq:deg-gimfr-Dk-sqrt}, and the appropriate maximal resource estimate control the three terms. Summing over the fixed set of binding resources proves \eqref{eq:deg-gimfr-terminal-envelope}.
\end{proof}

\subsubsection{Finishing the proof}

\begin{proof}
  We proceed in three steps.

  \noindent\textbf{Step 1: entrance paths and the stopped comparison.}
  Use the entrance event, initial post-entry ratio and deviation, regular-entry
  event, and analytical stop defined in the setup preceding
  Lemma~\ref{lem:deg-gimfr-epoch-control}.
  Lemma~\ref{lem:entrance-bookkeeping} gives
  \begin{equation}
    \expect[\onesideddeviation[0]\indicator\{\event[\horizon]\}]
    =\bigO(1),
    \qquad
    \horizon\prob(\event[\horizon][\mathrm c])=\bigO(1).
    \label{eq:deg-gimfr-entry-D0}
  \end{equation}
  It also shows that, on $\event[\horizon]$, the entrance discrepancy
  \[
    Z_{\mathrm{ent}}
    \coloneqq
    \horizon \stationaryfluid(\nominalratio)-\entrancereward
    -\potential(\epochhorizon[0],\entranceinventory,\entrancestate)
  \]
  is bounded above by a nonnegative random variable with $\bigO(1)$
  expectation.  In particular,
  \begin{equation}
    \expect[Z_{\mathrm{ent}}^+
      \indicator\{\event[\horizon]\}]=\bigO(1).
    \label{eq:deg-gimfr-entrance-potential-cost}
  \end{equation}

  On $\event[\horizon]\setminus\goodevent[0]$, the
  distance--deviation relation \eqref{eq:deg-gimfr-distance-deviation-relation}
  gives
  \[
    \epochhorizon[0]
    \le \frac{\onesideddeviation[0]}{\localradius[0]}.
  \]
  Since $\horizon=\entranceduration+\epochhorizon[0]$ after successful
  entrance, Assumption~\ref{ass:constant-entrance} and
  \eqref{eq:deg-gimfr-entry-D0} imply
  \begin{align}
    \expect[\horizon\indicator\{\goodevent[0][\mathrm c]\}]
    &\le \horizon\prob(\event[\horizon][\mathrm c])
      +\expect[\entranceduration]
      +\frac1{\localradius[0]}
        \expect[\onesideddeviation[0]\indicator\{\event[\horizon]\}]  \notag\\
    &=\bigO(1).
    \label{eq:deg-gimfr-irregular-entry-weight}
  \end{align}

  On $\goodevent[0]$, use the analytical stop controlled by
  Lemma~\ref{lem:deg-gimfr-epoch-control}.  Its resource near-depletion check is applied before
  an action when some inventory coordinate is at most $\maxconsumption$, so
  resource depletion cannot precede $\analyticalstop$.  Define a comparison
  policy as follows: after regular
  entry it agrees with IR up to $\analyticalstop$ and then always uses the
  null action; after a successful but nonregular entry it uses the null action
  immediately.  On entrance-failure paths it simply follows IR until the
  process stops.  Actual IR dominates this comparison pathwise, because the
  policies agree before analytical truncation and every later reward of IR is
  nonnegative.  For notational completeness, set
  $\analyticalstop=\entrancelast$ on
  $\event[\horizon]\setminus\goodevent[0]$.  On entrance-failure paths, any
  post-entry random variable below is assigned an arbitrary value, say zero;
  such variables only occur multiplied by $\indicator\{\goodevent[0]\}$.

  \noindent\textbf{Step 2: bound the accumulated dual slack.}
  Define the nonnegative nominal dual action and state slacks by
  \begin{align*}
    \dualloss[\state\req\act]
    &\coloneqq \statedual[\state\req][*]
      -\left[\reward(\state,\req,\act)
      -(\resourceprice[][*])^\top\consumption(\state,\req,\act)
      +\sum_{\state[][\prime]}
        \transition(\state[][\prime]\mid\state,\act)
        \statebias[\state[][\prime]][*]\right],\\
    \dualloss[\state]
    &\coloneqq \statebias[\state][*]+\averagegain[][*]
      -\sum_{\req}\prior[\req]\statedual[\state\req][*].
  \end{align*}
  Under a selected policy at parameter $\capacityratio$, its expected
  one-period slack at state $\state$ is
  \[
    \dualloss[\capacityratio](\state)
    \coloneqq \dualloss[\state]
      +\sum_{\req}\prior[\req]\sum_{\act}
        \policy[\capacityratio](\act\mid\state,\req)
        \dualloss[\state\req\act].
  \]
  If $\state\in\recstates$, complementary slackness makes the state slack
  zero and makes every action used by the nominal policy tight.  Hence
  \begin{align*}
    0\le\dualloss[\capacityratio](\state)
    &=\sum_{\req}\prior[\req]\sum_{\act}
      \bigl(\policy[\capacityratio](\act\mid\state,\req)
      -\policy[\nominalratio](\act\mid\state,\req)\bigr)
      \dualloss[\state\req\act] \\
    &= \bigO(\capacitydistance(\capacityratio)),
    \qquad \state\in\recstates,
  \end{align*}
  where the final estimate follows from \eqref{eq:deg-gimfr-policy-close} and the
  finiteness of the action-slack collection.  For $\state \in \outsidestates$, all slacks
  are bounded, so $\dualloss[\capacityratio](\state) = \bigO(1)$.
  Combining the two cases gives the pointwise bound
  \begin{equation}
    \dualloss[\capacityratio](\state)
    = \bigO(\capacitydistance(\capacityratio) + \indicator\{\state\in\outsidestates\}).
    \label{eq:deg-gimfr-dual-loss-all-states}
  \end{equation}

  Consider a reached epoch $1\le\epoch<\sqrtepoch$ whose resolving-point
  checks pass, and use the hypothetical full frozen continuation defined in the
  setup before Lemma~\ref{lem:deg-gimfr-epoch-control}.  Summing
  \eqref{eq:deg-gimfr-dual-loss-all-states} and then applying
  Lemma~\ref{lem:deg-gimfr-outside-domination} give
  \begin{align}
    &\expect\left[
      \sum_{\period=\period[\epoch]}^{\period[\epoch+1]-1}
      \dualloss[\capacityratio[][\epoch]]
        (\stateproc[][\period,\mathrm{hyp}])
      \ \middle|\ \filtration[][\period[\epoch]]\right]  \notag\\
    &\qquad\le
      \bigO\left(\blocklength[\epoch]
        \capacitydistance(\capacityratio[][\epoch])
      +\expect\left[
        \sum_{\period=\period[\epoch]}^{\period[\epoch+1]-1}
        \indicator\{\stateproc[][\period,\mathrm{hyp}]
          \in\outsidestates\}
        \ \middle|\ \filtration[][\period[\epoch]]\right]\right)  \notag\\
    &\qquad\le
      \bigO \left(\blocklength[\epoch]
        \capacitydistance(\capacityratio[][\epoch])\right)
      \le\bigO(\onesideddeviation[\epoch]).
    \label{eq:deg-gimfr-epoch-dual-loss}
  \end{align}
  The last inequality uses $\blocklength[\epoch]\le\epochhorizon[\epoch]$
  and \eqref{eq:deg-gimfr-distance-deviation-relation}.  If resource near-depletion stops the
  actual epoch early, its realized loss is a prefix of this nonnegative
  full-epoch loss, so the same estimate applies.

  The protected epoch $0$ has zero dual slack.  Indeed, the entrance state is
  in $\recstates$, stationarity of $\stateoccupation[][*]$ implies that
  $\recstates$ is closed under the nominal policy, and complementary slackness
  makes both the state slack and every used action slack zero there.  Set
  $\dualloss[][\period]=0$ in that epoch; in a later epoch $\epoch$, set
  $\dualloss[][\period]
   =\dualloss[\capacityratio[][\epoch]](\stateproc[][\period])$.  Thus
  \eqref{eq:deg-gimfr-epoch-dual-loss} and
  \eqref{eq:deg-gimfr-sum-Dk} yield
  \begin{equation}
    \expect\left[
      \sum_{\period=\entrancelast}^{\analyticalstop-1}
      \dualloss[][\period]
      \indicator\{\goodevent[0]\}\right]
    = \bigO(\sqrt{\horizon}).
    \label{eq:deg-gimfr-total-dual-loss}
  \end{equation}

  \noindent\textbf{Step 3: finish by the stopped-potential reduction.}
  Direct substitution of the slack definitions into the nominal dual
  constraints gives, on $\goodevent[0]$ and at every round before
  $\analyticalstop$,
  \begin{equation*}
    \expect\left[
      \rewardproc[][\period]
      +\potential(\subhorizon[][\period]-1,
        \inventory[][\period+1],\stateproc[][\period+1])
      \ \middle|\ \filtration[][\period]\right]
    =\potential(\subhorizon[][\period],\inventory[][\period],
      \stateproc[][\period])-\dualloss[][\period].
  \end{equation*}
  This includes the protected epoch because its slack is zero.  Let
  \[
    R_{\mathrm{post}}
    \coloneqq\sum_{\period=\entrancelast}^{\analyticalstop-1}
      \rewardproc[][\period]
  \]
  be the stopped comparison's post-entry reward.  Since
  $\analyticalstop\le\horizon+1$ is bounded, conditional finite-horizon
  telescoping from $\entrancelast$, followed by the zero-reward null actions
  after truncation and deletion of the negative part of the terminal
  potential, gives, on $\goodevent[0]$,
  \begin{equation}
    \begin{aligned}
      &\potential(\epochhorizon[0],\entranceinventory,\entrancestate)
      -\expect\left[R_{\mathrm{post}}
        \mid\filtration[][\entrancelast]\right]\\
      &\qquad\le
      \expect\left[
        \potential(\subhorizon[][\analyticalstop],
          \inventory[][\analyticalstop],
          \stateproc[][\analyticalstop])^+
        +\sum_{\period=\entrancelast}^{\analyticalstop-1}
          \dualloss[][\period]
        \ \middle|\ \filtration[][\entrancelast]\right].
    \end{aligned}
    \label{eq:deg-gimfr-stopped-potential}
  \end{equation}

  For completeness, complementary slackness gives
  $\resourceprice[\res][*]=0$ for $\res\notin\bindingset$.  Hence the positive
  terminal potential satisfies
  \[
    \potential(\subhorizon[][\analyticalstop],
      \inventory[][\analyticalstop],\stateproc[][\analyticalstop])^+
    \le \sum_{\res\in\bindingset}
      \resourceprice[\res][*]\inventory[\res][\analyticalstop]
      +|\averagegain[][*]|\subhorizon[][\analyticalstop]
      +\|\statebias[][*]\|_\infty.
  \]
  The terminal-envelope estimate
  \eqref{eq:deg-gimfr-terminal-envelope} therefore gives
  \begin{equation}
    \expect\left[
      \potential(\subhorizon[][\analyticalstop],
        \inventory[][\analyticalstop],\stateproc[][\analyticalstop])^+
      \indicator\{\goodevent[0]\}\right]
    \le\bigO(\sqrt{\horizon}).
    \label{eq:deg-gimfr-residual-potential}
  \end{equation}

  The event $\goodevent[0]$ is measurable at entrance.  Multiplying
  \eqref{eq:deg-gimfr-stopped-potential} by its indicator, taking
  expectations, and using \eqref{eq:deg-gimfr-residual-potential} and
  \eqref{eq:deg-gimfr-total-dual-loss} give
  \begin{equation}
    \expect\left[
      \bigl(\potential(\epochhorizon[0],\entranceinventory,\entrancestate)
      -R_{\mathrm{post}}\bigr)
      \indicator\{\goodevent[0]\}\right]
    \le\bigO(\sqrt{\horizon}).
    \label{eq:deg-gimfr-post-entry-potential-gap}
  \end{equation}

  Finally, the stopped comparison earns
  $\entrancereward+R_{\mathrm{post}}$ on $\goodevent[0]$, earns a
  nonnegative reward on its complement, and is dominated pathwise by IR\@.
  Therefore
  \begin{align*}
    &\horizon\stationaryfluid(\nominalratio)
      -\valuefn[\horizon][\infrequentalg]
        (\horizon\nominalratio,\state[][1]) \notag\\
    &\quad\le
      \expect\left[
        (\horizon \stationaryfluid(\nominalratio)-\entrancereward-R_{\mathrm{post}})
        \indicator\{\goodevent[0]\}\right]
      +\stationaryfluid(\nominalratio)\expect[\horizon\indicator\{\goodevent[0][\mathrm c]\}] \notag\\
    &\quad=
      \expect[Z_{\mathrm{ent}}\indicator\{\goodevent[0]\}]
      +\expect\left[
        \bigl(\potential(\epochhorizon[0],\entranceinventory,
          \entrancestate)-R_{\mathrm{post}}\bigr)
        \indicator\{\goodevent[0]\}\right]
      +\stationaryfluid(\nominalratio)\expect[\horizon\indicator\{\goodevent[0][\mathrm c]\}] \notag\\
    &\quad\le\bigO(\sqrt{\horizon}),
  \end{align*}
  where the last line follows from
  \eqref{eq:deg-gimfr-entrance-potential-cost},
  \eqref{eq:deg-gimfr-post-entry-potential-gap}, and
  \eqref{eq:deg-gimfr-irregular-entry-weight}; here we also used
  $\goodevent[0]\subseteq\event[\horizon]$ to bound the first expectation
  by \eqref{eq:deg-gimfr-entrance-potential-cost}.

  Hence the stationary-benchmark gap is $\bigO(\sqrt{\horizon})$.
  Lemma~\ref{lem:boundary-comparison} and
  $\valuefn[\horizon][*]\le
    \transientfluid[\horizon](\horizon\nominalratio,\state[][1])$
  now give \eqref{eq:deg-gimfr-fluid-sqrt} and
  \eqref{eq:deg-gimfr-optimal-sqrt}.
\end{proof}

%% file: appendices/07_2_boundary_cases_proofs.tex
\subsection{Proof of Proposition~\ref{prop:unique-degenerate-gmfr-linear}}
\label{app:proof:prop:unique-degenerate-gmfr-linear}

\begin{proof}
  There is one request type, one resource, and three states
  \[
    \states = \{\hardstateA, \hardstateB, \hardstateC\}.
  \]
  The initial state is $\hardstateA$, the initial inventory is $\inventory[][1] = \horizon$, and $\nominalratio = 1$.  The common action set contains the null action $\act[][0]$ and non-null actions $\act, \excursionact, \recoveryact$.  The relevant state-action pairs are
  \begin{equation*}
    \begin{array}{c|c|c|c|c}
      \text{state} & \text{action} & \text{reward} & \text{cost} & \text{next state}\\ \hline
      \hardstateA & \act & 1 & 3 / 2 & \hardstateB\\
      \hardstateB & \act & 1 & 1 / 2 & \hardstateA\\
      \hardstateB & \excursionact & 1 & 1 / 2 & \hardstateC\\
      \hardstateC & \recoveryact & 3 / 4 & 1 / 2 & \hardstateA.
    \end{array}
  \end{equation*}
  The null action has zero reward and cost and transitions according to
  \begin{equation*}
    \hardstateA \xrightarrow {\act[][0]}\hardstateB,
    \qquad
    \hardstateB \xrightarrow {\act[][0]}\hardstateA,
    \qquad
    \hardstateC \xrightarrow {\act[][0]}\hardstateC.
  \end{equation*}
  Every unspecified non-null state-action pair has reward zero, positive cost,
  and the same transition as the null action in that state, so it is weakly
  dominated by the null action and can be omitted when computing the fluid
  value.  This omission does not create a uniqueness issue.  Indeed, whenever
  the reduced optimizer identified below is capacity-binding, a full optimizer
  cannot place positive mass on an unspecified action: replacing that mass by
  the same mass on the null action preserves reward and state flow and strictly
  lowers resource use.  The resulting point would have to be the unique reduced
  optimizer, which binds the resource constraint, making the original point
  infeasible.  For ratios
  above one, uniqueness instead follows directly from the equality case in the
  reward bound.

  \noindent\textbf{Reduction of the stationary LP\@.}
  Let $\excursionmass$ be the stationary state-request-action mass of action $\excursionact$ in state $\hardstateB$, let $\nullmass[\hardstateC]$ be the null-action mass in state $\hardstateC$, and let $\nullmass[\hardstateA], \nullmass[\hardstateB]$ be the null-action masses in states $\hardstateA, \hardstateB$.  State flow gives
  \begin{equation*}
    \occupation[\hardstateC, \recoveryact] = \excursionmass,
    \qquad
    \stateoccupation[\hardstateA] = \stateoccupation[\hardstateB] = \stateoccupation \coloneqq \frac{1 - \excursionmass - \nullmass[\hardstateC]}{2},
    \qquad
    \stateoccupation[\hardstateC] = \excursionmass + \nullmass[\hardstateC].
  \end{equation*}
  The remaining relevant stationary state-request-action masses are
  \begin{equation*}
    \occupation[\hardstateA, \act] = \stateoccupation - \nullmass[\hardstateA],
    \qquad
    \occupation[\hardstateB, \act] = \stateoccupation - \excursionmass - \nullmass[\hardstateB],
    \qquad
    \occupation[\hardstateB, \excursionact] = \excursionmass.
  \end{equation*}
  Accordingly, the stationary reward and resource consumption are
  \begin{align}
    \stationaryreward
    & = (\stateoccupation - \nullmass[\hardstateA]) + (\stateoccupation - \excursionmass - \nullmass[\hardstateB]) + \excursionmass + \frac34\excursionmass
    = 1 - \frac14\excursionmass - \nullmass[\hardstateC] - \nullmass[\hardstateA] - \nullmass[\hardstateB],
    \label{eq:unique-degenerate-reward}\\
    \sum_{\state, \req, \act}\consumption(\state, \req, \act)\occupation[\state\req\act] & = \frac32(\stateoccupation - \nullmass[\hardstateA]) + \frac12(\stateoccupation - \excursionmass - \nullmass[\hardstateB]) + \frac12\excursionmass + \frac12\excursionmass
    = 1 - \frac12\excursionmass - \nullmass[\hardstateC] - \frac32\nullmass[\hardstateA] - \frac12\nullmass[\hardstateB].
    \notag
  \end{align}
  For a capacity ratio $\capacityratio = 1 - \capacityshortfall$, resource feasibility is equivalent to
  \begin{equation}
    \frac12\excursionmass + \nullmass[\hardstateC] + \frac32\nullmass[\hardstateA] + \frac12\nullmass[\hardstateB] \ge \capacityshortfall.
    \label{eq:unique-degenerate-feasibility}
  \end{equation}

  \noindent\textbf{The nominal optimizer is unique.}
  At $\capacityshortfall = 0$, \eqref{eq:unique-degenerate-reward} gives
  \begin{equation*}
    \stationaryreward
    = 1 - \frac14\excursionmass - \nullmass[\hardstateC] - \nullmass[\hardstateA] - \nullmass[\hardstateB] \le 1.
  \end{equation*}
  Equality requires $\excursionmass = \nullmass[\hardstateC] = \nullmass[\hardstateA] = \nullmass[\hardstateB] = 0$; it also rules out every unspecified action, whose reward is zero.  Hence the nominal optimizer is unique in the full action space and is given by
  \begin{equation}
    \stateoccupation[\hardstateA][*] = \stateoccupation[\hardstateB][*] = \frac12,
    \qquad \stateoccupation[\hardstateC][*] = 0,
    \qquad \occupation[\hardstateA, \act][*] = \occupation[\hardstateB, \act][*] = \frac12,
    \qquad \stationaryfluid(1) = 1.
    \label{eq:unique-degenerate-nominal-solution}
  \end{equation}

  \noindent\textbf{Every small negative perturbation activates state $\hardstateC$.}
  For $\capacityshortfall > 0$, \eqref{eq:unique-degenerate-feasibility} implies
  \begin{align*}
    1 - \stationaryreward
    & = \frac14\excursionmass + \nullmass[\hardstateC] + \nullmass[\hardstateA] + \nullmass[\hardstateB]\\
    & = \frac12\left(
    \frac12\excursionmass + \nullmass[\hardstateC] + \frac32\nullmass[\hardstateA] + \frac12\nullmass[\hardstateB]\right)
    +\frac12\nullmass[\hardstateC] + \frac14\nullmass[\hardstateA] + \frac34\nullmass[\hardstateB]\\
    & \ge \frac \capacityshortfall2.
  \end{align*}
  Therefore
  \begin{equation*}
    \stationaryreward \le 1 - \frac \capacityshortfall2.
  \end{equation*}
  Equality requires $\nullmass[\hardstateC] = \nullmass[\hardstateA] = \nullmass[\hardstateB] = 0$ and $\excursionmass = 2\capacityshortfall$.  For $0 < \capacityshortfall \le 1 / 6$, this solution is feasible and hence is the unique optimizer:
  \begin{equation*}
    \nullmass[\hardstateC] = \nullmass[\hardstateA] = \nullmass[\hardstateB] = 0,
    \qquad \excursionmass = 2\capacityshortfall,
    \qquad \stateoccupation[\hardstateA] = \stateoccupation[\hardstateB] = \frac{1 - 2\capacityshortfall}{2},
    \qquad \stateoccupation[\hardstateC] = 2\capacityshortfall.
  \end{equation*}
  Indeed, the remaining mass of action $\act$ in state $\hardstateB$ is $(1 - 6\capacityshortfall) / 2 \ge 0$, including the endpoint $\capacityshortfall = 1 / 6$.  Thus, throughout $0 < \capacityshortfall \le 1 / 6$,
  \begin{equation}
    \stationaryfluid(1 - \capacityshortfall) = 1 - \frac12\capacityshortfall,
    \qquad
    \policy[1 - \capacityshortfall](\excursionact \mid \hardstateB)
    = \frac{\excursionmass}{\stateoccupation[\hardstateB]}
    = \frac{4\capacityshortfall}{1 - 2\capacityshortfall}.
    \label{eq:unique-degenerate-excursion-probability}
  \end{equation}

  For every $\capacityratio \ge 1$, the nominal solution in
  \eqref{eq:unique-degenerate-nominal-solution} remains feasible and uniquely
  optimal: the reward is at most one, with equality only when
  $\excursionmass = \nullmass[\hardstateC] = \nullmass[\hardstateA]
  = \nullmass[\hardstateB] = 0$ and no unspecified zero-reward action is
  used.  In particular,
  \begin{equation}
    \stateoccupation[\hardstateC](\capacityratio) = 0
    \qquad\text{for every }\capacityratio \ge 1.
    \label{eq:unique-degenerate-c-absent-above}
  \end{equation}

  The nominal solution is degenerate.  To see this without choosing a particular redundant-row deletion, recall Proposition~\ref{prop:local-lp-stability}: a unique nondegenerate nominal optimum would have a fixed positive-variable support throughout some right-hand-side neighborhood.  Here $\occupation[\hardstateB, \excursionact] = \stateoccupation[\hardstateC] = \occupation[\hardstateC, \recoveryact] = 0$ at $\capacityshortfall = 0$ but all three are positive for every $\capacityshortfall > 0$.  Hence the unique nominal optimum cannot be nondegenerate.

  \noindent\textbf{Both dynamic assumptions hold.}
  The optimal target state class is
  \[
    \recstates = \{\hardstateA, \hardstateB\}.
  \]
  The optimal induced kernel on this class is
  \begin{equation*}
    \inducedkernel[][*]
    = \begin{pmatrix}0 & 1\\1 & 0\end{pmatrix},
  \end{equation*}
  which is irreducible, although periodic, so
  Assumption~\ref{ass:support-irreducibility} holds.  The initial state
  $\hardstateA$ belongs to $\recstates$, so the entrance phase is empty and
  Assumption~\ref{ass:constant-entrance} also holds.

  \noindent\textbf{Round-by-round behavior of FR\@.}
  Write
  \[
    \subhorizon[][\period] = \horizon - \period + 1,
    \qquad
    \inventorydeviation[][\period] \coloneqq \inventory[][\period] - \subhorizon[][\period].
  \]
  Initially $\inventorydeviation[][1] = 0$.  At the nominal ratio, FR takes action $\act$ in state $\hardstateA$, earns one, and consumes $3 / 2$.  If $\horizon = 1$, this is the terminal action and there is no subsequent ratio to evaluate.  If $\horizon \ge 2$, the next round is active in state $\hardstateB$, with
  \[
    \inventorydeviation[][\period] = -\frac12,
    \qquad
    \capacityratio[][\period] = 1 - \frac{1}{2\subhorizon[][\period]}.
  \]
  Using \eqref{eq:unique-degenerate-excursion-probability}, whenever such a $\hardstateB$-round has $\subhorizon[][\period] \ge 3$, the conditional probability of selecting $\excursionact$ is
  \begin{equation}
    \policy[1 - 1 / (2\subhorizon[][\period])](\excursionact \mid \hardstateB)
    = \frac{2}{\subhorizon[][\period] - 1},
    \qquad \subhorizon[][\period] \ge 3.
    \label{eq:unique-degenerate-round-excursion-probability}
  \end{equation}
  For the remaining cases $\subhorizon[][\period] \in \{1, 2\}$, set $\capacityshortfall = 1 / (2\subhorizon[][\period])$. Nonnegativity of the baseline mass in state $\hardstateB$ gives
  \[
    3\excursionmass + \nullmass[\hardstateC] + 2\nullmass[\hardstateB] \le 1.
  \]
  Together with \eqref{eq:unique-degenerate-feasibility}, this implies
  \begin{align*}
    1 - \stationaryreward
    & = \frac23\left(\frac12\excursionmass + \nullmass[\hardstateC] + \frac32\nullmass[\hardstateA] + \frac12\nullmass[\hardstateB]\right)
    -\frac1{12}\excursionmass + \frac13\nullmass[\hardstateC] + \frac23\nullmass[\hardstateB]\\
    & \ge \frac23\capacityshortfall - \frac1{36}
    +\frac{13}{36}\nullmass[\hardstateC] + \frac{13}{18}\nullmass[\hardstateB].
  \end{align*}
  Equality is attained uniquely by
  \[
    \excursionmass = \frac13, \qquad
    \nullmass[\hardstateC] = \nullmass[\hardstateB] = 0, \qquad
    \nullmass[\hardstateA] = \frac23\left(\capacityshortfall - \frac16\right).
  \]
  The last mass equals $1 / 18$ when $\subhorizon[][\period] = 2$ and $2 / 9$ when $\subhorizon[][\period] = 1$, so both solutions are feasible.  In each case $\stateoccupation[\hardstateB] = \excursionmass = 1 / 3$, and FR chooses $\excursionact$ with probability one.  The endpoint $\subhorizon[][\period] = 3$ also has excursion probability one by \eqref{eq:unique-degenerate-round-excursion-probability}.

  Actions $\act$ and $\excursionact$ both consume $1 / 2$.  Therefore, after an active $\hardstateB$-round, $\inventorydeviation[][\period + 1] = 0$ in either case.  If $\act$ is selected, the state returns to $\hardstateA$ and the same two-round cycle begins again.  If $\excursionact$ is selected and a round remains, the next state is $\hardstateC$ and the next ratio is exactly one.  For even $\horizon \ge 4$, a trajectory that has not made an earlier excursion must do so at the $\hardstateB$-round with three rounds remaining; for odd $\horizon \ge 3$, it must do so with two rounds remaining.  Thus a $\hardstateB$-round with only one round remaining occurs only when $\horizon = 2$; its $\excursionact$ action exhausts the remaining inventory $1 / 2$ and is terminal.

  At ratio one the unique optimizer in \eqref{eq:unique-degenerate-nominal-solution} assigns zero stationary probability to state $\hardstateC$.  By the FR convention, when the selected stationary LP solution assigns zero probability to the current state, the algorithm chooses the null action.  Suppose there are $\subhorizon \ge 1$ rounds and $\subhorizon$ units of inventory when $\hardstateC$ is entered.  After $0 \le \prefixidx < \subhorizon$ null rounds, before the next active round, the inventory remains $\subhorizon$, the remaining horizon is $\subhorizon - \prefixidx$, and hence
  \[
    \capacityratio = \frac{\subhorizon}{\subhorizon - \prefixidx} \ge 1.
  \]
  By \eqref{eq:unique-degenerate-c-absent-above}, every newly solved optimizer still assigns zero stationary probability to $\hardstateC$.  Therefore, FR continues choosing the null action until the horizon ends.  Thus entry into $\hardstateC$ is a permanent trap without any analytical truncation.

  \noindent\textbf{The trapping event has constant probability.}
  Let $\batchsize=\lfloor\horizon/4\rfloor$ and take $\horizon\ge8$.
  Consider the first $\batchsize$ visits to state $\hardstateB$. Conditional on
  no earlier excursion, visit $\poststep$ occurs with remaining horizon
  \[
    \subhorizon[][\poststep] = \horizon - 2\poststep + 1,
  \]
  and by \eqref{eq:unique-degenerate-round-excursion-probability}, its conditional excursion probability is $2 / (\horizon - 2\poststep)$.
  The action randomizations are fresh from round to round. Hence
  \begin{align*}
    \prob(\text{no excursion in the first $\batchsize$ visits})
    & = \prod_{\poststep = 1}^{\batchsize}\left(1 - \frac{2}{\horizon - 2\poststep}\right)\notag\\
    & = \prod_{\poststep = 1}^{\batchsize}\frac{\horizon - 2\poststep - 2}{\horizon - 2\poststep}
    = \frac{\horizon-2\batchsize-2}{\horizon-2}.
  \end{align*}
  Therefore
  \[
    \prob(\text{an excursion in the first $\batchsize$ visits})
    =\frac{2\batchsize}{\horizon-2}
    \ge\frac{\horizon/2-2}{\horizon-2}
    \ge\frac13.
  \]
  The $\batchsize$-th visit occurs in round
  $2\batchsize\le\horizon/2$.  Hence, on this event, FR enters
  $\hardstateC$ with at least
  $\horizon-2\batchsize\ge\horizon/2$ rounds remaining.

  \noindent\textbf{A linearly better admissible policy.}
  Define $\frequentcomparisonpolicy$ to follow FR until the first time action $\excursionact$ is selected. If state $\hardstateC$ is reached with an active round remaining, $\frequentcomparisonpolicy$ selects $\recoveryact$ in that round and thereafter repeatedly selects the baseline action $\act$ in states $\hardstateA$ and $\hardstateB$ until the process terminates.

  Suppose the excursion occurs by round $\horizon / 2$, and let $\subhorizon \ge \horizon / 2$ be the number of rounds remaining upon entering $\hardstateC$. Since $\inventorydeviation = 0$ after the excursion, the remaining inventory is exactly $\subhorizon$. Policy $\frequentcomparisonpolicy$ first takes $\recoveryact$, earning $3 / 4$, consuming $1 / 2$, and returning to $\hardstateA$. It then uses action $\act$ in each of the remaining $\subhorizon - 1$ rounds. Starting from $\hardstateA$, the successive baseline costs are
  \[
    \frac32, \frac12, \frac32, \frac12, \ldots.
  \]
  The sum of the first $\subhorizon - 1$ baseline costs is $\subhorizon - 1$ when $\subhorizon - 1$ is even and $\subhorizon - 1 / 2$ when $\subhorizon - 1$ is odd. Including the $1 / 2$ cost of $\recoveryact$, the total suffix consumption is at most $\subhorizon$. Thus $\frequentcomparisonpolicy$ is admissible and acts in every remaining round. Its suffix reward is
  \begin{equation}
    \frac34 + (\subhorizon - 1) = \subhorizon - \frac14
    \label{eq:unique-degenerate-recovery-reward}
  \end{equation}
  over a suffix on which FR earns zero.  The same calculation applies after
  every nonterminal excursion, while after a terminal excursion the two
  policies coincide; hence $\frequentcomparisonpolicy$ never earns less than
  FR on the complementary event.  Combining the probability bound above with
  \eqref{eq:unique-degenerate-recovery-reward},
  \begin{equation}
    \valuefn[\horizon][\frequentcomparisonpolicy]
      - \valuefn[\horizon][\frequentalg]
    \ge \frac13\left(\frac \horizon2 - \frac14\right)
    = \frac \horizon6 - \frac1{12}
    = \bigOmega(\horizon).
    \label{eq:unique-degenerate-comparison-linear}
  \end{equation}
  This proves \eqref{eq:unique-degenerate-comparison-linear}. Since $\valuefn[\horizon][*] \ge \valuefn[\horizon][\frequentcomparisonpolicy]$ and $\transientfluid[\horizon] \ge \valuefn[\horizon][*]$, \eqref{eq:unique-degenerate-gmfr-linear} follows.

  Finally, before the trapping event,
  \begin{equation*}
    |\capacityratio[][\period] - \nominalratio|
    = \begin{cases}
      0, & \stateproc[][\period] = \hardstateA, \\[1mm]
      \dfrac1{2\subhorizon[][\period]}, & \stateproc[][\period] = \hardstateB,
    \end{cases}
    \le \frac1{2\subhorizon[][\period]}.
  \end{equation*}
  Thus even a much stronger estimate than $\expect|\capacityratio[][\period] - \nominalratio| = \bigO((\subhorizon[][\period])^{-1 / 2})$ would not prevent the linear loss.  The obstruction is not a large capacity-ratio deviation; it is the disappearance, at the next resolving point, of a state that was reached under a nearby optimal policy.
\end{proof}

\subsection{Proof of Proposition~\ref{prop:linear-gap-no-irreducibility}}
\label{app:proof:prop:linear-gap-no-irreducibility}

\begin{proof}
  There is one request type, one resource, two states $\states = \{1, 2\}$, and two actions: the null action $\act[][0]$ and a non-null action $\act$.  Both states are absorbing under both actions.  The non-null reward and cost are
  \[
    \begin{array}{c|cc}
      & \reward(\state, \act) & \consumption(\state, \act)\\ \hline
      \state = 1 & 3 / 2 & 1\\
      \state = 2 & 2 & 2.
    \end{array}
  \]
  The null action has zero reward and cost.  Let $\nominalratio = 3 / 2$ and $\state[][1] = 1$.

  Let $\flowmass[\state]$ be the stationary mass of the non-null action in state $\state$.  Eliminating null-action and state variables, the stationary LP is
  \begin{equation*}
    \max\left\{\frac32\flowmass[1] + 2\flowmass[2]:
    \flowmass[1] + \flowmass[2] \le 1,
    \ \flowmass[1] + 2\flowmass[2] \le \frac32,
    \ \flowmass[1], \flowmass[2] \ge 0\right\}.
  \end{equation*}
  Its relevant vertices are
  \[
    (1, 0), \qquad (1 / 2, 1 / 2), \qquad (0, 3 / 4),
  \]
  with respective values $3 / 2, 7 / 4, 3 / 2$.  Hence its unique optimum is
  \[
    \flowmass[1][*] = \flowmass[2][*] = \frac12,
    \qquad
    \stationaryfluid(\nominalratio) = \frac74.
  \]

  For completeness, add null-action masses $\nullmass[1], \nullmass[2]$, stationary state probabilities $\stateoccupation[1], \stateoccupation[2]$, and the resource slack $\resourceslack$.  A full-row-rank standard form is
  \[
    \begin{aligned}
      \flowmass[1] + \nullmass[1] - \stateoccupation[1] & = 0, \\
      \flowmass[2] + \nullmass[2] - \stateoccupation[2] & = 0, \\
      \stateoccupation[1] + \stateoccupation[2] & = 1, \\
      \flowmass[1] + 2\flowmass[2] + \resourceslack & = \frac32.
    \end{aligned}
  \]
  At the optimum the variables $\flowmass[1], \flowmass[2], \stateoccupation[1], \stateoccupation[2]$ all equal $1 / 2$, while $\nullmass[1], \nullmass[2], \resourceslack$ are zero.  The basis matrix formed by the columns of $\flowmass[1], \flowmass[2], \stateoccupation[1], \stateoccupation[2]$ is
  \[
    \begin{pmatrix}
      1 & 0 & -1 & 0\\
      0 & 1 & 0 & -1\\
      0 & 0 & 1 & 1\\
      1 & 2 & 0 & 0
    \end{pmatrix},
    \qquad
    \det = 1.
  \]
  Every basic variable is strictly positive.  The strict objective comparisons above also show uniqueness.  Assumptions~\ref{ass:nominal-uniqueness} and~\ref{ass:nominal-nondegeneracy} therefore hold.

  The optimal induced kernel is
  \[
    \inducedkernel[][*] = \begin{pmatrix}1 & 0\\0 & 1\end{pmatrix},
  \]
  which is reducible on $\recstates = \{1, 2\}$.  Consequently,
  Assumption~\ref{ass:support-irreducibility} fails.  In contrast, entrance into
  $\recstates$ takes zero time and consumes no resources, so
  Assumption~\ref{ass:constant-entrance} holds.  Since the actual process starts
  in state $1$ and never leaves it, no policy can earn more than $3 / 2$ in a
  round.  The capacity $3\horizon / 2$ is sufficient to take action $\act$ in all
  $\horizon$ rounds, so in fact
  \[
    \valuefn[\horizon][*](\horizon\nominalratio, 1) = \frac32\horizon.
  \]
  Since every policy has value at most $\valuefn[\horizon][*]$, combining
  this with
  $\horizon\stationaryfluid(\nominalratio)=7\horizon/4$ proves the
  stationary-benchmark inequality in \eqref{eq:linear-stationary-gap} for
  every policy.  The time-aggregated transient LP has the same linear
  obstruction.  For $\horizon \ge 2$, set
  \[
    \nonterminalocc[1, \act] = \horizon / 2 - 1, \qquad \terminalocc[1, \act] = 1, \qquad
    \nonterminalocc[2, \act] = \horizon / 2, \qquad \terminalocc[2, \act] = 0,
  \]
  and $\transientstateocc[1] = \transientstateocc[2] = \horizon / 2$,
  with every unspecified transient variable set to zero.  These variables
  satisfy the transient request, state-flow, and horizon constraints.  The
  charged resource use is
  \[
    \nonterminalocc[1, \act] + 2\nonterminalocc[2, \act] = 3\horizon / 2 - 1 \le 3\horizon / 2,
  \]
  while the objective is
  \[
    \frac32(\nonterminalocc[1, \act] + \terminalocc[1, \act]) + 2(\nonterminalocc[2, \act] + \terminalocc[2, \act]) = \frac74\horizon.
  \]
  Thus
  \begin{equation*}
    \transientfluid[\horizon](\horizon\nominalratio, 1) - \valuefn[\horizon][*](\horizon\nominalratio, 1) \ge \horizon / 4.
  \end{equation*}
  Since
  $\valuefn[\horizon][\policy](\horizon\nominalratio,1)
    \le\valuefn[\horizon][*](\horizon\nominalratio,1)$
  for every nonanticipating policy $\policy$, this proves the transient-fluid
  inequality in \eqref{eq:linear-stationary-gap} as well.
  The transient relaxation admits a source-free stationary circulation in the unreachable absorbing class because its state-flow constraints are aggregated over time.
\end{proof}

%% file: appendices/08_1_unknown_request_prior_proofs.tex

\section{Analysis of the unknown request prior}
\label{app:unknown-prior-proofs}

This appendix gives the supporting results and proofs for
Section~\ref{sec:unknown-prior}.  Throughout the positive analysis,
Assumptions~\ref{ass:support-irreducibility} and~\ref{ass:constant-entrance} hold, and
UIR-UP is supplied with the true optimal target state class $\recstates$ and the
entrance policy $\entrancepolicy$ fixed under
Assumption~\ref{ass:constant-entrance};
Assumption~\ref{ass:nominal-nondegeneracy} is imposed only where stated.
Conditional on the realized entrance history, for the analysis only, write
$\horizon = \entrancerest$ and abbreviate
$\inventory[][\entrancelast + \completedrounds]$,
$\stateproc[][\entrancelast + \completedrounds]$,
$\subhorizon[][\entrancelast + \completedrounds]$, and
$\filtration[][\entrancelast + \completedrounds]$ by
$\inventory[][\completedrounds]$, $\stateproc[][\completedrounds]$,
$\subhorizon[][\completedrounds]$, and $\filtration[][\completedrounds]$,
respectively.  Here $\completedrounds$ counts completed post-entry rounds; it
does not replace the original round index in the algorithm.  This convention is
used only in the possibly degenerate case.  Under
Assumption~\ref{ass:nominal-nondegeneracy},
Proposition~\ref{prop:nondegenerate-full-state-support} gives
$\recstates=\states$; in that case $\horizon$ retains its original meaning and
the completed-round clock starts at $0$.

\subsection{Proof of Theorem~\ref{thm:unknown-universal}}

Throughout the proof, we suppose that Assumption~\ref{ass:nominal-uniqueness} holds and do not mention it again.

We start with several preliminary properties of estimation concentration and the construction of the phase-2 re-solving schedule.

\subsubsection{Estimation concentration}
\label{app:unknown-preliminaries}

\begin{proposition}
  \label{prop:unknown-prior-concentration}
  There are constants $C_{\mathrm{est}},c_{\mathrm{est}} > 0$, depending
  only on $|\requests|$, such that, conditional on any realized entrance
  history, the estimator in \eqref{eq:unknown-empirical-prior} satisfies, for
  every integer $\completedrounds \ge 1$,
  \begin{align}
    \expect\|\empiricalprior[][\completedrounds] - \prior\|_1
    & \le C_{\mathrm{est}}\completedrounds[][-1 / 2],
    \label{eq:unknown-mean-prior-error}\\
    \prob\bigl(\|\empiricalprior[][\completedrounds] - \prior\|_1 \ge \deviationlevel\bigr)
    & \le C_{\mathrm{est}}\exp(-c_{\mathrm{est}}
      \completedrounds \deviationlevel[][2]),
    \qquad 0 < \deviationlevel \le 1.
    \label{eq:unknown-prior-tail}
  \end{align}
\end{proposition}

\begin{proof}
  Conditional on the realized entrance history, the requests counted by
  $\requestcount[\req](\completedrounds)$ remain i.i.d.\ with distribution
  $\prior$.  Write $d \coloneqq |\requests|$,
  and let $\widetilde p_n$ be their ordinary empirical distribution.  If $u$
  denotes the uniform distribution on $\requests$, then the Laplace-smoothed
  estimator satisfies
  \begin{equation*}
    \empiricalprior[][n]
    = \frac{n}{n+d}\widetilde p_n + \frac{d}{n+d}u.
  \end{equation*}
  Since $\|u-\prior\|_1 \le 2$, it follows that
  \begin{equation}
    \|\empiricalprior[][n]-\prior\|_1
    \le \|\widetilde p_n-\prior\|_1 + \frac{2d}{n+d}.
    \label{eq:unknown-smoothed-empirical-bias}
  \end{equation}

  For the mean bound, the variance of the $\req$-th coordinate of
  $\widetilde p_n$ is $\prior[\req](1-\prior[\req])/n$.  Hence, by
  Cauchy--Schwarz and Jensen's inequality,
  \[
    \expect\|\widetilde p_n-\prior\|_1
    \le \sum_{\req\in\requests}
      \sqrt{\frac{\prior[\req](1-\prior[\req])}{n}}
    \le \sqrt{\frac{d}{n}}.
  \]
  Combining this with \eqref{eq:unknown-smoothed-empirical-bias} and
  $n\ge1$ gives
  \[
    \expect\|\empiricalprior[][n]-\prior\|_1
    \le (\sqrt d+2d)n^{-1/2}.
  \]

  For the tail bound, the Bretagnolle--Huber--Carol inequality
  \cite[Proposition~A.6.6]{vdVW96} gives, for every $x>0$,
  \[
    \prob\bigl(\|\widetilde p_n-\prior\|_1>x\bigr)
    \le 2^d\exp(-nx^2/2).
  \]
  If $\deviationlevel>4d/(n+d)$, then
  \eqref{eq:unknown-smoothed-empirical-bias} and the preceding inequality,
  with $x=\deviationlevel/2$, yield
  \[
    \prob\bigl(\|\empiricalprior[][n]-\prior\|_1>\deviationlevel\bigr)
    \le 2^d\exp(-n\deviationlevel^2/8).
  \]
  If instead $\deviationlevel\le 4d/(n+d)$, then
  \[
    n\deviationlevel^2
    \le \frac{16d^2n}{(n+d)^2}
    \le 4d,
  \]
  so the trivial probability bound $1$ is at most
  $\exp(d/2)\exp(-n\deviationlevel^2/8)$.  Thus the two asserted bounds hold,
  for example, with
  \[
    c_{\mathrm{est}}=\frac18,
    \qquad
    C_{\mathrm{est}}
    =\max\{\sqrt d+2d,\,2^d,\,\exp(d/2)\},
  \]
  which depend only on $d=|\requests|$.
\end{proof}

\subsubsection{Sparse-balanced scheduling}
\label{app:proof:lem:sparse-balanced-schedule}

Under a unique but degenerate optimum, a frozen empirical policy can have first-order error $\bigO(\|\empiricalprior[][\completedrounds] - \prior\|_1)$ per round. Therefore, if an epoch starts after $\completedrounds$ observations and lasts $\blocklength$ rounds, its statistical contribution can be of order $\blocklength / \sqrt {\completedrounds}$. The purpose of the sparse-balanced schedule is to keep the sum of these terms of order $\bigO(\sqrt {\horizon})$ with only $\bigO(\log\log \horizon)$ re-solves.

\begin{lemma}
  \label{lem:sparse-balanced-schedule}
  Let $1 \le \completedrounds < \scheduleendpoint$ be integers.  For every integer $\midSchedule$ satisfying
  \[
    \min\left\{\scheduleendpoint - \completedrounds,
    \max\left\{1, \left\lceil\log_2\log
    \frac{\scheduleendpoint}{\completedrounds}\right\rceil + 2\right\}\right\}
    \le \midSchedule \le \scheduleendpoint - \completedrounds,
  \]
  there is a deterministic integer partition $\completedrounds = \period[0] < \cdots < \period[\midSchedule] = \scheduleendpoint$ with
  \begin{equation*}
    \sum_{\epoch = 0}^{\midSchedule - 1}
    \frac{\period[\epoch + 1] - \period[\epoch]}{\sqrt{\period[\epoch]}}
    \le 9\sqrt{\scheduleendpoint}.
  \end{equation*}
  In particular, the minimizing schedule in \eqref{eq:sparse-scheduling-cost}, with the clipped choice \eqref{eq:sparse-mid-count}, satisfies this bound.  Every resolving time is a deterministic function of $(\completedrounds, \scheduleendpoint, \midSchedule)$.
\end{lemma}

\begin{proof}
  If $\midSchedule = \scheduleendpoint - \completedrounds$, use all integer times.  Its cost is
  \[
    \sum_{\period = \completedrounds}^{\scheduleendpoint - 1}\frac1{\sqrt{\period}}
    \le \int_{\completedrounds - 1}^{\scheduleendpoint - 1}\frac{d\genericvalue}{\sqrt{\genericvalue}}
    \le 2\sqrt{\scheduleendpoint}.
  \]
  Otherwise the clipping at $\scheduleendpoint - \completedrounds$ is inactive, so
  $\midSchedule \ge \max\{1, \lceil\log_2\log(\scheduleendpoint / \completedrounds)\rceil + 2\}$.
  Write $\horizonratio \coloneqq \scheduleendpoint / \completedrounds$.  If $\horizonratio < \exp(1)$, every feasible integer partition has cost at most
  \[
    \frac{\scheduleendpoint - \completedrounds}{\sqrt{\completedrounds}}
    = \frac{\horizonratio - 1}{\sqrt{\horizonratio}}\sqrt{\scheduleendpoint}
    < 2\sqrt{\scheduleendpoint}.
  \]
  Thus it remains to treat $\log\horizonratio \ge 1$.

  Set
  \[
    \midSchedule[0] \coloneqq \lfloor\log_2\log\horizonratio\rfloor + 2,
    \qquad
    \schedulecoef \coloneqq
    2^{\,1 - \midSchedule[0] / (2^{\midSchedule[0]} - 1)}
    \horizonratio[][1 / (2^{\midSchedule[0] + 1} - 2)].
  \]
  Then $2 \le \midSchedule[0] \le \midSchedule$.  Define real times by
  \begin{equation}
    \period[0][\mathrm c] \coloneqq \completedrounds,
    \qquad
    \period[\epoch + 1][\mathrm c] \coloneqq \schedulecoef\,2^{\epoch - \midSchedule[0] + 1}
    \sqrt{\scheduleendpoint\,\period[\epoch][\mathrm c]},
    \quad 0 \le \epoch < \midSchedule[0].
    \label{eq:sparse-real-construction}
  \end{equation}
  Unrolling the logarithmic recurrence, using
  \[
    \sum_{\epoch = 0}^{\midSchedule[0] - 1}2^{-\epoch}
    = 2(1 - 2^{-\midSchedule[0]}),
    \qquad
    \sum_{\epoch = 0}^{\midSchedule[0] - 1}\epoch2^{-\epoch}
    = 2 - (\midSchedule[0] + 1)2^{1 - \midSchedule[0]},
  \]
  gives
  \[
    \log\frac{\period[\midSchedule[0]][\mathrm c]}{\scheduleendpoint}
    = -2^{-\midSchedule[0]}\log\horizonratio
    +2(1 - 2^{-\midSchedule[0]})\log\schedulecoef
    -\bigl[2 - (\midSchedule[0] + 1)2^{1 - \midSchedule[0]}\bigr]\log2
    = 0.
  \]
  Hence the endpoint is exactly $\period[\midSchedule[0]][\mathrm c] = \scheduleendpoint$.

  All the real times are strictly increasing.  For the first step, direct substitution gives
  \[
    \log\frac{\period[1][\mathrm c]}{\period[0][\mathrm c]}
    = \frac{2^{\midSchedule[0] - 1}}{2^{\midSchedule[0]} - 1}
    \left(\log\horizonratio -
    [2(\midSchedule[0] - 2) + 2^{2 - \midSchedule[0]}]\log2\right) > 0.
  \]
  Indeed, $\log\horizonratio \ge 2^{\midSchedule[0] - 2}$; when $\midSchedule[0] = 2$ the asserted positivity follows from $1 > \log2$, and when $\midSchedule[0] \ge 3$ it follows from
  \[
    2(\midSchedule[0] - 2) + 2^{2 - \midSchedule[0]}
    \le \frac54\,2^{\midSchedule[0] - 2},
    \qquad \frac54\log2 < 1.
  \]
  The first inequality uses $2\resolveindex \le 2^{\resolveindex}$ and
  $2^{-\resolveindex} \le 2^{\resolveindex} / 4$ for integers $\resolveindex \ge 1$.
  To propagate strict increase, compare the coefficients in the recurrence: if
  $\period[\epoch + 1][\mathrm c] > \period[\epoch][\mathrm c]$, then
  \[
    \sqrt{\period[\epoch + 1][\mathrm c] / \scheduleendpoint}
    < \schedulecoef\,2^{\epoch - \midSchedule[0] + 1}
    < \schedulecoef\,2^{\epoch - \midSchedule[0] + 2},
  \]
  so \eqref{eq:sparse-real-construction} yields
  $\period[\epoch + 2][\mathrm c] > \period[\epoch + 1][\mathrm c]$.
  Consequently every interior real time lies strictly between
  $\completedrounds$ and $\scheduleendpoint$.

  Since $\log\horizonratio < 2^{\midSchedule[0] - 1}$, the coefficient satisfies
  $\schedulecoef \le 2\exp(1 / 2)$.  Thus the real schedule costs at most
  \[
    \sum_{\epoch = 0}^{\midSchedule[0] - 1}
    \frac{\period[\epoch + 1][\mathrm c] - \period[\epoch][\mathrm c]}
    {\sqrt{\period[\epoch][\mathrm c]}}
    \le \schedulecoef\sqrt{\scheduleendpoint}
    \sum_{\epoch = 0}^{\midSchedule[0] - 1}2^{\epoch - \midSchedule[0] + 1}
    \le 4\exp(1 / 2)\sqrt{\scheduleendpoint}.
  \]
  Round every real time upward,
  $\period[\epoch][\mathrm r] \coloneqq \lceil\period[\epoch][\mathrm c]\rceil$.
  The endpoints are unchanged, the resulting sequence is nondecreasing, and
  \[
    \frac{\period[\epoch + 1][\mathrm r] - \period[\epoch][\mathrm r]}
    {\sqrt{\period[\epoch][\mathrm r]}}
    \le
    \frac{\period[\epoch + 1][\mathrm c] - \period[\epoch][\mathrm c]}
    {\sqrt{\period[\epoch][\mathrm c]}}
    +\frac1{\sqrt{\period[\epoch][\mathrm c]}}.
  \]
  The added cost is at most $\midSchedule[0] / \sqrt{\completedrounds} \le 2\sqrt{\scheduleendpoint}$:
  for $\log\horizonratio \ge 1$,
  $\midSchedule[0] \le \log_2\log\horizonratio + 2 \le \log\horizonratio + 2 \le 2\sqrt{\horizonratio}$.
  Deleting repeated rounded times removes only zero-length intervals and does not change the cost.  There are then at most $\midSchedule[0] \le \midSchedule$ intervals.

  If necessary, insert the smallest unused interior integer times until exactly $\midSchedule$ intervals remain.  This is possible because $\midSchedule \le \scheduleendpoint - \completedrounds$.  Inserting a point cannot increase cost: for $\period < \period[][\prime] < \period[][\prime\prime]$,
  \[
    \frac{\period[][\prime] - \period}{\sqrt{\period}}
    +\frac{\period[][\prime\prime] - \period[][\prime]}{\sqrt{\period[][\prime]}}
    \le \frac{\period[][\prime\prime] - \period}{\sqrt{\period}}.
  \]
  The constructed integer partition therefore costs at most
  $(4\exp(1 / 2) + 2)\sqrt{\scheduleendpoint} < 9\sqrt{\scheduleendpoint}$.
  The exact minimizing partition has no greater cost, proving the claim.
\end{proof}

\subsubsection{Local sensitivity with an estimated request prior}
\label{subsec:unknown-sensitivity}

Use the nominal notation from \eqref{eq:nominal-supports} and \eqref{eq:one-sided-capacity-distance}.  For a capacity ratio $\capacityratio$ and an estimated prior $\prior[][\prime]$, define the joint perturbation size
\begin{equation*}
  \jointdistance(\capacityratio, \prior[][\prime])
  \coloneqq \capacitydistance(\capacityratio) + \|\prior[][\prime] - \prior\|_1.
\end{equation*}
Upward changes in an inactive capacity are omitted for the same reason as in \eqref{eq:one-sided-capacity-distance}: they relax a constraint that is strictly slack at $\selectedsolution[][*]$.
For a stationary LP vector $\selectedsolution = (\occupation, \stateoccupation)$, write $\reward[][\top]\selectedsolution \coloneqq \sum_{\state, \req, \act}\reward(\state, \req, \act)\occupation[\state\req\act]$ for its stationary reward.

\paragraph{Joint Lipschitz sensitivity.}
\label{app:proof:lem:unknown-joint-sensitivity}

For any empirical LP solution
$\selectedsolution(\capacityratio, \prior[][\prime])$ whose core projection is
well defined, let
$\inducedkernel[\capacityratio, \prior[][\prime]][\mathrm{proj}]$ denote the
transition kernel generated by that projected policy under the true request
prior $\prior$.  Thus, for $\state,\state[][\prime]\in\recstates$,
\[
  \inducedkernel[\capacityratio, \prior[][\prime]][\mathrm{proj}]
    (\state,\state[][\prime])
  \coloneqq
  \sum_{\req}\prior[\req]\sum_{\act}
  \coreprojection(\selectedsolution(\capacityratio,\prior[][\prime]))
    (\act\mid\state,\req)
  \transition(\state[][\prime]\mid\state,\act).
\]

\begin{lemma}
  \label{lem:unknown-joint-sensitivity}
  There is a radius $\localradius > 0$ such that, whenever $\jointdistance(\capacityratio, \prior[][\prime]) < \localradius$, every optimal solution $\selectedsolution(\capacityratio, \prior[][\prime])$ of the full empirical LP $\stationaryfluid[][\prior[][\prime]](\capacityratio)$ satisfies
  \begin{equation}
    \|\selectedsolution(\capacityratio, \prior[][\prime]) - \selectedsolution[][*]\|_1
    = \bigO(\jointdistance(\capacityratio, \prior[][\prime])).
    \label{eq:unknown-joint-sensitivity}
  \end{equation}
  In particular, there is $\positivitymargin>0$ such that $\stateoccupation[\state](\capacityratio, \prior[][\prime]) \ge \positivitymargin$ for every $\state \in \recstates$, and the core projection \eqref{eq:unknown-core-projection} is well defined. Moreover,
  \begin{align}
    \|\coreprojection(\selectedsolution(\capacityratio, \prior[][\prime]))(\cdot \mid \state, \req) - \policy[\nominalratio](\cdot \mid \state, \req)\|_1
    & = \bigO(\jointdistance(\capacityratio, \prior[][\prime])),
    \qquad \state \in \recstates,\ \req \in \requests,
    \label{eq:unknown-projected-policy-close}\\
    \|\inducedkernel[\capacityratio, \prior[][\prime]][\mathrm{proj}](\state, \cdot) - \inducedkernel[][*](\state, \cdot)\|_1
    & = \bigO(\jointdistance(\capacityratio, \prior[][\prime])),
    \qquad \state \in \recstates.
    \label{eq:unknown-projected-kernel-close}
  \end{align}
\end{lemma}

\begin{proof}
  The proof mirrors the proof of Lemma~\ref{lem:deg-gimfr-polyhedral-sensitivity}.  Fix $(\capacityratio, \prior[][\prime])$ in the stated neighborhood and abbreviate
  \[
    \perturbation \coloneqq \jointdistance(\capacityratio, \prior[][\prime]).
  \]

  \noindent\textbf{Step 1: one fixed comparison polytope.}
  Let $\bindingpolytope$ be the stationary flow polytope under the true prior $\prior$ in which the resource constraints are retained only for $\res \in \bindingset$; inactive resource constraints are dropped.  The set is compact because $\sum_{\state} \stateoccupation[\state] = 1$ and $\sum_{\state, \req, \act}\occupation[\state\req\act] = 1$.  We claim that $\selectedsolution[][*]$ is the unique maximizer of the reward over $\bindingpolytope$.  Indeed, if $\comparatorsolution \in \bindingpolytope$ had reward at least $\stationaryfluid[][\prior](\nominalratio)$, then for sufficiently small $\mixtureweight > 0$ the convex combination
  \[
    (1 - \mixtureweight)\selectedsolution[][*] + \mixtureweight \comparatorsolution
  \]
  would satisfy every inactive resource constraint at $\nominalratio$, because $\selectedsolution[][*]$ has the strict slack $\resourceslack[\res][*] > 0$ there.  The binding constraints are already satisfied by both points.  A strict reward improvement would contradict optimality at $\nominalratio$, and equal reward with $\comparatorsolution\ne \selectedsolution[][*]$ would contradict uniqueness.  Hence $\selectedsolution[][*]$ is the unique maximizer on $\bindingpolytope$.

  A linear objective over a fixed compact polytope with a unique maximizer is weakly sharp. When $\bindingpolytope\ne\{\selectedsolution[][*]\}$, set
  \[
    \sharpnessconstant
    \coloneqq \min_{\substack{\vertex\text{ a vertex of }\bindingpolytope\\\vertex\ne\selectedsolution[][*]}}
    \frac{\stationaryfluid[][\prior](\nominalratio) - \reward[][\top]\vertex}{\|\vertex - \selectedsolution[][*]\|_1} > 0.
  \]
  Every point of $\bindingpolytope$ is a convex combination of its vertices, so linearity of the objective and the triangle inequality give
  \begin{equation}
    \stationaryfluid[][\prior](\nominalratio) - \reward[][\top]\comparatorsolution
    \ge \sharpnessconstant\|\comparatorsolution - \selectedsolution[][*]\|_1,
    \qquad \comparatorsolution \in \bindingpolytope.
    \label{eq:unknown-weak-sharp}
  \end{equation}

  \noindent\textbf{Step 2: a perturbed feasible point within $\bigO(\jointdistance)$ of $\selectedsolution[][*]$.}
  Keep the nominal policy $\policy[\nominalratio]$ on $\recstates$.  Under a nearby prior $\prior[][\prime]$, its induced kernel
  \[
    \inducedkernel[\prior[][\prime]][*](\state, \state[][\prime])
    \coloneqq \sum_{\req}\prior[\req][\prime]\sum_{\act}\policy[\nominalratio](\act \mid \state, \req)\transition(\state[][\prime] \mid \state, \act)
  \]
  remains irreducible on $\recstates$, because the positive graph of $\inducedkernel[][*]$ persists when $\prior[][\prime]$ is close to $\prior$.  Its stationary distribution $\stateoccupation[][*](\prior[][\prime])$ is Lipschitz in $\prior[][\prime]$ on the finite state space $\recstates$; this follows, for example, from the same fundamental-matrix representation used in Lemma~\ref{lem:poisson-bounds}.  Define
  \[
    \occupation[\state\req\act][0](\prior[][\prime]) \coloneqq \prior[\req][\prime]\stateoccupation[\state][*](\prior[][\prime])\policy[\nominalratio](\act \mid \state, \req)
  \]
  on $\recstates$ and zero outside.  Then $\priortransport(\prior[][\prime]) = (\occupation[][0](\prior[][\prime]), \stateoccupation[][*](\prior[][\prime]))$ is a stationary flow solution for prior $\prior[][\prime]$ and
  \begin{equation*}
    \|\priortransport(\prior[][\prime]) - \selectedsolution[][*]\|_1 = \bigO(\|\prior[][\prime] - \prior\|_1).
  \end{equation*}
  Its resource rates differ from the nominal rates by $\bigO(\|\prior[][\prime] - \prior\|_1)$.

  Let $\nullsolution(\prior[][\prime])$ be any stationary flow solution that always uses the null action; it has zero reward and zero resource use. Because $\nominalratio[\res] > 0$ for every binding resource, one can, for sufficiently small $\perturbation$, choose $\mixtureweight = \bigO(\perturbation)$ so that
  \[
    \perturbedcomparator \coloneqq (1 - \mixtureweight)\priortransport(\prior[][\prime]) + \mixtureweight \nullsolution(\prior[][\prime])
  \]
  is feasible for $\stationaryfluid[][\prior[][\prime]](\capacityratio)$.  Binding-resource deficits of size at most $\capacitydistance(\capacityratio)$ are absorbed by the $\bigO(\perturbation)$ dilution, while for $\res \notin \bindingset$ the nominal strict slack $\resourceslack[\res][*]$ absorbs every permitted small downward perturbation.  Therefore
  \begin{equation}
    \|\perturbedcomparator - \selectedsolution[][*]\|_1 = \bigO(\perturbation),
    \qquad
    \reward[][\top]\perturbedcomparator \ge \stationaryfluid[][\prior](\nominalratio) - \bigO(\perturbation).
    \label{eq:unknown-near-feasible-comparator}
  \end{equation}
  If $\selectedsolution = \selectedsolution(\capacityratio, \prior[][\prime])$ is any perturbed optimum, optimality gives $\reward[][\top]\selectedsolution \ge \reward[][\top]\perturbedcomparator$.

  \noindent\textbf{Step 3: transfer the perturbed optimizer to the fixed nominal polytope.}
  The vector $\selectedsolution$ satisfies the nominal state-flow and normalization equations exactly.  Its request-balance residual under the true prior is
  \[
    \sum_{\act} \occupation[\state\req\act] - \prior[\req]\stateoccupation[\state] = (\prior[\req][\prime] - \prior[\req])\stateoccupation[\state],
  \]
  so its total equality violation is $\bigO(\|\prior[][\prime] - \prior\|_1)$.  For $\res \in \bindingset$ its resource violation relative to $\nominalratio[\res]$ is at most $\capacitydistance(\capacityratio)$, while inactive resources are not constrained in $\bindingpolytope$.  The ordinary Hoffman bound for the \emph{fixed} polytope $\bindingpolytope$ therefore gives a point $\comparatorsolution \in \bindingpolytope$ with
  \begin{equation}
    \|\selectedsolution - \comparatorsolution\|_1 = \bigO(\perturbation).
    \label{eq:unknown-hoffman-transfer}
  \end{equation}
  If $\bindingpolytope = \{\selectedsolution[][*]\}$, then $\comparatorsolution = \selectedsolution[][*]$ and \eqref{eq:unknown-hoffman-transfer} already proves \eqref{eq:unknown-joint-sensitivity}. Otherwise, using \eqref{eq:unknown-near-feasible-comparator} and \eqref{eq:unknown-hoffman-transfer},
  \[
    \stationaryfluid[][\prior](\nominalratio) - \reward[][\top]\comparatorsolution = \bigO(\perturbation).
  \]
  The weak-sharpness inequality \eqref{eq:unknown-weak-sharp} implies $\|\comparatorsolution - \selectedsolution[][*]\|_1 = \bigO(\perturbation)$, and the triangle inequality proves \eqref{eq:unknown-joint-sensitivity}.

  Finally, $\stateoccupation[\state][*] > 0$ on $\recstates$ and $\prior[\req] > 0$.  Every nominal positive action stays in $\recstates$, so
  \[
    \sum_{\act \in \actions[\recstates](\state)}\occupation[\state\req\act][*] = \prior[\req]\stateoccupation[\state][*] > 0.
  \]
  The sensitivity bound therefore makes every denominator in \eqref{eq:unknown-core-projection} uniformly positive locally.  Normalizing the safe state-request-action masses gives \eqref{eq:unknown-projected-policy-close}, and the kernel estimate \eqref{eq:unknown-projected-kernel-close} follows by linearity and the fixed true prior $\prior$.
\end{proof}

\paragraph{Basis stability under uniqueness and nondegeneracy.}
\label{app:proof:lem:unknown-joint-basis}

For any capacity ratio $\capacityratio$, prior $\prior[][\prime]$, and empirical
stationary LP solution
$\selectedsolution(\capacityratio,\prior[][\prime])
=(\occupation(\capacityratio,\prior[][\prime]),
\stateoccupation(\capacityratio,\prior[][\prime]))$,
define its stationary resource-consumption rates by
\[
  \stationaryconsumption[\res][\prior[][\prime]](\capacityratio)
  \coloneqq
  \sum_{\state,\req,\act}\consumption[\res](\state,\req,\act)
  \occupation[\state\req\act](\capacityratio,\prior[][\prime]).
\]

\begin{lemma}
  \label{lem:unknown-joint-basis}
  Under Assumption~\ref{ass:nominal-nondegeneracy}, there is a $\localradius[\mathrm{nd}] > 0$ such that, whenever $\jointdistance(\capacityratio, \prior[][\prime]) < \localradius[\mathrm{nd}]$, the full empirical $\stationaryfluid[][\prior[][\prime]](\capacityratio)$ has the nominal basis as its unique optimal basis. Its positive $\occupation$- and $\stateoccupation$-supports are exactly $\actsupport$ and $\recstates$, respectively. The core projection does nothing, every used action is nominal-dual tight, and for some $\positivitymargin[\mathrm{nd}] > 0$,
  \begin{align}
    \stationaryconsumption[\res][\prior[][\prime]](\capacityratio) & = \capacityratio[\res], & & \res \in \bindingset,
    \label{eq:unknown-empirical-binding-rate}\\
    \stationaryconsumption[\res][\prior[][\prime]](\capacityratio) & \le \capacityratio[\res] - \positivitymargin[\mathrm{nd}], & & \res \notin \bindingset.
    \label{eq:unknown-empirical-inactive-rate}
  \end{align}
  The same statement holds for the restricted empirical LP used in the protective phase.
\end{lemma}

\begin{proof}
  First verify that the row deletion fixed in Section~\ref{sec:regularity} remains valid as the prior changes. Write all variables on the left-hand sides of \textup{(SF1)--(SF3)}, and add the resource slacks to \textup{(SF4)}. In a linear dependence among these full equality rows, let $\requestrowcoefficient[\state\req]$, $\stateflowrowcoefficient[\state]$, and $w$ be the coefficients of the request, state-flow, and normalization rows. Every resource-row coefficient is zero because its slack variable occurs in that row alone. The $\occupation$- and $\stateoccupation$-columns therefore give, respectively,
  \[
    \requestrowcoefficient[\state\req] = \sum_{\state[][\prime]}\transition(\state[][\prime] \mid \state, \act)\stateflowrowcoefficient[\state[][\prime]]
    \quad\text{for every }(\state, \req, \act),
    \qquad
    -\sum_{\req}\prior[\req][\prime]\requestrowcoefficient[\state\req] + \stateflowrowcoefficient[\state] + w = 0.
  \]
  The request-independent transition kernel makes $\requestrowcoefficient[\state\req]$ a common value $\requestrowcoefficient[\state]$ for all requests at a state.  Since $\sum_{\req}\prior[\req][\prime] = 1$, the second equality becomes $\requestrowcoefficient[\state] = \stateflowrowcoefficient[\state] + w$.  Using the first equality for the null action and averaging against any stationary distribution of its finite-state transition matrix gives $w = 0$.  Consequently the full row-dependency space is exactly
  \[
    \requestrowcoefficient[\state\req] = \stateflowrowcoefficient[\state],
    \qquad
    \sum_{\state[][\prime]}\transition(\state[][\prime] \mid \state, \act)\stateflowrowcoefficient[\state[][\prime]] = \stateflowrowcoefficient[\state]
    \quad\text{for every }(\state, \act),
    \qquad w = 0,
  \]
  with zero resource-row coefficients. These conditions are independent of $\prior[][\prime]$, and every such dependence also annihilates the right-hand side: the balance rows have zero right-hand sides, while the normalization and resource coefficients vanish. Thus the same retained rows form an equivalent full-row-rank system for every probability vector $\prior[][\prime]$ and every $\capacityratio$.

  In that fixed representation, the entries of the nominal basis matrix and of the nonbasic columns depend continuously on $\prior[][\prime]$, while the resource right-hand sides depend continuously on $\capacityratio$.  The nominal basis matrix is nonsingular, all nominal basic variables are strictly positive, and all nominal nonbasic reduced costs are strictly negative by Proposition~\ref{prop:local-lp-stability}.  Hence these properties persist for sufficiently small two-sided perturbations of $(\nominalratio, \prior)$.  For an inactive resource, its positive nominal slack remains basic, and its column is the same resource-row unit vector for every $\prior[][\prime]$.  Increasing that resource's right-hand side therefore increases only its own basic slack, leaving all other basic variables and all reduced costs unchanged.  Applying the two-sided neighborhood result first with each inactive capacity replaced by $\min\{\capacityratio[\res], \nominalratio[\res]\}$, and then restoring all upward increases, proves the result throughout the stated one-sided joint neighborhood.  The binding and uniform inactive-slack conclusions follow from the continued basic solution, after shrinking the neighborhood if necessary.  Thus the basis remains uniquely optimal with exactly the nominal positive support, so every used action is nominal-dual tight and the core projection is unchanged.  Deleting variables outside $\recstates$ leaves the same nominal basis feasible and uniquely optimal in the restricted LP\@.
\end{proof}

\paragraph{Frozen core-projected epochs.}
\label{app:proof:lem:unknown-frozen-epoch}

For the next lemma, take a pair $(\capacityratio,\prior[][\prime])$ for which
the core projection is well defined, a positive integer $\blocklength$, and an
initial state in $\recstates$.  Run the frozen core-projected policy
$\coreprojection(\selectedsolution(\capacityratio,\prior[][\prime]))$ under
the true request prior $\prior$, and number the epoch's pre-request instants by
$\period=1,\ldots,\blocklength$.  Let $\filtration[][\period]$ denote the full
history before request $\period$.  Fix a bounded functional of the current
state, fresh request, selected action, and next state, and denote its value in
round $\period$ by $\oneperiodquantity[][\period]$.  Under the frozen projected
policy, define its statewise conditional mean to be the one-step kernel
expectation from state $\state$:
\[
  \conditionalmean[\capacityratio,\prior[][\prime]](\state)
  \coloneqq
  \expect[\oneperiodquantity[][\period]\mid\stateproc[][\period]=\state].
\]
Thus
$\expect[\oneperiodquantity[][\period]\mid\filtration[][\period]]
=\conditionalmean[\capacityratio,\prior[][\prime]](\stateproc[][\period])$.
Under the nominal policy $\policy[\nominalratio]$, let
$\conditionalmean[][*](\state)$ denote the corresponding statewise mean and
let $\nominalmean$ denote its stationary mean.  Rewards and resource
consumptions are examples of such functionals.

Under Assumption~\ref{ass:nominal-nondegeneracy}, for pairs in the joint-basis neighborhood
of Lemma~\ref{lem:unknown-joint-basis}, write
$\policy[\capacityratio,\prior[][\prime]]$ for the local empirical policy,
which equals its core projection, and let
$\trueresourcemean[\res](\capacityratio,\prior[][\prime])$ denote its stationary
rate of resource-$\res$ consumption when it is evaluated under the true
request prior.

\begin{lemma}
  \label{lem:unknown-frozen-epoch}
  If $\jointdistance(\capacityratio, \prior[][\prime]) < \localradius$, then
  \begin{equation}
    \expect\left[
    \max_{1 \le \completedrounds \le \blocklength}
    \left|\sum_{\period = 1}^{\completedrounds}(\oneperiodquantity[][\period] - \nominalmean)\right|
    \right]
    = \bigO\bigl(\sqrt {\blocklength} + \blocklength\jointdistance(\capacityratio, \prior[][\prime]) + 1\bigr),
    \label{eq:unknown-frozen-bound}
  \end{equation}
  and, for every $\deviationlevel > 0$,
  \begin{equation}
    \prob\left(
    \max_{1 \le \completedrounds \le \blocklength}
    \left|\sum_{\period = 1}^{\completedrounds}(\oneperiodquantity[][\period] - \nominalmean)\right|
    > \bigTheta(\blocklength\jointdistance(\capacityratio, \prior[][\prime]) + 1) + \deviationlevel
    \right)
    \le \bigO\left(\exp\left(-\bigOmega\!\left(\frac{\deviationlevel[][2]}{\blocklength}\right)\right)\right).
    \label{eq:unknown-frozen-tail}
  \end{equation}
  All constants are uniform over the local neighborhood.

  Under Assumption~\ref{ass:nominal-nondegeneracy}, the true-prior stationary resource rates satisfy
  \begin{align}
    |\trueresourcemean[\res](\capacityratio, \prior[][\prime]) - \capacityratio[\res]| & = \bigO(\|\prior[][\prime] - \prior\|_1),
    & & \res \in \bindingset,
    \label{eq:unknown-true-binding-bias}\\
    \trueresourcemean[\res](\capacityratio, \prior[][\prime]) & \le \capacityratio[\res] - \frac12\positivitymargin[\mathrm{nd}],
    & & \res \notin \bindingset,
    \label{eq:unknown-true-inactive-slack}
  \end{align}
  whenever the neighborhood is small enough, and
  \begin{equation}
    \prob\left(
    \max_{1 \le \completedrounds \le \blocklength}
    \left|\sum_{\period = 1}^{\completedrounds}(\consumptionproc[\res][\period] - \trueresourcemean[\res](\capacityratio, \prior[][\prime]))\right| > \deviationlevel
    \right)
    \le \bigO\left(\exp\left(-\bigOmega\!\left(\frac{\deviationlevel[][2]}{\blocklength}\right)\right)\right)
    \label{eq:unknown-centered-resource-tail}
  \end{equation}
  for $\deviationlevel$ above a fixed constant.
\end{lemma}

\begin{proof}
  The proof partially mirrors the proof of Lemma~\ref{lem:deg-gimfr-frozen-additive}. 
  Abbreviate $\perturbation \coloneqq \jointdistance(\capacityratio, \prior[][\prime])$.  Because the core projection uses only actions in $\actions[\recstates](\state)$, the state remains in $\recstates$.  By \eqref{eq:unknown-projected-policy-close} and \eqref{eq:unknown-projected-kernel-close}, both $\conditionalmean[\capacityratio, \prior[][\prime]]$ and the true kernel $\inducedkernel[\capacityratio, \prior[][\prime]][\mathrm{proj}]$ differ by $\bigO(\perturbation)$ from their nominal counterparts on $\recstates$.  Let $\poissoncorrector$ be the bounded centered solution of the nominal Poisson equation
  \[
    (\identitymatrix - \inducedkernel[][*])\poissoncorrector = \conditionalmean[][*] - \nominalmean\onesvector.
  \]
  With
  \[
    \martingale[][\period + 1] \coloneqq \oneperiodquantity[][\period] - \conditionalmean[\capacityratio, \prior[][\prime]](\stateproc[][\period])
    +\poissoncorrector(\stateproc[][\period + 1]) - (\inducedkernel[\capacityratio, \prior[][\prime]][\mathrm{proj}]\poissoncorrector)(\stateproc[][\period]),
  \]
  we have $\expect[\martingale[][\period + 1] \mid \filtration[][\period]] = 0$, and the martingale has uniformly bounded increments. The pathwise decomposition is
  \begin{align*}
    \sum_{\period = 1}^{\completedrounds}(\oneperiodquantity[][\period] - \nominalmean)
    = {} & \sum_{\period = 1}^{\completedrounds}\martingale[][\period + 1] + \poissoncorrector(\stateproc[][1]) - \poissoncorrector(\stateproc[][\completedrounds + 1])\\
    & +\sum_{\period = 1}^{\completedrounds}
    \bigl[\conditionalmean[\capacityratio, \prior[][\prime]](\stateproc[][\period]) - \nominalmean - \poissoncorrector(\stateproc[][\period])
    +(\inducedkernel[\capacityratio, \prior[][\prime]][\mathrm{proj}]\poissoncorrector)(\stateproc[][\period])\bigr].
  \end{align*}
  The last summand is uniformly $\bigO(\perturbation)$. Doob's square-integrable maximal inequality proves \eqref{eq:unknown-frozen-bound}, and the maximal Azuma--Hoeffding inequality proves \eqref{eq:unknown-frozen-tail}.

  Under nominal uniqueness and nondegeneracy, the empirical policy has the fixed support $\actsupport$. Its kernel under either $\prior[][\prime]$ or $\prior$ contains the same uniformly positive strongly connected graph from Proposition~\ref{prop:support-regularity}. Hence both kernels are uniformly irreducible and their stationary distributions are Lipschitz in the request weights. Replacing $\prior[][\prime]$ by $\prior$ while holding the policy fixed changes the stationary resource rate by $\bigO(\|\prior[][\prime] - \prior\|_1)$. Combining this fact with \eqref{eq:unknown-empirical-binding-rate} and \eqref{eq:unknown-empirical-inactive-rate} gives \eqref{eq:unknown-true-binding-bias} and \eqref{eq:unknown-true-inactive-slack}. Uniform Poisson bounds for this compact family of true kernels, followed by the same martingale argument as in Lemma~\ref{lem:frozen-concentration}, give \eqref{eq:unknown-centered-resource-tail}.
\end{proof}

\subsubsection{Inventory control for the universal upper bounds}
\label{app:unknown-inventory-control}

We start with a general analysis. 

At every empirical resolving point $\period[\epoch]$, define
\begin{equation}
  \begin{aligned}
    \subhorizon[\epoch] & \coloneqq \subhorizon[][\period[\epoch]],
    & \inventorydeviation[][\epoch] & \coloneqq \inventory[][\period[\epoch]] - \subhorizon[\epoch]\nominalratio,
    & \onesideddeviation[\epoch] & \coloneqq \inventorydeviationnorm(\inventorydeviation[][\epoch]), \\
    \capacityratio[][\epoch] & \coloneqq \frac{\inventory[][\period[\epoch]]}{\subhorizon[\epoch]},
    & \priorerror[\epoch] & \coloneqq \|\empiricalprior[][\period[\epoch]] - \prior\|_1,
    & \jointdistance[\epoch] & \coloneqq \jointdistance\!\left(\frac{\inventory[][\period[\epoch]]}{\subhorizon[][\period[\epoch]]}, \empiricalprior[][\period[\epoch]]\right).
  \end{aligned}
  \label{eq:unknown-epoch-errors}
\end{equation}
The definitions give
\begin{equation}
  \subhorizon[\epoch]\capacitydistance(\inventory[][\period[\epoch]] / \subhorizon[\epoch]) \le \onesideddeviation[\epoch],
  \qquad
  \jointdistance[\epoch] \le \frac{\onesideddeviation[\epoch]}{\subhorizon[\epoch]} + \priorerror[\epoch].
  \label{eq:unknown-D-controls-ratio}
\end{equation}

This appendix uses four of the five check types in the analytical-check
convention of Appendix~\ref{app:algorithm-proofs}: locality, resource
near-depletion, epoch movement, and prior error.  Deterministic midpoint,
square-root, and fixed-cutoff endpoints are stopping boundaries rather than
checks.  We first define the two pathwise checks whose thresholds do not
depend on a particular local neighborhood.  The
\emph{resource near-depletion check} is applied before each action and fires at
completed-round time $\completedrounds$ if
\begin{equation}
  \min_{\res\in[\numres]}
  \inventory[\res][\completedrounds]\le\maxconsumption.
  \label{eq:unknown-resource-near-depletion}
\end{equation}
The stopped comparison then terminates before taking that action.  This is not
the actual resource-depletion rule \eqref{eq:stopping}.  Since each coordinate
of one-period consumption is at most $\maxconsumption$, actual depletion cannot
occur before this check fires.  This is also the resource near-depletion check
used in Appendix~\ref{app:boundary-proofs}.

Given a positive buffer $\epochbuffer[\epoch]$, the \emph{epoch-movement check} at
a reached endpoint $\period[\epoch+1]$ fails if
\begin{equation}
  \left|\capacityratio[\res][\epoch+1]
    -\capacityratio[\res][\epoch]\right|>\epochbuffer[\epoch]
  \quad\text{for some }\res\in\bindingset,
  \quad\text{or}\quad
  \capacityratio[\res][\epoch+1]
    <\capacityratio[\res][\epoch]-\epochbuffer[\epoch]
  \quad\text{for some }\res\notin\bindingset.
  \label{eq:unknown-epoch-movement-check}
\end{equation}
This is the separate epoch-movement check used in
Lemma~\ref{lem:one-epoch-bounds}.
Only the late nondegenerate comparison uses this second check.

The next lemma gives stopped square-root control through the midpoint.

\paragraph{Control through the midpoint.}
\label{app:proof:lem:unknown-first-half-control}

The \emph{locality check} is applied at a resolving point before an action.  It
passes precisely when the resolving pair is \emph{local}, meaning that
$\jointdistance(\capacityratio,\empiricalprior)<\localradius$ and its core
projection is defined; a nonlocal pair is a failure of this check.  For a
post-entry horizon $\horizon$, set
\[
  \protectlength\coloneqq\lceil(\log\horizon)^2\rceil,
  \qquad
  \scheduleendpoint\coloneqq\lfloor\horizon/2\rfloor.
\]
Given an entrance history, start the sparse-balanced phase at completed-round
time $\protectlength$.  Let $\analyticalstop$ be the first nonlocal resolving
point strictly before $\scheduleendpoint$, the first activation of the resource
near-depletion check \eqref{eq:unknown-resource-near-depletion} before
$\scheduleendpoint$, or $\scheduleendpoint$, whichever occurs first.  The
locality and resource near-depletion checks at the midpoint belong to the
late-phase analysis.
For this stopped path, define
\[
  \onesideddeviation[][*]
  \coloneqq
  \max_{\protectlength\le\completedrounds\le\analyticalstop}
  \inventorydeviationnorm\bigl(
    \inventory[][\completedrounds]
    -\subhorizon[][\completedrounds]\nominalratio
  \bigr).
\]
The maximum includes the first failed resolving point or resource
near-depletion prefix and, if reached, the midpoint; the path is not continued
beyond a resource near-depletion stop.

Index rewards, actions, and consumptions by completed-round time as well.
Before analytical truncation, define the one-period nominal-dual loss
$\dualloss[][\period]$ by
\begin{equation}
  \expect\left[
    \rewardproc[][\period]
    +\potential(\subhorizon[][\period]-1,
      \inventory[][\period+1],\stateproc[][\period+1])
    \ \middle|\ \filtration[][\period]\right]
  =\potential(\subhorizon[][\period],\inventory[][\period],
    \stateproc[][\period])-\dualloss[][\period].
  \label{eq:unknown-dual-loss-identity}
\end{equation}
The nominal dual inequalities give $\dualloss[][\period]\ge0$.  In a local
projected epoch, Lemma~\ref{lem:unknown-joint-sensitivity} gives
$\dualloss[][\period]=\bigO(\jointdistance[\epoch])$; if
Lemma~\ref{lem:unknown-joint-basis} applies, every used action is nominally
tight and $\dualloss[][\period]=0$.  For the stopped first-half path, set
\begin{equation*}
  \mathcal L_{\mathrm{sp}}
  \coloneqq
  \sum_{\period=\protectlength}^{\analyticalstop-1}
    \dualloss[][\period].
\end{equation*}
Because the locality and resource near-depletion checks are applied before an
action, $\mathcal L_{\mathrm{sp}}$ is exactly the nominal-dual loss over the
executed rounds of reached local sparse-balanced epochs, including the
executed prefix of an epoch interrupted by resource near-depletion.

\begin{lemma}
  \label{lem:unknown-first-half-control}
  Fix an entrance history and take the post-entry horizon sufficiently large
  that $\completedrounds[0] < \protectlength < \scheduleendpoint$. Suppose the
  initial inventory is large enough, relative to the bounded one-period costs,
  to preclude activation of the resource near-depletion check through
  $\protectlength$, regardless of the protective actions. Then the following
  four bounds hold for the stopped path:
  \begin{align}
    \expect\!\left[
      \onesideddeviation[][*]
      \ \middle|\ \filtration[][0]\right]
    &= \bigO\!\left(\sqrt {\horizon}
      + \inventorydeviationnorm(\entranceinventorydeviation)\right),
    \label{eq:unknown-first-half-Dmax}\\
    \expect\!\left[
      \subhorizon[][\analyticalstop]
      \indicator\{\analyticalstop < \scheduleendpoint\}
      \ \middle|\ \filtration[][0]\right]
    &= \bigO\!\left(\sqrt {\horizon}
      + \inventorydeviationnorm(\entranceinventorydeviation)\right), \notag\\
    \expect\!\left[
      \potential(\subhorizon[][\analyticalstop],
        \inventory[][\analyticalstop],
        \stateproc[][\analyticalstop])^+
      \indicator\{\analyticalstop < \scheduleendpoint\}
      \ \middle|\ \filtration[][0]\right]
    &= \bigO\!\left(\sqrt {\horizon}
      + \inventorydeviationnorm(\entranceinventorydeviation)\right), \notag\\
    \expect\!\left[
      \mathcal L_{\mathrm{sp}}
      \ \middle|\ \filtration[][0]\right]
    &= \bigO\!\left(\sqrt {\horizon}
      + \inventorydeviationnorm(\entranceinventorydeviation)\right). \notag
  \end{align}
  These conditional bounds have constants uniform over entrance histories
  satisfying the hypotheses.  When they are averaged over the regular-entry
  histories on which the lemma is invoked,
  \eqref{eq:entrance-imbalance-bound} makes the contribution of the entrance
  imbalance $\bigO(1)$.
\end{lemma}

\begin{proof}
  During the first $\protectlength$ post-entry rounds, all costs are bounded, so deterministically
  \begin{equation}
    \inventorydeviationnorm(\inventory[][\protectlength] - \subhorizon[][\protectlength]\nominalratio)
    = \bigO(\protectlength) + \inventorydeviationnorm(\entranceinventorydeviation).
    \label{eq:unknown-protective-deviation}
  \end{equation}
  Here the entrance imbalance is fixed under the entrance conditioning.

  Consider the sparse-balanced epochs before the analytical stop. All epoch sums below are restricted to epochs reached with a local resolving pair. In these sums, use the full scheduled epoch lengths as upper bounds where convenient. For estimates involving the prior, define every scheduled empirical prior from an infinite i.i.d.\ request sequence, including observations after any analytical stop; dropping the reach indicator then only enlarges a nonnegative bound. For resource $\res$, apply the proof of Lemma~\ref{lem:unknown-frozen-epoch} with the nominal resource Poisson solution $\poissoncorrector[\res][*]$. Because the same nominal corrector is used in every epoch, the martingale differences concatenate across resolving times and the Poisson boundary term telescopes globally. Set these martingales to zero through time $\protectlength$, stop them at $\analyticalstop$, and extend them constantly thereafter. Thus, there exist a bounded-increment martingale $\martingale[\res]$ and a remainder $\globalremainder[\res]$ such that, for every $\protectlength \le \completedrounds \le \analyticalstop$,
  \begin{equation*}
    \begin{aligned}
      \sum_{\period = \protectlength}^{\completedrounds - 1}
      (\consumptionproc[\res][\period] - \stationaryconsumption[\res](\nominalratio))
      & = \martingale[\res](\completedrounds)
      +\poissoncorrector[\res][*](\stateproc[][\protectlength])
      -\poissoncorrector[\res][*](\stateproc[][\completedrounds])
      +\globalremainder[\res](\completedrounds), \\
      |\globalremainder[\res](\completedrounds)|
      & \le \bigO\left(\sum_{\epoch:\period[\epoch] < \completedrounds}
      \bigl(\min\{\completedrounds, \period[\epoch + 1]\} - \period[\epoch]\bigr)
      \jointdistance[\epoch]\right).
    \end{aligned}
  \end{equation*}
  Here $\stationaryconsumption[\res](\nominalratio) = \nominalratio[\res]$ for $\res \in \bindingset$, whereas
  \[
    \stationaryconsumption[\res](\nominalratio) = \nominalratio[\res] - \resourceslack[\res][*]\qquad(\res \notin \bindingset).
  \]
  Consequently, with $\maxmartingale \coloneqq \max_{\res}\max_{\protectlength \le \completedrounds \le \scheduleendpoint}|\martingale[\res](\completedrounds)|$, every such prefix satisfies
  \begin{equation}
    \begin{aligned}
      \inventorydeviationnorm(\inventory[][\completedrounds] - \subhorizon[][\completedrounds]\nominalratio)
      \le {} & \bigO\left(
      1 + \inventorydeviationnorm(\inventory[][\protectlength] - \subhorizon[][\protectlength]\nominalratio) + \maxmartingale\right)\\
      & +\bigO\left(\sum_{\epoch:\period[\epoch] < \completedrounds}\blocklength[\epoch]\left(\frac{\onesideddeviation[\epoch]}{\subhorizon[\epoch]} + \priorerror[\epoch]\right)\right),
    \end{aligned}
    \label{eq:unknown-pathwise-recursion}
  \end{equation}
  where $\onesideddeviation[\epoch]$, $\priorerror[\epoch]$, and $\jointdistance[\epoch]$ are defined in \eqref{eq:unknown-epoch-errors} and $\onesideddeviation[\epoch] \le \onesideddeviation[][*]$. Since $\subhorizon[\epoch] \ge \horizon / 2$ in this phase and the scheduled blocks occupy at most the first half of the horizon,
  \[
    \sum_{\epoch}\frac{\blocklength[\epoch]}{\subhorizon[\epoch]}
    \le \frac{2}{\horizon}\sum_{\epoch}\blocklength[\epoch]
    \le 1.
  \]
  Set
  \[
    A\coloneqq\bigTheta\left(
      1+\inventorydeviationnorm(\inventory[][\protectlength]
        -\subhorizon[][\protectlength]\nominalratio)
      +\maxmartingale
      +\sum_{\epoch}\blocklength[\epoch]\priorerror[\epoch]
    \right),
    \qquad
    a_{\epoch}\coloneqq
      \bigTheta\left(\frac{\blocklength[\epoch]}{\subhorizon[\epoch]}\right).
  \]
  At a reached resolving point $\period[\epoch]$,
  \eqref{eq:unknown-pathwise-recursion} contains only deviations from earlier
  resolving points, and hence
  \[
    \onesideddeviation[\epoch]
    \le A+\sum_{\epoch'<\epoch}a_{\epoch'}
      \onesideddeviation[\epoch'].
  \]
  Direct induction therefore gives, also at a first failed resolving point,
  \[
    \onesideddeviation[\epoch]
    \le A\prod_{\epoch'<\epoch}(1+a_{\epoch'}).
  \]
  Indeed, the induction step is exactly the identity
  \[
    1+\sum_{\epoch'<\epoch}a_{\epoch'}
      \prod_{\ell<\epoch'}(1+a_{\ell})
    =\prod_{\epoch'<\epoch}(1+a_{\epoch'}).
  \]
  Moreover, the total amplification is bounded because
  \[
    \prod_{\epoch}(1+a_{\epoch})
    \le \exp\left(\sum_{\epoch}a_{\epoch}\right)
    =\bigO(1).
  \]
  Therefore, substituting the resolving-point bounds back into
  \eqref{eq:unknown-pathwise-recursion} gives, for every stopped prefix
  $\completedrounds$, including a final epoch interrupted by resource
  near-depletion,
  \[
    \inventorydeviationnorm(\inventory[][\completedrounds]
      -\subhorizon[][\completedrounds]\nominalratio)
    \le A\prod_{\epoch:\period[\epoch]<\completedrounds}(1+a_{\epoch})
    =\bigO(A).
  \]
  Taking the maximum over $\completedrounds$ yields
  \begin{equation}
    \onesideddeviation[][*]
    \le \bigO\left(
    1 + \inventorydeviationnorm(\inventory[][\protectlength] - \subhorizon[][\protectlength]\nominalratio) + \maxmartingale + \sum_{\epoch}\blocklength[\epoch]\priorerror[\epoch]\right).
    \label{eq:unknown-deviation-bound}
  \end{equation}
  Doob's square-integrable maximal inequality gives
  \[
    \expect \maxmartingale = \bigO(\sqrt {\horizon}).
  \]
  By \eqref{eq:unknown-mean-prior-error} and Lemma~\ref{lem:sparse-balanced-schedule},
  \begin{equation*}
    \expect\sum_{\epoch}\blocklength[\epoch]\priorerror[\epoch] \le \bigO\left(\sum_{\epoch}\frac{\blocklength[\epoch]}{\sqrt{\period[\epoch]}}\right)
    = \bigO(\sqrt {\horizon}).
  \end{equation*}
  Together with \eqref{eq:unknown-protective-deviation}, this proves \eqref{eq:unknown-first-half-Dmax}.

  We next charge the analytical stop.  A prior-estimation failure $\priorerror[\epoch] \ge \localradius / 2$ at some sparse resolving point has probability at most
  \[
    \midSchedule\exp(-\bigOmega(\localradius[][2]\protectlength)) = \exp(-\bigOmega((\log \horizon)^2)),
  \]
  so its remaining-horizon contribution is $\bigO(1)$.  If instead the capacity ratio is the first nonlocal component, \eqref{eq:unknown-D-controls-ratio} gives
  \[
    \subhorizon[\epoch] \le \frac{2\onesideddeviation[\epoch]}{\localradius} \le \frac{2\onesideddeviation[][*]}{\localradius},
  \]
  and its expected remaining horizon is $\bigO(\sqrt {\horizon} + \inventorydeviationnorm(\entranceinventorydeviation))$.  Finally, at a resource near-depletion stop, some coordinate satisfies $\inventory[\res][\analyticalstop] \le \maxconsumption$.  The definition of the prefix maximum therefore gives pathwise
  \[
    \subhorizon[][\analyticalstop]\min_{\res}\nominalratio[\res] \le \maxconsumption + \onesideddeviation[][*].
  \]
  Its expectation is therefore $\bigO(\sqrt {\horizon} + \inventorydeviationnorm(\entranceinventorydeviation))$ as well.

  The same bounds control the positive terminal potential. At every stop, including a resource near-depletion prefix, each binding inventory satisfies $\inventory[\res][\analyticalstop] \le \subhorizon[][\analyticalstop]\nominalratio[\res] + \onesideddeviation[][*]$. Since $\resourceprice[\res][*] = 0$ for inactive resources, direct substitution into $\potential$ gives, for the first-half analytical stopping time $\analyticalstop$,
  \begin{equation*}
    \expect\!\left[
    \potential(\subhorizon[][\analyticalstop], \inventory[][\analyticalstop], \stateproc[][\analyticalstop])^+
    \indicator\{\analyticalstop < \scheduleendpoint\}\right]
    = \bigO\bigl(\sqrt {\horizon} + \inventorydeviationnorm(\entranceinventorydeviation)\bigr).
  \end{equation*}
  At a prior-estimation stop, the envelope is $\bigO(1 + \horizon + \onesideddeviation[][*])$; its horizon term has exponentially small probability, and the deviation term is covered by \eqref{eq:unknown-first-half-Dmax}. On successful midpoint paths, the potential is carried into the next phase and is not charged here.

  For reward comparison, fix any nominal optimal dual solution and define its nonnegative action slacks as in the proof of Theorem~\ref{thm:deg-gimfr-sqrt}.  On $\recstates$ the state slack is zero.  By \eqref{eq:unknown-projected-policy-close}, a local projected policy places all but $\bigO(\jointdistance[\epoch])$ probability on nominally tight actions, so \eqref{eq:unknown-dual-loss-identity} gives $\dualloss[][\period]=\bigO(\jointdistance[\epoch])$ in epoch $\epoch$.  Hence
  \[
    \expect[\mathcal L_{\mathrm{sp}}]
    \le \bigO\left(\expect\sum_{\epoch}\blocklength[\epoch]\jointdistance[\epoch]\right)
    \le \bigO\left(\expect\left[\onesideddeviation[][*]\sum_{\epoch}\frac{\blocklength[\epoch]}{\subhorizon[\epoch]}\right]
    +\expect\sum_{\epoch}\blocklength[\epoch]\priorerror[\epoch]\right) = \bigO\bigl(\sqrt {\horizon} + \inventorydeviationnorm(\entranceinventorydeviation)\bigr),
  \]
  which proves the final claim.
\end{proof}

\paragraph{Late-half control.}
\label{app:proof:lem:unknown-late-sqrt}

Fix an entrance history and write $\horizon$ for its post-entry horizon.  Let
$\reachedevent[0]$ be the event that the first-half comparison reaches the
midpoint without an earlier permanent stop and that
$\stateproc[][\scheduleendpoint]\in\recstates$.  Define the midpoint inventory
deviation on this event by
\[
  \onesideddeviation[0]
  \coloneqq
  \inventorydeviationnorm\bigl(
    \inventory[][\scheduleendpoint]
    -\subhorizon[][\scheduleendpoint]\nominalratio
  \bigr),
  \qquad\text{on }\reachedevent[0],
\]
and set $\onesideddeviation[0]=0$ on $\reachedevent[0]^{\mathrm c}$.  Define
\[
  \sqrtepoch
  \coloneqq
  \min\{\epoch\ge0:\schedulepower[][\epoch]\le1/2\}.
\]
For $0\le\epoch\le\sqrtepoch$, define the late epoch boundaries by
\begin{equation*}
  \period[\epoch]
  \coloneqq
  \horizon-\epochhorizon[\epoch],
  \qquad
  \period[0]=\scheduleendpoint,
\end{equation*}
and, for $0\le\epoch<\sqrtepoch$, define
\begin{equation*}
  \blocklength[\epoch]
  \coloneqq
  \period[\epoch+1]-\period[\epoch]
  =\epochhorizon[\epoch]-\epochhorizon[\epoch+1].
\end{equation*}
For the proof-only late comparison starting on $\reachedevent[0]$, let
$\analyticalstop$ be the first nonlocal resolving instant, the first activation
of the resource near-depletion check
\eqref{eq:unknown-resource-near-depletion}, or the completed-round instant
$\period[\sqrtepoch]$, whichever occurs first.  This stopping
rule includes the locality and resource near-depletion checks at the midpoint
itself.  Set $\analyticalstop=\scheduleendpoint$ on
$\reachedevent[0]^{\mathrm c}$, and define the stopped late-phase nominal-dual
loss by
\begin{equation*}
  \mathcal L_{\mathrm{late}}
  \coloneqq
  \indicator\{\reachedevent[0]\}
  \sum_{\period=\scheduleendpoint}^{\analyticalstop-1}
    \dualloss[][\period].
\end{equation*}
Because every check is applied before an action, this sum contains exactly the
executed late-comparison rounds.  It is empty if a check fails at the midpoint
and includes the executed prefix of an epoch interrupted by resource
near-depletion.

\begin{lemma}
  \label{lem:unknown-late-sqrt}
  Suppose $\horizon$ is sufficiently large that the late schedule has at least
  $\sqrtepoch$ epochs. Then the following three bounds hold for the comparison and stopping rule
  defined above:
  \begin{align*}
    \expect\!\left[
      \subhorizon[][\analyticalstop]
      \indicator\{\reachedevent[0]\}
      \ \middle|\ \filtration[][0]\right]
    &=\bigO\left(\sqrt{\horizon}
    +\expect\!\left[
      \onesideddeviation[0]\indicator\{\reachedevent[0]\}
      \ \middle|\ \filtration[][0]\right]\right),\\
    \expect\!\left[
      \potential(\subhorizon[][\analyticalstop],
        \inventory[][\analyticalstop],
        \stateproc[][\analyticalstop])^+
      \indicator\{\reachedevent[0]\}
      \ \middle|\ \filtration[][0]\right]
    &=\bigO\left(\sqrt{\horizon}
    +\expect\!\left[
      \onesideddeviation[0]\indicator\{\reachedevent[0]\}
      \ \middle|\ \filtration[][0]\right]\right),\\
    \expect\!\left[
      \mathcal L_{\mathrm{late}}
      \ \middle|\ \filtration[][0]\right]
    &=\bigO\left(\sqrt{\horizon}
    +\expect\!\left[
      \onesideddeviation[0]\indicator\{\reachedevent[0]\}
      \ \middle|\ \filtration[][0]\right]\right).
  \end{align*}
\end{lemma}

\begin{proof}
  All expectations and probabilities below are conditional on
  $\filtration[][0]$ unless a later filtration is displayed explicitly.  The
  number $\sqrtepoch$ of relevant late epochs depends only on
  $\schedulepower$ and is therefore fixed.  At every late resolving time at
  least $\horizon / 2 - \bigO(1)$ requests have already been observed, so, for
  $0\le\epoch<\sqrtepoch$,
  \begin{equation}
    \expect[\priorerror[\epoch]] = \bigO(\horizon[][-1 / 2]),
    \qquad
    \prob(\priorerror[\epoch] > \localradius / 2)
    \le \exp(-\bigOmega(\horizon)).
    \label{eq:unknown-late-prior}
  \end{equation}
  Retain the event $\reachedevent[0]$ defined above.  For
  $1\le\epoch<\sqrtepoch$, let $\reachedevent[\epoch]$ be the event that epoch
  $\epoch$ is reached before its locality and resource near-depletion checks.
  For $0\le\epoch<\sqrtepoch$, let
  $\goodevent[\epoch]\subseteq\reachedevent[\epoch]$ be the event that those
  checks pass.  Finally, let $\reachedevent[\sqrtepoch]$ be the event that the
  deterministic cutoff $\period[\sqrtepoch]$ is reached; no new checks are
  assigned there.  Each event defined at $\period[\epoch]$ is
  $\filtration[][\period[\epoch]]$-measurable, and
  $\reachedevent[\epoch+1]\subseteq\goodevent[\epoch]$ for
  $0\le\epoch<\sqrtepoch$.

  For $0\le\epoch\le\sqrtepoch$, set
  $\onesideddeviation[\epoch]=0$ off $\reachedevent[\epoch]$.  For
  $0\le\epoch<\sqrtepoch$, define $\priorerror[\epoch]$ from the full i.i.d.\
  request sequence even on stopped paths.  On $\goodevent[\epoch]$, freeze the selected policy on an
  infinite-resource copy and let
  $\onesideddeviation[\epoch, \mathrm{hyp}][*]$ be the maximum of its one-sided
  inventory deviation over all prefixes of this epoch, including its initial
  and terminal points.  Write
  $\onesideddeviation[\epoch + 1, \mathrm{hyp}]$ for the deviation at that
  terminal point, and set both hypothetical deviations to zero off
  $\goodevent[\epoch]$.  For $0\le\epoch<\sqrtepoch$,
  Lemma~\ref{lem:unknown-frozen-epoch}, with the nominal drift
  retained for inactive resources, gives
  \begin{equation*}
    \expect[\onesideddeviation[\epoch + 1, \mathrm{hyp}] \mid \filtration[][\period[\epoch]]]
    \le \expect[\onesideddeviation[\epoch, \mathrm{hyp}][*] \mid \filtration[][\period[\epoch]]]
    \le \bigO\left(
    \onesideddeviation[\epoch] + \sqrt{\epochhorizon[\epoch]}
    +\epochhorizon[\epoch]\priorerror[\epoch] + 1\right).
  \end{equation*}
  The realized $\priorerror[\epoch]$ must remain in this conditional bound.
  Multiply by $\indicator\{\goodevent[\epoch]\}$ and then take expectations.
  Only at this point use \eqref{eq:unknown-late-prior}, obtaining
  \[
    \expect[\onesideddeviation[\epoch + 1]\indicator\{\reachedevent[\epoch + 1]\}]
    \le \bigO\left(
    \expect[\onesideddeviation[\epoch]\indicator\{\reachedevent[\epoch]\}]
    +\sqrt{\epochhorizon[\epoch]}
    +\frac{\epochhorizon[\epoch]}{\sqrt{\horizon}} + 1\right).
  \]
  Because $\sqrtepoch$ is fixed, induction bounds all event-weighted endpoint deviations and prefix maxima:
  \[
    \expect[\onesideddeviation[\epoch]\indicator\{\reachedevent[\epoch]\}]
    +
    \expect[\onesideddeviation[\epoch, \mathrm{hyp}][*]\indicator\{\goodevent[\epoch]\}]
    = \bigO(\sqrt{\horizon} + \expect[\onesideddeviation[0]\indicator\{\reachedevent[0]\}]).
  \]

  If the capacity ratio is the first component to become nonlocal while $\priorerror[\epoch] < \localradius / 2$, \eqref{eq:unknown-D-controls-ratio} gives $\subhorizon[][\analyticalstop] \le 2\onesideddeviation[\epoch] / \localradius$. A prior nonlocality has exponentially small probability by \eqref{eq:unknown-late-prior}, which is used without conditioning on $\reachedevent[\epoch]$. At a resource near-depletion stop inside a local epoch, some inventory is at most $\maxconsumption$, so
  \[
    \subhorizon[][\analyticalstop]\min_{\res}\nominalratio[\res] \le \maxconsumption + \onesideddeviation[\epoch, \mathrm{hyp}][*].
  \]
  The same estimate when the resource near-depletion check fires at an epoch start uses $\onesideddeviation[\epoch]$.  Finally, the clipped power recursion gives $\epochhorizon[\sqrtepoch] = \bigO(\sqrt{\horizon})$.  These observations give the explicit remaining-horizon bound
  \[
    \expect\!\left[
      \subhorizon[][\analyticalstop]
      \indicator\{\reachedevent[0]\}\right]
    =\bigO\left(\sqrt{\horizon}
    +\expect\!\left[
      \onesideddeviation[0]\indicator\{\reachedevent[0]\}
      \ \middle|\ \filtration[][0]\right]\right).
  \]

  For binding resources the stopped inventory satisfies
  $\inventory[\res][\analyticalstop] \le \subhorizon[][\analyticalstop]\nominalratio[\res] + \onesideddeviation[\epoch, \mathrm{hyp}][*]$
  at a within-epoch stop, and the corresponding bound with $\onesideddeviation[\epoch]$ at a resolving stop.  Since inactive resource prices vanish, direct substitution in $\potential$ gives the same bound on expected positive terminal potential.  In the rare prior-failure case, bound the positive terminal potential by $\bigO(1 + \horizon + \onesideddeviation[\epoch])$ using the preceding event-weighted estimates.  Thus
  \[
    \expect\!\left[
      \potential(\subhorizon[][\analyticalstop],
        \inventory[][\analyticalstop],
        \stateproc[][\analyticalstop])^+
      \indicator\{\reachedevent[0]\}\right]
    =\bigO\left(\sqrt{\horizon}
    +\expect\!\left[
      \onesideddeviation[0]\indicator\{\reachedevent[0]\}
      \ \middle|\ \filtration[][0]\right]\right).
  \]

  Finally, every pre-$\analyticalstop$ action in a late epoch $\epoch$ is
  taken only on $\goodevent[\epoch]$, and its one-period dual loss is
  $\bigO(\jointdistance[\epoch])$.  Using the full scheduled block length to
  upper-bound a possibly interrupted terminal epoch gives
  \begin{align*}
    \expect[\mathcal L_{\mathrm{late}}]
    \le \bigO\left(\expect\sum_{\epoch < \sqrtepoch}
      \indicator\{\goodevent[\epoch]\}
      \blocklength[\epoch]\jointdistance[\epoch]\right)
    &\le \bigO\left(\sum_{\epoch < \sqrtepoch}
    \left(\expect[\onesideddeviation[\epoch]\indicator\{\reachedevent[\epoch]\}]
    +\frac{\epochhorizon[\epoch]}{\sqrt{\horizon}}\right)\right) \\
    &= \bigO\left(\sqrt{\horizon}
    +\expect\!\left[
      \onesideddeviation[0]\indicator\{\reachedevent[0]\}
      \ \middle|\ \filtration[][0]\right]\right).
  \end{align*}
  This proves the lemma.
\end{proof}

The unique, nondegenerate case needs a stronger, horizon-weighted stability statement.  Fix $\concentrationpower$ as in \eqref{eq:zeta-kappa}, so
\[
  0 < \concentrationpower < 1 - \frac1{2\schedulepower},
  \qquad
  \concentrationexponent = 2\schedulepower(1 - \concentrationpower) - 1 > 0.
\]

\paragraph{Stabilization through the midpoint under uniqueness and nondegeneracy.}
\label{app:proof:lem:unknown-nd-midpoint}

By Proposition~\ref{prop:nondegenerate-full-state-support},
$\recstates=\states$.  We therefore use completed-round notation on the
original horizon, starting from
$\inventory[][0]=\horizon\nominalratio$ at time $0$.

\begin{lemma}
  \label{lem:unknown-nd-midpoint}
  Under Assumptions~\ref{ass:nominal-nondegeneracy}
  and~\ref{ass:support-irreducibility},
  \begin{equation}
    \expect\!\left[
      \sum_{\period=0}^{\protectlength-1}\dualloss[][\period]
    \right]=\bigO(1).
    \label{eq:unknown-nd-protective-loss}
  \end{equation}
  Moreover, for every prescribed $\tailpower>0$, the
  sparse-balanced phase reaches the midpoint without activating the resource
  near-depletion check and satisfies
  \begin{equation}
    \capacitydistance(\inventory[][\scheduleendpoint] / \subhorizon[][\scheduleendpoint]) < \localradius[\mathrm{nd}] / 4,
    \label{eq:unknown-nd-midpoint-local}
  \end{equation}
  except with probability $\bigO(\horizon[][-\tailpower])$, for all
  sufficiently large $\horizon$.
\end{lemma}

\begin{proof}
  Shrink $\localradius[\mathrm{nd}]$ in advance so that
  $\localradius[\mathrm{nd}]<\frac14\min_{\res}\nominalratio[\res]$.
  Since $\inventory[][0]/\horizon=\nominalratio$, bounded one-period costs and
  the ratio recursion give, for every $\completedrounds\le\protectlength$,
  \begin{align*}
    \left|\frac{\inventory[\res][\completedrounds]}{\subhorizon[][\completedrounds]} - \nominalratio[\res]\right|
    & = \bigO\!\left(\frac{\protectlength}{\horizon - \protectlength}\right), & & \res \in \bindingset, \\
    \frac{\inventory[\res][\completedrounds]}{\subhorizon[][\completedrounds]}
    & \ge \nominalratio[\res] - \bigO\!\left(\frac{\protectlength}{\horizon - \protectlength}\right), & & \res \notin \bindingset.
  \end{align*}
  Thus, for all sufficiently large $\horizon$, every protective resolving ratio
  has $\capacitydistance(\inventory[][\completedrounds] /
  \subhorizon[][\completedrounds])<\localradius[\mathrm{nd}]/8$.  If, in
  addition,
  $\|\empiricalprior[][\learningresolve[\resolveindex]] - \prior\|_1 < \localradius[\mathrm{nd}] / 2$, Lemma~\ref{lem:unknown-joint-basis} shows that the restricted empirical solve uses exactly the nominal support and therefore satisfies the exact nominal Bellman-potential identity. If this prior event fails, charge the entire following protective epoch. Its length is $\bigO(\learningresolve[\resolveindex][2])$, and every one-period nominal dual slack is uniformly bounded in the finite model. Hence, \eqref{eq:unknown-prior-tail} gives
  \begin{equation*}
    \sum_{\resolveindex} \learningresolve[\resolveindex][2]\prob\!\left(
    \|\empiricalprior[][\learningresolve[\resolveindex]] - \prior\|_1 \ge \localradius[\mathrm{nd}] / 2\right)
    \le \bigO\left(\sum_{\resolveindex} \learningresolve[\resolveindex][2]\exp(-\bigOmega(\learningresolve[\resolveindex]))\right) = \bigO(1).
  \end{equation*}
  The initial inventory equals $\horizon\nominalratio$, and the first
  $\protectlength = o(\horizon)$ rounds consume at most
  $\maxconsumption\protectlength$ of each resource.  Thus the resource
  near-depletion check cannot fire during these phases for sufficiently large
  $\horizon$, even when some
  protective empirical estimates are bad. The comparison continues through the resulting
  bad epoch and resumes at the next resolving point; it does not discard the remaining
  horizon. The fixed warm start costs only $\bigO(1)$. This proves the
  protective-phase assertion. Moreover,
  \begin{equation}
    \inventorydeviationnorm(\inventory[][\protectlength] - \subhorizon[][\protectlength]\nominalratio)
    = \bigO(\protectlength)
    \label{eq:unknown-nd-protective-deviation}
  \end{equation}
  whether or not a protective empirical solve was good, because all protective
  and safe actions have bounded costs.

  For the sparse phase, apply the basis-locality check at every sparse
  resolving point before acting: it fails when
  \[
    \jointdistance\!\left(
      \frac{\inventory[][\period[\epoch]]}{\subhorizon[][\period[\epoch]]},
      \empiricalprior[][\period[\epoch]]
    \right)\ge\localradius[\mathrm{nd}].
  \]
  Stop analytically at its first failure or at the first activation of the
  resource near-depletion check.  The midpoint locality and resource
  near-depletion checks are assigned to the late comparison.  On this stopped
  process, use the global resource martingales from the proof of
  Lemma~\ref{lem:unknown-first-half-control}. For $\highprobparameter > 0$, let $\highprobevent[\horizon](\highprobparameter)$ be the event that the following simultaneously hold:
  \begin{align}
    \max_{\res}\max_{\period \le \scheduleendpoint}|\martingale[\res](\period)|
    & \le \highprobparameter\sqrt{\horizon\log \horizon},
    \notag\\
    \|\empiricalprior[][\period[\epoch]] - \prior\|_1
    & \le \highprobparameter\sqrt{\frac{\log \horizon}{\period[\epoch]}}
    \quad\text{at every sparse resolving point }\period[\epoch].
    \label{eq:unknown-nd-sparse-prior-event}
  \end{align}
  The maximal Azuma--Hoeffding inequality for the stopped martingales and
  \eqref{eq:unknown-prior-tail}, followed by a union bound over
  $\bigO(\log\log \horizon)$ re-solves, give
  \begin{equation}
    \prob\bigl(\highprobevent[\horizon](\highprobparameter)^{\mathrm c}\bigr)
    = \bigO(\horizon[][-\tailpower]),
    \qquad\text{for each fixed }\tailpower > 0
    \text{ and sufficiently large }\highprobparameter.
    \label{eq:unknown-highprob-firsthalf}
  \end{equation}
  Thus the exponent can be made larger than any prescribed $\tailpower$.

  On $\highprobevent[\horizon](\highprobparameter)$, the pathwise deviation
  bound \eqref{eq:unknown-deviation-bound} and
  \eqref{eq:unknown-nd-protective-deviation} give
  \begin{align*}
    \onesideddeviation[][*]
    & = \bigO\left(
    1 + \protectlength
    +\highprobparameter\sqrt{\horizon\log \horizon}
    +\highprobparameter\sqrt{\log \horizon}\sum_{\epoch}\frac{\blocklength[\epoch]}{\sqrt{\period[\epoch]}}
    \right)\\
    & = \bigO\bigl(\highprobparameter\sqrt{\horizon\log \horizon}\bigr).
  \end{align*}
  For sufficiently large $\horizon$, the right-hand side is smaller than
  $\localradius[\mathrm{nd}]\horizon/8$. Since
  $\subhorizon[\epoch]\ge\horizon/2$ throughout the sparse phase,
  \begin{equation*}
    \capacitydistance(\inventory[][\period[\epoch]] / \subhorizon[][\period[\epoch]])
    \le \frac{\onesideddeviation[][*]}{\subhorizon[\epoch]} < \localradius[\mathrm{nd}] / 4.
  \end{equation*}
  Also, $\period[\epoch] \ge \protectlength \ge (\log \horizon)^2$, so the bound in \eqref{eq:unknown-nd-sparse-prior-event} is $\littleo(1)$ uniformly over sparse resolving points and is eventually smaller than $\localradius[\mathrm{nd}] / 4$. Hence every stopped resolving pair actually satisfies
  $\jointdistance(\inventory[][\period[\epoch]] / \subhorizon[][\period[\epoch]], \empiricalprior[][\period[\epoch]]) < \localradius[\mathrm{nd}] / 2$, contradicting a joint-basis exit. Since $\onesideddeviation[][*]$ includes every stopped prefix and $\localradius[\mathrm{nd}] < \frac14\min_{\res}\nominalratio[\res]$, at every such instant and for every resource,
  \[
    \inventory[\res][\completedrounds] \ge \subhorizon[][\completedrounds]\nominalratio[\res] - \onesideddeviation[][*]
    > \frac\horizon2\min_{\res}\nominalratio[\res] - \frac{\localradius[\mathrm{nd}]\horizon}{8}.
  \]
  For sufficiently large $\horizon$, this exceeds $\maxconsumption$, excluding
  a first activation of the resource near-depletion check as well. Therefore,
  \eqref{eq:unknown-nd-midpoint-local} holds on
  $\highprobevent[\horizon](\highprobparameter)$. Combining this fact with
  \eqref{eq:unknown-highprob-firsthalf} proves the final assertion. The finitely
  many small values of $\horizon$ are absorbed into the constants.
\end{proof}

\paragraph{Summable late exits.}
\label{app:proof:lem:unknown-nd-late}

For the late comparison, set
\[
  \scheduleendpoint\coloneqq\lfloor\horizon/2\rfloor,
  \qquad
  \epochhorizon[0]\coloneqq\lceil\horizon/2\rceil.
\]
Let $\goodevent[0]$ be the event that the comparison reaches the midpoint and
that
\[
  \frac{\inventory[][\scheduleendpoint]}{\epochhorizon[0]}
  \in\capacityneighborhood[\localradius[\mathrm{nd}]/4].
\]
In this late comparison, $\cutrest$ denotes a fixed, instance-dependent
integer that may be enlarged below but never depends on $\horizon$.
For every nonfinal late epoch $\epoch$, define
\[
  \period[\epoch]\coloneqq\horizon-\epochhorizon[\epoch],
  \qquad
  \epochhorizon[\epoch+1]
    \coloneqq\left\lceil\epochhorizon[\epoch][\schedulepower]\right\rceil,
  \qquad
  \blocklength[\epoch]
    \coloneqq\epochhorizon[\epoch]-\epochhorizon[\epoch+1],
\]
and define its endpoint-movement buffer by
\[
  \epochbuffer[\epoch]
  \coloneqq\epochhorizon[\epoch+1][-\concentrationpower].
\]
Fix $\thresholdfraction>0$, to be chosen sufficiently small, and define the
epoch-$\epoch$ prior-error tolerance, for every epoch opened above the fixed
cutoff, by
\[
  \priortolerance[\epoch]
  \coloneqq
  \thresholdfraction
  \frac{\epochbuffer[\epoch]\epochhorizon[\epoch+1]}
       {\blocklength[\epoch]}.
\]
The \emph{prior-error check} at the start of such an epoch is the pre-action
test
\[
  \priorerror[\epoch]
  =\|\empiricalprior[][\period[\epoch]]-\prior\|_1
  \le\priortolerance[\epoch];
\]
it fails when the displayed inequality is violated.
At a reached endpoint $\period[\epoch+1]$, the checks are applied in the
following order.  First apply the epoch-movement check for epoch $\epoch$.
If it passes but the fixed cutoff has been reached, stop without opening a new
epoch.  Otherwise, regard the same time as the start of epoch $\epoch+1$ and
apply the prior-error check and then the resource near-depletion check for that
new epoch.  Thus an epoch-movement failure at $\period[\epoch+1]$ belongs to
epoch $\epoch$, whereas either pre-action check failure there belongs to epoch
$\epoch+1$.  At a pre-action time that is not a reached epoch endpoint, the
fixed cutoff is tested before resource near-depletion.

On $\goodevent[0]$, let $\analyticalstop$ be the first of the following times,
using the priority just specified when several triggers share the same
completed-round time:
\begin{enumerate}[label=\textup{(\alph*)},leftmargin=2.2em]
  \item a reached endpoint $\period[\epoch+1]$ at which the epoch-movement
  check \eqref{eq:unknown-epoch-movement-check} fails;
  \item the first pre-action time at which the remaining horizon is at most
  the fixed cutoff $\cutrest$;
  \item a resolving time $\period[\epoch]$ at which the prior-error check
  $\priorerror[\epoch]\le\priortolerance[\epoch]$ fails, with the check applied
  before acting in epoch $\epoch$;
  \item a pre-action time
  $\period\in\{\period[\epoch],\ldots,\period[\epoch+1]-1\}$ at which the
  resource near-depletion check \eqref{eq:unknown-resource-near-depletion}
  fires; this failure belongs to epoch $\epoch$.
\end{enumerate}
A trigger in items \textup{(a)}, \textup{(c)}, or \textup{(d)} during epoch $\epoch$ is called an
\emph{analytical late failure in epoch $\epoch$}; reaching the fixed cutoff is
not a failure.  The union of items \textup{(a)} and \textup{(d)} is called a
\emph{dynamic failure}; thus this term includes both an epoch-movement failure
and activation of the resource near-depletion check, but not a prior-error
failure.

\begin{lemma}
  \label{lem:unknown-nd-late}
  Under Assumption~\ref{ass:nominal-nondegeneracy}, the fixed cutoff $\cutrest$ can be enlarged and $\thresholdfraction$ can be chosen sufficiently small so that, for every late epoch $\epoch$,
  \begin{equation}
    \prob\bigl(
    \goodevent[0], \ \text{the first analytical late failure occurs in epoch }\epoch
    \bigr)
    \le \exp(-\bigOmega(\epochhorizon[\epoch][\concentrationexponent])).
    \label{eq:unknown-nd-late-bad-tail}
  \end{equation}
  Consequently,
  \begin{equation}
    \expect[\subhorizon[][\analyticalstop]\indicator\{\goodevent[0]\}] = \bigO(1),
    \qquad
    \expect[\potential(\subhorizon[][\analyticalstop], \inventory[][\analyticalstop], \stateproc[][\analyticalstop])^+\indicator\{\goodevent[0]\}] = \bigO(1).
    \label{eq:unknown-nd-late-weighted-residual}
  \end{equation}
  Before $\analyticalstop$, every empirical LP remains in the nominal joint basis cell and its core projection is the identity.
\end{lemma}

\begin{proof}
  Enlarge $\cutrest$ and reduce $\thresholdfraction$ so that $\priortolerance[\epoch] \le \min\{1, \localradius[\mathrm{nd}] / 4\}$ for every nonfinal late epoch and so that the cumulative movement bound from Lemma~\ref{lem:cumulative-stability} is smaller than $\localradius[\mathrm{nd}] / 4$. Every late resolving point occurs after at least $\scheduleendpoint = \lfloor \horizon / 2\rfloor$ requests have been completed. Therefore \eqref{eq:unknown-prior-tail} gives, without conditioning on the epoch-start history,
  \begin{align*}
    \prob\bigl(\|\empiricalprior[][\period[\epoch]] - \prior\|_1 > \priortolerance[\epoch]\bigr)
    & \le \exp(-\bigOmega(\horizon\priortolerance[\epoch][2]))\\
    & \le \exp\bigl(-\bigOmega(\epochhorizon[\epoch][2\schedulepower(1 - \concentrationpower) - 1])\bigr)
    = \exp(-\bigOmega(\epochhorizon[\epoch][\concentrationexponent])).
  \end{align*}
  The second inequality uses $\blocklength[\epoch] \le \epochhorizon[\epoch]$, $\epochhorizon[\epoch + 1] \ge \epochhorizon[\epoch][\schedulepower]$, and $\horizon \ge \epochhorizon[\epoch]$.

  Now fix an epoch-start history on which all preceding analytical checks passed and
  $\priorerror[\epoch] \coloneqq \|\empiricalprior[][\period[\epoch]] - \prior\|_1 \le \priortolerance[\epoch]$. The cumulative movement bound keeps the capacity ratio in the joint basis neighborhood, while $\priortolerance[\epoch] \le \localradius[\mathrm{nd}] / 4$ keeps the prior there. Lemma~\ref{lem:unknown-joint-basis} therefore applies, the core projection is the identity, and \eqref{eq:unknown-true-binding-bias} gives, for every binding resource,
  \[
    \blocklength[\epoch]\bigl|\trueresourcemean[\res](\capacityratio[][\epoch], \empiricalprior[][\period[\epoch]]) - \capacityratio[\res][\epoch]\bigr|
    = \bigO(\blocklength[\epoch]\priorerror[\epoch])
    \le \frac14\epochbuffer[\epoch]\epochhorizon[\epoch + 1]\]
  after reducing $\thresholdfraction$.  Hence a binding endpoint movement larger than $\epochbuffer[\epoch]$ requires a centered resource fluctuation of order $\epochbuffer[\epoch]\epochhorizon[\epoch + 1]$.  By \eqref{eq:unknown-centered-resource-tail}, uniformly over the good history,
  \begin{equation*}
    \prob\left(
    \text{binding endpoint movement} > \epochbuffer[\epoch] \mid \filtration[][\period[\epoch]]\right)
    \le \exp\left(-\bigOmega\left(\frac{\epochbuffer[\epoch][2]\epochhorizon[\epoch + 1][2]}{\blocklength[\epoch]}\right)\right)
    \le \exp(-\bigOmega(\epochhorizon[\epoch][\concentrationexponent])).
  \end{equation*}
  For inactive resources, \eqref{eq:unknown-true-inactive-slack} supplies a fixed favorable drift, so the same one-sided estimate controls a downward endpoint movement larger than $\epochbuffer[\epoch]$.

  For the resource near-depletion check, abbreviate the frozen true resource rate by $\frozenresourcerate[\res] \coloneqq \trueresourcemean[\res](\capacityratio[][\epoch], \empiricalprior[][\period[\epoch]])$.  The preceding binding-bias bound and the favorable inactive drift give, for every resource and every prefix $0 \le \prefixidx \le \blocklength[\epoch]$,
  \begin{align*}
    \epochhorizon[\epoch]\capacityratio[\res][\epoch] - \prefixidx\frozenresourcerate[\res]
    & \ge \epochhorizon[\epoch + 1]\capacityratio[\res][\epoch] - \blocklength[\epoch](\frozenresourcerate[\res] - \capacityratio[\res][\epoch])_+\\
    & \ge \epochhorizon[\epoch + 1](\nominalratio[\res] - \localradius[\mathrm{nd}] - \epochbuffer[\epoch] / 4)
    \ge \frac12\epochhorizon[\epoch + 1]\nominalratio[\res],
  \end{align*}
  after enlarging the fixed cutoff so that $\epochbuffer[\epoch] \le \min_{\res}\nominalratio[\res]$, using $\localradius[\mathrm{nd}] < \frac14\min_{\res}\nominalratio[\res]$.  The actual inventory at a prefix before or at the first activation of the resource near-depletion check equals this mean-reference inventory minus the centered cumulative cost of the same prefix of the hypothetical frozen continuation.  Thus an inventory level of at most $\maxconsumption$ requires a centered prefix deviation of order $\epochhorizon[\epoch + 1]$.  By \eqref{eq:unknown-centered-resource-tail}, resource near-depletion therefore has conditional probability at most
  $\exp(-\bigOmega(\epochhorizon[\epoch][2\schedulepower - 1]))$, which is smaller because $\concentrationexponent < 2\schedulepower - 1$.

  Let $\mathcal G_{\epoch}^{\mathrm{pre}}$ denote the event that
  $\goodevent[0]$ holds and no analytical late failure occurred before epoch
  $\epoch$; for $\epoch=0$, the latter condition is vacuous.  Splitting the
  empirical event from the conditional dynamic event gives
  \begin{align*}
    & \prob\bigl(\mathcal G_{\epoch}^{\mathrm{pre}}, \ \text{first analytical late failure occurs in epoch }\epoch\bigr)\\
    & \qquad \le
    \prob(\priorerror[\epoch] > \priortolerance[\epoch])
    +\expect\!\left[
    \indicator\{\mathcal G_{\epoch}^{\mathrm{pre}}, \priorerror[\epoch] \le \priortolerance[\epoch]\}
    \prob(\text{dynamic failure in epoch }\epoch \mid \filtration[][\period[\epoch]])
    \right]\\
    & \qquad \le \exp(-\bigOmega(\epochhorizon[\epoch][\concentrationexponent])),
  \end{align*}
  which proves \eqref{eq:unknown-nd-late-bad-tail}. Note that empirical concentration is used only in the first, unconditional term; after conditioning on $\filtration[][\period[\epoch]]$, the empirical error is treated as a fixed quantity. If no failure occurs in epoch $\epoch$, the endpoint movement is at most $\epochbuffer[\epoch]$. Lemma~\ref{lem:cumulative-stability} and the choice of $\cutrest$ then keep every later capacity ratio inside the joint basis cell. At the beginning of the next epoch, the new empirical error is checked afresh before any action is taken.

  On $\mathcal G_{\epoch}^{\mathrm{pre}}$, the binding inventories at the beginning of epoch $\epoch$ are $\bigO(\epochhorizon[\epoch])$, and costs are nonnegative.  Thus, whether the first failure is an estimation stop, an epoch-movement failure, or a resource near-depletion stop,
  \begin{equation}
    \subhorizon[][\analyticalstop] + \potential(\subhorizon[][\analyticalstop], \inventory[][\analyticalstop], \stateproc[][\analyticalstop])^+
    = \bigO(1 + \epochhorizon[\epoch])
    \label{eq:unknown-nd-late-residual-envelope}
  \end{equation}
  on that first-failure event; inactive inventories do not enter the dual potential because $\resourceprice[\res][*] = 0$ there.  Finally,
  \begin{equation}
    \sum_{\epoch}(1 + \epochhorizon[\epoch])\exp(-\bigOmega(\epochhorizon[\epoch][\concentrationexponent])) = \bigO(1).
    \label{eq:unknown-nd-late-summability}
  \end{equation}
  This bound is uniform in $\horizon$: reading the power schedule backward from the fixed cutoff, as in the proof of Lemma~\ref{lem:cumulative-stability}, makes $\epochhorizon[\epoch]$ grow faster than geometrically, whereas the exponential factor decays super-polynomially.  The final cutoff has bounded horizon.  Combining \eqref{eq:unknown-nd-late-bad-tail}, \eqref{eq:unknown-nd-late-residual-envelope} and \eqref{eq:unknown-nd-late-summability} proves \eqref{eq:unknown-nd-late-weighted-residual} and the lemma.
\end{proof}

\subsubsection{Finishing the proof}
\label{app:proof:thm:unknown-universal}

\begin{proof}
  Fix the true nominal dual solution and, in each case, compare UIR-UP with the
  proof-only policy $\unknowninfrequentalg[][\mathrm g]$ that takes null actions
  after its permanent analytical stop.  We use
  \eqref{eq:unknown-dual-loss-identity} only before that stop.

  \noindent\textbf{Unique, possibly degenerate case.}
  Fix a sufficiently small one-sided entrance neighborhood, below
  $\frac12\min_{\res}\nominalratio[\res]$, and stop the comparison immediately
  after any successful entrance outside that neighborhood.  On this irregular
  entry event,
  \[
    \entrancerest=\bigO(\inventorydeviationnorm(\entranceinventorydeviation)),
    \qquad
    \potential(\entrancerest,\entranceinventory,\entrancestate)^+
    =\bigO(1+\entrancerest+
      \inventorydeviationnorm(\entranceinventorydeviation)).
  \]
  Lemma~\ref{lem:entrance-bookkeeping} therefore charges entrance failure,
  irregular entry, and the entrance potential discrepancy by $\bigO(1)$ in
  expectation.  On the remaining entrance histories, condition on the entrance
  history and use the post-entry notation.  The regular-entry radius gives every
  initial inventory coordinate at least
  $\horizon\min_{\res}\nominalratio[\res]/2$.  Hence, for every sufficiently
  large post-entry horizon, bounded one-period costs preclude resource
  near-depletion through the protective phase and
  $\completedrounds[0]<\protectlength<\scheduleendpoint$.  The algorithm's safe
  fallback handles the remaining bounded post-entry horizons, whose entire
  contribution is $\bigO(1)$.  Finally, the warm start and protective phase have
  length $\bigO((\log\horizon)^2)=\littleo(\sqrt{\horizon})$, so their bounded
  one-period dual losses are $\littleo(\sqrt{\horizon})$.

  Concatenate the stopped comparisons from
  Lemmas~\ref{lem:unknown-first-half-control}
  and~\ref{lem:unknown-late-sqrt}.  Those lemmas bound the accumulated local
  dual loss and the positive terminal potential by
  $\bigO(\sqrt{\horizon}+
  \inventorydeviationnorm(\entranceinventorydeviation))$.  Telescoping
  \eqref{eq:unknown-dual-loss-identity} from entrance to the permanent stop,
  then averaging over entrance histories and using
  $\sqrt{\entrancerest}\le\sqrt{\horizon}$, gives
  \[
    \horizon\stationaryfluid[][\prior](\nominalratio) - \valuefn[\horizon][\unknowninfrequentalg[][\mathrm g]] = \bigO(\sqrt {\horizon}).
  \]
  Pathwise dominance and the true-prior boundary comparison \eqref{eq:boundary-comparison} give \eqref{eq:unknown-sqrt-result} for the guard-free UIR-UP policy.

  \noindent\textbf{Unique, nondegenerate case.}
  By Proposition~\ref{prop:nondegenerate-full-state-support},
  $\recstates=\states$.  Thus the analysis starts at the first round with
  $\inventory[][0]=\horizon\nominalratio$.  Let $\analyticalstop$ be the
  permanent stop obtained by concatenating the midpoint and late comparisons,
  including the fixed cutoff.  Choose $\tailpower>2$ in
  Lemma~\ref{lem:unknown-nd-midpoint}.  A first-half failure then has
  probability $\bigO(\horizon[][-\tailpower])$.  Since binding inventories never
  exceed $\bigO(\horizon)$ and inactive resource prices vanish, the positive
  terminal potential on this event is at most
  $\bigO(1+\horizon)$; hence its expected contribution is
  $\bigO(1)$.  On midpoint-success paths,
  Lemma~\ref{lem:unknown-nd-late} gives an $\bigO(1)$ expected positive terminal
  potential, including late failures and the cutoff.

  On every good protective, sparse, and late epoch,
  Lemma~\ref{lem:unknown-joint-basis} gives the nominal action support, so the
  loss in \eqref{eq:unknown-dual-loss-identity} is zero.  The fixed warm start
  and bad protective epochs have total expected loss $\bigO(1)$ by
  \eqref{eq:unknown-nd-protective-loss}.
  Therefore \eqref{eq:initial-potential} and finite-horizon telescoping from the
  first round give
  \[
    \begin{aligned}
      &\horizon\stationaryfluid[][\prior](\nominalratio)
      -\valuefn[\horizon][\unknowninfrequentalg[][\mathrm g]] \\
      &\qquad\le
      \expect\!\left[
        \potential(\subhorizon[][\analyticalstop],
          \inventory[][\analyticalstop],
          \stateproc[][\analyticalstop])^+
        +\sum_{\period=0}^{\analyticalstop-1}\dualloss[][\period]
      \right]
      =\bigO(1).
    \end{aligned}
  \]
  Applying pathwise dominance and \eqref{eq:boundary-comparison} proves \eqref{eq:unknown-constant-result} for the guard-free UIR-UP policy.
\end{proof}


%% file: appendices/08_2_unknown_request_prior_proofs.tex
\subsection{Proof of Proposition~\ref{prop:unknown-R-linear}}
\label{app:proof:prop:unknown-R-linear}

\begin{proof}
  There are three states $\{0, \hardstateA, \hardstateB\}$, two request types $\highreq, \lowreq$, and one resource with $\nominalratio = 1$; the initial state is $0$. Consider the two priors
  \begin{equation*}
    \prior[][+](\highreq) = \frac34,
    \qquad
    \prior[][-](\highreq) = \frac14,
  \end{equation*}
  with complementary probabilities on $\lowreq$.

  At state $0$, there are two actions, $\act[\hardstateA]$ and $\act[\hardstateB]$, each with reward zero and cost one; their deterministic transitions are $0\to \hardstateA$ and $0\to \hardstateB$, respectively. The null action also has zero reward and cost and moves to $\hardstateA$ or $\hardstateB$ with probability $1 / 2$ each. Thus, no action can remain at state $0$ to collect additional request samples before committing to one of the two classes. Once $\hardstateA$ or $\hardstateB$ is reached, it is absorbing under all available actions. Every unspecified non-null state-action pair has zero reward, cost one, and the same next-state law as the null action in that state, so it is dominated by the null action and does not affect the stationary calculation.

  In state $\hardstateA$, there is a productive action $\act[\hardstateA]$ of cost one with reward
  \[
    \reward(\hardstateA, \highreq, \act[\hardstateA]) = 1,
    \qquad
    \reward(\hardstateA, \lowreq, \act[\hardstateA]) = \lowerreward.
  \]
  In state $\hardstateB$, there is a productive action $\act[\hardstateB]$ of cost one with
  \[
    \reward(\hardstateB, \lowreq, \act[\hardstateB]) = 1,
    \qquad
    \reward(\hardstateB, \highreq, \act[\hardstateB]) = \lowerreward.
  \]
  The null action has zero reward and cost and leaves $\hardstateA$ and
  $\hardstateB$ absorbing.  The productive action has strictly larger reward
  for both request types, and at the nominal ratio $\nominalratio=1$, using it
  on every request consumes exactly one unit of resource per period.

  No transition enters state $0$, so stationarity forces
  $\stateoccupation[0]=0$ and
  $\stateoccupation[\hardstateA]+\stateoccupation[\hardstateB]=1$.
  For any feasible stationary flow with such a state distribution, replacing
  every null or unspecified action in the two absorbing classes by the corresponding
  productive action is feasible: after all replacements, total resource use
  is exactly one.  Since all productive rewards are strictly positive, every
  optimum must therefore use the productive action for both requests in every
  state having positive stationary mass.

  Under $\prior[][+]$, the resulting per-period rewards in classes $\hardstateA$ and $\hardstateB$ are, respectively,
  \[
    \classreward[\hardstateA][+] = \frac34 + \frac14\lowerreward,
    \qquad
    \classreward[\hardstateB][+] = \frac14 + \frac34\lowerreward,
  \]
  and $\classreward[\hardstateA][+] - \classreward[\hardstateB][+] = (1 - \lowerreward) / 2 > 0$.  The objective under $\prior[][+]$ is therefore
  \[
    \stateoccupation[\hardstateA]\classreward[\hardstateA][+]
    +\stateoccupation[\hardstateB]\classreward[\hardstateB][+],
  \]
  so its unique maximizer has
  $\stateoccupation[\hardstateA]=1$ and
  $\stateoccupation[\hardstateB]=0$.  Thus the unique stationary fluid
  optimum under $\prior[][+]$ takes $\act[\hardstateA]$ for both request
  types.  Symmetrically, the unique optimum under $\prior[][-]$ assigns all
  stationary state probability to $\hardstateB$ and takes
  $\act[\hardstateB]$ for both request types.  The respective optimal target
  state classes are therefore $\{\hardstateA\}$ and $\{\hardstateB\}$.
  The optimal induced kernel on each singleton is irreducible, so
  Assumption~\ref{ass:support-irreducibility} holds separately for both priors.
  The correct entrance action reaches that singleton in one round with unit
  resource cost, so Assumption~\ref{ass:constant-entrance} also holds separately
  for both priors.

  The nominal optimum is also degenerate.  Indeed, the unique optimum uses the resource at rate exactly one, so the resource is binding.  If, say, under $\prior[][+]$, the capacity ratio is decreased to $\capacityratio = 1 - \capacityshortfall$ for arbitrarily small $\capacityshortfall > 0$, the nominal positive $\occupation$-support cannot remain feasible: request balance still requires total stationary state probability one, while taking the productive action on every request would consume one unit per period.  An optimizer must therefore put positive mass on the null action.  Thus the positive $\occupation$-support changes under arbitrarily small downward right-hand-side perturbations.  By Proposition~\ref{prop:local-lp-stability}, a unique nondegenerate nominal basic optimum would have locally constant positive support.  The unique nominal optimum is therefore degenerate.

  Fix a horizon $\horizon \ge 1$ and any policy $\policy[\horizon]$ that does not know which of the two priors governs requests.  The policy may depend on $\horizon$.  All value functions below are evaluated from initial inventory $\horizon\nominalratio$ and initial state $0$.  The policy sees $\reqproc[][1]$ before choosing its first action, but after that action the recurrent class is fixed forever.  Let
  \[
    \classAprob[\highreq] \coloneqq \prob(\text{eventual class is }\hardstateA \mid \reqproc[][1] = \highreq),
    \qquad
    \classAprob[\lowreq] \coloneqq \prob(\text{eventual class is }\hardstateA \mid \reqproc[][1] = \lowreq),
  \]
  where the probabilities include the algorithm's internal randomization and also cover use of the null action.  These two conditional class-selection probabilities are the same under both priors, because the first observed request and all other supplied information coincide under that conditioning.  They may depend on $\horizon$.  The wrong-class probabilities are
  \begin{align*}
    \wrongclassprob[+] & = \frac34(1 - \classAprob[\highreq]) + \frac14(1 - \classAprob[\lowreq]), \\
    \wrongclassprob[-] & = \frac14\classAprob[\highreq] + \frac34\classAprob[\lowreq].
  \end{align*}
  Therefore
  \[
    \wrongclassprob[+] + \wrongclassprob[-] = 1 - \frac12\classAprob[\highreq] + \frac12\classAprob[\lowreq] \ge \frac12.
  \]
  At least one prior has wrong-class probability at least $1 / 4$.  Put
  \[
    R_{\mathrm{good}}
      \coloneqq\frac34+\frac14\lowerreward,
    \qquad
    R_{\mathrm{wrong}}
      \coloneqq\frac14+\frac34\lowerreward.
  \]
  Under either prior, the good and wrong recurrent classes differ in
  per-round expected productive reward by exactly
  \[
    R_{\mathrm{good}}-R_{\mathrm{wrong}}
    = \frac{1 - \lowerreward}{2}.
  \]
  The optimal known-prior policy chooses the correct entrance action in round one and
  then takes the productive action in every remaining round.  This is feasible
  with $\nominalratio = 1$: the entrance action consumes one unit, leaving
  $\horizon - 1$ units for the remaining $\horizon - 1$ productive rounds,
  with the terminal round allowed to bring the inventory to zero.  No policy
  in a fixed class can exceed its per-period productive reward, and future
  requests are independent of the first-round commitment.  Consequently, for
  either sign $\sigma\in\{+,-\}$,
  \[
    \valuefn[\horizon][*,\prior[][\sigma]]
      =(\horizon-1)R_{\mathrm{good}},
    \qquad
    \valuefn[\horizon][\policy[\horizon],\prior[][\sigma]]
      \le(\horizon-1)
        \left((1-\wrongclassprob[\sigma])R_{\mathrm{good}}
          +\wrongclassprob[\sigma]R_{\mathrm{wrong}}\right).
  \]
  Choosing a sign with $\wrongclassprob[\sigma]\ge1/4$ gives
  \[
    \max_{\prior[][\prime] \in \{\prior[][+], \prior[][-]\}}
    \left\{\valuefn[\horizon][*, \prior[][\prime]] - \valuefn[\horizon][\policy[\horizon], \prior[][\prime]]\right\}
    \ge \frac14 \cdot \frac{1 - \lowerreward}{2}(\horizon - 1)
    = \frac{1 - \lowerreward}{8}(\horizon - 1).
  \]
  Here the superscript $\prior[][\prime]$ records the true law used to evaluate
  a policy.  For each prior separately, the transient fluid upper bound gives
  \[
    \transientfluid[\horizon][\prior[][\prime]]
      (\horizon\nominalratio,0)
      -\valuefn[\horizon][\policy[\horizon],\prior[][\prime]]
    \ge
    \valuefn[\horizon][*,\prior[][\prime]]
      -\valuefn[\horizon][\policy[\horizon],\prior[][\prime]].
  \]
  Hence the maximum transient-fluid gap obeys the same lower bound.  This
  proves the proposition for every $\horizon$ and $\policy[\horizon]$.  The
  maximizing member of the fixed pair may depend on the horizon and policy;
  the argument does not require a single prior to be worst for every horizon
  in a horizon-dependent policy family.
\end{proof}

%% file: references.bib
@book{bertsimas1997introduction,
  title={Introduction to linear optimization},
  author={Bertsimas, Dimitris and Tsitsiklis, John N.},
  volume={6},
  year={1997},
  publisher={Athena Scientific},
  address={Belmont, MA}
}

@article{Freedman75,
  author  = {Freedman, David A.},
  title   = {On tail probabilities for martingales},
  journal = {The Annals of Probability},
  volume  = {3},
  number  = {1},
  pages   = {100--118},
  year    = {1975},
  doi     = {10.1214/aop/1176996452},
  url     = {https://doi.org/10.1214/aop/1176996452}
}

@book{Roch24,
  author    = {Roch, S{\'e}bastien},
  title     = {Modern Discrete Probability: An Essential Toolkit},
  publisher = {Cambridge University Press},
  address   = {Cambridge},
  year      = {2024},
  doi       = {10.1017/9781009305129},
  url       = {https://doi.org/10.1017/9781009305129}
}

@book{vdVW96,
  author    = {van der Vaart, Aad W. and Wellner, Jon A.},
  title     = {Weak Convergence and Empirical Processes: With Applications to Statistics},
  series    = {Springer Series in Statistics},
  publisher = {Springer-Verlag},
  address   = {New York},
  year      = {1996},
  doi       = {10.1007/978-1-4757-2545-2},
  url       = {https://doi.org/10.1007/978-1-4757-2545-2}
}

@article{BW20,
  author  = {Bumpensanti, Pornpawee and Wang, He},
  title   = {A re-solving heuristic with uniformly bounded loss for network revenue management},
  journal = {Management Science},
  volume  = {66},
  number  = {7},
  pages   = {2993--3009},
  year    = {2020},
  url     = {https://doi.org/10.1287/mnsc.2019.3365},
  note    = {Online appendix: \url{https://pubsonline.informs.org/doi/suppl/10.1287/mnsc.2019.3365/suppl_file/mnsc.2019.3365.sm1.pdf}}
}

@article{CLY24,
  author  = {Chen, Guanting and Li, Xiaocheng and Ye, Yinyu},
  title   = {{An} improved analysis of {LP}-based control for revenue management},
  journal = {Operations Research},
  volume  = {72},
  number  = {3},
  pages   = {1124--1138},
  year    = {2024},
  url     = {https://doi.org/10.1287/opre.2022.2358}
}

@inproceedings{CDMK24,
  author    = {Chen, Zhaohua and Ai, Rui and Yang, Mingwei and Pan, Yuqi and Wang, Chang and Deng, Xiaotie},
  title     = {Contextual decision-making with knapsacks beyond the worst case},
  booktitle = {Advances in Neural Information Processing Systems 37},
  year      = {2024},
  url       = {https://doi.org/10.52202/079017-2798}
}

@unpublished{FV07,
  author = {Farias, Vivek F. and Van Roy, Benjamin},
  title  = {An approximate dynamic programming approach to network revenue management},
  note   = {Working paper, 23 April},
  year   = {2007},
  url    = {https://web.mit.edu/vivekf/www/papers/ADP-rm-07-03.pdf}
}

@article{Jasin15,
  author  = {Jasin, Stefanus},
  title   = {Performance of an {LP}-based control for revenue management with unknown demand parameters},
  journal = {Operations Research},
  volume  = {63},
  number  = {4},
  pages   = {909--915},
  year    = {2015},
  url     = {https://doi.org/10.1287/opre.2015.1390}
}

@article{JK12,
  author  = {Jasin, Stefanus and Kumar, Sunil},
  title   = {A re-solving heuristic with bounded revenue loss for network revenue management with customer choice},
  journal = {Mathematics of Operations Research},
  volume  = {37},
  number  = {2},
  pages   = {313--345},
  year    = {2012},
  url     = {https://doi.org/10.1287/moor.1120.0537}
}

@misc{LWZ25,
  author = {Li, Guokai and Wang, Zizhuo and Zhang, Jingwei},
  title  = {Infrequent resolving algorithm for online linear programming},
  year   = {2026},
  note   = {arXiv preprint arXiv:2408.00465v7, revised 1 September 2026},
  url    = {https://arxiv.org/abs/2408.00465v7}
}

@article{Secomandi08,
  author  = {Secomandi, Nicola},
  title   = {An analysis of the control-algorithm re-solving issue in inventory and revenue management},
  journal = {Manufacturing \& Service Operations Management},
  volume  = {10},
  number  = {3},
  pages   = {468--483},
  year    = {2008},
  url     = {https://doi.org/10.1287/msom.1070.0184}
}

@article{VB21,
  author  = {Vera, Alberto and Banerjee, Siddhartha},
  title   = {The {Bayesian} prophet: A low-regret framework for online decision making},
  journal = {Management Science},
  volume  = {67},
  number  = {3},
  pages   = {1368--1391},
  year    = {2021},
  url     = {https://doi.org/10.1287/mnsc.2020.3624}
}

@article{AgrawalWangYe14,
  author  = {Agrawal, Shipra and Wang, Zizhuo and Ye, Yinyu},
  title   = {A dynamic near-optimal algorithm for online linear programming},
  journal = {Operations Research},
  volume  = {62},
  number  = {4},
  pages   = {876--890},
  year    = {2014},
  url     = {https://doi.org/10.1287/opre.2014.1289}
}

@article{BalseiroLuMirrokni23,
  author  = {Balseiro, Santiago R. and Lu, Haihao and Mirrokni, Vahab},
  title   = {The best of many worlds: Dual mirror descent for online allocation problems},
  journal = {Operations Research},
  volume  = {71},
  number  = {1},
  pages   = {101--119},
  year    = {2023},
  url     = {https://doi.org/10.1287/opre.2021.2242}
}

@article{BrownZhang22,
  author  = {Brown, David B. and Zhang, Jingwei},
  title   = {Dynamic programs with shared resources and signals: Dynamic fluid policies and asymptotic optimality},
  journal = {Operations Research},
  volume  = {70},
  number  = {5},
  pages   = {3015--3033},
  year    = {2022},
  url     = {https://doi.org/10.1287/opre.2021.2181}
}

@article{BrownZhang25,
  author  = {Brown, David B. and Zhang, Jingwei},
  title   = {Fluid policies, reoptimization, and performance guarantees in dynamic resource allocation},
  journal = {Operations Research},
  volume  = {73},
  number  = {2},
  pages   = {1029--1045},
  year    = {2025},
  url     = {https://doi.org/10.1287/opre.2022.0601}
}

@article{Jasin14,
  author  = {Jasin, Stefanus},
  title   = {Reoptimization and self-adjusting price control for network revenue management},
  journal = {Operations Research},
  volume  = {62},
  number  = {5},
  pages   = {1168--1178},
  year    = {2014},
  url     = {https://doi.org/10.1287/opre.2014.1297}
}

@article{JK13,
  author  = {Jasin, Stefanus and Kumar, Sunil},
  title   = {Analysis of deterministic {LP}-based booking limit and bid price controls for revenue management},
  journal = {Operations Research},
  volume  = {61},
  number  = {6},
  pages   = {1312--1320},
  year    = {2013},
  url     = {https://doi.org/10.1287/opre.2013.1216}
}

@inproceedings{Jiang26,
  author    = {Jiang, Jiashuo},
  title     = {Constant Approximation for Network Revenue Management with {Markovian}-Correlated Customer Arrivals},
  booktitle = {Web and Internet Economics---21st International Conference, {WINE} 2025, Proceedings},
  series    = {Lecture Notes in Computer Science},
  volume    = {16266},
  pages     = {704--746},
  publisher = {Springer},
  year      = {2026},
  doi       = {10.1007/978-3-032-18660-7_38},
  url       = {https://doi.org/10.1007/978-3-032-18660-7_38}
}

@article{LiYe22,
  author  = {Li, Xiaocheng and Ye, Yinyu},
  title   = {Online linear programming: Dual convergence, new algorithms, and regret bounds},
  journal = {Operations Research},
  volume  = {70},
  number  = {5},
  pages   = {2948--2966},
  year    = {2022},
  url     = {https://doi.org/10.1287/opre.2021.2164}
}

@unpublished{LRTZ25,
  author = {Li, Weiyuan and Rusmevichientong, Paat and Topaloglu, Huseyin and Zhang, Jingwei},
  title  = {History-Dependent Fluid Approximations and Performance Guarantees for Revenue Management with {Markov}-Modulated Demands},
  year   = {2025},
  note   = {{SSRN} working paper 5107682, 21 January 2025},
  doi    = {10.2139/ssrn.5107682},
  url    = {https://doi.org/10.2139/ssrn.5107682}
}

@article{MaglarasMeissner06,
  author  = {Maglaras, Constantinos and Meissner, Joern},
  title   = {Dynamic pricing strategies for multiproduct revenue management problems},
  journal = {Manufacturing \& Service Operations Management},
  volume  = {8},
  number  = {2},
  pages   = {136--148},
  year    = {2006},
  url     = {https://doi.org/10.1287/msom.1060.0105}
}

@article{MaCaoTsangXia25,
  author  = {Ma, Wanteng and Cao, Ying and Tsang, Danny H. K. and Xia, Dong},
  title   = {Optimal regularized online allocation by adaptive re-solving},
  journal = {Operations Research},
  volume  = {73},
  number  = {4},
  pages   = {2079--2096},
  year    = {2025},
  url     = {https://doi.org/10.1287/opre.2022.0486}
}

@article{ReimanWang08,
  author  = {Reiman, Martin I. and Wang, Qiong},
  title   = {An asymptotically optimal policy for a quantity-based network revenue management problem},
  journal = {Mathematics of Operations Research},
  volume  = {33},
  number  = {2},
  pages   = {257--282},
  year    = {2008},
  url     = {https://doi.org/10.1287/moor.1070.0288}
}

@article{VBG21,
  author  = {Vera, Alberto and Banerjee, Siddhartha and Gurvich, Itai},
  title   = {Online allocation and pricing: Constant regret via {Bellman} inequalities},
  journal = {Operations Research},
  volume  = {69},
  number  = {3},
  pages   = {821--840},
  year    = {2021},
  url     = {https://doi.org/10.1287/opre.2020.2061}
}

@article{ArlottoGurvich19,
  author  = {Arlotto, Alessandro and Gurvich, Itai},
  title   = {Uniformly bounded regret in the multisecretary problem},
  journal = {Stochastic Systems},
  volume  = {9},
  number  = {3},
  pages   = {231--260},
  year    = {2019}
}

@article{ArlottoXie20,
  author  = {Arlotto, Alessandro and Xie, Xinchang},
  title   = {Logarithmic regret in the dynamic and stochastic knapsack problem with equal rewards},
  journal = {Stochastic Systems},
  volume  = {10},
  number  = {2},
  pages   = {170--191},
  year    = {2020}
}

@article{BalseiroBesbesPizarro24,
  author  = {Balseiro, Santiago R. and Besbes, Omar and Pizarro, Dana},
  title   = {Survey of dynamic resource-constrained reward collection problems: Unified model and analysis},
  journal = {Operations Research},
  volume  = {72},
  number  = {5},
  pages   = {2168--2189},
  year    = {2024},
  doi     = {10.1287/opre.2023.2441},
  url     = {https://doi.org/10.1287/opre.2023.2441}
}

@article{BesbesKanoriaKumar22,
  author  = {Besbes, Omar and Kanoria, Yash and Kumar, Akshit},
  title   = {Dynamic resource allocation: Algorithmic design principles and spectrum of achievable performances},
  journal = {Operations Research},
  volume  = {73},
  number  = {3},
  pages   = {1273--1288},
  year    = {2025},
  doi     = {10.1287/opre.2022.0504},
  url     = {https://doi.org/10.1287/opre.2022.0504}
}

@article{Bray24,
  author  = {Bray, Robert L.},
  title   = {Logarithmic regret in multisecretary and online linear programs with continuous valuations},
  journal = {Operations Research},
  volume  = {73},
  number  = {4},
  pages   = {2188--2203},
  year    = {2025},
  doi     = {10.1287/opre.2022.0036},
  url     = {https://doi.org/10.1287/opre.2022.0036}
}

@inproceedings{HanZengWangXiangZhang22,
  author    = {Han, Yuxuan and Zeng, Jialin and Wang, Yang and Xiang, Yang and Zhang, Jiheng},
  title     = {Optimal contextual bandits with knapsacks under realizability via regression oracles},
  booktitle = {Proceedings of the 26th International Conference on Artificial Intelligence and Statistics},
  series    = {Proceedings of Machine Learning Research},
  volume    = {206},
  pages     = {5011--5035},
  year      = {2023},
  url       = {https://proceedings.mlr.press/v206/han23b.html}
}

@misc{JiangZhang20,
  author = {Jiang, Jiashuo and Zhang, Jiawei},
  title  = {Online resource allocation with stochastic resource consumption},
  year   = {2020},
  note   = {arXiv preprint arXiv:2012.07933},
  url    = {https://arxiv.org/abs/2012.07933}
}

@misc{JiangLiZhang20,
  author = {Jiang, Jiashuo and Li, Xiaocheng and Zhang, Jiawei},
  title  = {Online stochastic optimization with {Wasserstein}-based non-stationarity},
  year   = {2020},
  note   = {arXiv preprint arXiv:2012.06961},
  url    = {https://arxiv.org/abs/2012.06961}
}

@article{JiangMaZhang25,
  author  = {Jiang, Jiashuo and Ma, Will and Zhang, Jiawei},
  title   = {Degeneracy is {OK}: Logarithmic regret for network revenue management with indiscrete distributions},
  journal = {Operations Research},
  volume  = {73},
  number  = {6},
  pages   = {3405--3420},
  year    = {2025},
  doi     = {10.1287/opre.2022.0641},
  url     = {https://doi.org/10.1287/opre.2022.0641}
}

@inproceedings{LiSunYe20,
  author    = {Li, Xiaocheng and Sun, Chunlin and Ye, Yinyu},
  title     = {Simple and fast algorithm for binary integer and online linear programming},
  booktitle = {Advances in Neural Information Processing Systems 33},
  pages     = {9412--9421},
  year      = {2020}
}

@article{MiaoWang21,
  author  = {Miao, Sentao and Wang, Yining},
  title   = {Network revenue management with nonparametric demand learning: $\sqrt{T}$-regret and polynomial dimension dependency},
  journal = {Mathematics of Operations Research},
  year    = {2025},
  note    = {Articles in Advance},
  doi     = {10.1287/moor.2022.0086},
  url     = {https://doi.org/10.1287/moor.2022.0086}
}

@article{SlivkinsZhouSankararamanFoster24,
  author  = {Slivkins, Aleksandrs and Zhou, Xingyu and Sankararaman, Karthik Abinav and Foster, Dylan J.},
  title   = {Contextual bandits with packing and covering constraints: A modular {Lagrangian} approach via regression},
  journal = {Journal of Machine Learning Research},
  volume  = {25},
  number  = {394},
  pages   = {1--37},
  year    = {2024},
  url     = {https://jmlr.org/papers/v25/24-1220.html}
}

@misc{SunWangZhou20,
  author = {Sun, Rui and Wang, Xinshang and Zhou, Zijie},
  title  = {Near-optimal primal-dual algorithms for quantity-based network revenue management},
  year   = {2020},
  note   = {arXiv preprint arXiv:2011.06327},
  url    = {https://arxiv.org/abs/2011.06327}
}

@article{VeraArlottoGurvichLevin25,
  author  = {Vera, Alberto and Arlotto, Alessandro and Gurvich, Itai and Levin, Eli},
  title   = {Dynamic resource allocation: The geometry and robustness of constant regret},
  journal = {Mathematics of Operations Research},
  volume  = {50},
  number  = {4},
  pages   = {2834--2872},
  year    = {2025},
  doi     = {10.1287/moor.2020.0334},
  url     = {https://doi.org/10.1287/moor.2020.0334}
}

@article{WangWang22,
  author  = {Wang, Yining and Wang, He},
  title   = {Constant regret resolving heuristics for price-based revenue management},
  journal = {Operations Research},
  volume  = {70},
  number  = {6},
  pages   = {3538--3557},
  year    = {2022},
  doi     = {10.1287/opre.2021.2219},
  url     = {https://doi.org/10.1287/opre.2021.2219}
}

@inproceedings{WuSrikantLiuJiang15,
  author    = {Wu, Huasen and Srikant, Rayadurgam and Liu, Xin and Jiang, Chong},
  title     = {Algorithms with logarithmic or sublinear regret for constrained contextual bandits},
  booktitle = {Advances in Neural Information Processing Systems 28},
  year      = {2015}
}

@article{GaoGeSunXueYe26,
  author  = {Gao, Wenzhi and Ge, Dongdong and Sun, Chunlin and Xue, Chenyu and Ye, Yinyu},
  title   = {Beyond {$\mathcal{O}(\sqrt{T})$} regret: Decoupling learning and decision making in online linear programming},
  journal = {Operations Research},
  volume  = {74},
  number  = {4},
  pages   = {1932--1944},
  year    = {2026},
  doi     = {10.1287/opre.2024.1575},
  url     = {https://doi.org/10.1287/opre.2024.1575}
}

@article{Gupta24,
  author  = {Gupta, Varun},
  title   = {Greedy algorithm for multiway matching with bounded regret},
  journal = {Operations Research},
  volume  = {72},
  number  = {3},
  pages   = {1139--1155},
  year    = {2024},
  doi     = {10.1287/opre.2022.2400},
  url     = {https://doi.org/10.1287/opre.2022.2400}
}

@unpublished{HeWeiXuYu25,
  author = {He, Siqi and Wei, Yehua and Xu, Jiaming and Yu, Sophie H.},
  title  = {Online resource allocation without re-solving: The effectiveness of primal-dual policies},
  year   = {2025},
  note   = {{SSRN} working paper 5133857, revised 11 September 2025},
  doi    = {10.2139/ssrn.5133857},
  url    = {https://doi.org/10.2139/ssrn.5133857}
}

@inproceedings{CastiglioniCelliKroer22,
  author    = {Castiglioni, Matteo and Celli, Andrea and Kroer, Christian},
  title     = {Online learning with knapsacks: the best of both worlds},
  booktitle = {Proceedings of the 39th International Conference on Machine Learning},
  series    = {Proceedings of Machine Learning Research},
  volume    = {162},
  pages     = {2767--2783},
  publisher = {PMLR},
  year      = {2022},
  url       = {https://proceedings.mlr.press/v162/castiglioni22a.html}
}

@article{BanerjeeFreund25,
  author  = {Banerjee, Siddhartha and Freund, Daniel},
  title   = {Good prophets know when the end is near},
  journal = {Management Science},
  volume  = {71},
  number  = {6},
  pages   = {4877--4894},
  year    = {2025},
  doi     = {10.1287/mnsc.2023.04307},
  url     = {https://doi.org/10.1287/mnsc.2023.04307}
}

@inproceedings{BanerjeeFreund20,
  author    = {Banerjee, Siddhartha and Freund, Daniel},
  title     = {Uniform loss algorithms for online stochastic decision-making with applications to bin packing},
  booktitle = {Abstracts of the 2020 {SIGMETRICS}/Performance Joint International Conference on Measurement and Modeling of Computer Systems},
  pages     = {1--2},
  publisher = {Association for Computing Machinery},
  year      = {2020},
  doi       = {10.1145/3393691.3394224},
  url       = {https://doi.org/10.1145/3393691.3394224}
}
